\documentclass[letterpaper]{article} % DO NOT CHANGE THIS
\usepackage{aaai24}  % DO NOT CHANGE THIS
\usepackage{times}  % DO NOT CHANGE THIS
\usepackage{helvet}  % DO NOT CHANGE THIS
\usepackage{courier}  % DO NOT CHANGE THIS
\usepackage[hyphens]{url}  % DO NOT CHANGE THIS
\usepackage{graphicx} % DO NOT CHANGE THIS
\usepackage{natbib}  % DO NOT CHANGE THIS AND DO NOT ADD ANY OPTIONS TO IT
\usepackage{caption} % DO NOT CHANGE THIS AND DO NOT ADD ANY OPTIONS TO IT
\usepackage{algorithm}
\usepackage{algorithmic}

\usepackage{newfloat}
\usepackage{listings}
\DeclareCaptionStyle{ruled}{labelfont=normalfont,labelsep=colon,strut=off} % DO NOT CHANGE THIS
\floatstyle{ruled}
\newfloat{listing}{tb}{lst}{}
\floatname{listing}{Listing}
\newcommand{\modelname}{GraS$^2$P}
\usepackage{amsmath}
\usepackage{amsfonts}

\usepackage{commath}
\usepackage{bm}
\usepackage{amsthm}
\newtheorem{theorem}{Theorem}[section]
\newtheorem{lemma}[theorem]{Lemma}

\theoremstyle{definition}
\newtheorem{definition}{Definition}[section]
\theoremstyle{remark}

\usepackage{booktabs}
\usepackage{cleveref}
\usepackage{multirow}
\usepackage{subfigure}
\makeatletter
\DeclareRobustCommand\bfseries{%
	\not@math@alphabet\bfseries\mathbf
	\fontseries\bfdefault\selectfont
	\boldmath % <-- added
}
\usepackage{xcolor}

\title{Continuous-time Graph Representation with Sequential Survival Process}
\author{
    Abdulkadir \c{C}elikkanat, 
    Nikolaos Nakis, 
    Morten Mørup
}
\affiliations{
Department of Applied Mathematics and Computer Science\\
Technical University of Denmark\\
Kongens Lyngby, 2800\\
abce@dtu.dk, nnak@dtu.dk, mmor@dtu.dk 
}

\usepackage{bibentry}
\begin{document}

\maketitle

\begin{abstract}
Over the past two decades, there has been a tremendous increase in the growth of representation learning methods for graphs, with numerous applications across various fields, including bioinformatics, chemistry, and the social sciences. However, current dynamic network approaches focus on discrete-time networks or treat links in continuous-time networks as instantaneous events. Therefore, these approaches have limitations in capturing the persistence or absence of links that continuously emerge and disappear over time for particular durations. To address this, we propose a novel stochastic process relying on survival functions to model the durations of links and their absences over time. This forms a generic new likelihood specification explicitly accounting for intermittent edge-persistent networks, namely \textsc{\modelname}: Graph Representation with Sequential Survival Process. We apply the developed framework to a recent continuous time dynamic latent distance model characterizing network dynamics in terms of a sequence of piecewise linear movements of nodes in latent space. We quantitatively assess the developed framework in various downstream tasks, such as link prediction and network completion, demonstrating that the developed modeling framework accounting for link persistence and absence well tracks the intrinsic trajectories of nodes in a latent space and captures the underlying characteristics of evolving network structure.
\end{abstract}
\section{Introduction}\label{sec:introduction}

In diverse fields spanning physical and social sciences, entities ranging from minuscule scales like microorganisms and proteins to larger scales such as humans to scales of celestial objects like planets and galaxies always exert influence upon and interact with one another. These evolving and intricate interconnections inherently translate into networks, providing a versatile framework to encapsulate the subtle interplay of relationships. In this regard, networks (or graphs) have become essential for investigating and comprehending the intricate dynamics of these complex systems evolving over time \citep{newman2003structure}.

Representation learning models on graphs have gained popularity due to their ability to effectively extract knowledge from networks and achieve various objectives like predicting linkage and node properties \citep{hamilton_book,GRL-survey-ieeebigdata20}. However, their primary emphasis has been on static networks. The early works relied either on random walks \citep{deepwalk-perozzi14, grover2016node2vec, line} or matrix factorization techniques \cite{netmf-wsdm18, netsmf-www2019}. In recent years, Graph Neural Network (GNN) architectures have become a prominent way to address network embedding problems \cite{gnn_survey}, and a plethora of methods have also been developed to address a variety of network types, such as signed networks \citep{li2020learning,nakis23aistats} and knowledge graphs \citep{dai2020survey}, or to serve diverse purposes like encoding the hierarchical structure of networks in learning node embeddings \cite{louvainNE-wsdm20, nakis22hbdm}.  

Lately, there has been a growing interest in modeling and learning latent representations of temporal networks, encompassing the transient nature of node interactions \citep{xue2022dynamic}. The evolving focus seeks to unveil a richer understanding of node interactions throughout time, accounting for relationships' complex and evolving dynamics. Importantly, dynamic network modeling can thereby reveal intricate patterns within network structures that static approaches cannot adequately address. Initially, these dynamic modeling approaches focused on discrete time networks \citep{expl2,ishiguro2010dynamic, heaukulani2013dynamic, herlau2013modeling, yang2021discrete}. However, in recent years, substantial attention has also been devoted to the modeling of continuous-time networks. Prominent works have utilized Poisson \citep{fan2021continuous,dyrep, pivem} and Hawkes processes \citep{hp1,hp2,hawkes_1,hawkes_2,hawkes_3,zuo2018embedding,lu2019temporal,huang2022mutually,yang2017decoupling} in order to define principled learning procedures under continuous-time network likelihoods of event-based data. Contrary to the previous studies, which work on a network block level, the \textsc{HTNE} \citep{zuo2018embedding} extends the Hawkes process modeling to account for node-level embeddings. Furthermore, the GNN extensions for continuous-time dynamic networks, \textsc{TGN} \citep{rossi2020temporal}, and the temporal-point process of \textsc{M$^2$DNE} \citep{lu2019temporal} use a case-control approach optimizing a binary cross-entropy loss. In particular, \textsc{M$^2$DNE} takes into account both pairwise interactions at the micro level and broader network-wide dynamics at the macro level. Finally, non-likelihood-based procedures utilizing dynamic random walks such as (CTDNE) \citep{nguyen2018continuous} perform temporal random walks based on the observed continuous-time interactions. 

However, the currently existing approaches designed for modeling continuous-time dynamic networks exhibit significant limitations. In particular, when utilizing the event-based Poisson Process or extended Hawkes Process, they treat network links as instantaneous events, whereas the case-control approach using binary cross-entropy does not explicitly account for edge persistence in the likelihood. Nevertheless, numerous continuous-time dynamic networks in real-world scenarios surpass these perspectives. Links within these real-world networks often showcase intermittent patterns as interactions persist and dissipate consecutively over time. This nuanced nature of network dynamics necessitates a more comprehensive approach to accurately account for the persistent presence and absence of edges between nodes.

There are many prominent examples in which we can observe intermittent persistent linkage structures in real-world scenarios. For instance, consider a social media platform where users can follow or unfollow each other and thereby form a connection with each other over different time periods or contact and collaboration networks in which people can respectively be together or collaborate for extended periods of time. These intermittent persistent pairwise dynamics challenge traditional continuous-time dynamic network models that only account for the event of a tie but not its persistence and static models that assume constant and steady relationships. Although there are recent efforts modeling networks of intermittent link characteristics, they rely on certain assumptions, such as the stochastic block model \citep{rastelli2020stochastic,xu2015stochastic}, and are unable to produce continuous-time latent representations. There is, therefore, a need for new continuous-time dynamic network modeling approaches that can explicitly account for network connections that persist and dissipate consecutively over time.

In this paper, we introduce the continuous-time Graph Sequential Survival Process (\textsc{\modelname}). Specifically, we extend the traditional usage of Survival analysis to the realm of network science by developing a \textit{Sequential Survival} process that can capture the dynamic persistence of links and their absence in networks. To the best of our knowledge, this is the first approach capable of explicitly characterizing networks featuring intermittent time-persistent linkage structures. The main contributions can be outlined as:
\begin{itemize}
\item \textbf{A Novel Counting Process}. We introduce a novel stochastic process by leveraging the survival analysis to model the intermittent time-persistent linkage structure of the networks forming the \textsc{\modelname} model.	
 We further highlight the utility of the \textsc{\modelname} model considering the recently proposed continuous-time node embedding procedure \cite{pivem}.
 
\item \textbf{Experimental Validation}. We conduct extensive experiments on diverse real-world datasets to evaluate \textsc{\modelname}. The results showcase its effectiveness in capturing intricate characteristics of networks by explicitly accounting for intermittent edge persistence, and outperforming baseline methods in downstream tasks.
\item \textbf{Visualization Tool}. We show that the proposed approach can embed continuous-time edge persistent dynamic complex networks in low dimensional spaces accurately, thereby also serving as a visualization tool to get insights into the intricate temporal dynamics of link-persistent networks.
\end{itemize}

\noindent\textbf{Implementation}. The source codes and other details can be found at \url{https://abdcelikkanat.github.io/projects/grassp}
\section{Proposed Model}\label{sec:model}

In this section, we present our proposed approach, but before delving into the details, we will first establish the notations used throughout the paper. Without loss of generality,  we can suppose that the timeline begins at time $0$ and ends at $T$, and we will use $[T]$ to denote the time interval $[0,T)$. We employ the conventional notation, $\mathcal{G}=(\mathcal{V},\mathcal{E})$ to indicate a graph where $\mathcal{V}=\{1,\ldots,N\}$ is the vertex set and $\mathcal{E}:=\cup_{i,j\in\mathcal{V}}\mathcal{E}_{ij}$ refers to the edge set of the network, comprising of pairwise temporal links, $\mathcal{E}_{ij}$, for each pair $(i,j)\in\mathcal{V}^2$.

Again, it is worth emphasizing that we assume a pair of nodes consists of sequential links indicating intermittent interactions over time \cite{holme201297}. An existing link (i.e., interaction) can disappear and then emerge again. In this regard, we will utilize tuple $(i,j,t_k,t_{k+1})$ to denote a link between nodes $i$ and $j$ for the interval from $t_k\in[T]$ up to $t_{k+1}\in[T]$. We provide the formal description of the networks considered in this work in Definition \ref{defn:network} below:

% \begin{definition}[Continuous-time Intermittent Edge Persistent Graph]\label{defn:network}
\begin{definition}[Continuous-time Interval Graphs]\label{defn:network}
A \textit{continuous-time interval network} over a timeline $[T] := [0, T]$ is an ordered pair $\mathcal{G}=(\mathcal{V}, \mathcal{E})$ where $\mathcal{V} = \{1,\ldots,N\}$ is the set of nodes and $\mathcal{E}:=\{ (i,j, t_k, t_{k+1}) \in \mathcal{V}^2\times [T]^2 \mid t_k < t_{k+1} \}$ denotes the set of non-overlapping temporal links, i.e. if $(i,j,t_k, t_{k+1})$ and $(i,j,\bar{t}_l, \bar{t}_{l+1})$ are distinct links of $(i,j)$ pair, then it satisfies $[t_k, t_{k+1}) \cap [\bar{t}_l, \bar{t}_{l+1}) = \emptyset$.
\end{definition}

We call the initial and the last time of each link period as an \textit{event time}, and for practical purposes, we always suppose $0$ is also an event time for each node pair. In addition, we introduce the \textit{state} function, $s:\mathcal{V}^2 \times [T]\rightarrow\mathcal{S}$ as an indicator of the presence or absence of a link for a given time $t\in[T]$ where $\mathcal{S} := \{1, -1\}$ is the state space ($+1$ symbolizes the existence of the link and $-1$ its absence). Note that the state of each pair is constant until the next event time; thereby, we omit the input variable from the state function, $s(t)$, for convenience, and we write $s$ to denote the state of the interaction for the corresponding interval. In this regard, we can partition the timeline with respect to the values of the state function for each node pair (i.e., depending on whether a link exists or not), so if a pair consists of $M$ events $e_0 = 0 < e_1 < \cdots < e_m < \cdots < e_{M-1} < T$ then there must exist $M$ consecutive intervals, $\{[e_m, e_{m+1} )\subseteq [T] : \forall m\in  \{0,\ldots, M-1 \} \}$, having different states. 

% Even though networks with sporadic interactions over time are prevalent in real-world contexts such as contact or social networks, to the best of our knowledge, they have not been studied previously.
Even though networks with sporadic interactions over time are prevalent in real-world contexts such as contact or social networks, to the best of our knowledge, they have not been considered previously to learn the continuous-time latent node representations.

\subsection{Sequential Survival Process}\label{sec:sec_survival_process}
Many research fields strongly emphasize modeling the time duration required for an event to unfold, such as investigating the lifespan of living organisms or analyzing the reliability of mechanical systems. The term "survival" is mostly employed in those works to describe the duration leading up to the occurrence of death or failure, which is a measure that encapsulates the essence of lifetime estimation and plays a fundamental role in understanding the dynamics of complex systems. More formally, for a given continuous random variable, $T$, representing the lifetime of an object or a system, the survival function is given by 
\begin{align*}
S(t) := \mathbb{P}\{ T>t \} = \int_{t}^{\infty}f(u)\mathrm{d}u = 1 - F(t)
\end{align*}
where $F(t)$ and $f(t)$ indicate the cumulative distribution and probability density functions. It can also be reparameterized using an associated hazard function $\lambda:[0,\infty)\rightarrow \mathbb{R}^+$ as:
\begin{align*}
S(t) = \exp\left( -\int_{0}^{t}\lambda(t^{\prime})dt^{\prime} \right).
\end{align*}
For further details on survival analysis, we recommend unfamiliar readers refer to the work of \citep{jenkins2005survival_book}.
%for in-depth details. 

Our approach characterizes node interactions by utilizing the power of survival functions. As mentioned before, we always assume that the state of the model alters after each event time point. Therefore, we design our \textit{Sequential Survival} process as consecutive survival functions denoting "surviving" and "dying" events. In other words, for a given initial state, $s_0\in\mathcal{S}$, we define the process, $\{ M(t): t\geq 0 \}$ as a counting process showing the total number of occurrences or events that have happened up to time $t$. Hence, we write the probability of the random variable $M(t)$ being equal to $m$ as follows:
\begin{align}
p_{M(t)}(m) = \int\limits_{\bm{\xi}\in\mathcal{R}}\prod_{k=1}^{m}\frac{\int_{e_{k-1}}^{e_k}\lambda(s_k,t^{\prime})\mathrm{d}t^{\prime}}{ \exp\left(\int_{e_{k-1}}^{e_k}\lambda(s_k,t^{\prime})\mathrm{d}t^{\prime}\right) }\mathrm{d}\bm{\xi}\label{eq:seq_surv_prob}
\end{align}
where $\lambda(s_k, t)$ is the hazard rate for given time $t\in[T]$ and state $s_k \in \mathcal{S}$, and $\mathcal{R} := \{ (t_1, \ldots, t_m)\in [T]^m: 0\leq t_1 < t_2 < \cdots t_m < T) \}$ is the domain of the integration. 

We can also write the likelihood function of the process from the probability given in Eq. \eqref{eq:seq_surv_prob}. Let $\Xi = (\Phi, M(t))$ be a random variable where $M(t)$ denotes the number of events up to time $t$, and $\Phi$ is the corresponding ordered event sequence. Then, we can write the marginal distribution of $M(t)$ by integrating over all possible ordered sequences in the set $\mathcal{R}$. In other words,
\begin{align*}
p_{M(t)}(m) = \int_{\bm{\xi} \in \mathcal{R}} p_{(\Phi,M(t))}(\bm{\xi}, m)d\bm{\xi}, 
\end{align*}
and by using the fundamental theorem of calculus, we can obtain the probability density function of the random variable, $\Xi = (\Phi,M(t))$, evaluated at $((e_1,\ldots,e_m), m)$ as: %follows:
\begin{align}\label{eq:likelihood}
\!\!\!\!p_{(\Phi,M(t)}((e_1,\ldots,e_m), m) \!= \!\!\prod_{k=1}^{m}\!\frac{\lambda(s_k,e_k) }{ \exp\!\left(\int\limits_{e_{k-1}}^{e_k}\!\!\!\lambda(s_k,t^{\prime})\mathrm{d}t^{\prime}\!\!\right) }.
\end{align}

\subsection{Problem Formulation}
Our objective is to learn continuous-time node representations in a metric space $(\mathsf{X}, d_{\mathsf{X}})$ to uncover underlying temporal patterns of a network so the pairwise distances among nodes in a latent space should acquire the temporal changes within the network. We will use, $\mathbf{r}(i,t)$ or simply $\mathbf{r}_{i}(t)$, to denote the embedding of node $i\in\mathcal{V}$ at time $t\in[T]$ in a $ D$-dimensional space ($ D \ll |\mathcal{V}|$).
More specifically, we desire to obtain a map $\mathbf{r}:\mathcal{V}\times[T] \rightarrow \mathsf{X}$ satisfying 
\begin{align}
&\int_{e_l}^{e_u}\!\!\psi^{s}\left(d_{\mathsf{X}}\big(\mathbf{r}(i,t),\mathbf{r}(j,t)\big)\right)\mathrm{d}t \approx \int_{e_l}^{e_u}\!\!\bm{\lambda}_{ij}(s,t^{\prime})\mathrm{d}t^{\prime} && 
\end{align}
for a continuous function $\psi^{s}:\mathbb{R} \rightarrow \mathbb{R}^+$ and all $(i,j,s)\in\mathcal{V}^2\times\mathcal{S}$ where $\bm{\lambda}{(s, t)}$ indicates the \textit{true hazard rate} between $i$ and $j$ at time $t\in[T]$ and state $s\in\mathcal{S}$.  

Since we assume that a node pair has connections of alternating states (i.e., link or non-link periods) over time, we utilize the Sequential Survival process introduced in the previous part to characterize these intermittent persistent edges. In this regard, by using Eq \eqref{eq:likelihood}, we can write the log-likelihood function for the whole network as follows:
\begin{align}\label{eq:log_likelihood}
\!\!\!\!\mathcal{L}(\Omega|\mathcal{G})\nonumber \!:=& \log p(\mathcal{G}|\Omega)\nonumber
\\
\!=&\!\!\!\sum_{i,j\in\mathcal{V}}\sum_{\substack{m=1}}^{|\mathcal{E}_{ij}|} \!\Bigg( \!\!\log\lambda_{ij}(s_m,e_m) \!- \!\!\!\!\!\int\limits_{e_m}^{e_{m+1}}\!\!\!\!\lambda_{ij}(s_m,t)\mathrm{d}t \Bigg)
\end{align}
where $\Omega$ refers to the set of model hyper-parameters.

To learn continuous-time node dynamics, we consider the latent distance modeling framework \cite{exp1}.  We leverage the hazard functions to define the latent representations uncovering the evolving relationships between nodes in the network. Based on our assumption, when a pair of nodes has a link or interaction at a particular time, it is expected to dissolve eventually. As a result, their latent positions should also naturally drift apart from each other over time to reflect their temporal connection strength. Conversely, when they do not have any connection, they might interact in the future, so their latent positions should also approach each other to reflect the potential for a coming link. In this regard, we define the hazard function, $\lambda_{ij}(s, t)$ as follows:
\begin{align}\label{eq:hazard_function}
\lambda_{ij}(s, t) := \exp\left( \beta(s) + s\| \mathbf{r}_i(t) - \mathbf{r}_j(t) \|^2 \right).
\end{align}
for node pair $(i,j)\in\mathcal{V}^2$ and state $s\in\mathcal{S}:=\{-1, 1\}$. We incorporate bias terms ($\beta(s)$) for each state value in the definition of the hazard function given in Eq. \eqref{eq:hazard_function}, and they are responsible for capturing the global information in the network \citep{krivitsky2009representing,pivem}. We further use the squared Euclidean metric \cite{LSM_geo,poincare}. Using this formulation, Lemma \ref{lemma:bounds} ensures the latent representations of nodes will be positioned close enough or significantly distant from each other depending on the state (i.e., link or non-link periods) of the node pairs.

\begin{lemma}\label{lemma:bounds}
Let $e_k$ and $e_{k+1}$ be a consecutive event times following a Sequential Survival process for node pair $(i,j)\in\mathcal{V}^2$. Then, the average squared distance between nodes during interval $[e_k, e_{k+1})$ associated with survival function $S(\cdot)$ and state $s\in\{-1,1\}$ can be bounded by 
\begin{align*}
b(-1) &\leq \frac{1}{(e_{k+1}\!-\!e_k)}\!\!\!\int\limits_{e_k}^{e_{k+1}}\!\!\!\! \| \mathbf{r}_i(t) - \mathbf{r}_j(t) \|^2 dt \leq b(+1)
\end{align*}
where $b(s)\!:=\!\!-2s\log(e_{k+1}\!-e_k)-\!s\log{ \!S(e_{k+1}\!)}\!-\!s\beta{(s)}$.
% Let $e_0=0 < e_1 < \cdots < e_{M-1} < T$ be a sequence following a Sequential Survival process for node pair $(i,j)\in\mathcal{V}^2$. Then, the average squared distance between nodes during interval $[e_m, e_{m+1})$ associated with survival function $S_m(\cdot)$ and state $s_m\in\{-1,1\}$ can be bounded by 
% \begin{align*}
% b_m(-1) &\leq \frac{1}{(e_{m+1}\!-\!e_m)}\!\!\!\int\limits_{e_m}^{e_{m+1}}\!\!\!\! \| \mathbf{r}_i(t) - \mathbf{r}_j(t) \|^2 dt \leq b_m(+1)
% \end{align*}
% where $b_m(s) \!:=\! -2s\log(e_{m+1}\!-\!e_m)+\log{ S(e_{m+1})}-s\beta{(s)}$.
\end{lemma}
\begin{proof}
% Please refer to the appendix for the proof.
Please refer to the appendix on the project page.
\end{proof}

\subsection{Continuous-time Node Representations using Piecewise Linear Approximations }
 For the embedding vectors $\{\mathbf{r}_{i}(t): i\in\mathcal{V}, t\in[T]\}$ we consider the continuous-time extension of the latent distance model proposed in \cite{pivem} in the context of event-based (Poisson Process likelihood) graphs using analytically tractable piecewise linear approximations of latent dynamics. Specifically, we define each node embedding as a linear function depending on time: 
\begin{align}
\mathbf{r}_{i}(t) := \mathbf{x}_i + \mathbf{v}_i t
\end{align}
The definition can be understood as assigning the initial position ($\mathbf{x}_i$) and velocity ($\mathbf{v}_i$) to each node, enabling us to locate the node's position in the latent space at any given time. However, it also constrains the motion capacity of nodes in the embedding space, as they are limited to moving in a single direction. To overcome this limitation, the model is extended by dividing the timeline into $B$ equal-sized bins, introducing bin-specific velocity vectors. More specifically, the latent position of node $i\in\mathcal{V}$ at time $t\in[T]$ is given by
\begin{align}\label{eq:piecewise_definition}
\mathbf{r}_i(t) &:= \mathbf{x}^{(0)}_i + \Delta_B\mathbf{v}_i^{(1)} + \Delta_B\mathbf{v}_i^{(2)} +\cdots+ \nonumber 
\\
&\quad +\Delta_B\mathbf{v}_i^{(b)}+ \cdots + \text{mod}(t,\Delta_B)\mathbf{v}_i^{\left(\lfloor t/\Delta_B \rfloor+1\right)}
\end{align}
where $\Delta_B$ is the bin width (i.e. $T / B$), and $\text{mod}(\cdot, \cdot)$ is the modulo operation giving the remaining time. Importantly, employing such a piecewise interpretation of the timeline enables tracking the paths of nodes in the embedding space effectively, and by augmenting the number of bins, more accurate trajectories can be obtained. In particular, the use of finer-grained divisions in the timeline allows for a more detailed and precise representation of node movements, leading to improved accuracy in capturing their dynamics within the embedding space \cite{pivem}.

\subsubsection{Regularization.} In order to control the nodes' mobility in the latent space, we incorporate a prior distribution for the velocity vectors. Imagine a situation for a pair of nodes only interacting with each other during a period; the model situates them closely in the embedding space when they have a link. Nevertheless, their distance in the latent space tends towards infinity as the link is inactive. Therefore, we assume the velocity vectors, $\mathbf{v}\in\mathbb{R}^{B\times N\times D}$, follow a multivariate normal distribution with zero mean:
\begin{align*}
\text{vect}(\mathbf{v}) \sim \mathcal{N}(\bm{0}, \lambda^2\Sigma)
\end{align*}
where $\lambda$ is the scaling coefficient, and $\Sigma\in\mathbb{R}^{BND \times BND}$ is a diagonal matrix defined as a Kronecker product of three other vectors. In other words, $\Sigma:= \text{diag}(\bm{\sigma}_B \otimes \bm{\sigma}_N \otimes \bm{\sigma}_D)$ where the vectors, $\bm{\sigma}_B$, $\bm{\sigma}_N$ and $\bm{\sigma}_D$ are responsible for the influence of the model's bins, nodes, and dimensions, respectively. Here, the notation, $\otimes$, symbolizes the Kronecker product, and $\text{vect}(\mathbf{z})$ represents the vectorization operator converting the given tensor into a vector form. We constrain $\bm{\sigma}_B$ and $\bm{\sigma}_N$ within the standard $(B-1)$ and $(N-1)$-simplex sets, and we define $\bm{\sigma}_D$ as $\bm{1}_D=(1,1,\ldots,1)\in\mathbb{R}^D$ to have uncorrelated dimensions. To sum up, we can restrain the embedding space by utilizing the prior distribution since it allows us to control the motions of nodes. For higher values of the scaling factor $\lambda$, the model can have more flexibility, enabling more dynamic node movements in the latent space, whereas lower values restrict node mobility, resulting in more static node representations. Notably, with this regularization, the model can be considered a continuous-time extension of the discrete-time latent distance model based on diffusion considered in \cite{expl2} in which the diffusion between time-bins propagates continuously.

\subsection{Inference}
Our objective function defined by the log-likelihood given in Eq. \eqref{eq:log_likelihood} together with the log-prior regularization term is a non-convex function, so the learning strategy applied for inferring the model's hyper-parameters is crucial to avoid poor-quality local minima in the resulting representations. 

The position vectors ($\mathbf{x}$) are initialized uniformly within the $[-1, 1]$ range at random. The bias terms $(\bm{\beta})$ and velocities ($\mathbf{v}$) are sampled from the standard normal distribution. The prior parameters, ($\bm{\sigma}_{B}$, $\bm{\sigma}_{N}$) are set to $\mathbf{1}_B/B$ and $\mathbf{1}_N/N$ at the beginning. We follow a sequential learning strategy for training the model, i.e., we optimize different sets of parameters in stages. Firstly, we optimize the velocities ($\mathbf{v}$) for $100$ epochs. Then, we include the initial positions ($\mathbf{x}$) into the optimization procedure, and we continue to train the model by optimizing these two parameters ($\mathbf{x}$, $\mathbf{v}$) together for another $100$ epochs. Finally, we incorporate the bias and prior parameters and optimize all model hyper-parameters together. In total, we use $300$ epochs for the whole learning procedure, and the \textit{Adam optimizer} \cite{kingma2017adam} is employed with an initial learning rate of $0.1$. In the experiments, we set the number of bins ($B$) to $100$ to ensure sufficient capacity for tracking nodes in the latent space ($D=2$).

\begin{figure*}[!ht]
\centering
\subfigure[$t=350$]{\includegraphics[trim={5cm 6cm 5cm 6cm},clip,scale=0.22]{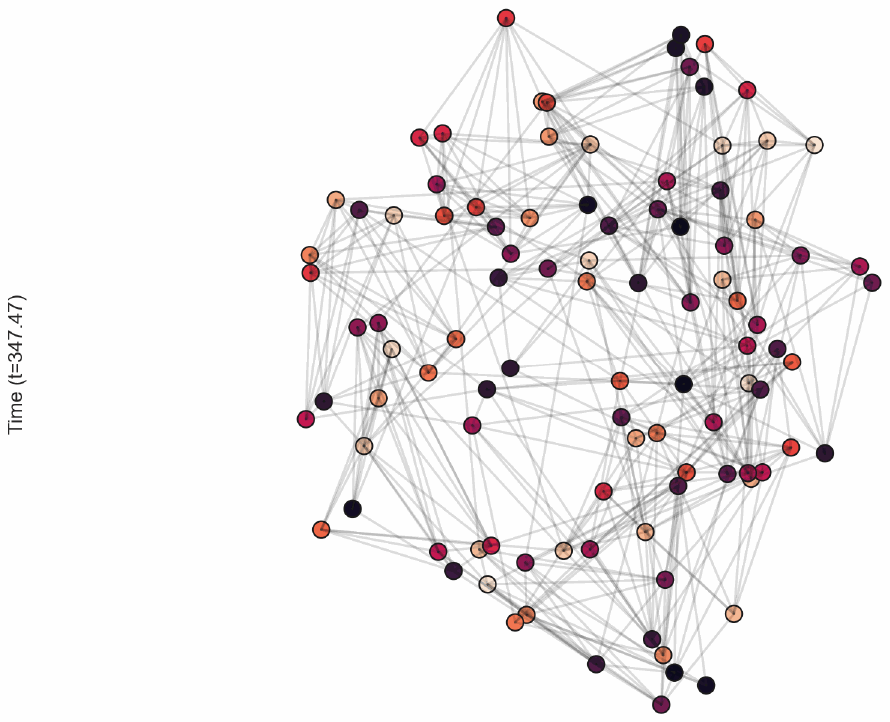}}
\hfill
\subfigure[$t=400$]{\includegraphics[trim={5cm 6cm 5cm 6cm},clip,scale=0.22]{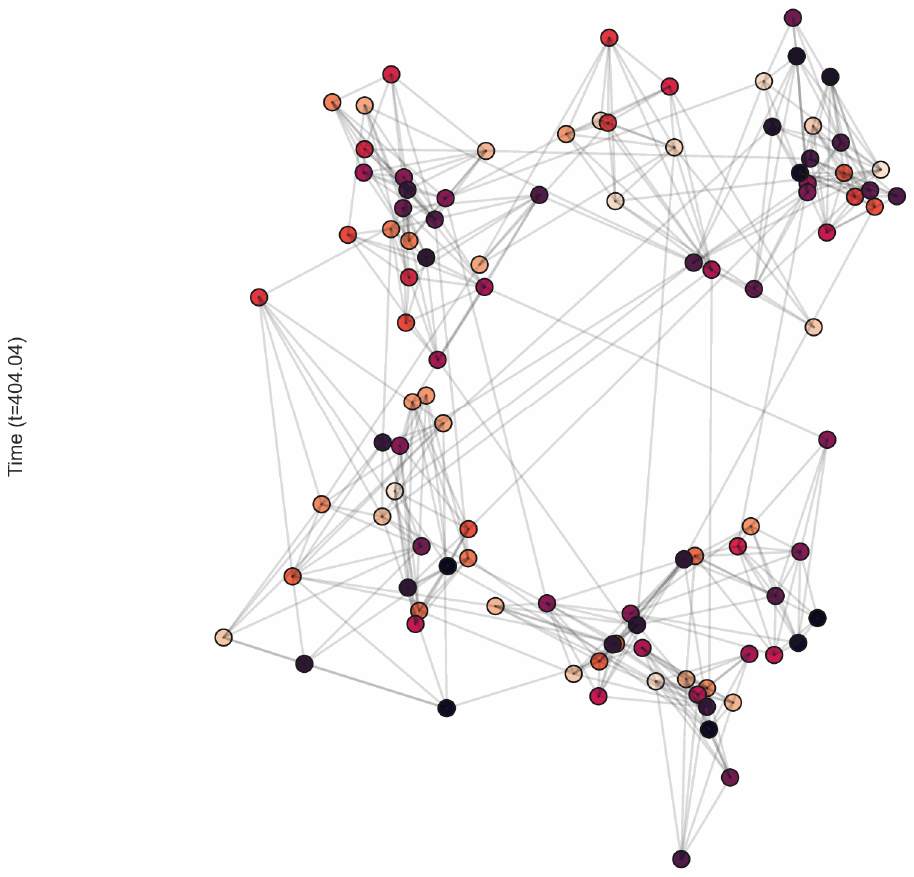}}
\hfill
\subfigure[$t=450$]{\includegraphics[trim={5cm 6cm 5cm 6cm},clip,scale=0.22]{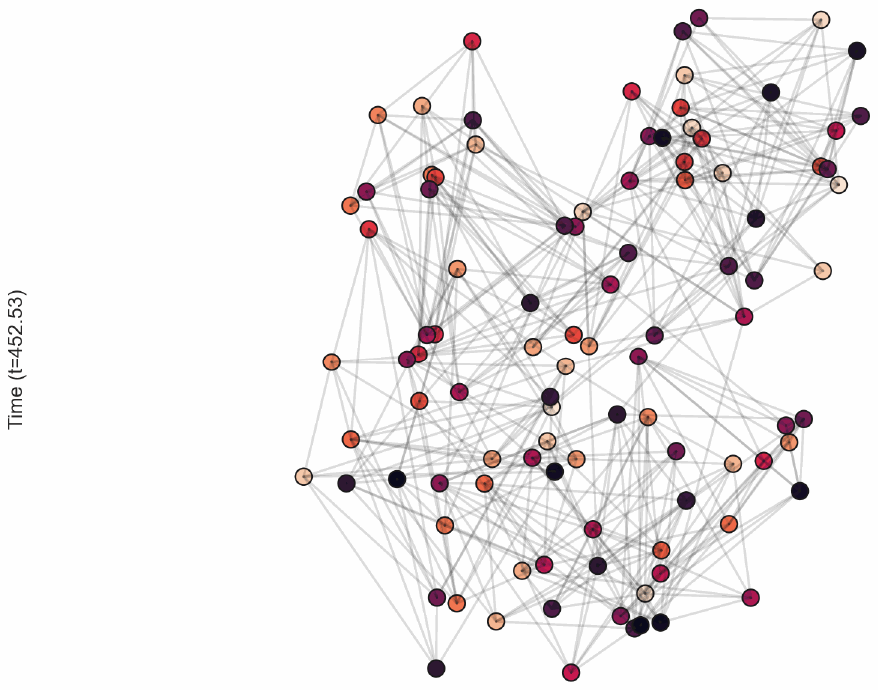}}
\hfill
\subfigure[$t=500$]{\includegraphics[trim={5cm 6cm 5cm 6cm},clip,scale=0.22]{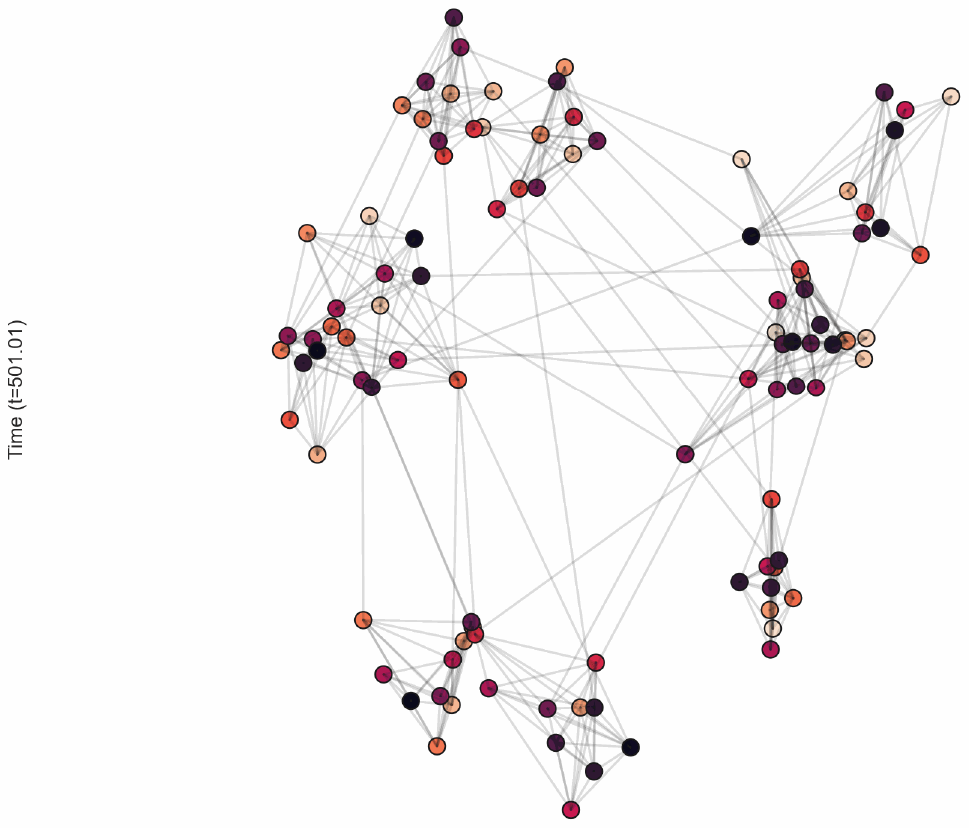}}
%%%%%%%%
\hfill
\subfigure[$t=550$]{\includegraphics[trim={5cm 6cm 5cm 6cm},clip,scale=0.22]{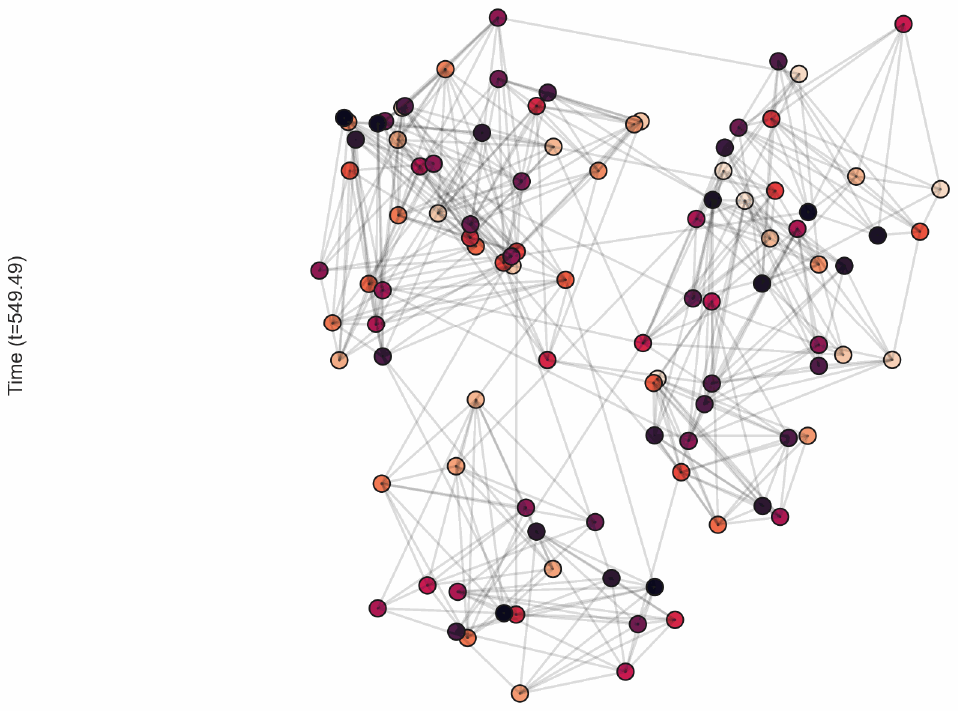}}
\hfill
\subfigure[$t=600$]{\includegraphics[trim={5cm 6cm 5cm 6cm},clip,scale=0.22]{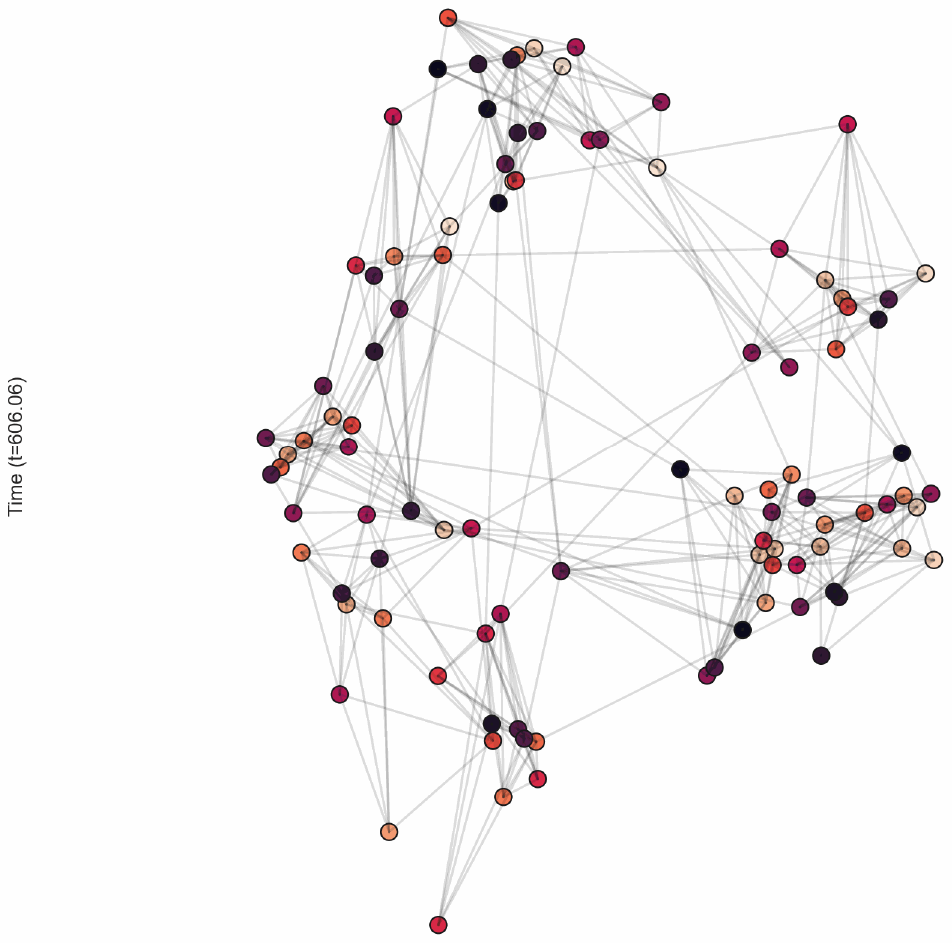}}
\caption{Snapshots of the continuous-time embeddings learned by \textsc{\modelname} for various time points over \textsl{Synthetic-$\beta$}.}\label{fig:visualization}
\end{figure*}

% \subsubsection{Complexity Analysis.}

\section{Experimental Evaluation}\label{sec:experiments}
In this section, we will examine the performance of the proposed model over a diverse range of networks varying in size and characteristics. But before delving into the experimental evaluations, we will first present details regarding the experimental setup, considered datasets, and baseline approaches.

\subsection{Experimental Setup}
We first split the networks into two sets, such that the events taking place within the last 10\% of the timeline are considered for the future prediction task. Furthermore, we randomly choose 20\% node pairs among all possible dyads in the initial first part, and they are divided into two equal-sized groups to design the validation and testing sets. The residual network does not contain any link from these dyads, and it forms the training set. If there is any node pair without any link period during the training time but included in the prediction set, it is also excluded from the network.

We need to generate the labeled data to perform link prediction tasks. For this purpose, we divide the timeline of each sampled dyad into segments based on the state values. Within these segments, we randomly select time $t$ to define a sample interval $[t-\epsilon, t+\epsilon]$, and $\epsilon$ is set to $10^{-2}\times T$ where $T$ is the dataset's timeline length. We deliberately exclude samples containing the event times where the state of the corresponding dyad changes since it is impractical to predict whether a link exists for the periods with multiple states.

We organized these generated samples into two categories as ``\textit{simple}'' and ``\textit{hard}'' sets. The ``\textit{hard}'' set consists of samples for the node pairs having at least one linked and non-linked period over time \citep{huang2023temporal}. Contrarily, dyads having stable states throughout the timeline produce ``\textit{simple}'' sets of samples since predicting the labels (i.e., state) of these instances is relatively straightforward. Additionally, the samples generated for the future link prediction task are categorized based on the dyads' linkage during the training time by following the study \citep{poursafaei2022towards}. 

% We consider an equal number of $k$ link and non-link samples, and the maximum sample size is limited to $10^3$. Each link or non-link category contains $h/2$ elements picked up from the \textit{hard} set consisting of $h$ elements, and we randomly select $k-h/2$ samples from the residual \textit{hard} instances and the \textit{simple} set. For the experiments, we report both AUC-ROC and AUC-PR scores to comprehensively evaluate the models' performances across different aspects of true and false positives and precision-recall characteristics.
We consider an equal number of $k$ link and non-link samples, and $k := \min\{10^3, \text{link set size}, \text{non-link set size}\}$. Each link or non-link category contains $h/2$ elements picked up from the \textit{hard} set of size $h$, and we randomly select the remaining $k-h/2$ samples from the residual \textit{hard} and the \textit{simple} instances. For the experiments, we report both AUC-ROC and AUC-PR scores to comprehensively evaluate the models' performances across different aspects of true and false positives and precision-recall characteristics.

\begin{table*}[!ht]
\centering
\caption{The performance comparison of the models for the network completion experiment across diverse datasets.}
\label{tab:completion}
\resizebox{0.9\textwidth}{!}{%
\begin{tabular}{rcccccccc}
\toprule
 &  & \textsc{LDM} & \textsc{Node2Vec} & \textsc{CTDNE} & \textsc{HTNE} & \textsc{M$^2$DNE} & \textsc{PiVeM} & \textsc{\modelname} \\\cmidrule(rl){2-2}\cmidrule(rl){3-3}\cmidrule(rl){4-4}\cmidrule(rl){5-5}\cmidrule(rl){6-6}\cmidrule(rl){7-7}\cmidrule(rl){8-8}\cmidrule(rl){9-9}
\multirow{2}{*}{\textsl{Synthetic-$\alpha$}} & ROC & $.711\pm.004$ & $.743\pm.002$ & $.692\pm.007$ & $.698\pm.021$ & $.558\pm.008$ & \underline{$.744\pm.002$} & \textbf{$.810\pm.009$} \\
 & PR & $.630\pm.006$ & \underline{$.667\pm.009$} & $.650\pm.007$ & $.645\pm.019$ & $.582\pm.004$ & $.653\pm.004$ & \textbf{$.751\pm.011$} \\\cmidrule(rl){2-9}
\multirow{2}{*}{\textsl{Synthetic-$\beta$}} & ROC & $.491\pm.020$ & $.534\pm.008$ & $.502\pm.008$ & $.525\pm.004$ & $.517\pm.013$ & \underline{$.593\pm.006$} & \textbf{$.677\pm.018$} \\
 & PR & $.486\pm.016$ & $.498\pm.007$ & $.502\pm.010$ & $.517\pm.006$ & $.522\pm.015$ & \underline{$.587\pm.011$} & \textbf{$.646\pm.022$} \\\cmidrule(rl){2-9}
\multirow{2}{*}{\textsl{Contacts}} & ROC & $.508\pm.008$ & \underline{$.584\pm.004$} & $.564\pm.034$ & $.472\pm.024$ & $.486\pm.013$ & $.493\pm.006$ & \textbf{$.680\pm.013$} \\
 & PR & $.490\pm.004$ & \underline{$.555\pm.023$} & $.543\pm.036$ & $.477\pm.023$ & $.500\pm.008$ & $.492\pm.016$ & \textbf{$.641\pm.023$} \\\cmidrule(rl){2-9}
\multirow{2}{*}{\textsl{HyperText}} & ROC & \underline{$.541\pm.015$} & $.533\pm.012$ & $.462\pm.016$ & $.441\pm.017$ & $.461\pm.021$ & $.426\pm.013$ & \textbf{$.692\pm.010$} \\
 & PR & \underline{$.503\pm.010$} & $.490\pm.013$ & $.477\pm.016$ & $.449\pm.009$ & $.479\pm.023$ & $.437\pm.007$ & \textbf{$.656\pm.024$} \\\cmidrule(rl){2-9}
\multirow{2}{*}{\textsl{Infectious}} & ROC & \underline{$.689\pm.007$} & $.671\pm.003$ & $.639\pm.006$ & $.653\pm.013$ & $.554\pm.005$ & $.669\pm.004$ & \textbf{$.742\pm.026$} \\
 & PR & \underline{$.615\pm.007$} & $.601\pm.005$ & $.593\pm.005$ & $.596\pm.010$ & $.560\pm.009$ & $.598\pm.004$ & \textbf{$.673\pm.024$} \\\cmidrule(rl){2-9}
\multirow{2}{*}{\textsl{Facebook}} & ROC & $.717\pm.004$ & $.675\pm.001$ & $.539\pm.005$ & $.608\pm.001$ & $.570\pm.010$ & \underline{$.710\pm.002$} & \textbf{$.723\pm.010$} \\
 & PR & $.659\pm.006$ & $.603\pm.005$ & $.538\pm.013$ & $.575\pm.001$ & $.562\pm.009$ & \underline{$.662\pm.002$} & \textbf{$.671\pm.012$} \\\cmidrule(rl){2-9}
\multirow{2}{*}{\textsl{NeurIPS}} & ROC & $.679\pm.010$ & $.697\pm.005$ & $.558\pm.020$ & $.654\pm.025$ & $.531\pm.005$ & \textbf{$.748\pm.010$} & \underline{$.735\pm.029$} \\
 & PR & $.618\pm.016$ & $.606\pm.020$ & $.552\pm.025$ & $.613\pm.026$ & $.553\pm.011$ & \textbf{$.761\pm.020$} & \underline{$.749\pm.021$}\\\bottomrule
\end{tabular}%
}
\end{table*}

\begin{table*}[!ht]
\centering
\caption{The performance comparison of the models for the network reconstruction experiment across diverse datasets.}
\label{tab:reconstruction}
\resizebox{0.9\textwidth}{!}{%
\begin{tabular}{rcccccccc}
\toprule
 &  & \textsc{LDM} & \textsc{Node2Vec} & \textsc{CTDNE} & \textsc{HTNE} & \textsc{M$^2$DNE} & \textsc{PiVeM} & \textsc{\modelname} \\\cmidrule(rl){2-2}\cmidrule(rl){3-3}\cmidrule(rl){4-4}\cmidrule(rl){5-5}\cmidrule(rl){6-6}\cmidrule(rl){7-7}\cmidrule(rl){8-8}\cmidrule(rl){9-9}
\multirow{2}{*}{\textsl{Synthetic-$\alpha$}} & ROC & $.702\pm.002$ & $.693\pm.003$ & $.638\pm.006$ & $.675\pm.011$ & $.507\pm.002$ & \underline{$.749\pm.002$} & \textbf{$.845\pm.006$} \\
 & PR & $.654\pm.006$ & $.627\pm.011$ & $.596\pm.009$ & $.639\pm.007$ & $.566\pm.003$ & \underline{$.665\pm.002$} & \textbf{$.782\pm.009$} \\\cmidrule(rl){2-9}
\multirow{2}{*}{\textsl{Synthetic-$\beta$}} & ROC & $.564\pm.009$ & $.507\pm.006$ & $.512\pm.008$ & $.544\pm.007$ & $.511\pm.002$ & \underline{$.680\pm.006$} & \textbf{$.744\pm.019$} \\
 & PR & $.553\pm.006$ & $.494\pm.005$ & $.511\pm.007$ & $.528\pm.007$ & $.513\pm.002$ & \underline{$.652\pm.008$} & \textbf{$.701\pm.013$} \\\cmidrule(rl){2-9}
\multirow{2}{*}{\textsl{Contacts}} & ROC & \underline{$.593\pm.004$} & $.556\pm.004$ & $.534\pm.018$ & $.528\pm.004$ & $.534\pm.002$ & $.496\pm.006$ & \textbf{$.825\pm.008$} \\
 & PR & \underline{$.541\pm.003$} & $.523\pm.015$ & $.528\pm.017$ & $.510\pm.008$ & $.537\pm.004$ & $.465\pm.002$ & \textbf{$.754\pm.014$} \\\cmidrule(rl){2-9}
\multirow{2}{*}{\textsl{HyperText}} & ROC & \underline{$.550\pm.002$} & $.535\pm.004$ & $.477\pm.012$ & $.473\pm.011$ & $.489\pm.003$ & $.430\pm.002$ & \textbf{$.760\pm.004$} \\
 & PR & \underline{$.513\pm.003$} & $.507\pm.007$ & $.488\pm.010$ & $.470\pm.008$ & $.479\pm.004$ & $.431\pm.001$ & \textbf{$.689\pm.007$} \\\cmidrule(rl){2-9}
\multirow{2}{*}{\textsl{Infectious}} & ROC & \underline{$.701\pm.006$} & $.688\pm.003$ & $.667\pm.005$ & $.676\pm.009$ & $.579\pm.002$ & $.666\pm.005$ & \textbf{$.788\pm.015$} \\
 & PR & \underline{$.626\pm.008$} & $.602\pm.007$ & $.606\pm.008$ & $.613\pm.008$ & $.584\pm.005$ & $.577\pm.006$ & \textbf{$.697\pm.013$} \\\cmidrule(rl){2-9}
\multirow{2}{*}{\textsl{Facebook}} & ROC & \underline{$.682\pm.005$} & $.645\pm.003$ & $.544\pm.007$ & $.624\pm.002$ & $.582\pm.009$ & $.673\pm.003$ & \textbf{$.731\pm.010$} \\
 & PR & $.615\pm.009$ & $.589\pm.008$ & $.535\pm.008$ & $.590\pm.009$ & $.573\pm.009$ & \underline{$.617\pm.004$} & \textbf{$.667\pm.013$} \\\cmidrule(rl){2-9}
\multirow{2}{*}{\textsl{NeurIPS}} & ROC & \underline{$.760\pm.007$} & $.720\pm.003$ & $.598\pm.004$ & $.731\pm.008$ & $.594\pm.001$ & $.698\pm.002$ & \textbf{$.889\pm.013$} \\
 & PR & $.687\pm.010$ & $.631\pm.007$ & $.590\pm.010$ & $.659\pm.006$ & $.599\pm.003$ & \underline{$.711\pm.002$} & \textbf{$.819\pm.020$}\\\bottomrule
\end{tabular}%
}
\end{table*}

\begin{table*}[!ht]
\centering
\caption{The performance comparison of the models for the future link prediction experiment across diverse datasets.}
\label{tab:future_prediction}
\resizebox{0.9\textwidth}{!}{%
\begin{tabular}{rcccccccc}
\toprule
 &  & \textsc{LDM} & \textsc{Node2Vec} & \textsc{CTDNE} & \textsc{HTNE} & \textsc{M$^2$DNE} & \textsc{PiVeM} & \textsc{\modelname} \\\cmidrule(rl){2-2}\cmidrule(rl){3-3}\cmidrule(rl){4-4}\cmidrule(rl){5-5}\cmidrule(rl){6-6}\cmidrule(rl){7-7}\cmidrule(rl){8-8}\cmidrule(rl){9-9}
\multirow{2}{*}{\textsl{Synthetic-$\alpha$}} & ROC & $.748\pm.007$ & $.756\pm.005$ & $.652\pm.012$ & \underline{$.784\pm.013$} & $.654\pm.011$ & $.740\pm.007$ & \textbf{$.902\pm.011$} \\
 & PR & $.719\pm.012$ & $.700\pm.020$ & $.636\pm.019$ & \underline{$.800\pm.016$} & $.745\pm.008$ & $.741\pm.005$ & \textbf{$.918\pm.008$} \\\cmidrule(rl){2-9}
\multirow{2}{*}{\textsl{Synthetic-$\beta$}} & ROC & $.515\pm.018$ & $.538\pm.004$ & $.503\pm.020$ & $.560\pm.006$ & $.519\pm.012$ & \textbf{$.894\pm.005$} & \underline{$.880\pm.012$} \\
 & PR & $.525\pm.021$ & $.501\pm.007$ & $.494\pm.016$ & $.548\pm.004$ & $.554\pm.014$ & \textbf{$.845\pm.007$} & \underline{$.843\pm.014$} \\\cmidrule(rl){2-9}
\multirow{2}{*}{\textsl{Contacts}} & ROC & \textbf{$.821\pm.004$} & $.703\pm.002$ & $.635\pm.013$ & $.727\pm.002$ & $.590\pm.002$ & $.692\pm.005$ & \underline{$.793\pm.013$} \\
 & PR & \textbf{$.773\pm.005$} & $.648\pm.006$ & $.599\pm.014$ & $.689\pm.004$ & $.610\pm.006$ & $.675\pm.004$ & \underline{$.752\pm.019$} \\\cmidrule(rl){2-9}
\multirow{2}{*}{\textsl{HyperText}} & ROC & \textbf{$.663\pm.004$} & $.553\pm.003$ & $.503\pm.010$ & $.530\pm.018$ & $.548\pm.004$ & $.559\pm.003$ & \underline{$.654\pm.005$} \\
 & PR & \underline{$.609\pm.003$} & $.516\pm.008$ & $.503\pm.006$ & $.518\pm.014$ & $.529\pm.008$ & $.534\pm.002$ & \textbf{$.612\pm.010$} \\\cmidrule(rl){2-9}
\multirow{2}{*}{\textsl{Infectious}} & ROC & \textbf{$.958\pm.004$} & $.869\pm.002$ & $.847\pm.008$ & $.893\pm.013$ & $.655\pm.008$ & \underline{$.945\pm.006$} & $.943\pm.017$ \\
 & PR & \textbf{$.943\pm.008$} & $.818\pm.007$ & $.820\pm.014$ & $.853\pm.007$ & $.698\pm.009$ & \underline{$.932\pm.006$} & $.923\pm.025$ \\\cmidrule(rl){2-9}
\multirow{2}{*}{\textsl{Facebook}} & ROC & \textbf{$.781\pm.007$} & $.694\pm.003$ & $.564\pm.005$ & $.626\pm.003$ & $.609\pm.015$ & \underline{$.775\pm.002$} & $.705\pm.009$ \\
 & PR & \underline{$.765\pm.009$} & $.653\pm.004$ & $.557\pm.004$ & $.599\pm.011$ & $.603\pm.011$ & \textbf{$.766\pm.003$} & $.648\pm.009$ \\\cmidrule(rl){2-9}
\multirow{2}{*}{\textsl{NeurIPS}} & ROC & $.682\pm.019$ & \underline{$.695\pm.012$} & $.637\pm.007$ & $.676\pm.014$ & $.661\pm.006$ & $.623\pm.010$ & \textbf{$.820\pm.008$} \\
 & PR & $.634\pm.024$ & $.621\pm.018$ & $.615\pm.015$ & $.635\pm.025$ & \underline{$.674\pm.014$} & $.628\pm.006$ & \textbf{$.788\pm.018$}\\\bottomrule
\end{tabular}%
}
\end{table*}

\subsection{Datasets}
We treat the networks as undirected and employ the finest available temporal granularity for the input timestamps, including measurements at the level of seconds and milliseconds. We tailor the datasets according to the chosen baselines to enable a meaningful comparison. For instance, we transform dynamic networks into static weighted and unweighted networks by aggregating links over time for static baselines. Additionally, we exclude the non-link events for the baselines since they cannot process these data points. 

In the experiments, we have used several real datasets of diverse types, including a social network (\textsl{Facebook}) \citep{viswanath2009evolution}, collaboration graph (\textsl{NeurIPS}), and three contact networks \citep{genois2015data,isella2011s}. 
% Due to the constraints on the page number, we provide the details regarding the real datasets in the appendix. 
Due to the constraints on the page number, we provide the details regarding the real datasets on the project page.
We also generated two synthetic networks to examine the model's predictive behavior in controlled settings. The link and non-link event times of the \textsl{Synthetic-$\alpha$} graph are generated from the Sequential Survival process introduced in Section \ref{sec:model}, and the initial embedding locations and velocities are sampled from a multivariate normal distribution as described in \cite{pivem}. For the \textsl{Synthetic-$\beta$} network, we divide the timeline into equal-sized $8$ bins and randomly group the nodes into $10$ clusters for each bin. Then, we establish connections between nodes within the same cluster with a probability of $0.8$, while nodes belonging to different clusters are linked with a $10^{-2}$ probability. These links stay persistent until the next bin.

\subsection{Baselines}
Due to the lack of an analogous approach to our proposed model explicitly accounting for dynamic networks with intermittent persistent linkage structures, we establish several static and dynamic methods as baselines to compare with the performance of our model. (i) \textsc{LDM} \citep{krivitsky2009representing,hoff2005bilinear} is a static node embedding approach for unweighted graphs in which we used the Poisson formulation in \cite{nakis22hbdm} when modeling the links with node-specific random effect terms. (ii) \textsc{PiVeM} \citep{pivem} can be viewed as an expansion of the \textsc{LDM} model, which models event-based continuous-time dynamic networks using the Nonhomogeneous Poisson process as opposed to the presently considered \textsc{\modelname}. (iii) \textsc{Node2Vec} \citep{grover2016node2vec} is very well-known random-walk based node embedding and we here report best obtained result considering both unweighted and weighted graphs. (iv) \textsc{CTDNE} \citep{nguyen2018continuous} also perform random walks but it concentrates on the temporal networks. (v) \textsc{HTNE} \citep{zuo2018embedding} relies on the Hawkes process to learn node embeddings and to model the neighbor arrival events. Finally, (vi) \textsc{M$^2$DNE} \citep{lu2019temporal} embeds the dynamic graphs by capturing the micro and macro network properties using case-control inference with cross-entropy loss. For all embedding methods we set $D=2$.

\subsection{Link Prediction}
We assess performances on three different tasks with hyperparameters set to  best performance on the validation sets.

\subsubsection{Network Reconstruction.} Our goal is to see how effectively the models can capture the temporal structural changes within the network over time. In pursuit of this objective, we seek to reconstruct both link and non-link periods. As depicted in Table \ref{tab:reconstruction}, our method, $\textsc{\modelname}$, exhibits notably superior performance compared to the baseline approaches. This marked improvement is attributed to the incapability of the other models to represent intermittent persistent linkage structures accurately.

\subsubsection{Network Completion.} Many real networks often contain noisy or missing links for various reasons, such as problems in the data collection processes or privacy concerns preventing the full disclosure of ties. In this regard, our aim is to evaluate the models' capacity to generalize the linkage structure in the training set. Table \ref{tab:completion} illustrates that our model once more demonstrates a substantial performance advantage over the baselines, except for the \textsl{NeurIPS} dataset. Given the network's yearly time resolution, the event-based approach, \textsc{PiVeM}, effectively captures its temporal structure. Importantly, our approach also displays a comparable level of performance in this scenario.

\subsubsection{Future Link Prediction.} The absence of a predictable periodic linkage pattern in the networks poses a significant challenge in forecasting future connections, particularly those at more distant time points. This complexity is evident in Table \ref{tab:future_prediction}, where our model consistently surpasses all baselines on the \textsl{Synthetic-$\alpha$} dataset, but its performance for \textsl{Synthetic-$\beta$} is not optimal. It can be explained by the fact that the links in the \textsl{Synthetic-$\alpha$} network are sampled from the \textit{Sequential Survival} process, but the temporal clusters in the \textsl{Synthetic-$\beta$} network are randomly formed (please see the dataset section for details). For specific network structures, static embedding models also showcase satisfactory performance since they are able to capture the global network information due to the aggregation of the links over time. Similarly, by choosing small values for the covariance factor, $\lambda$, we can restrict the dynamics of our approach. 

\begin{figure}
\centering
\subfigure[Impact of dimension size]{\includegraphics[width=0.49\columnwidth]{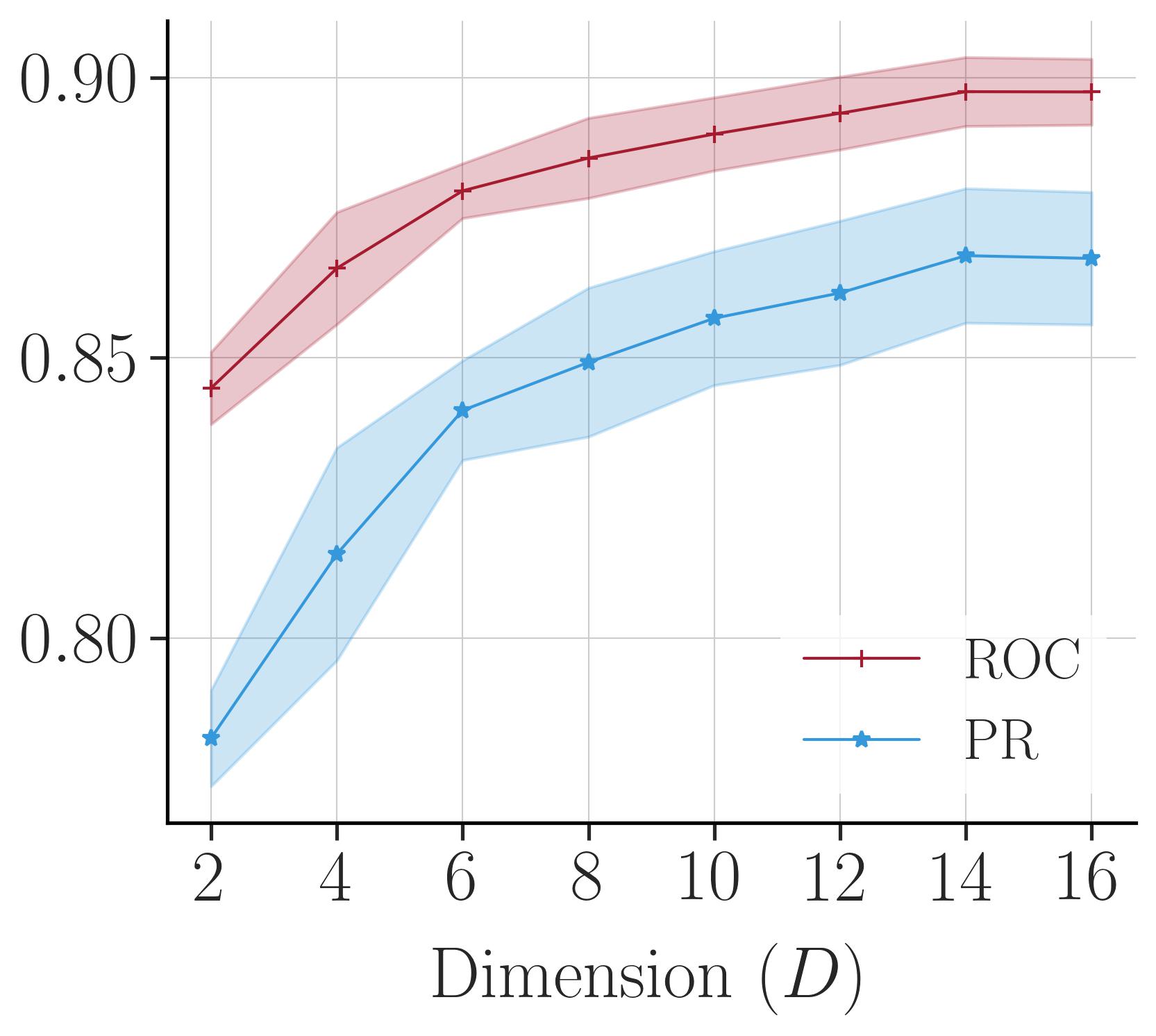}\label{fig:dim_influence}}
\hfill
\subfigure[Impact of bin count]{\includegraphics[width=0.49\columnwidth]{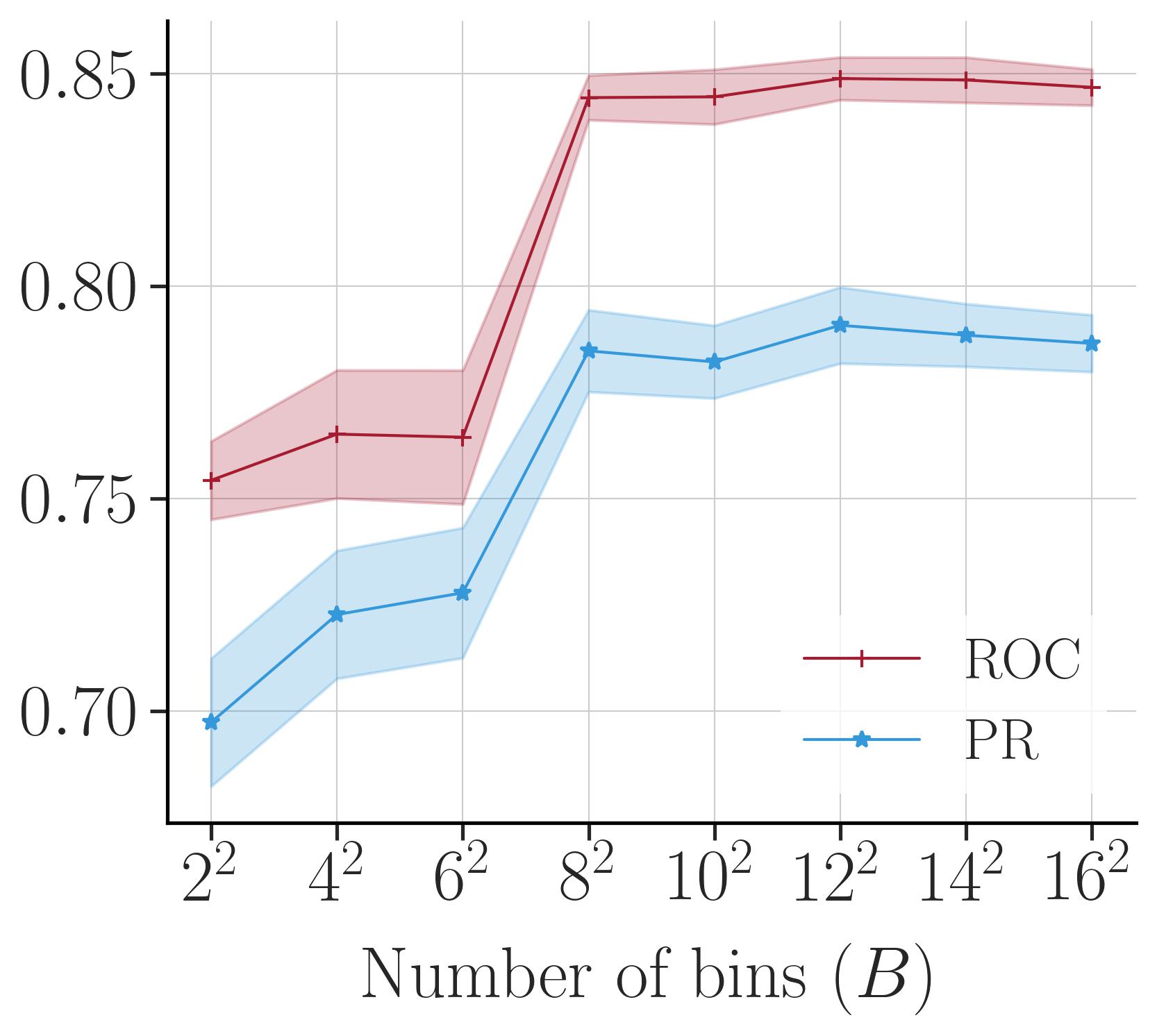}\label{fig:bin_influence}}
\caption{Influence of the hyper-parameters for the network reconstruction experiment over the \textsl{Synthetic-$\alpha$} dataset.}
\end{figure}

\subsubsection{Impact of dimension size.} Figure \ref{fig:dim_influence} shows a clear correlation between the increase in dimension size and performance improvement over the \textsl{synthetic-$\alpha$} network. With the introduction of each new dimension, the model's capacity augments, and the model becomes more adept at capturing intricate patterns within the network. For smaller dimensions, the model demonstrates performance on par with more high dimensions.  It is important to highlight that the two-dimensional representations sustain competitive performance enabling easy visualization and extraction of insights into the complex and dynamic nature of networks.

\subsubsection{Impact of bin count.} The impact of the number of bins on performance improvements is evident in Figure \ref{fig:bin_influence}, as the model's capacity to capture subtle temporal changes increases at finer granularity levels. Notably, the model's performance nearly reaches near-optimal performance around $B=64$, after which it demonstrates saturation.

\subsubsection{Continuous-time Dynamic Visualization.} Network visualization offers valuable insights for practitioners into the intricate architectures of complex networks. Nonetheless, numerous methods necessitate high-dimensional spaces to achieve satisfactory results for downstream tasks. Therefore, practitioners must utilize dimension-reduction techniques to generate visualizations conducive to human comprehension. Furthermore, many temporal models yield static embeddings, lacking the capacity to produce continuous-time node representations. In this regard, our model proves to be a versatile tool, effectively balancing the tradeoff between performance and dimension sizes. 

Figure \ref{fig:visualization} showcases the acquired embeddings across various selected time instances over the \textsl{Synthetic-$\beta$} dataset. The network takes an entirely new structural form for each time point $t=100k$ ($k\in\{4,\ldots,6\}$), including $10$ new clusters appearing randomly. Our model learns these temporal structures, particularly when interconnections between clusters remain relatively sparse. 
As time progresses, nodes within the latent space gradually adjust their positions to align with the evolving new random connections, making the clusters indistinguishable ($t=100k-50$ where $k\in\{4,\ldots,6\}$). 
% Additional visualizations are included in the appendix.
Additional visualizations are included on the project page.

\section{Conclusion}\label{sec:conclusion}

In this study, we introduced a novel representation learning model, \textsc{\modelname},  designed specifically for continuous-time networks exhibiting intermittent persistent linkage patterns over time. Our proposed approach characterizes node connections using the proposed Sequential Survival process. Our experimental results clearly demonstrate the superiority of \textsc{\modelname} over established baseline methods across multiple networks of varying properties. Notably, our model effectively balances the trade-off between dimensionality and performance, making it a valuable tool for visualizations. 

We aim to extend the methodology for diverse forms of temporal graphs, including weighted and signed networks. Moreover, we plan to tailor our approach to address large-scale networks while also capturing potentially recurring periodic structures within the networks considering different model specifications beyond piecewise linear dynamics.

\section*{Acknowledgments}
We appreciate the reviewers for their constructive feedback and their insightful comments. We gratefully acknowledge the Independent Research Fund Denmark for supporting this work [grant number: 0136-00315B].

\bibliography{references} %\bibliography{aaai24}

\clearpage
\appendix
% \section{Appendix}

\section{Theoretical Results}\label{sec:appendix_theoretical_results}

% \begin{lemma}\label{lemma:appendix_bounds_proof}
% Let $e_0=0 < e_1 < \cdots < e_{M-1} < T$ be a sequence following a Sequential Survival process for node pair $(i,j)\in\mathcal{V}^2$. Then, the average squared distance between nodes during interval $[e_k, e_{k+1})$ associated with survival function $S_k(\cdot)$ and state $s\in\{-1,1\}$ can be bounded by 
% \begin{align*}
% b_k(-1) &\leq \frac{1}{(e_{k+1}\!-\!e_k)}\!\!\!\int\limits_{e_k}^{e_{k+1}}\!\!\!\! \| \mathbf{r}_i(t) - \mathbf{r}_j(t) \|^2 dt \leq b_k(+1)
% \end{align*}
% where $b_k(s)\!:=\!\!-2s\log(e_{k+1}\!-e_k)-\!s\log{ \!S_k(e_{k+1}\!)}\!-\!s\beta{(s)}$.
% \end{lemma}
\begin{lemma}\label{lemma:appendix_bounds_proof}
Let $e_k$ and $e_{k+1}$ be a consecutive event times following a Sequential Survival process for node pair $(i,j)\in\mathcal{V}^2$. Then, the average squared distance between nodes during interval $[e_k, e_{k+1})$ associated with survival function $S(\cdot)$ and state $s\in\{-1,1\}$ can be bounded by 
\begin{align*}
b(-1) &\leq \frac{1}{(e_{k+1}\!-\!e_k)}\!\!\!\int\limits_{e_k}^{e_{k+1}}\!\!\!\! \| \mathbf{r}_i(t) - \mathbf{r}_j(t) \|^2 \mathrm{d}t \leq b(+1)
\end{align*}
where $b(s)\!:=\!\!-2s\log(e_{k+1}\!-e_k)-\!s\log{ \!S(e_{k+1}\!)}\!-\!s\beta{(s)}$.
\end{lemma}
\begin{proof}
Let $e_k$ and $e_{k+1}$ be a consecutive event times following a Sequential Survival process for node pair $(i,j)\in\mathcal{V}^2$. Due to the construction of the Sequential Survival process, we have an associated survival function, $S(t)$, for the interval $[e_k, e_{k+1})$ with state $s\in\mathcal{S}:=\{-1, 1\}$. Then, bounds for the survival function, $S(t)$, can be given by using Markov's inequality as:

\begin{align*}
S(t) &= \mathbb{P}(T_k \geq e_{k+1})
\\
&\leq \frac{\mathbb{E}[T_k]}{e_{k+1}-e_k}
\\
& = \!\frac{1}{(e_{k+1}\!-\!e_k)}\! \left\{ \!\!\int\limits_{e_k}^{e_{k+1}}\!\!\!\!\exp\left( \beta{(s)} \!+\! s\| \mathbf{r}_i(t) - \mathbf{r}_j(t) \|^2 \right)\mathrm{d}t  \right\}^{-1}
\end{align*}
where $T_k$ is the random variable showing event time after time $e_k$. The last line implies that
\begin{align}
&\frac{1}{(e_{k+1}\!-\!e_k)S(t)}\nonumber\\ 
&\qquad\geq \int\limits_{e_k}^{e_{k+1}}\!\!\!\exp\left( \beta{(s)} + s\| \mathbf{r}_i(t) - \mathbf{r}_j(t) \|^2 \right)\mathrm{d}t\nonumber
\\
&\qquad=\exp(\beta{(s)})\!\!\!\!\!\int\limits_{e_k}^{e_{k+1}}\!\!\!\!\exp\!\left(\! s\| \mathbf{r}_i(t) - \mathbf{r}_j(t) \|^2 \right)\!\mathrm{d}t\label{eq:appendix_lemma_bounds_proof_eq1}
\end{align}
Furthermore, we can apply Jensen's inequality for $\exp(s\cdot x)$ term since it is a convex function for any $s\in\{-1,1\}$ value. Hence, we can write
\begin{align}
&\frac{1}{(e_{k+1}\!-\!e_k)}\!\!\!\int\limits_{e_k}^{e_{k+1}}\!\!\!\exp\!\left( s\| \mathbf{r}_i(t) - \mathbf{r}_j(t) \|^2 \right)\mathrm{d}t \geq\nonumber
\\
&\qquad\quad\exp\left(\frac{s}{(e_{k+1}\!-\!e_k)}\!\!\!\int\limits_{e_k}^{e_{k+1}}\!\!\! \| \mathbf{r}_i(t) - \mathbf{r}_j(t) \|^2 \mathrm{d}t \right)\label{eq:appendix_lemma_bounds_proof_eq2}
\end{align}
By combining Eq. \eqref{eq:appendix_lemma_bounds_proof_eq1} and Eq. \eqref{eq:appendix_lemma_bounds_proof_eq2}, we can obtain the following inequality:
\begin{align*}
&\frac{1}{(e_{k+1}\!-\!e_{k})^2S(t)} \geq 
\\
&\qquad\quad\exp(\beta{(s)})\exp\!\left(\!\frac{s}{(e_{k+1}\!-\!e_k)}\!\!\!\int\limits_{e_k}^{e_{k+1}}\!\!\!\! \| \mathbf{r}_i(t) - \mathbf{r}_j(t) \|^2 \mathrm{d}t \right).
\end{align*}
The reorganization of the terms after taking the logarithm of the inequality yields:
\begin{align*}
&2\log(e_{k+1}\!-\!e_k)+\log{S(e_{k+1})}+\beta{(s)}
\\
&\qquad\qquad\qquad\qquad\leq-\frac{s}{(e_{k+1}\!-\!e_k)}\!\!\!\!\int\limits_{e_k}^{e_{k+1}}\!\!\!\!\| \mathbf{r}_i(t) - \mathbf{r}_j(t) \|^2 \mathrm{d}t 
\end{align*}
Finally, we can conclude that
\begin{align*}
b_k(-1) &\leq \frac{1}{(e_{k+1}\!-\!e_k)}\!\!\!\int\limits_{e_k}^{e_{k+1}}\!\!\!\! \| \mathbf{r}_i(t) - \mathbf{r}_j(t) \|^2 \mathrm{d}t \leq b_k(+1)
\end{align*}
where $b(s):=-2s\log(e_{k+1}\!-\!e_k)-s\log{ S(e_{k+1})}-s\beta{(s)}$.
\end{proof}

\begin{lemma}[Integral Computation]\label{lemma:appendix_integral}
The integral of the hazard function for a given pair $(i,j)\in\mathcal{V}^2$ having constant velocities and state during the interval $(e_l, e_u)$ is equal to
\begin{align*}
&\int\limits_{e_l}^{e_u}\lambda_{ij}(s,t) =\frac{\sqrt{\pi}}{2\|\Delta\mathbf{v}_{ij}\|} \exp\left(\beta(s) + s\|\Delta\mathbf{x}_{ij}\|^2 - s\rho_{ij}^2 \right) 
\\
&\quad\times E_{ij}(s, \tau(e_l),\tau(e_u))
\end{align*}
where $\rho_{ij} := \langle \Delta\mathbf{v}_{ij}, \Delta\mathbf{x}_{ij}\rangle / \| \Delta\mathbf{v}_{ij}\|$ and $E(\tau(e_l),\tau(e_u),s)$ is defined by 
\begin{align*}
	E_{ij}(s,e_l,e_u) := \begin{cases} \mathrm{erf}\big(\tau(e_u)\big) - \mathrm{erf}\big(\tau(e_l)\big) & \text{$s=+1$} \\ \mathrm{erfi}\big(\tau(e_u)\big) - \mathrm{erfi}\big(\tau(e_l)\big) & \text{$s=-1$}\end{cases}
\end{align*}
for position difference, $\Delta\mathbf{x}_{ij} := \mathbf{r}_i(e_l) - \mathbf{r}_j(e_l)$, at time $e_l$ and velocity difference $\Delta\mathbf{v}_{ij} := \mathbf{v}_i - \mathbf{v}_j$. The function $\tau:(e_l,e_u)\rightarrow \mathbb{R}$ is defined as $\|\Delta\mathbf{v}_{ij}\|t + \rho_{ij}$, and $\mathrm{erf}(\cdot)$, and $\mathrm{erfi}(\cdot)$ represent the error and the imaginary error functions.

\begin{proof}
Since it is supposed that the pair of nodes $(i,j)\in\mathcal{V}^2$ have constant velocities and state over the given interval $(e_l, e_u)$, the integral can be reexpressed in the following way:
\begin{align}
\int\limits_{e_l}^{e_u}\lambda_{ij}(s,t)d\mathrm{t} &= \!\!\int\limits_{e_l}^{e_u}\!\!\exp\!\left( \beta(s) + s\| \mathbf{r}_i(t) \!-\! \mathbf{r}_j(t) \|^2 \right)\!d\mathrm{t}\nonumber
\\
&= \!\!\int\limits_{e_l}^{e_u}\!\!\exp\!\left( \beta(s) + s\| \Delta\mathbf{x}_{ij} \!+\! \Delta\mathbf{v}_{ij}(t-e_l) \|^2 \right)\!d\mathrm{t}\nonumber
\\
&= \!\!\!\!\!\int\limits_{0}^{e_u-e_l}\!\!\!\!\exp\!\left( \beta(s) \!+\! s\| \Delta\mathbf{x}_{ij} \!+\! \Delta\mathbf{v}_{ij}t\|^2 \right)\!d\mathrm{t}\label{eq:appendix_integral_proof_eq_1}
\end{align}
where $\Delta\mathbf{x}_{ij} := \mathbf{r}_i(e_l) - \mathbf{r}_j(e_l)$ and $\Delta\mathbf{v}_{ij} := \mathbf{v}_i - \mathbf{v}_j$ indicate the differences between the positions and velocities, and the last line, Eq. \eqref{eq:appendix_integral_proof_eq_1}, is obtained by changing the boundaries of the integral.

Since the bias term, $\beta(s)$, does not vary by time, it can be moved outside the integral term: 
\begin{align}
&\int\limits_{0}^{e_u-e_l}\!\!\!\exp\left( \beta(s) + s\| \Delta\mathbf{x}_{ij} + \Delta\mathbf{v}_{ij}t\|^2 \right)d\mathrm{t}\nonumber
\\
&\quad\quad = \exp\left( \beta(s)\right)\!\!\!\!\int\limits_{0}^{e_u-e_l}\!\!\!\!\exp\!\left( s\| \Delta\mathbf{x}_{ij} + \Delta\mathbf{v}_{ij}t\|^2 \right)d\mathrm{t}\label{eq:appendix_integral_proof_eq_2}
\end{align}
Then, we can rewrite the integral term as follows:
\begin{align}
&\int\limits_{0}^{e_u-e_l} \exp\left(s\|\Delta \mathbf{x}_{ij} + \Delta \mathbf{v}_{ij}t \|^2\right)\nonumber 
\\
&\hspace{0.1cm}=\!\!\!\!\int\limits_{0}^{e_u-e_l}\!\!\!\! \exp\!\Big(s\|\Delta\mathbf{x}_{ij}\|^2 \!+\! 2s\!\left\langle\Delta\mathbf{x}_{ij}, \Delta\mathbf{v}_{ij} \right\rangle t \!+\! s\|\Delta\mathbf{v}_{ij}\|^2t^2\Big)\mathrm{d}t\nonumber
\\
&\hspace{0.1cm}=\!\!\!\!\int\limits_{0}^{e_u-e_l}\!\!\!\!\exp\!\Big( s\|\Delta\mathbf{x}_{ij}\|^2 -\!s\rho_{ij}^2 + s\left( \|\Delta\mathbf{v}_{ij}\|t +\! \rho_{ij} \right)^2 \Big)\mathrm{d}t\nonumber
\\
&\hspace{0.1cm}=\!\exp\!\Big(\!s\|\Delta\mathbf{x}_{ij}\|^2 \!\!-\!s\rho_{ij}^2 \Big)\!\!\!\!\!\!\int\limits_{0}^{e_u-e_l}\!\!\!\!\!\exp\!\Big( \!s\!\left( \|\Delta\mathbf{v}_{ij}\|t \!+\!\! \rho_{ij} \right)^2\!\Big)\mathrm{d}t\label{eq:appendix_integral_proof_eq_3}
\end{align}
where $\rho_{ij} := \langle \Delta\mathbf{v}_{ij}, \Delta\mathbf{x}_{ij} \rangle / \|\Delta\mathbf{v}_{ij}\|$.  A substitution $y = \|\Delta\mathbf{v}_{ij}\|t + \rho_{ij}$ gives us $\mathrm{d}y = \|\Delta\mathbf{v}_{ij}\|\mathrm{d}t$, so we can write
\begin{align}
&\int\limits_{0}^{e_u-e_l}\!\!\!\exp\Big(s\left( \|\Delta\mathbf{v}_{ij}\|t \!+\! \rho_{ij} \right)^2 \Big)\mathrm{d}t\nonumber 
\\
&\qquad\qquad =\frac{1}{\|\Delta\mathbf{v}_{ij}\|}\int_{\tau(e_l)}^{\tau(e_u)}\exp\left( sy^2 \right)\!\mathrm{d}y\nonumber
\\
&\qquad\qquad= \frac{1}{\|\Delta\mathbf{v}_{ij}\|}\frac{\sqrt{\pi}}{2} \left(\frac{2}{\sqrt{\pi}} \int_{\tau(e_l)}^{\tau(e_u)} \exp\left( sy^2 \right)\mathrm{d}y \right)\nonumber
\\
&\qquad\qquad= \frac{\sqrt{\pi}}{2\|\Delta\mathbf{v}_{ij}\|} E_{ij}(s,\tau(e_l), \tau(e_u))\label{eq:appendix_integral_proof_eq_4}
\end{align}
where $E_{ij}(s, \tau(e_l),\tau(e_u))$ is given by 
\begin{align*}
E_{ij}(s,e_l,e_u) := \begin{cases} \mathrm{erf}\big(\tau(e_u)\big) - \mathrm{erf}\big(\tau(e_l)\big) & \text{$s=+1$} \\ \mathrm{erfi}\big(\tau(e_u)\big) - \mathrm{erfi}\big(\tau(e_l)\big) & \text{$s=-1$}\end{cases}
\end{align*}
for $\tau(t) := \|\Delta\mathbf{v}_{ij}\|t + \rho_{ij}$. By combining all the results acquired in \Cref{eq:appendix_integral_proof_eq_1,eq:appendix_integral_proof_eq_2,eq:appendix_integral_proof_eq_3,eq:appendix_integral_proof_eq_4}, we can conclude that
\begin{align*}
&\int\limits_{e_l}^{e_u}\!\exp\!\left( \beta(s) + s\| \mathbf{r}_{ij} - \mathbf{r}_{ij}\|^2 \right)d\mathrm{t}\nonumber
\\
& \hspace{0.1cm}=\!\! \frac{\sqrt{\pi}}{2\|\Delta\mathbf{v}_{ij}\|}\! \exp\!\left(\!\beta(s) \!+\! s\|\Delta\mathbf{x}_{ij}\|^2 \!\!-\! s\rho_{ij}^2 \!\right) \!\!E_{ij}(s,\!\tau(e_l),\!\tau(e_u))
\end{align*}
\end{proof}
\end{lemma}

\section{Experimental Evaluation}
In this section, we provide further details regarding the experiments and dynamic visualizations.

\begin{table}[t]
\centering
\caption{Network statistics ($|\mathcal{V}|$: Number of nodes, $|\mathcal{E}_{min}|$: Min. number of events that a dyad has, $|\mathcal{E}_{max}|$: Max. number of events that a dyad has, $|\mathcal{E}|$: Total number of events).}
\label{tab:appendix_network_stats}
\resizebox{\columnwidth}{!}{%
\begin{tabular}{rccccc}
\toprule
 & $|\mathcal{V}|$ & $|\mathcal{E}_{min}|$ & $|\mathcal{E}_{max}|$ & $|\mathcal{E}|$ & Resolution \\\cmidrule(rl){2-2}\cmidrule(rl){3-3}\cmidrule(rl){4-4}\cmidrule(rl){5-5}\cmidrule(rl){6-6}
\textsl{Synthetic-$\alpha$} & 100 & 1 & 18 & 4286 & N/A \\
\textsl{Synthetic-$\beta$} & 100 & 2 & 12 & 8300 & N/A \\
\textsl{Contacts} & 92 & 1 & 197 & 5313 & Second \\
\textsl{HyperText} & 113 & 1 & 133 & 10450 & Second \\
\textsl{Infectious} & 410 & 1 & 29 & 9827 & Second \\
\textsl{Facebook} & 461 & 1 & 1 & 10222 & Second \\
\textsl{NeurIPS} & 327 & 1 & 6 & 1940 & Year\\\bottomrule
\end{tabular}%
}
\end{table}
\begin{table*}[t]
\centering
\caption{Performance of \textsc{LDM} for the weighted versions of the datasets in various prediction tasks.}
\label{tab:appendix_weighted_ldm_recontruction}
% \resizebox{\columnwidth}{!}{%
\begin{tabular}{rcccccc}
\toprule
 & \multicolumn{2}{c}{Reconstruction} & \multicolumn{2}{c}{Completion} & \multicolumn{2}{c}{Future Link Prediction} \\\cmidrule(lr){2-3}\cmidrule(lr){4-5}\cmidrule(lr){6-7}
 & AUC-ROC & AUC-PR & AUC-ROC & AUC-PR & AUC-ROC & AUC-PR \\\cmidrule(lr){2-2}\cmidrule(lr){3-3}\cmidrule(lr){4-4}\cmidrule(lr){5-5}\cmidrule(lr){6-6}\cmidrule(lr){7-7}
\textsl{Synthetic-$\alpha$} & $.701\pm.002$ & $.658\pm.004$ & $.689\pm.008$ & $.613\pm.010$ & $.769\pm.004$ & $.731\pm.005$ \\
\textsl{Synthetic-$\beta$} & $.561\pm.013$ & $.549\pm.015$ & $.485\pm.013$ & $.487\pm.007$ & $.497\pm.016$ & $.500\pm.016$ \\
\textsl{Contacts} & $.585\pm.003$ & $.538\pm.004$ & $.505\pm.010$ & $.490\pm.009$ & $.818\pm.003$ & $.764\pm.005$ \\
\textsl{HyperText} & $.533\pm.005$ & $.501\pm.007$ & $.519\pm.017$ & $.496\pm.009$ & $.605\pm.011$ & $.554\pm.013$ \\
\textsl{Infectious} & $.693\pm.008$ & $.612\pm.011$ & $.686\pm.004$ & $.616\pm.008$ & $.938\pm.005$ & $.903\pm.014$ \\
\textsl{Facebook} & $.681\pm.005$ & $.615\pm.005$ & $.714\pm.005$ & $.656\pm.008$ & $.726\pm.005$ & $.678\pm.005$ \\
\textsl{NeurIPS} & $.740\pm.007$ & $.666\pm.009$ & $.689\pm.011$ & $.620\pm.015$ & $.731\pm.010$ & $.671\pm.013$\\\bottomrule
\end{tabular}%
% }
\end{table*}

\subsubsection{Datasets.} In the experiments, we have considered networks of different types and sizes. More precisely, we have utilized the following four real networks:
\begin{itemize}
    \item \textsl{Facebook} \citep{viswanath2009evolution} is a friendship network signifying one user's presence within another's friend list. We considered users having at least $200$ links. 
    
    \item \textsl{NeurIPS} \citep{globerson2004euclidean} was formed through collaborations among authors whose works were presented at the NeurIPS conference covering the years from $1989$ to $2001$. We focused on authors with at least ten connections, assuming a year-long duration for each work. 
    
    \item \textsl{Contacts} \citep{genois2015data} consists of the interactions among individuals within an office building, encompassing a span of nine days in 2013. (vi) \textsl{HyperText} \citep{isella2011s} was collected during a conference in which participants wore radio badges tracking their face-to-face interactions, covering a period of approximately $2.5$ days. 
    
    \item \textsl{Infectious} \citep{isella2011s} is another interaction network collected during an event in Dublin. In the contact datasets (v-vii), each timestamp associated with a link corresponds to a $20$-second connection. When multiple link events occur within a two-minute window, we aggregate these events and treat them as a single link duration.
\end{itemize}
We also generated two artificial networks to examine the model's behavior in controlled settings:
\begin{itemize}
    \item \textsl{Synthetic-$\alpha$}. We have followed a similar procedure proposed by \cite{pivem} to sample the initial node positions and their velocities. The bias terms are chosen as $\beta(-1) = 3$ and $\beta(+1) = -0.25$ and other hyper-parameters ($\lambda$, $\sigma_{\Sigma}$, $\sigma_B$) are set to $30$, $10^{-2}$, and $10^{-6}$, respectively. Then, the link and non-link event times of the graph are generated from the proposed Sequential Survival process.
    
    \item \textsl{Synthetic-$\beta$}. The timeline was divided into equal-sized $8$ bins and the nodes are randomly grouped into $10$ clusters for each bin. Then, we add links among nodes within the same cluster with a probability of $0.8$, while nodes belonging to different clusters are linked with a $10^{-2}$ probability. These links stay persistent until the next bin event time.
\end{itemize}
We provide a concise overview of the networks in Table \ref{tab:appendix_network_stats}. 

\subsubsection{Continuous-time Dynamic Visualization.} We provide the snapshots of the learned embeddings for various timestamps in \Cref{fig:appendix_visualization_synthetic_mu,fig:appendix_visualization_synthetic_beta,fig:appendix_visualization_contacts,fig:appendix_visualization_hypertext,fig:appendix_visualization_infectious,fig:appendix_visualization_facebook} and the intermittent time-persistent linkage structures of the networks are depicted in Figure \ref{fig:appendix_link_structure}.

\subsubsection{Optimization.} In our experiments, we train all the models for $300$ epochs. The number of walks, walk length, and window size parameters are set to  $80$, $10$, and $10$ for \textsc{Node2Vec} and $10$, $80$, and $10$ for \textsc{CTDNE}. The negative sample sizes for \textsc{HTNE} and \textsc{M$^2$DNE} are selected as $10$. The learning rate and batch size for \textsc{LDM}, \textsc{PiVeM}, and \textsc{\modelname} are set to $0.1$ and $100$, respectively. We tune the scaling factor of the covariance matrix from the set $\{10^1,10^2,\ldots,10^{10}\}$, and the number of bins is chosen as $100$ for both \textsc{PiVeM} and \textsc{\modelname}. For all the other hyperparameters, we employed the default parameters. 

\subsubsection{Weighted-\textsc{LDM}.} In Table \ref{tab:appendix_weighted_ldm_recontruction}, we report the performance of the \textsc{LDM} model in terms of the AUC-ROC and AUC-PR scores for the weighted versions of the datasets.

\section{Table of Symbols}
We provide the list of the symbols used in the manuscript and their explanations  in Table 
\ref{tab:appendix_table_of_symbols}.

\begin{table}[t]
\centering
\caption{Table of symbols}
\label{tab:appendix_table_of_symbols}
% \resizebox{0.3\textwidth}{!}{%
% \rotatebox{90}{
\begin{tabular}{rl@{}}
\toprule
\textbf{Symbol} & \textbf{Description} \\ \midrule
$\mathcal{G}$ & Graph \\
$\mathcal{V}$ & Vertex set \\
$\mathcal{E}$ & Edge set \\
$\mathcal{E}_{ij}$ & Edge set of node pair $(i,j)$ \\
$N$ & Number of nodes \\
$D$ & Dimension size \\
$\mathcal{I}_T$ & Time interval \\
$T$ & Time length \\
$B$ & Number of bins \\ 
$\beta_i$ & Bias term of node $i$ \\
$\mathbf{x}$ & Initial position matrix \\
$\mathbf{v}^{(b)}$ & Velocity matrix for bin $b$ \\
$\mathbf{r}_{i}(t)$ & Latent representation of node $i$ at time $t$ \\
$\lambda_{ij}(t)$ & Intensity of node pair $(i,j)$ at time $t$\\
$e_{ij}$ & An event time of node pair $(i,j)$\\
$\mathrm{erf}$ & Error function\\
$\mathrm{erfi}$ & Imaginary Error function\\
\bottomrule
\end{tabular}%
\end{table}

\begin{figure*}[!ht]
\centering
\subfigure[$t=40$]{\includegraphics[trim={5cm 6cm 5cm 6cm},clip,width=0.16\textwidth]{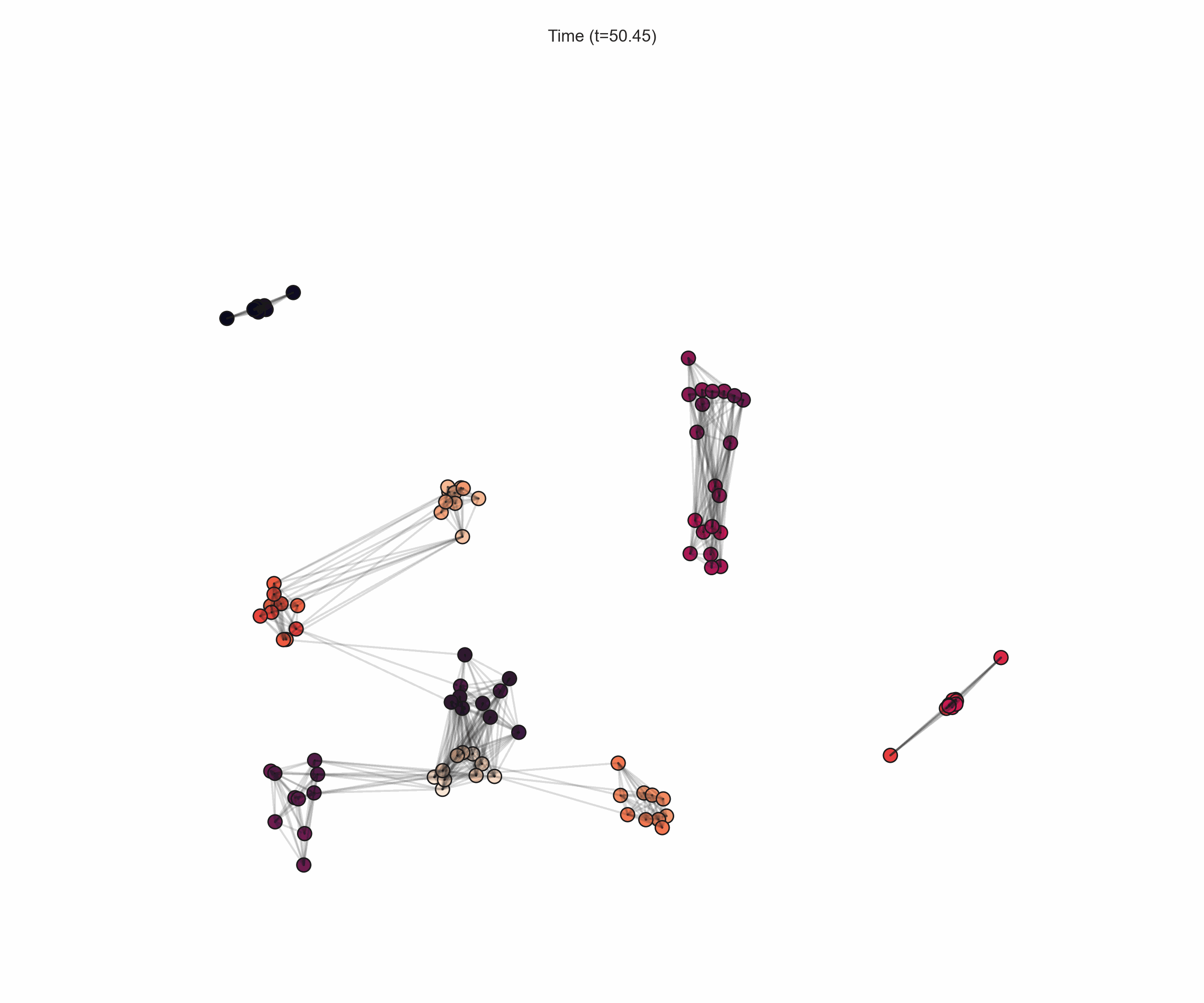}}
\hfill
\subfigure[$t=80$]{\includegraphics[trim={5cm 6cm 5cm 6cm},clip,width=0.16\textwidth]{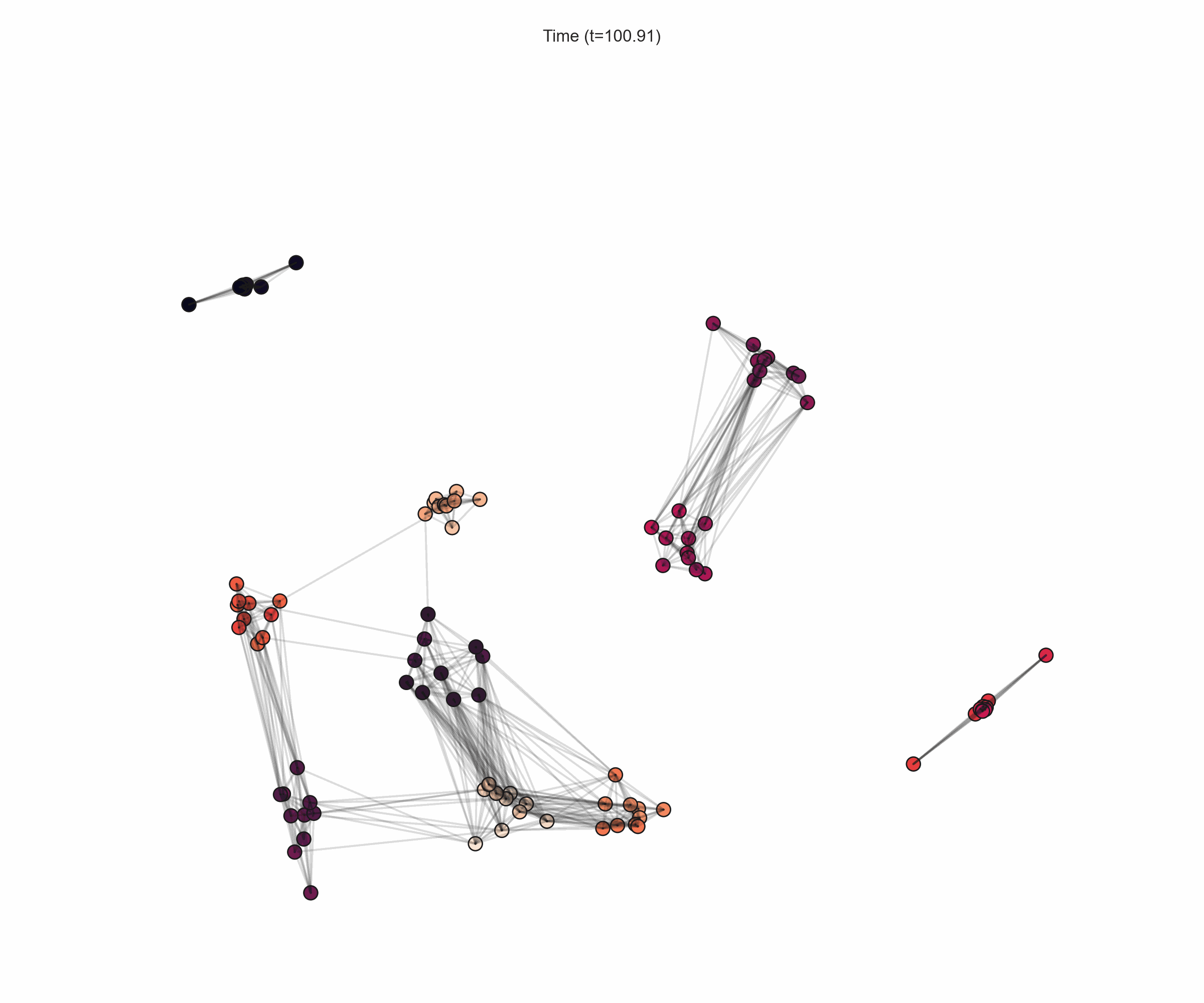}}
\hfill
\subfigure[$t=121$]{\includegraphics[trim={5cm 6cm 5cm 6cm},clip,width=0.16\textwidth]{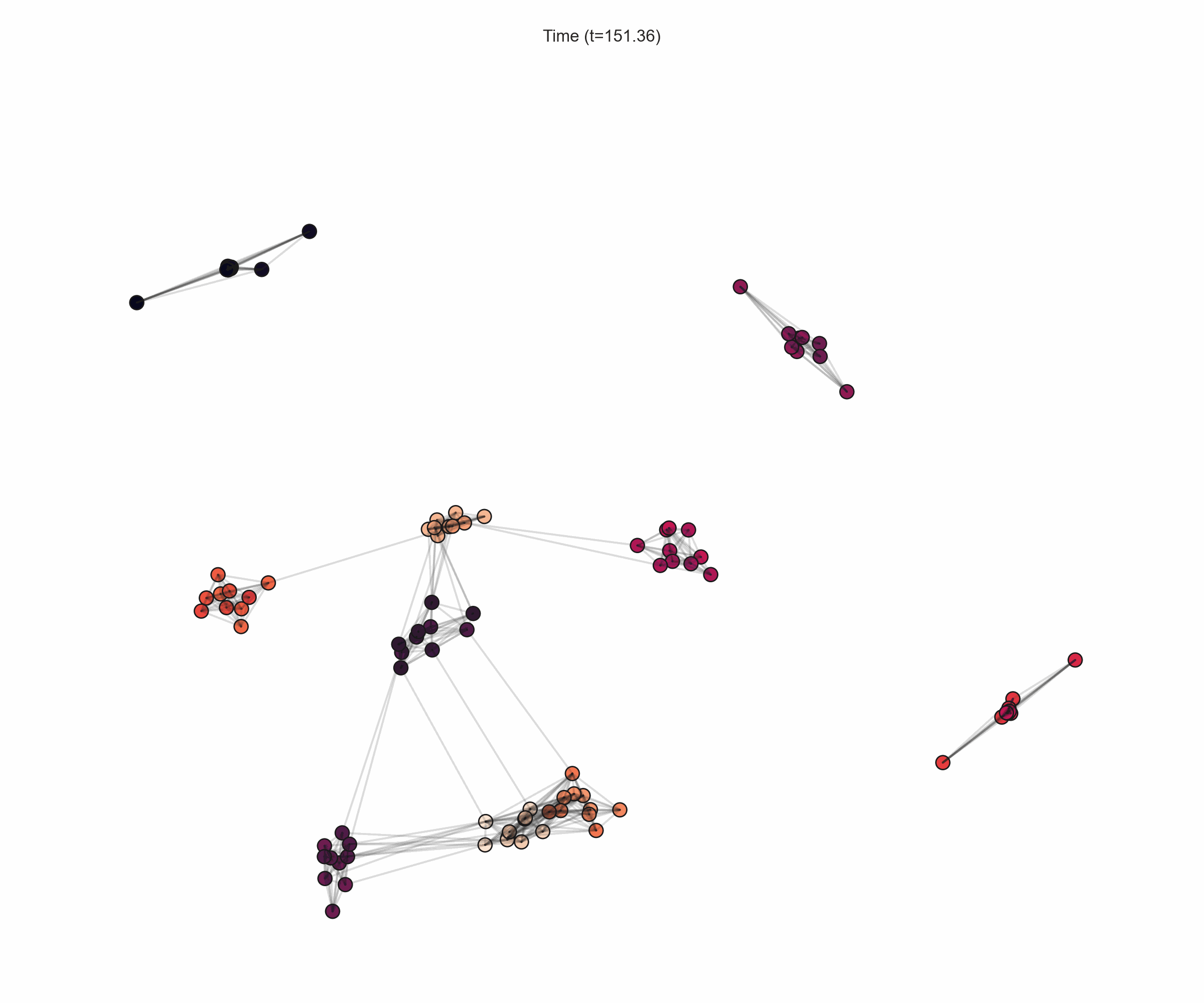}}
\hfill
\subfigure[$t=161$]{\includegraphics[trim={5cm 6cm 5cm 6cm},clip,width=0.16\textwidth]{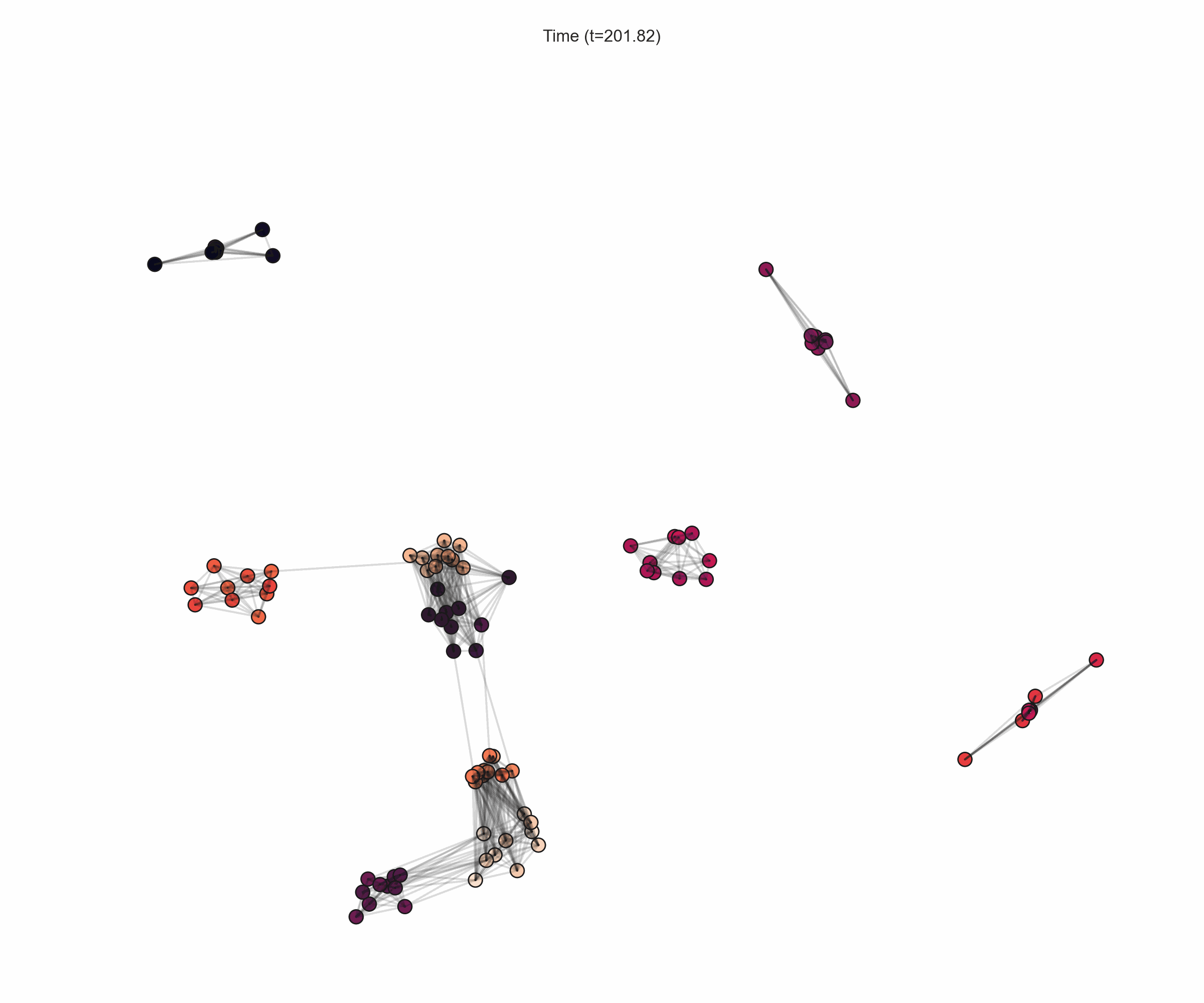}}
\hfill
\subfigure[$t=202$]{\includegraphics[trim={5cm 6cm 5cm 6cm},clip,width=0.16\textwidth]{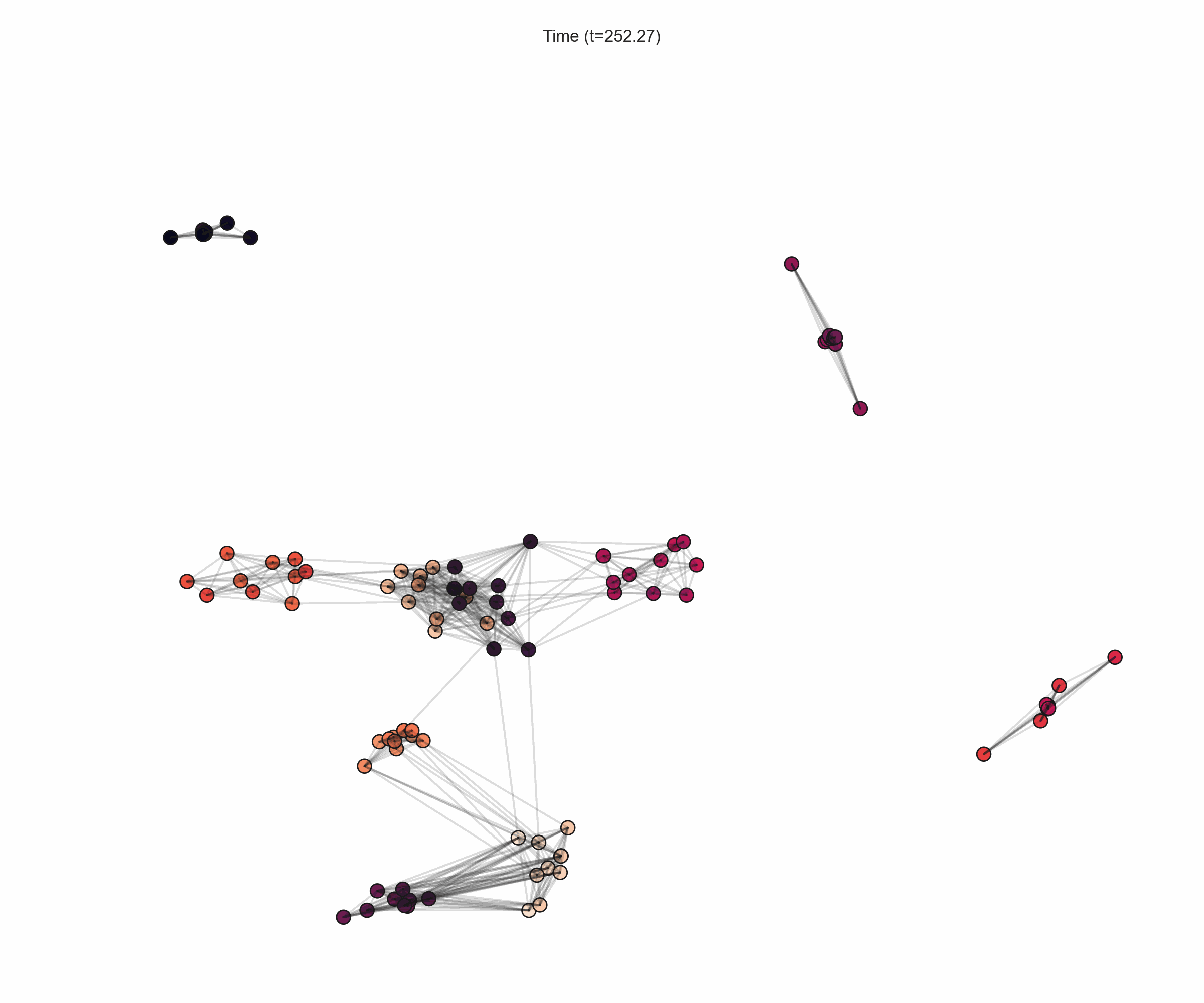}}
\hfill
\subfigure[$t=242$]{\includegraphics[trim={5cm 6cm 5cm 6cm},clip,width=0.16\textwidth]{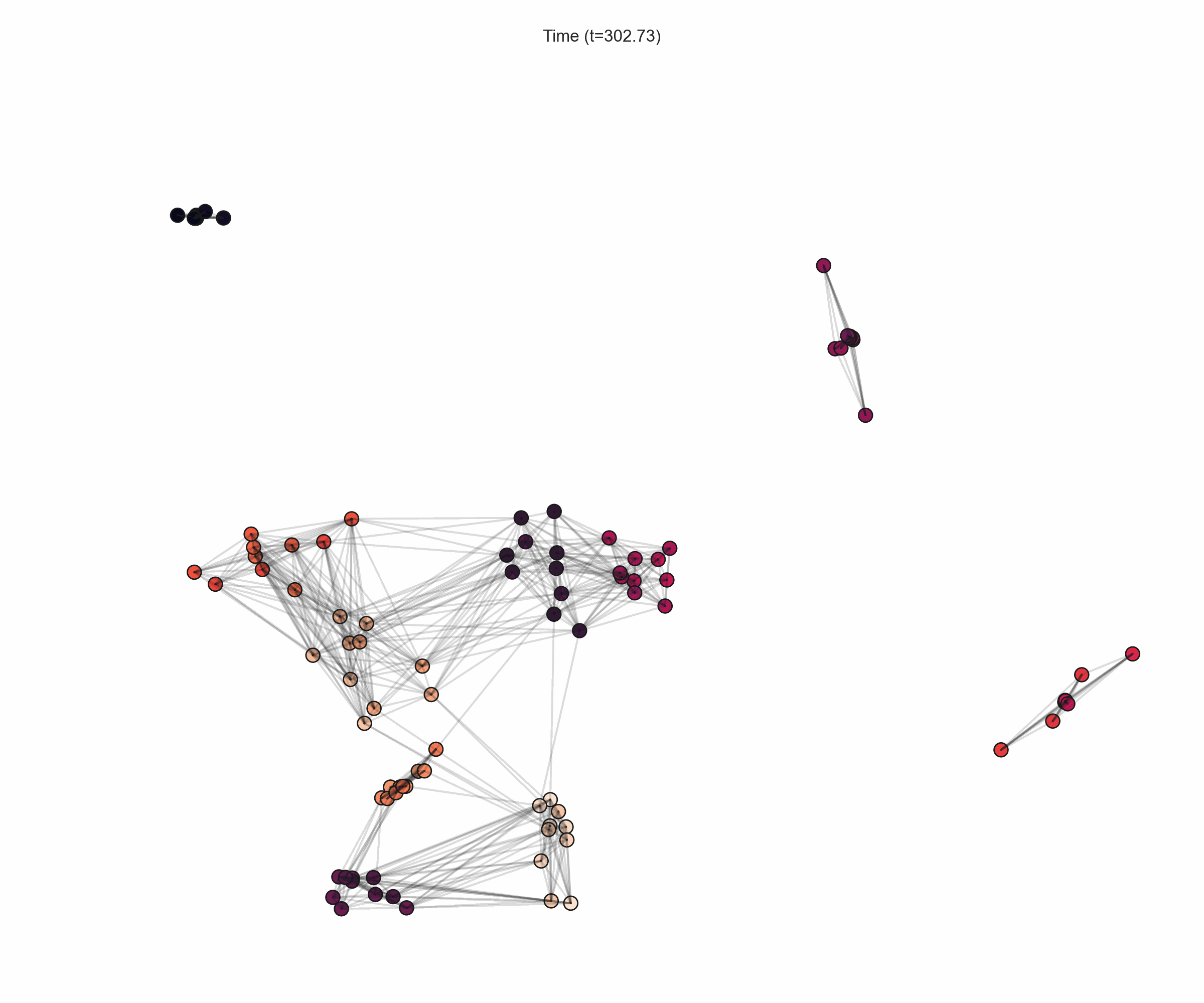}}
%%%%%%%%
\subfigure[$t=282$]{\includegraphics[trim={5cm 6cm 5cm 6cm},clip,width=0.16\textwidth]{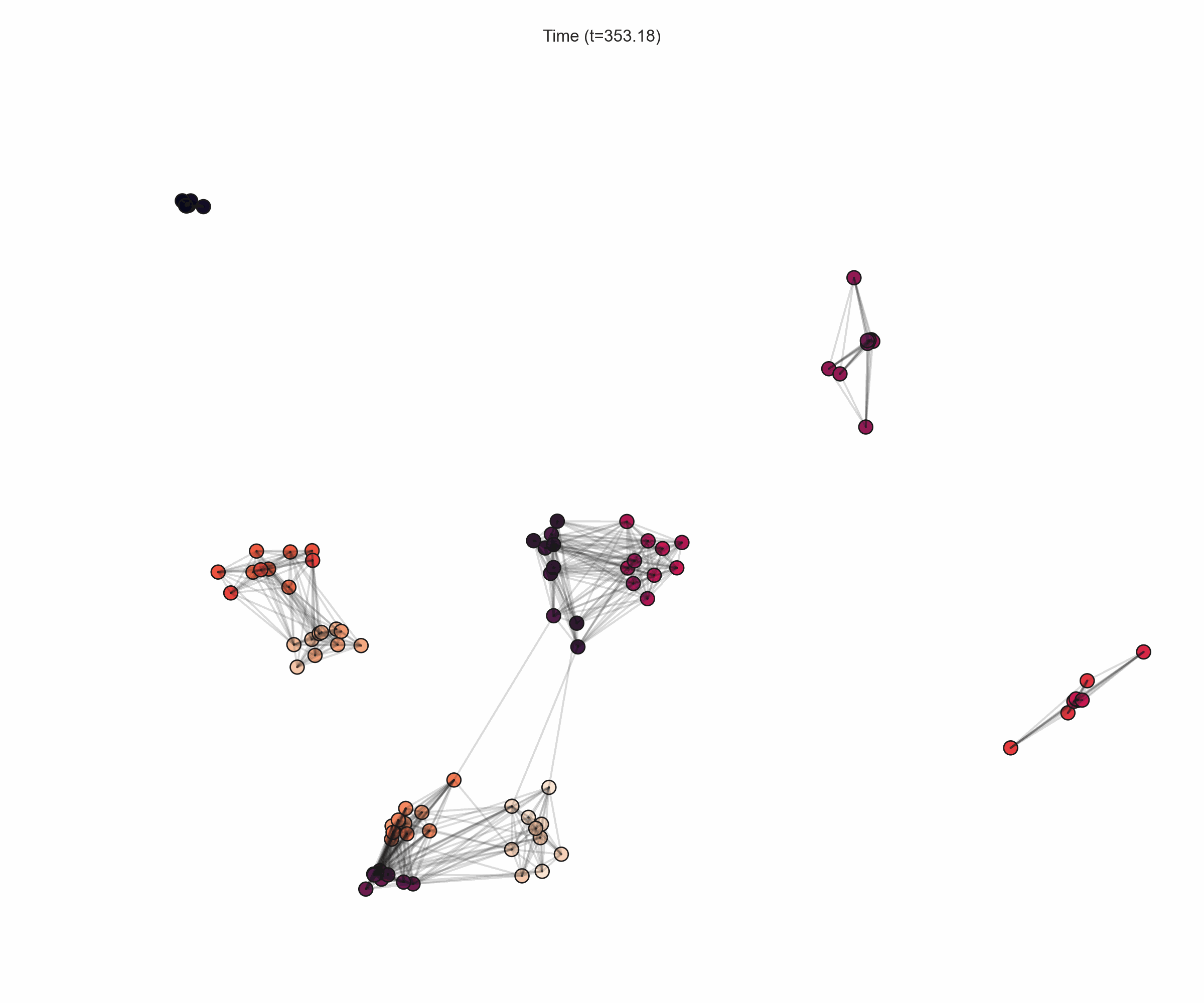}}
\hfill
\subfigure[$t=323$]{\includegraphics[trim={5cm 6cm 5cm 6cm},clip,width=0.16\textwidth]{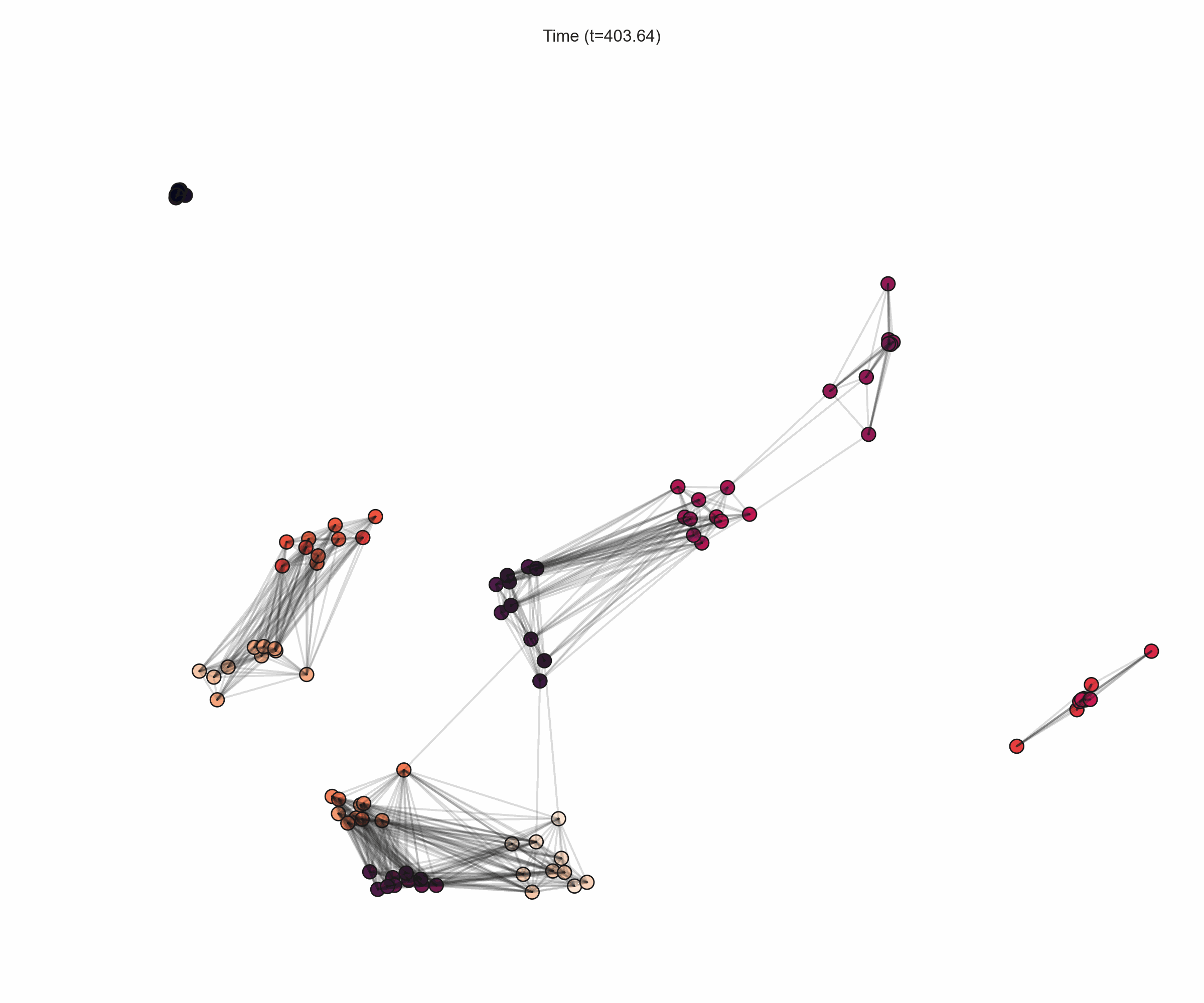}}
\hfill
\subfigure[$t=363$]{\includegraphics[trim={5cm 6cm 5cm 6cm},clip,width=0.16\textwidth]{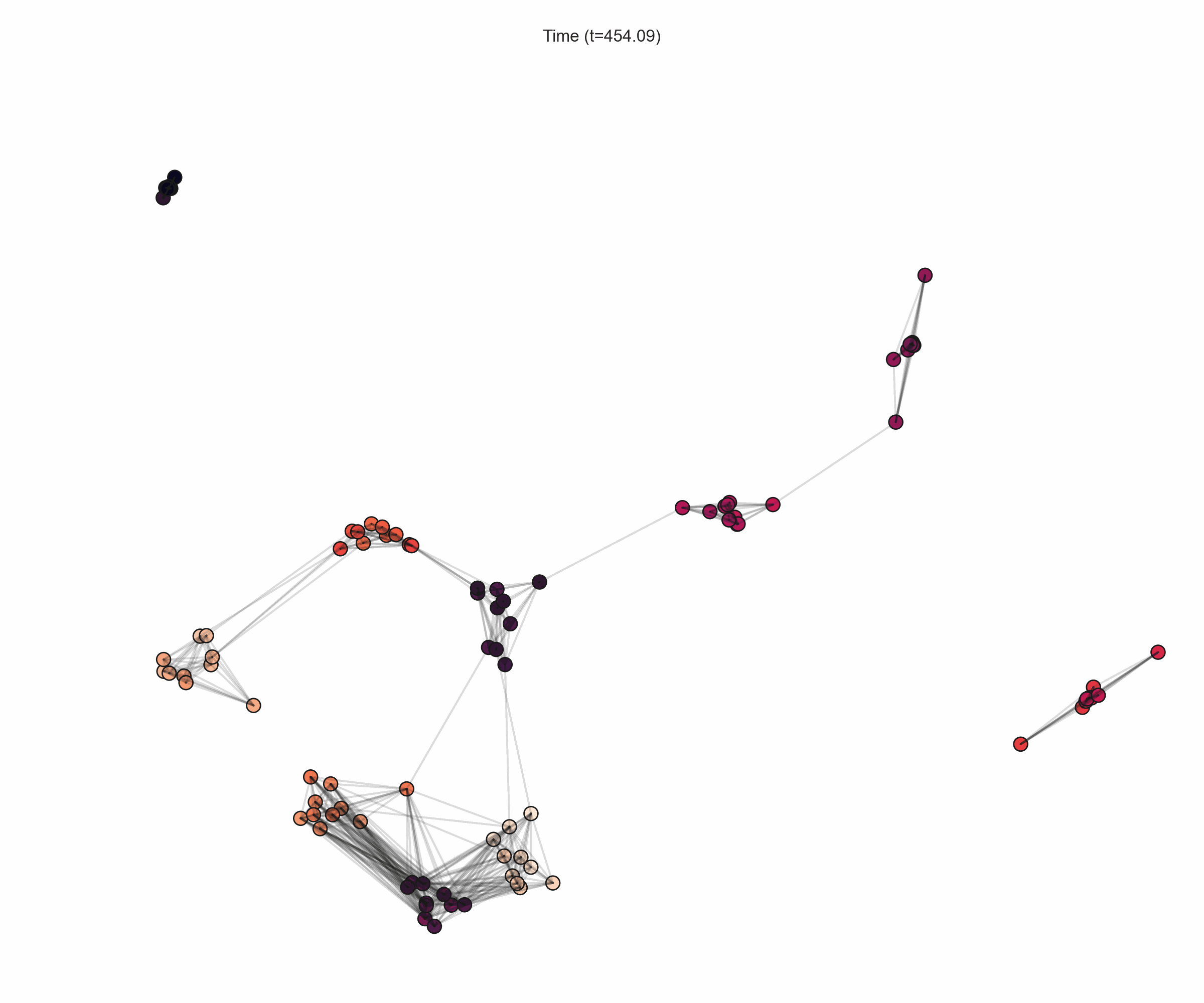}}
\hfill
\subfigure[$t=404$]{\includegraphics[trim={5cm 6cm 5cm 6cm},clip,width=0.16\textwidth]{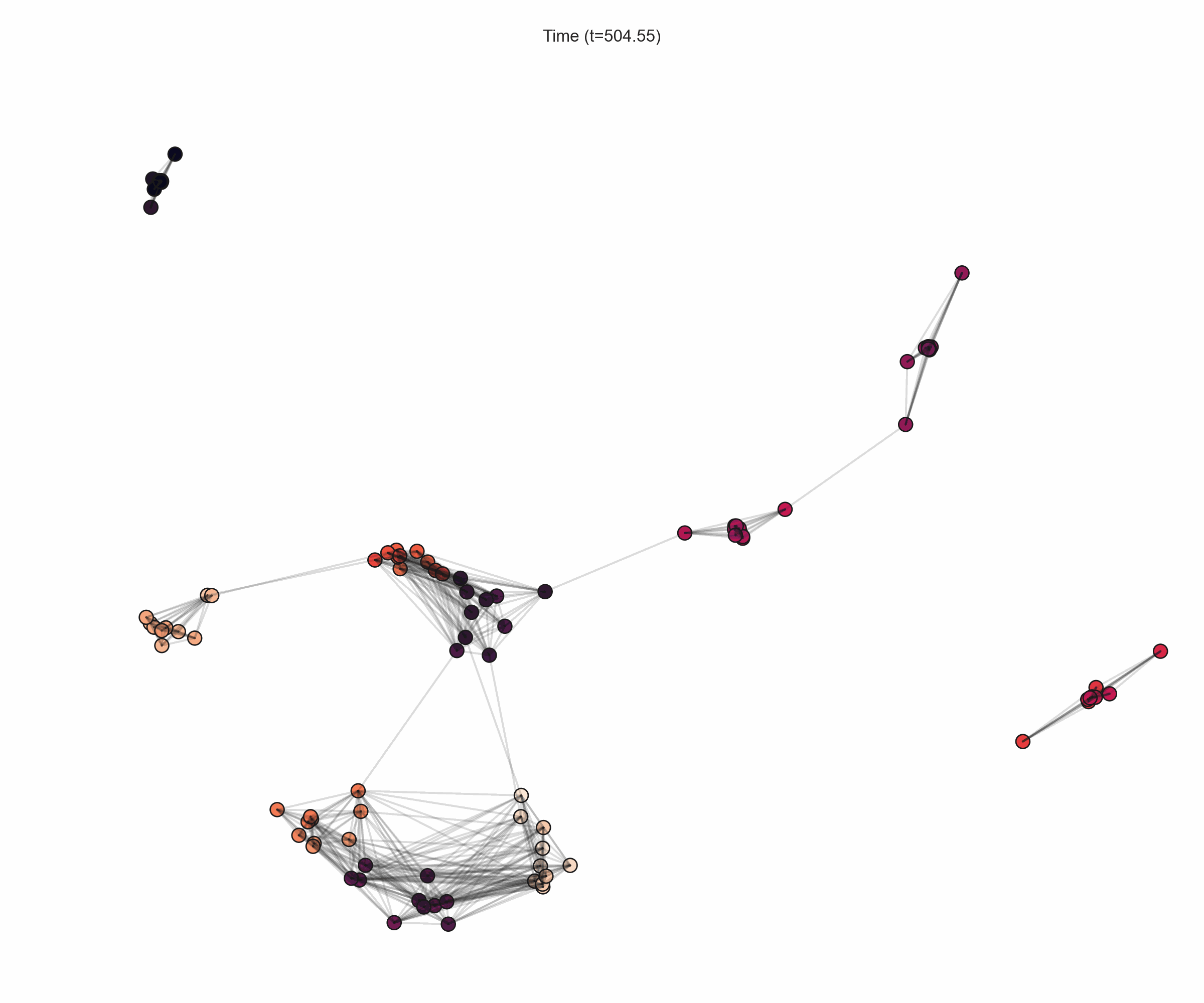}}
\hfill
\subfigure[$t=444$]{\includegraphics[trim={5cm 6cm 5cm 6cm},clip,width=0.16\textwidth]{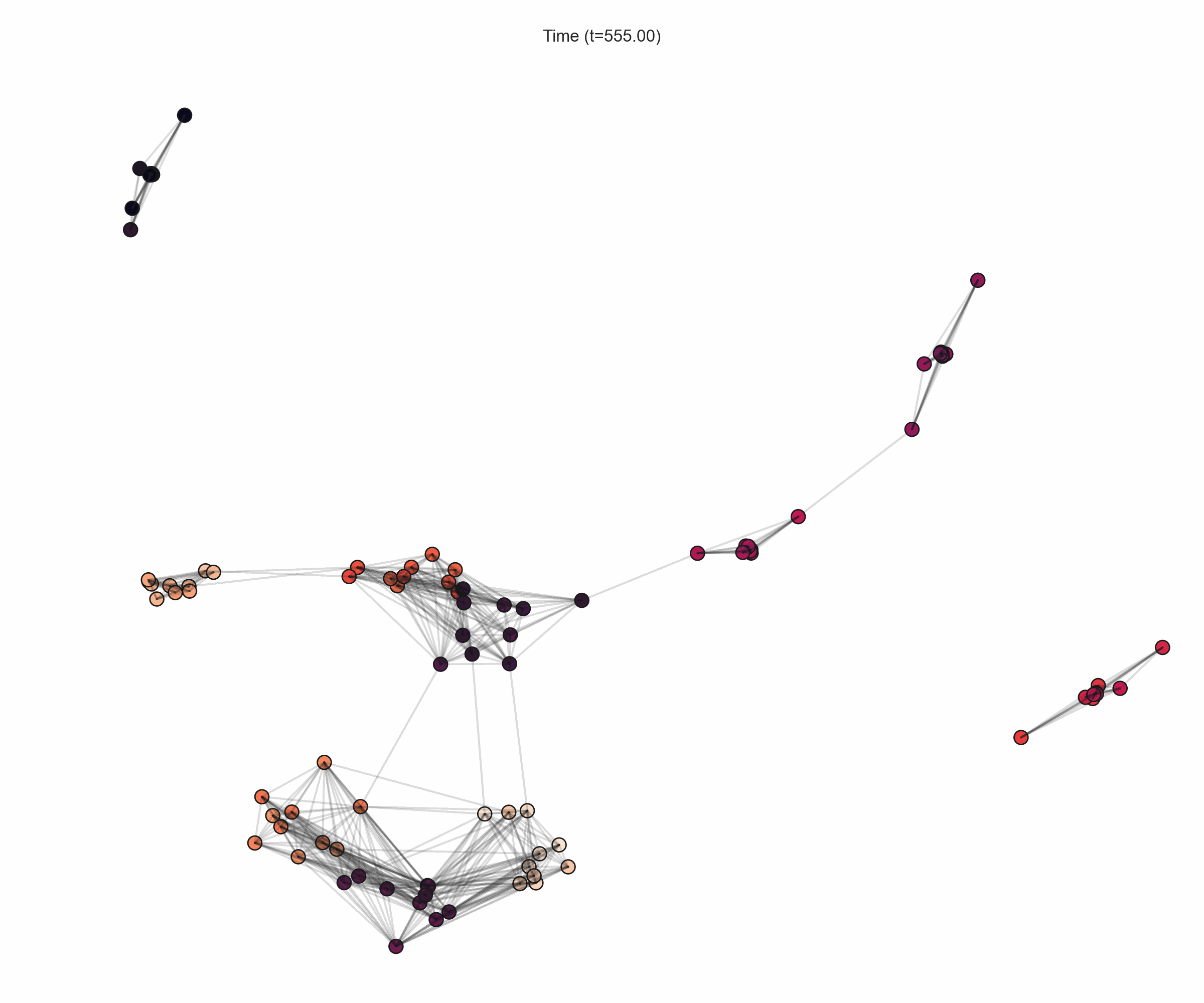}}
\hfill
\subfigure[$t=484$]{\includegraphics[trim={5cm 6cm 5cm 6cm},clip,width=0.16\textwidth]{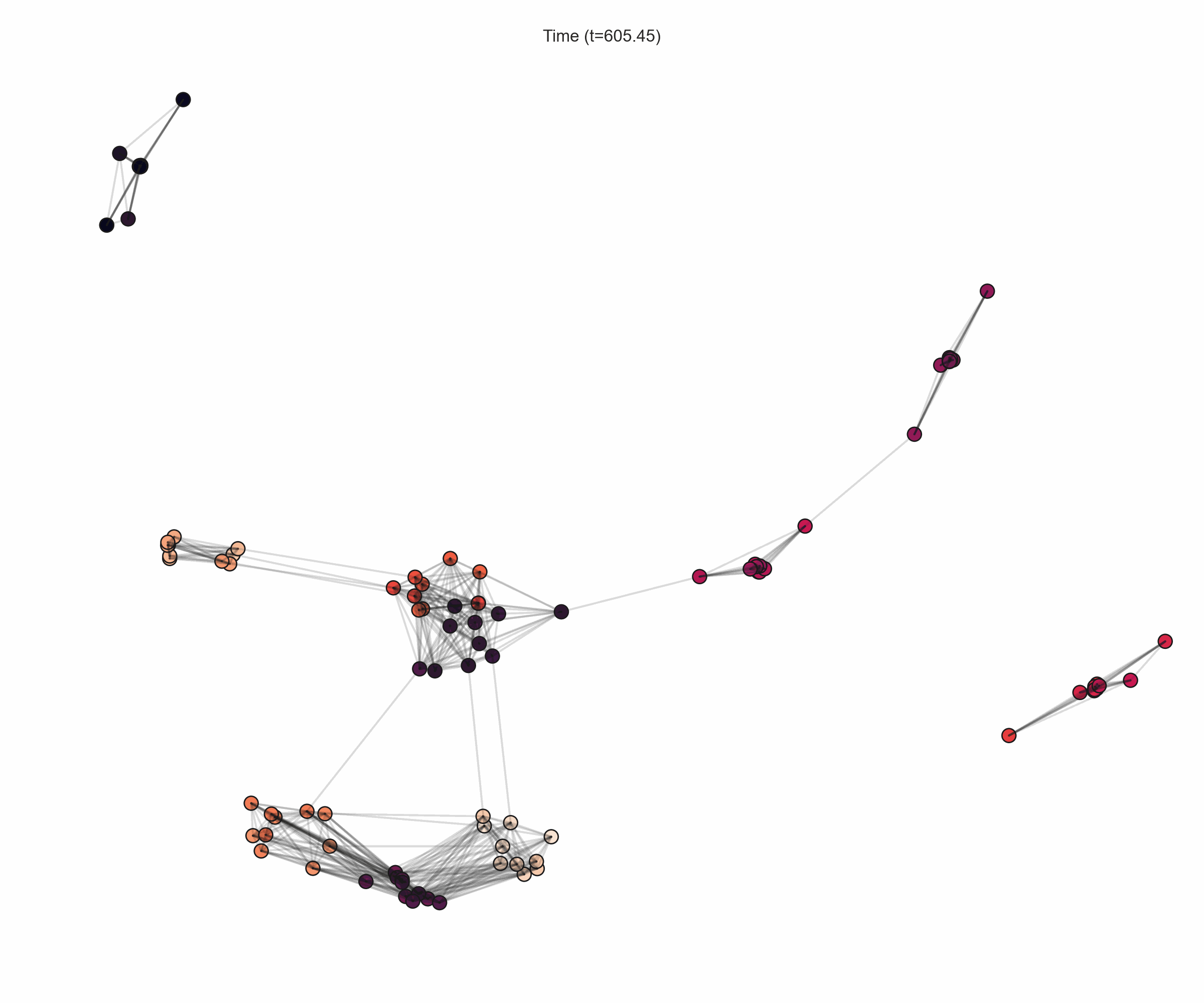}}
%%%%%%%%
\subfigure[$t=525$]{\includegraphics[trim={5cm 6cm 5cm 6cm},clip,width=0.16\textwidth]{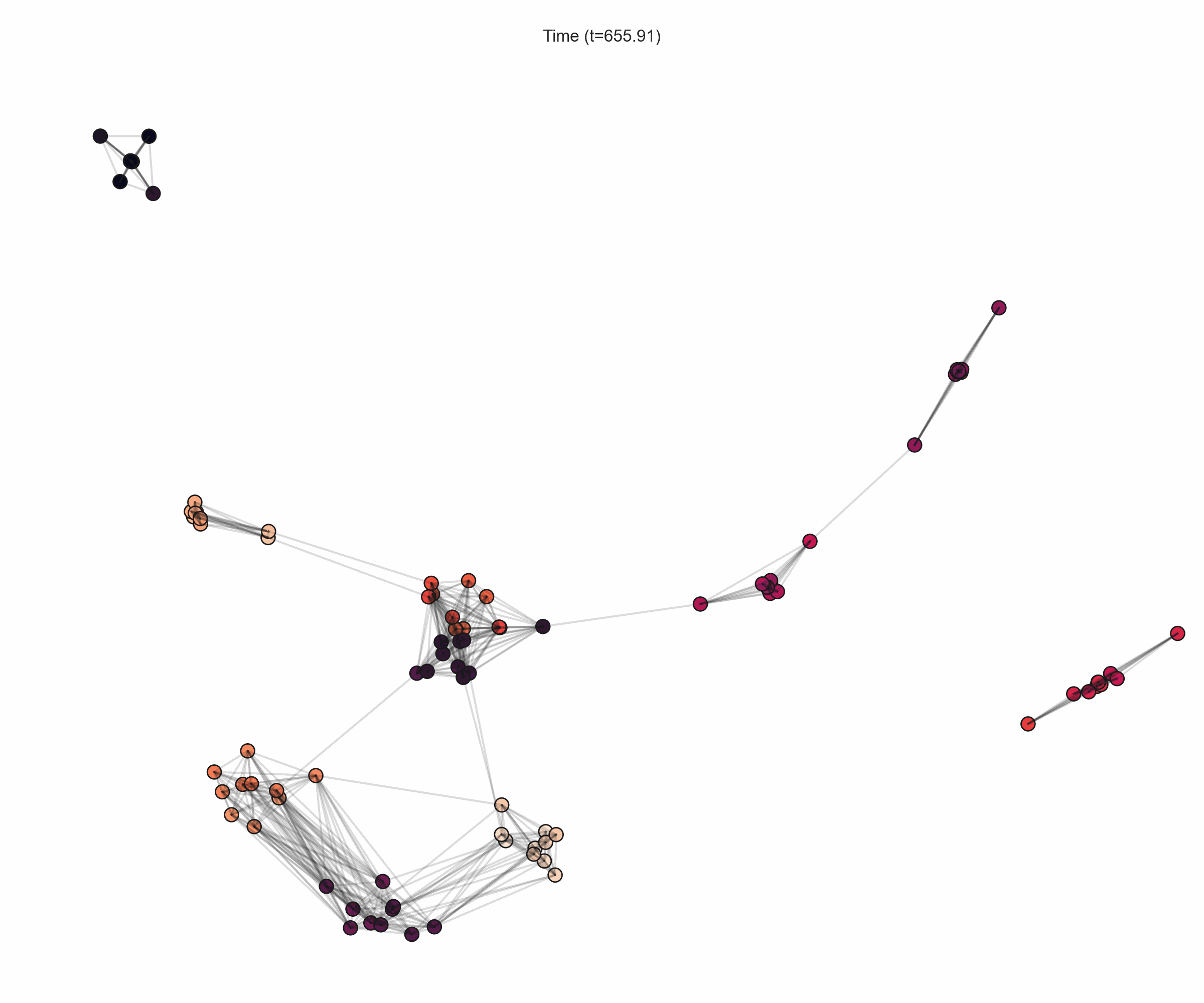}}
\hfill
\subfigure[$t=565$]{\includegraphics[trim={5cm 6cm 5cm 6cm},clip,width=0.16\textwidth]{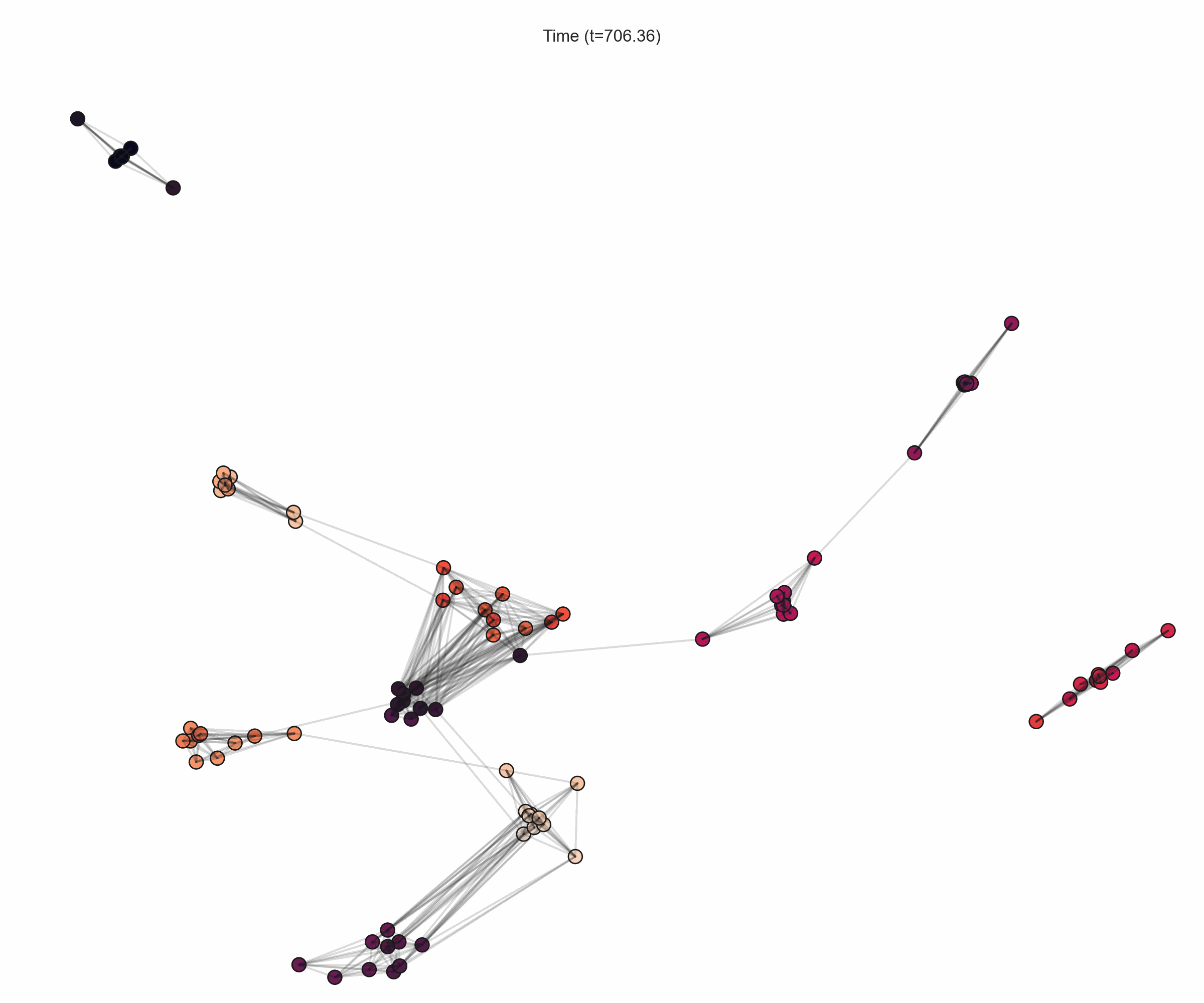}}
\hfill
\subfigure[$t=606$]{\includegraphics[trim={5cm 6cm 5cm 6cm},clip,width=0.16\textwidth]{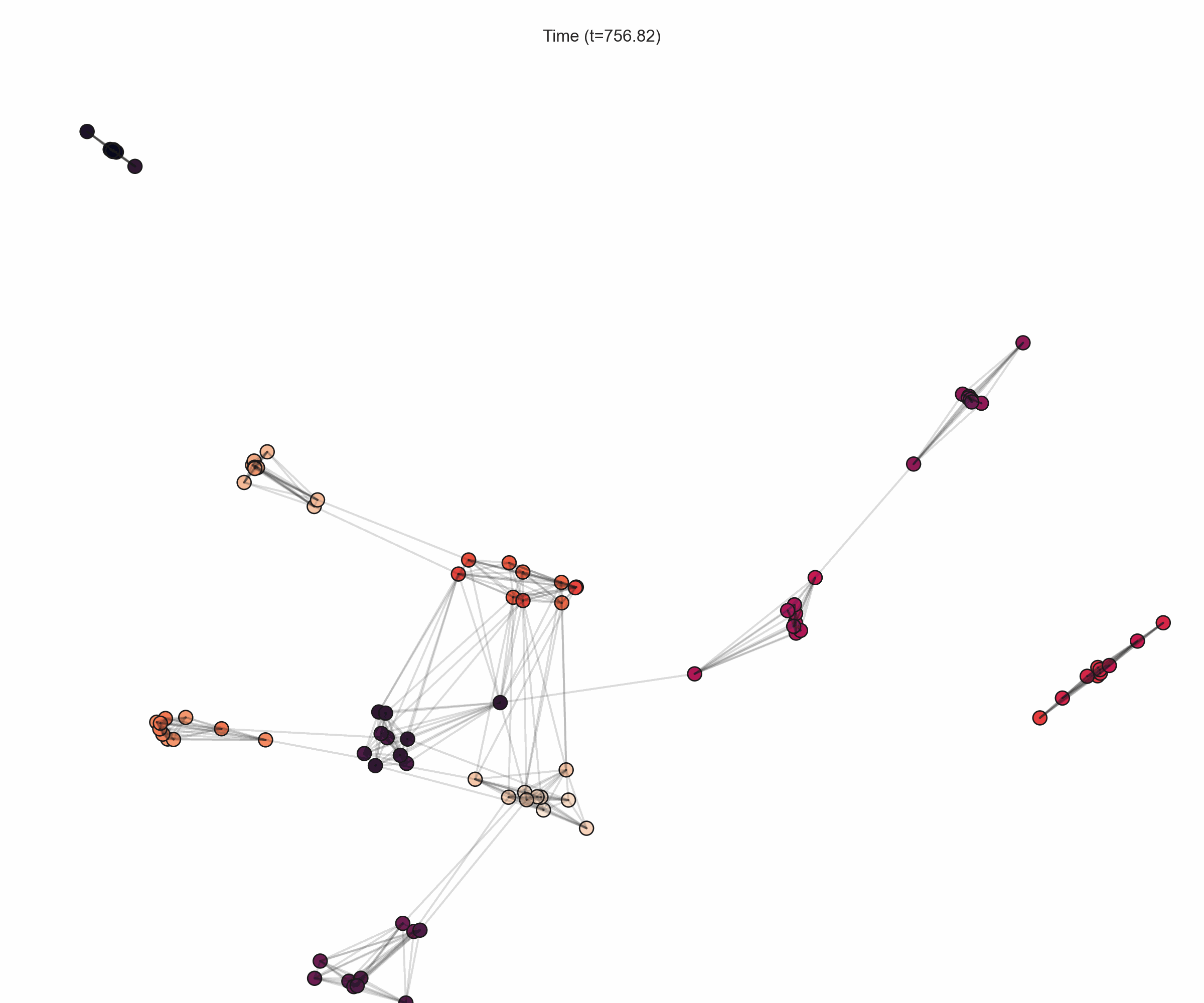}}
\hfill
\subfigure[$t=646$]{\includegraphics[trim={5cm 6cm 5cm 6cm},clip,width=0.16\textwidth]{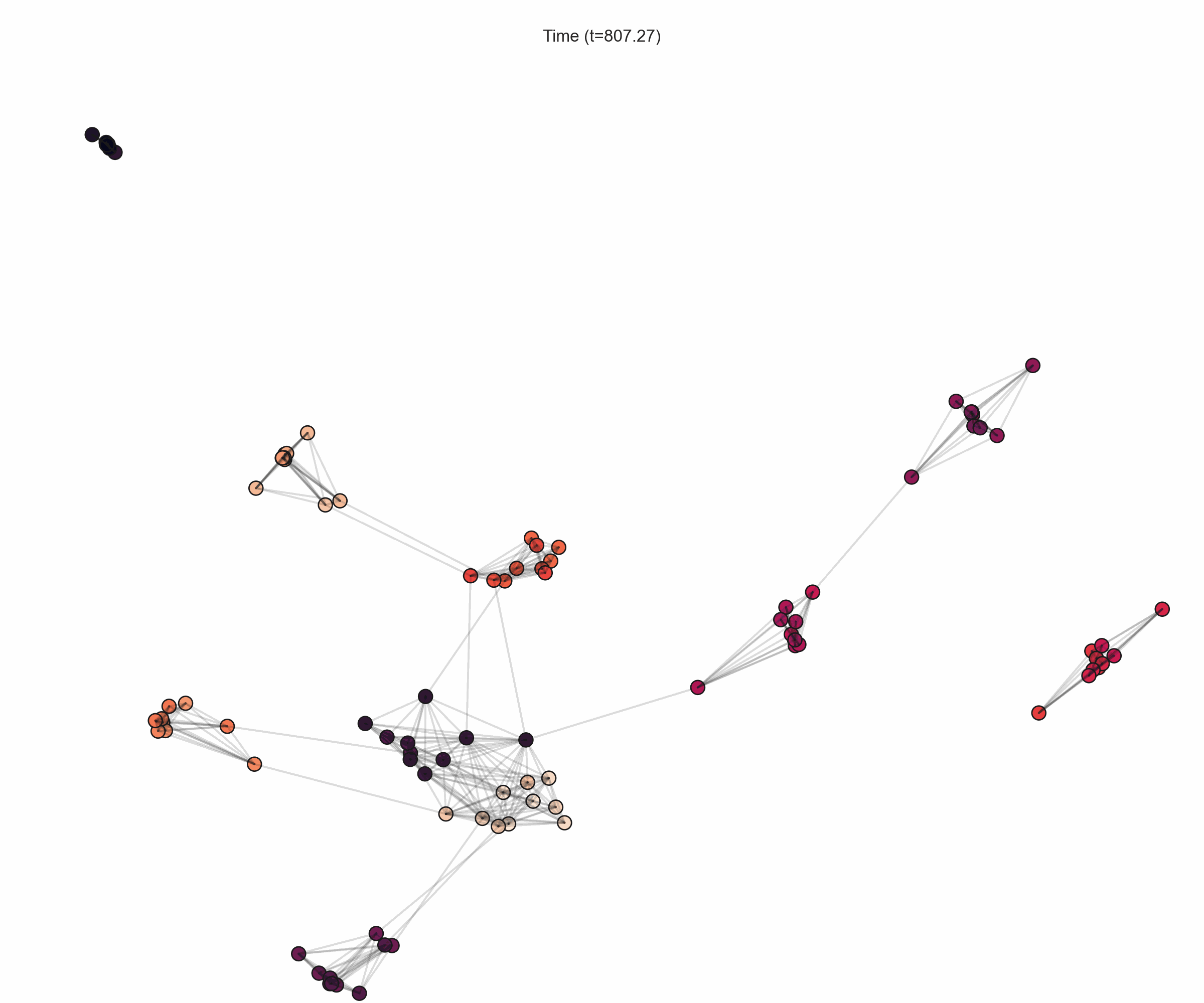}}
\hfill
\subfigure[$t=686$]{\includegraphics[trim={5cm 6cm 5cm 6cm},clip,width=0.16\textwidth]{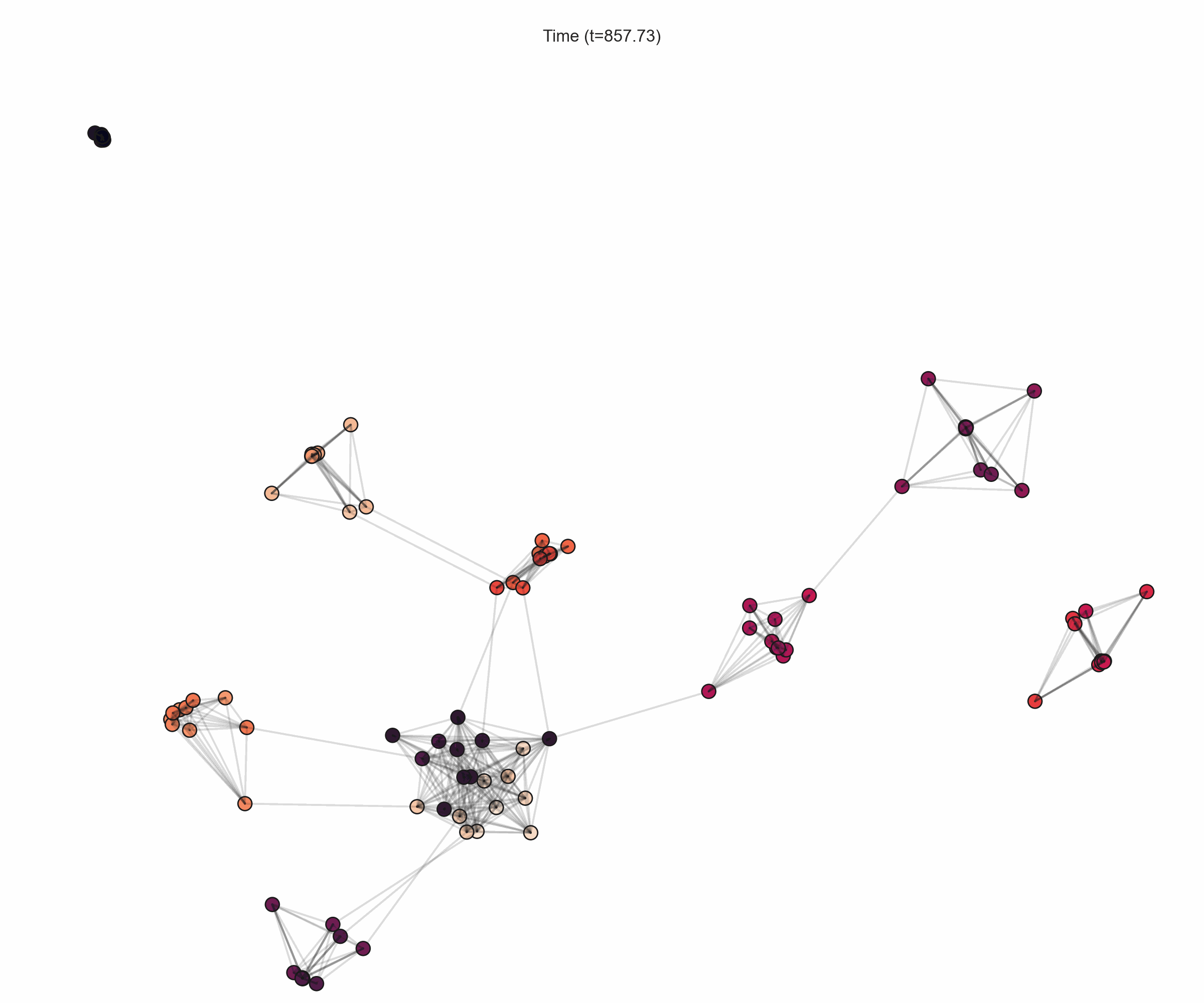}}
\hfill
\subfigure[$t=727$]{\includegraphics[trim={5cm 6cm 5cm 6cm},clip,width=0.16\textwidth]{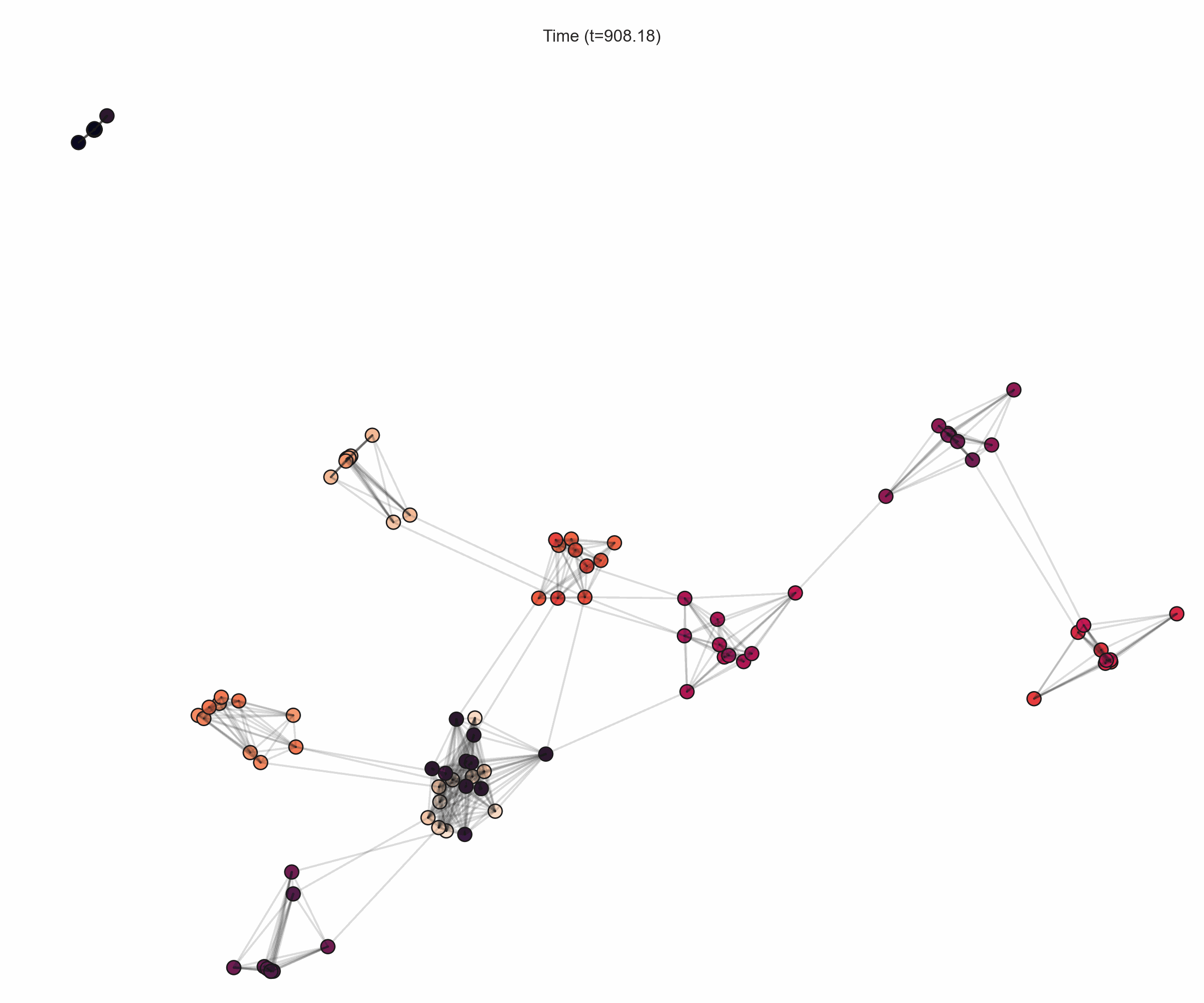}}
\caption{Snapshots of the continuous-time embeddings learned by \textsc{\modelname} for various time points over \textsl{Synthetic-$\alpha$}.}\label{fig:appendix_visualization_synthetic_mu}
\end{figure*}
%%%%%%%%%%%%%
\begin{figure*}[!ht]
\centering
\subfigure[$t=40$]{\includegraphics[trim={5cm 6cm 5cm 6cm},clip,width=0.16\textwidth]{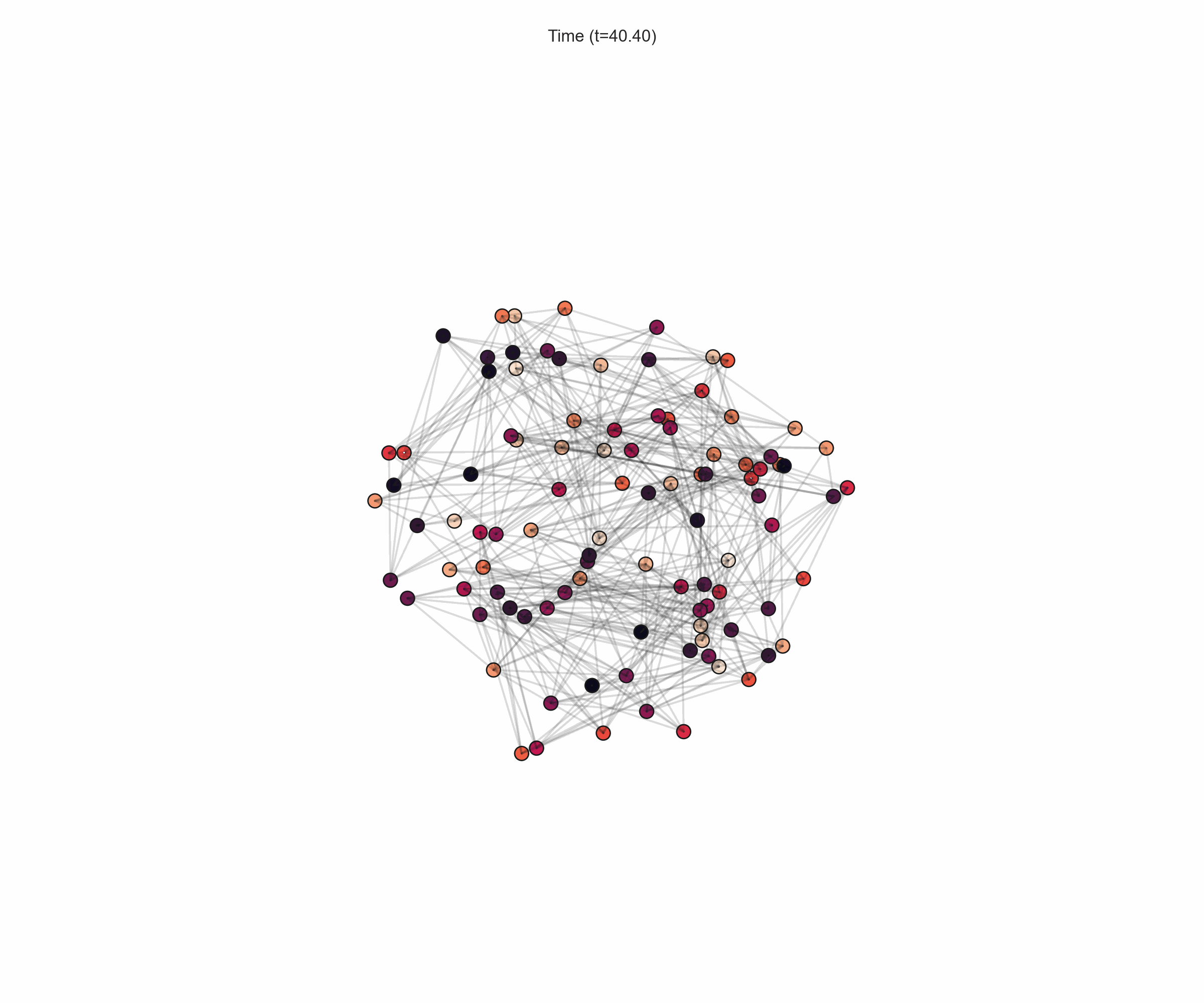}}
\hfill
\subfigure[$t=80$]{\includegraphics[trim={5cm 6cm 5cm 6cm},clip,width=0.16\textwidth]{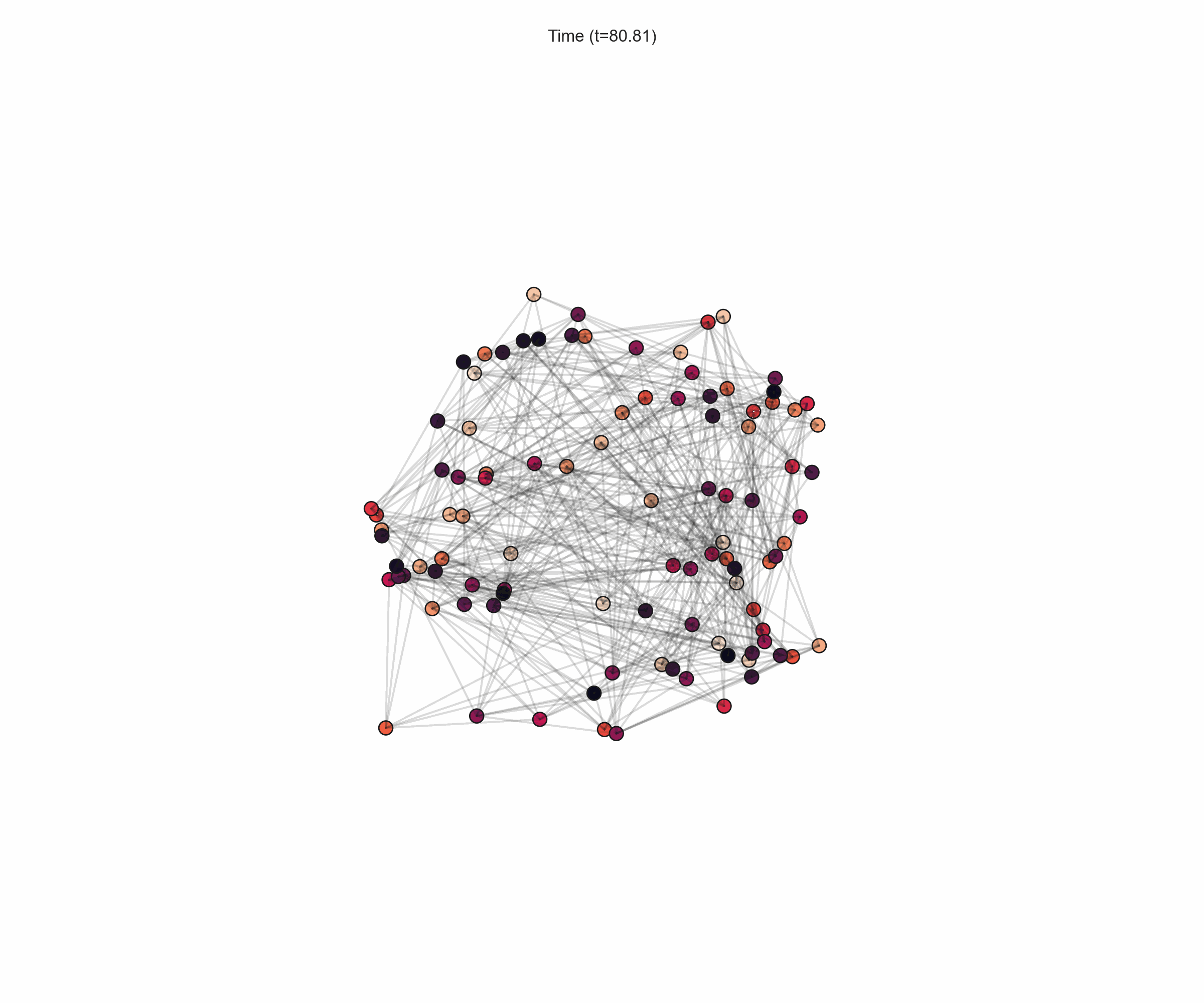}}
\hfill
\subfigure[$t=121$]{\includegraphics[trim={5cm 6cm 5cm 6cm},clip,width=0.16\textwidth]{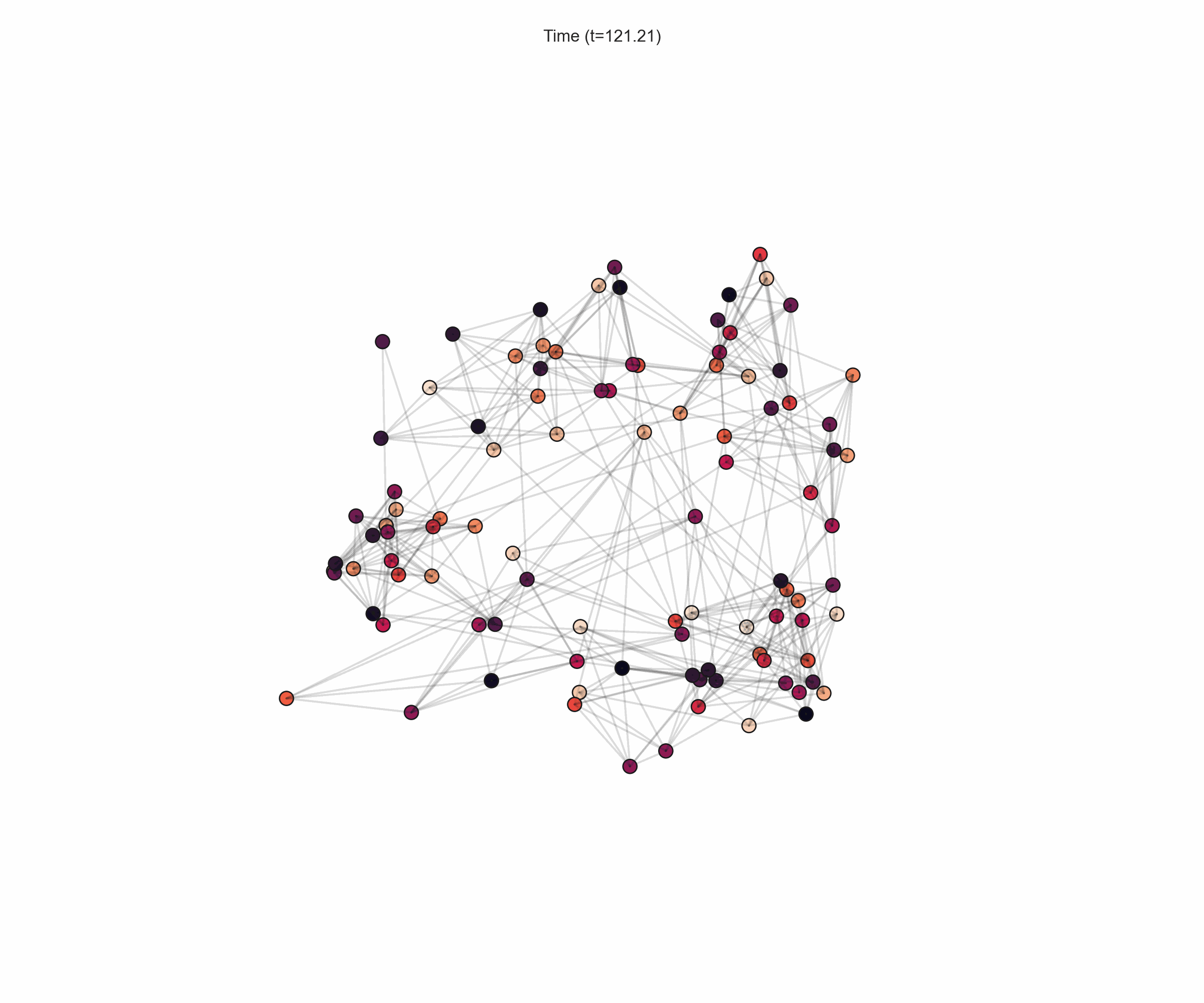}}
\hfill
\subfigure[$t=161$]{\includegraphics[trim={5cm 6cm 5cm 6cm},clip,width=0.16\textwidth]{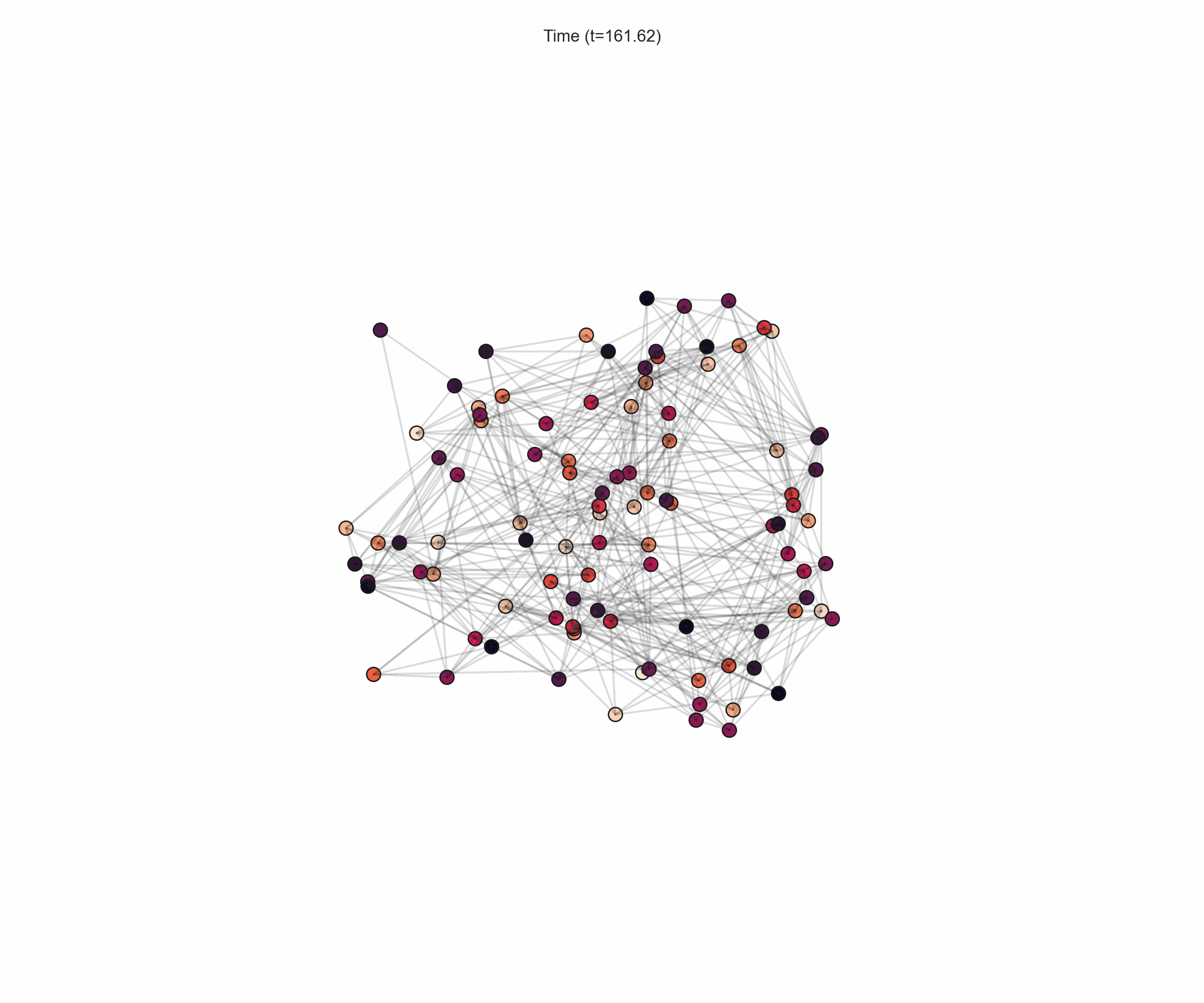}}
\hfill
\subfigure[$t=202$]{\includegraphics[trim={5cm 6cm 5cm 6cm},clip,width=0.16\textwidth]{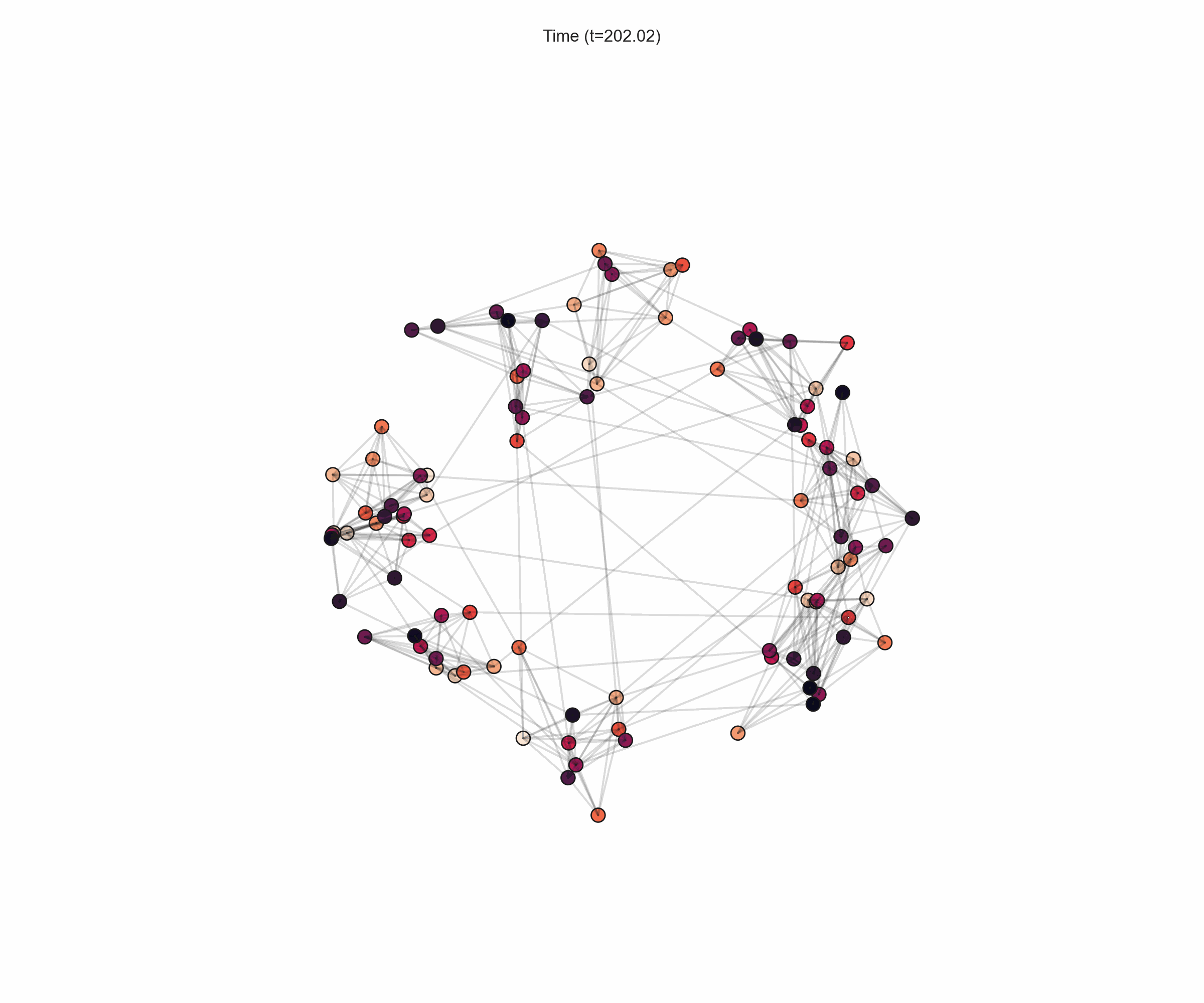}}
\hfill
\subfigure[$t=242$]{\includegraphics[trim={5cm 6cm 5cm 6cm},clip,width=0.16\textwidth]{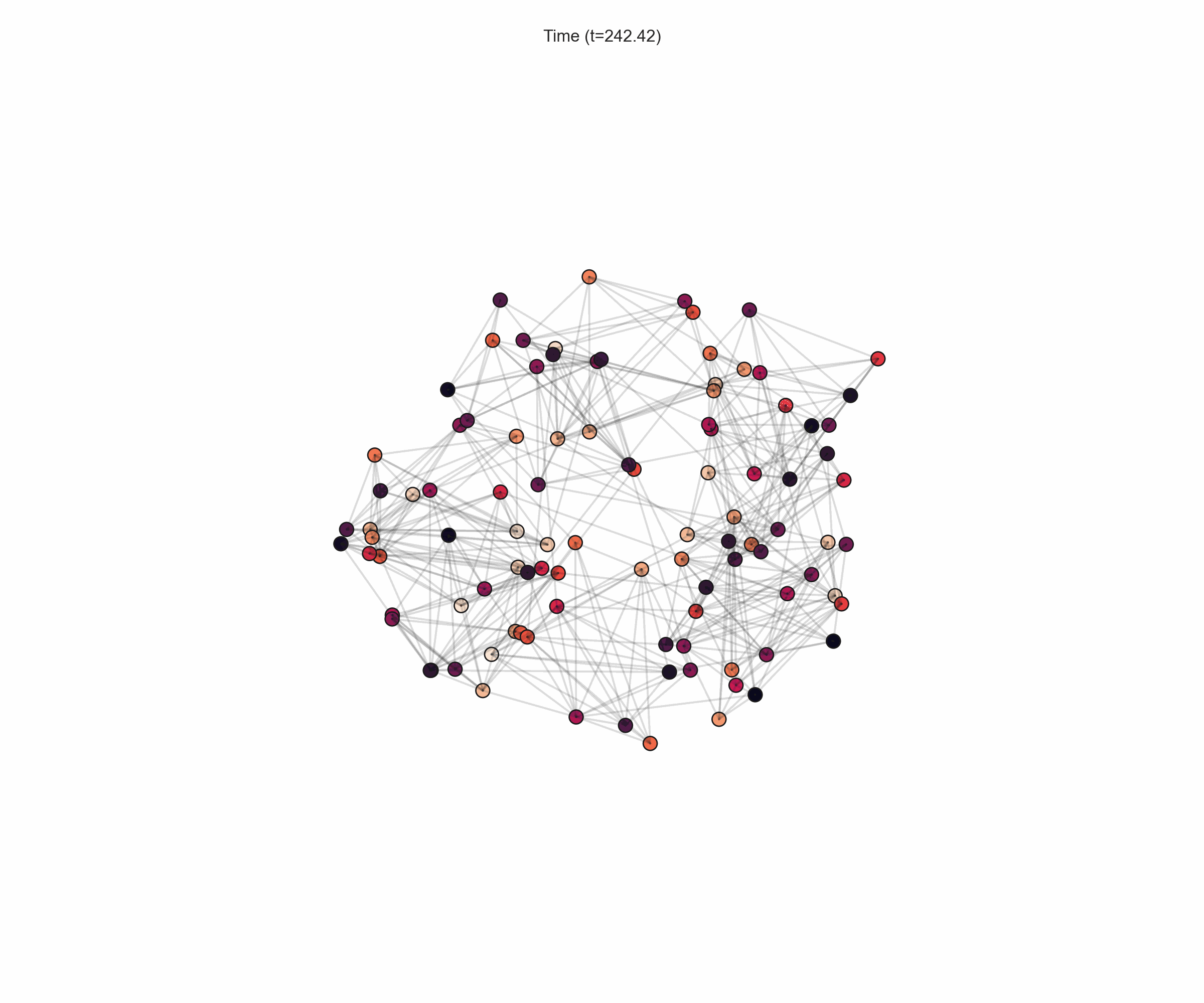}}
%%%%%%%%
\subfigure[$t=282$]{\includegraphics[trim={5cm 6cm 5cm 6cm},clip,width=0.16\textwidth]{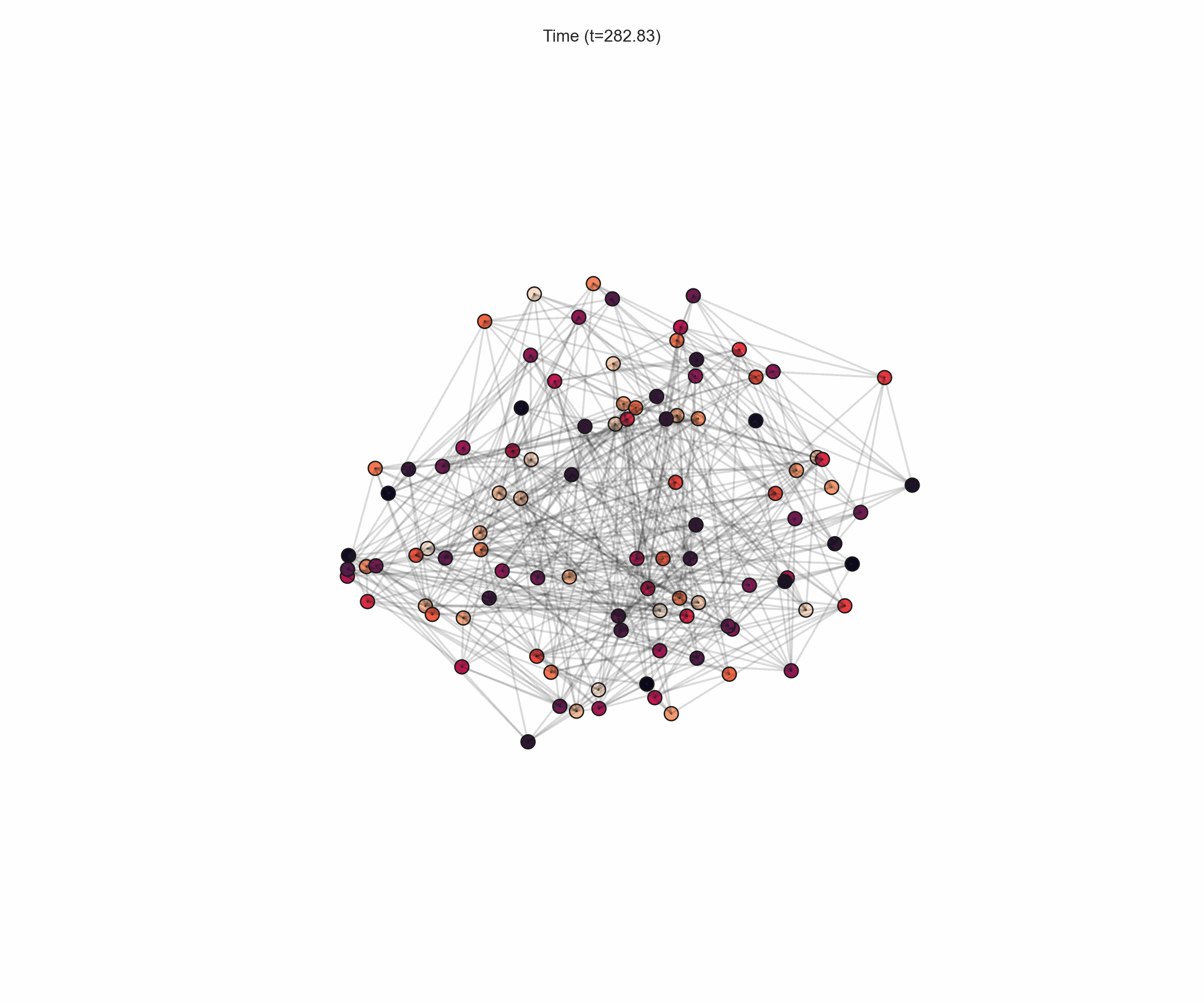}}
\hfill
\subfigure[$t=323$]{\includegraphics[trim={5cm 6cm 5cm 6cm},clip,width=0.16\textwidth]{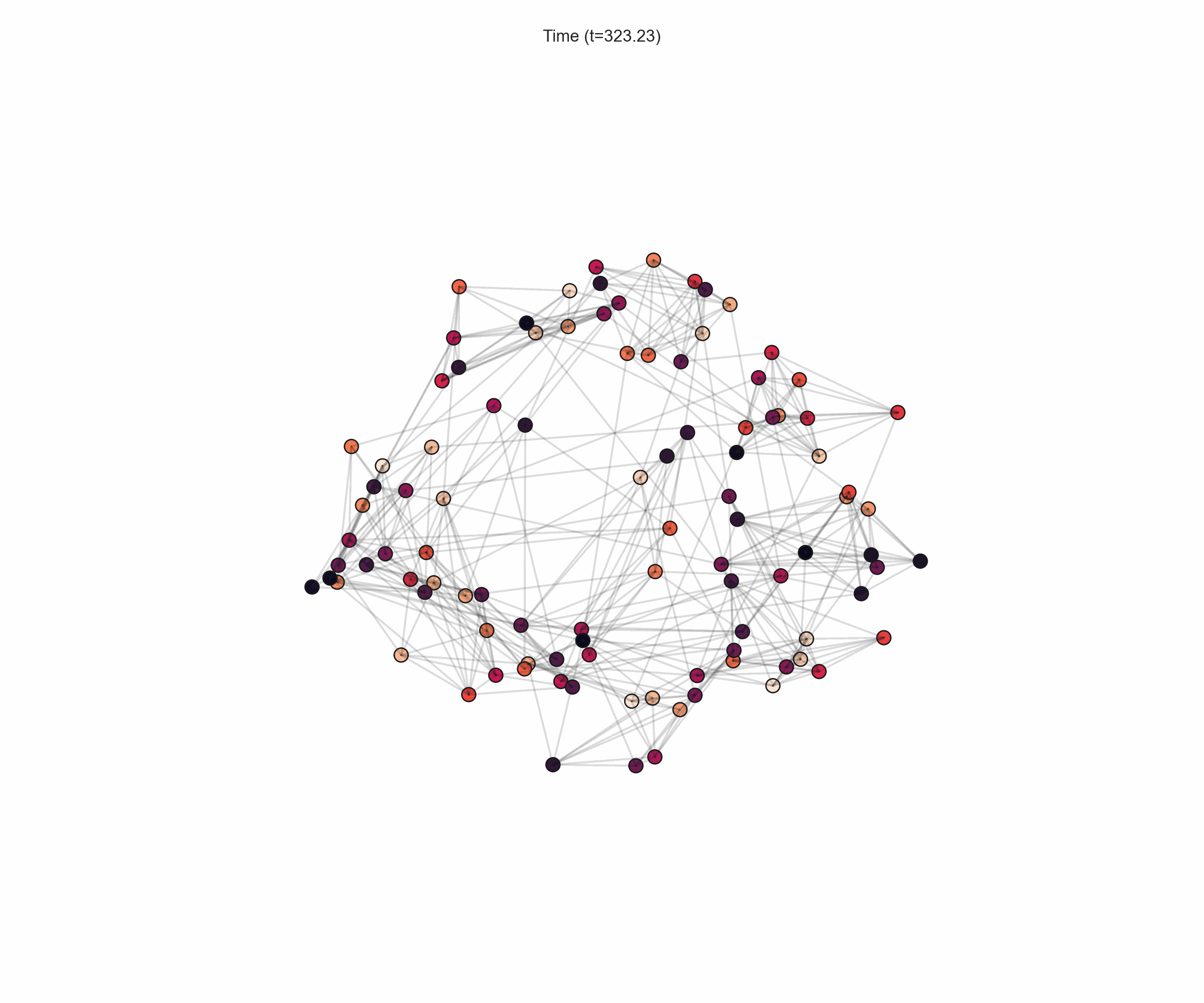}}
\hfill
\subfigure[$t=363$]{\includegraphics[trim={5cm 6cm 5cm 6cm},clip,width=0.16\textwidth]{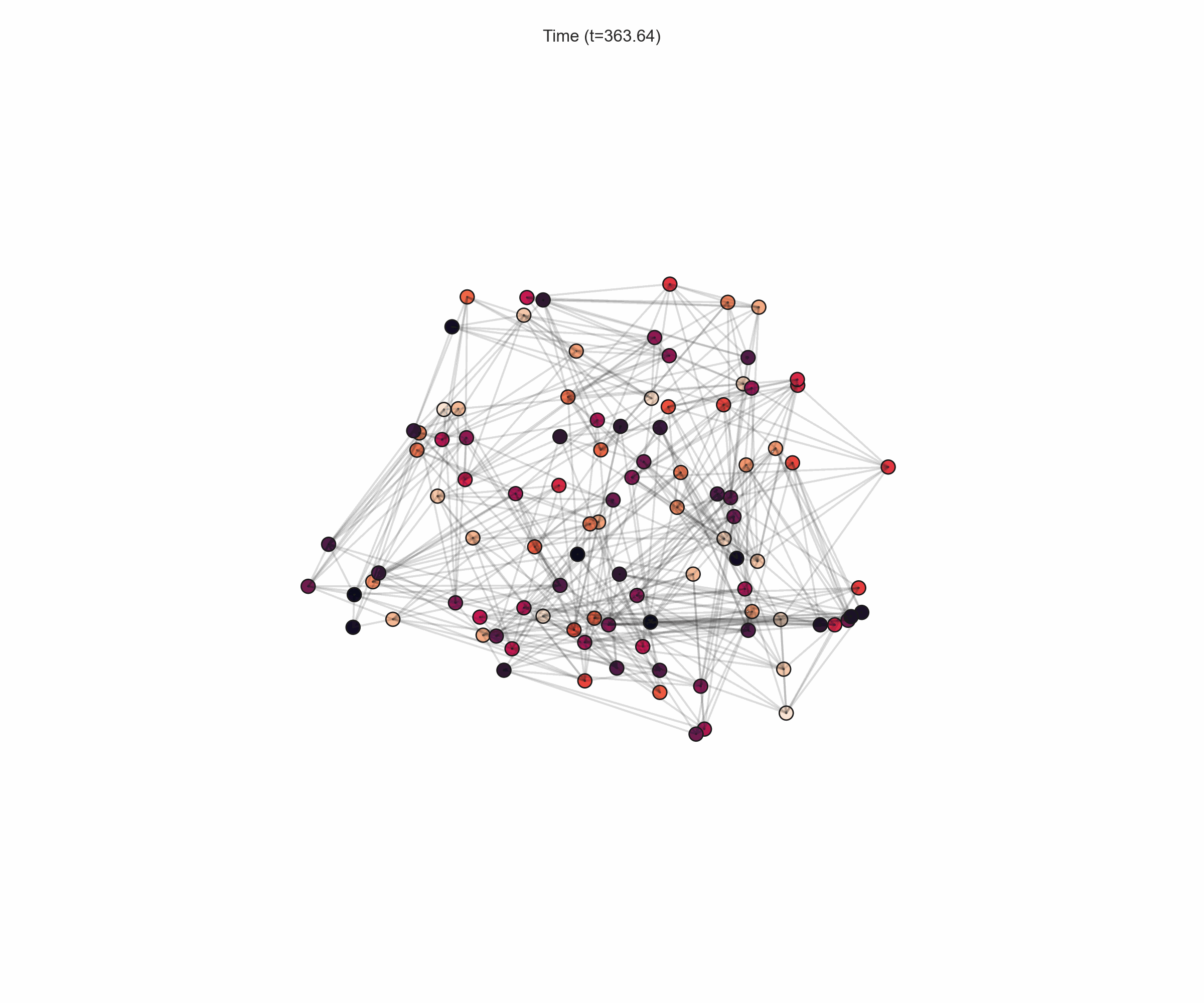}}
\hfill
\subfigure[$t=404$]{\includegraphics[trim={5cm 6cm 5cm 6cm},clip,width=0.16\textwidth]{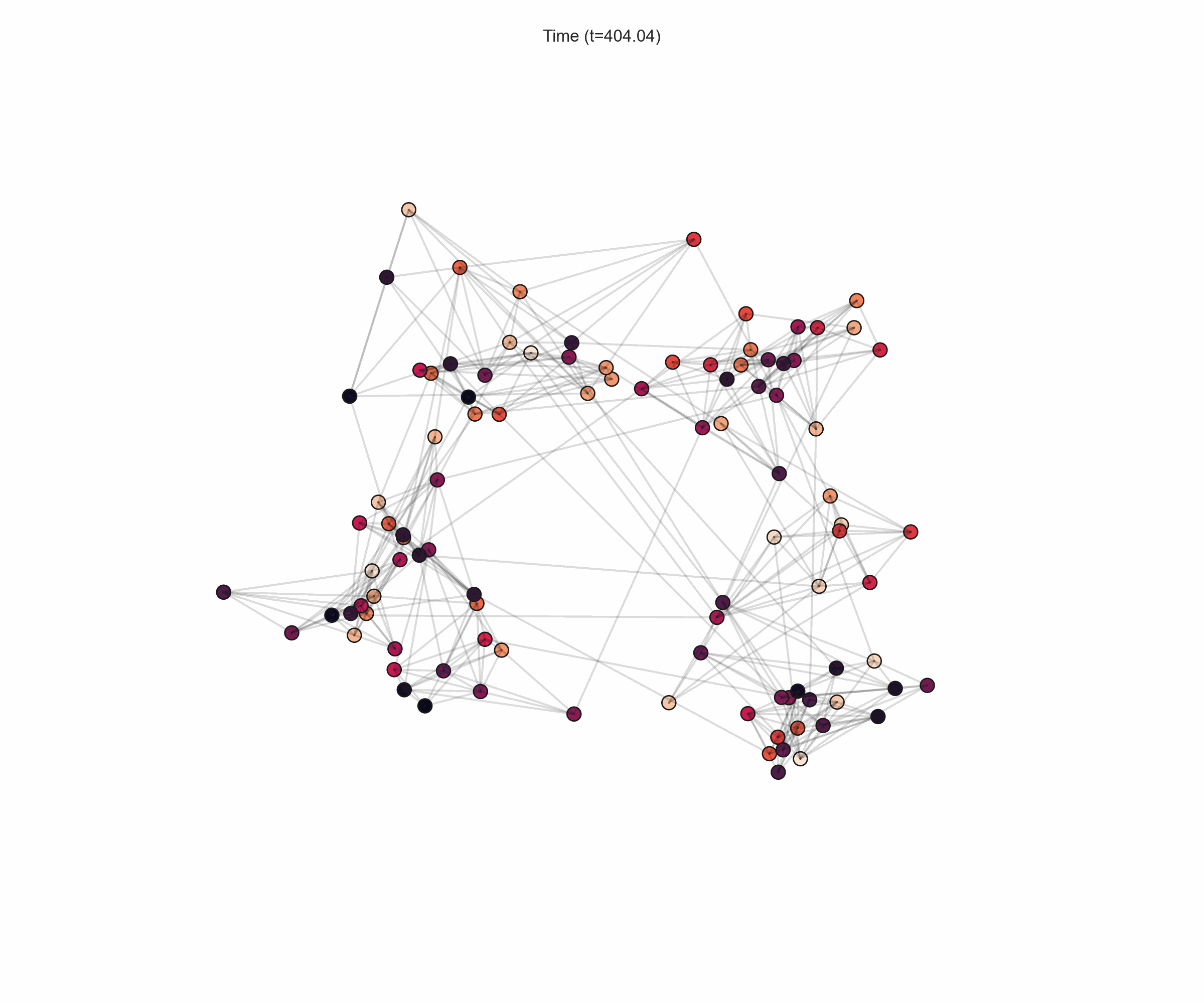}}
\hfill
\subfigure[$t=444$]{\includegraphics[trim={5cm 6cm 5cm 6cm},clip,width=0.16\textwidth]{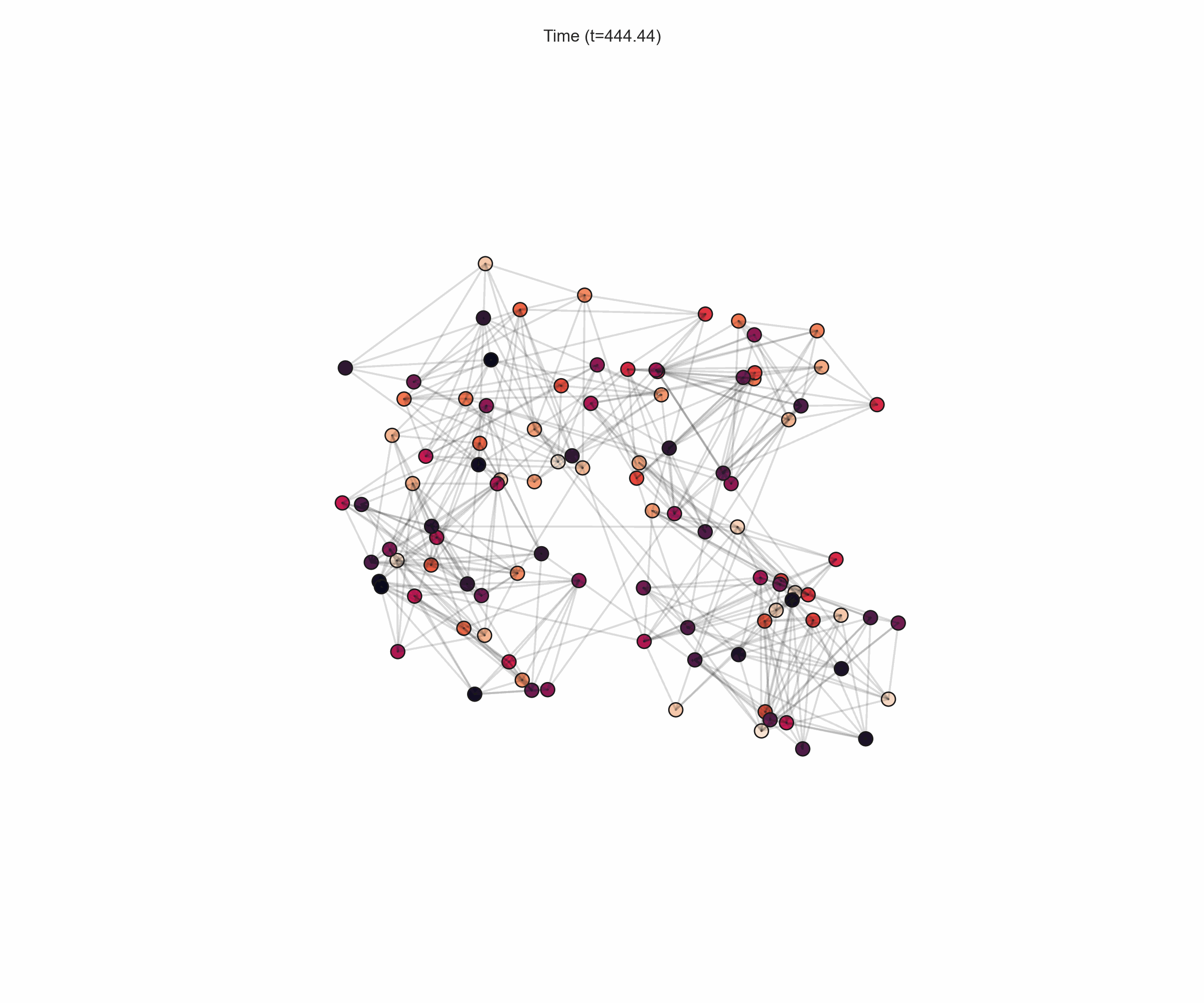}}
\hfill
\subfigure[$t=484$]{\includegraphics[trim={5cm 6cm 5cm 6cm},clip,width=0.16\textwidth]{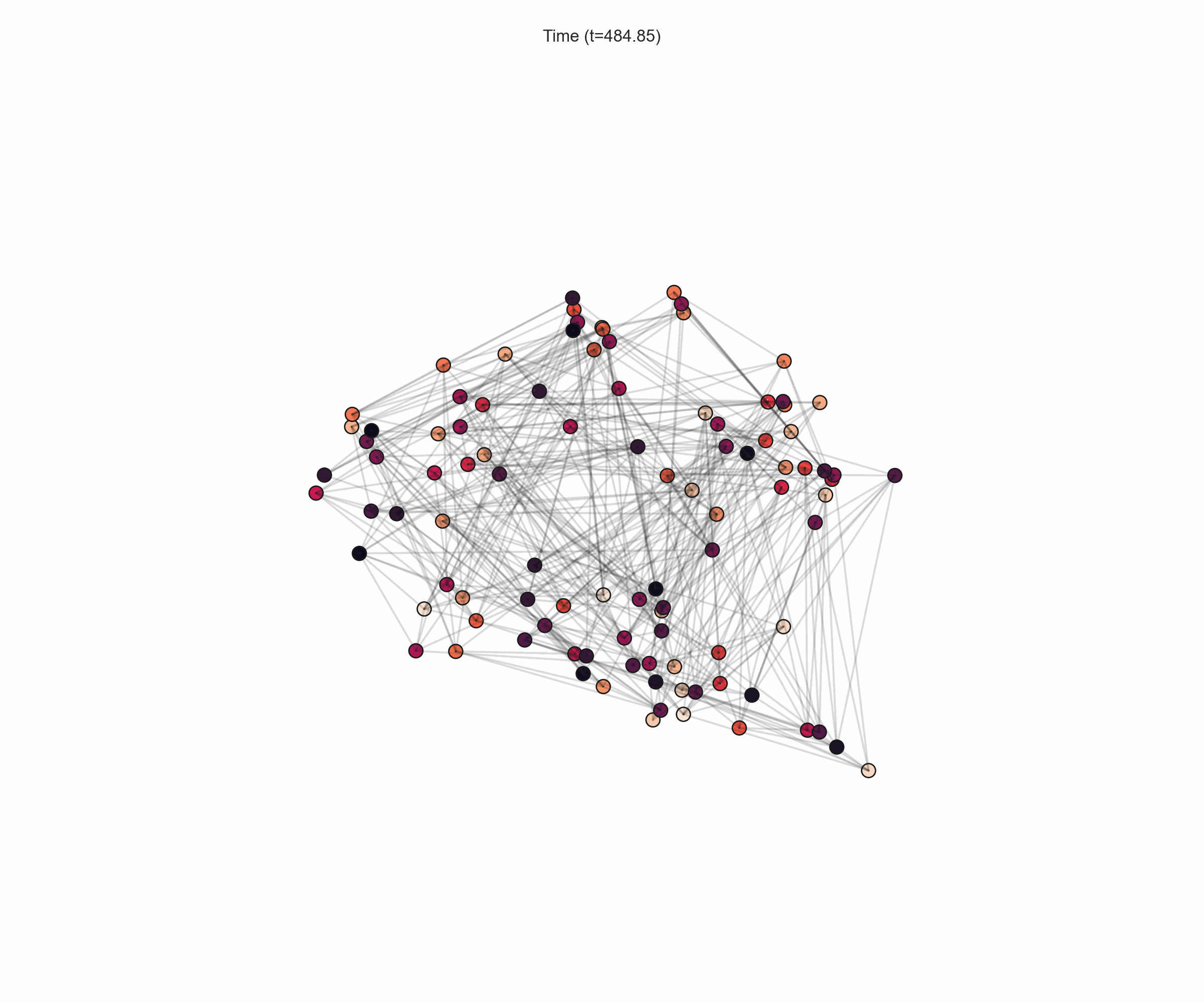}}
%%%%%%%%
\subfigure[$t=525$]{\includegraphics[trim={5cm 6cm 5cm 6cm},clip,width=0.16\textwidth]{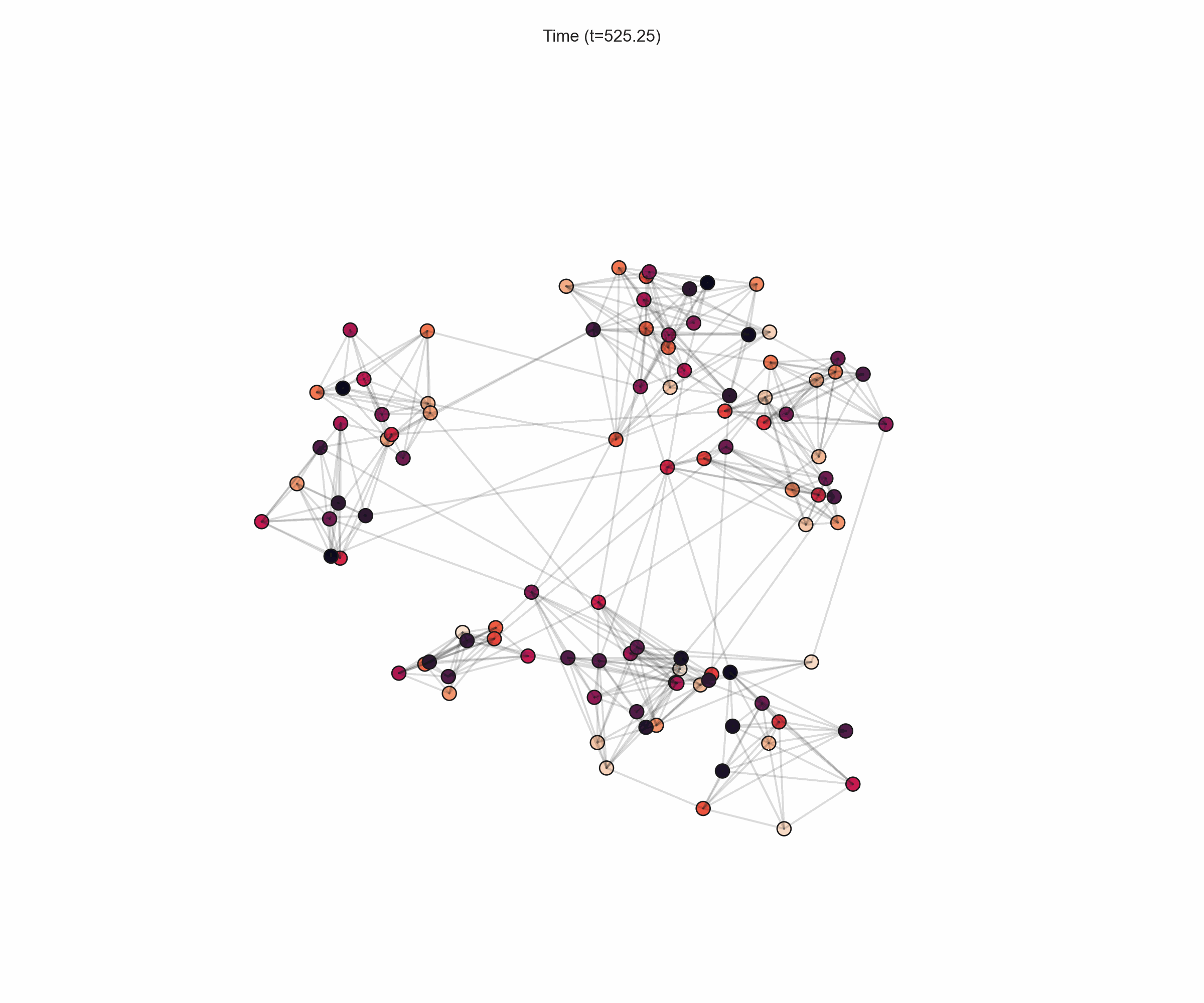}}
\hfill
\subfigure[$t=565$]{\includegraphics[trim={5cm 6cm 5cm 6cm},clip,width=0.16\textwidth]{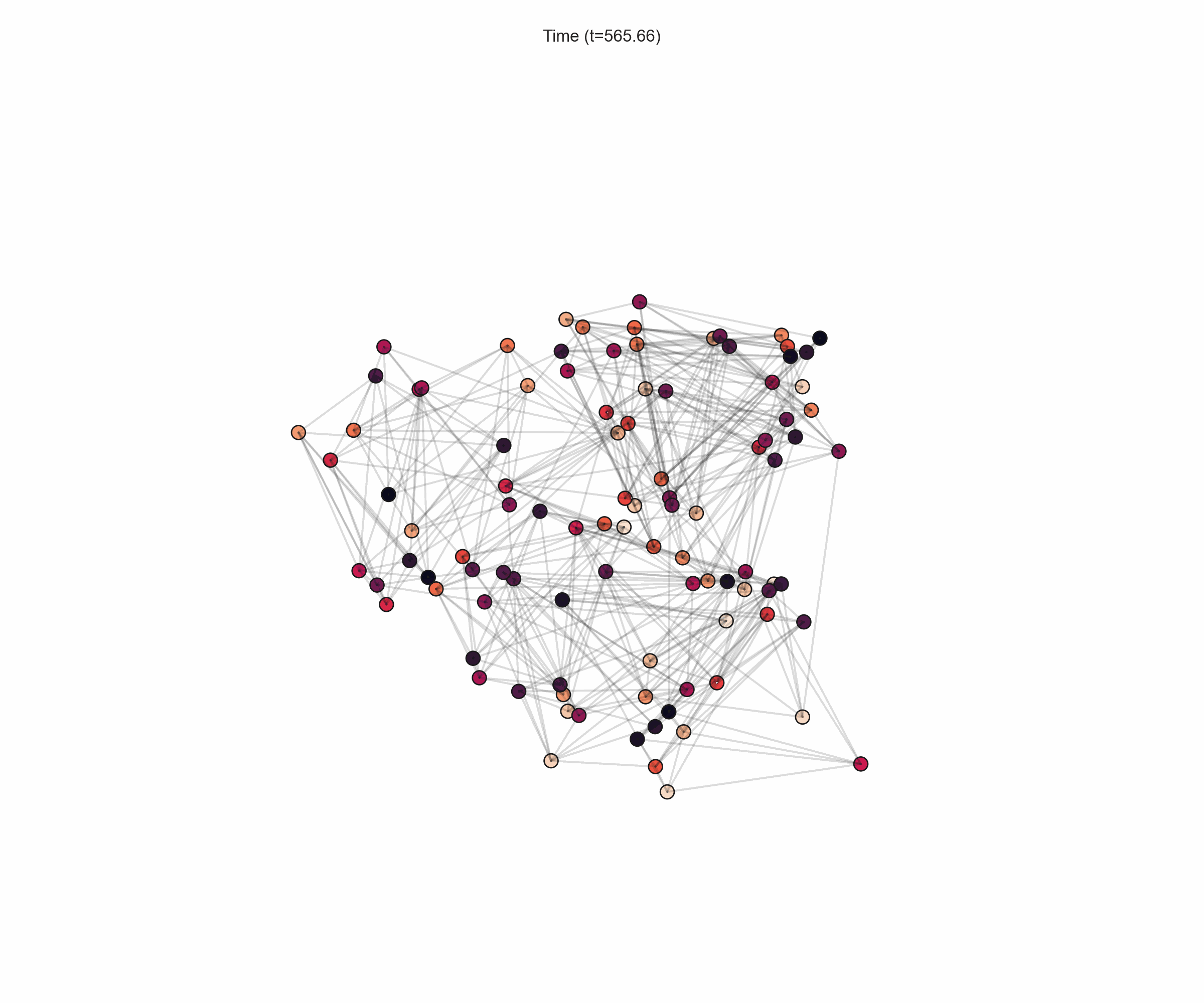}}
\hfill
\subfigure[$t=606$]{\includegraphics[trim={5cm 6cm 5cm 6cm},clip,width=0.16\textwidth]{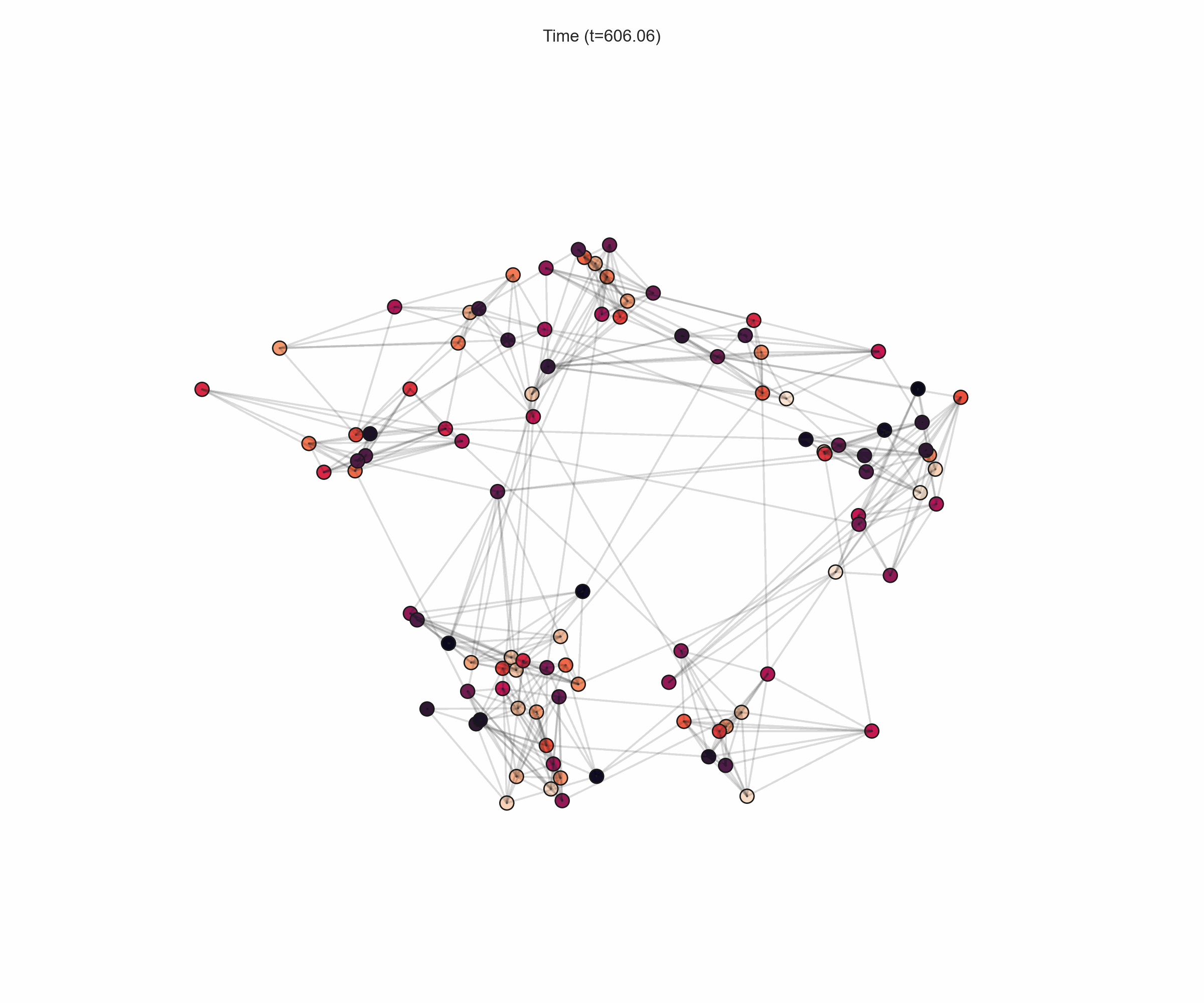}}
\hfill
\subfigure[$t=646$]{\includegraphics[trim={5cm 6cm 5cm 6cm},clip,width=0.16\textwidth]{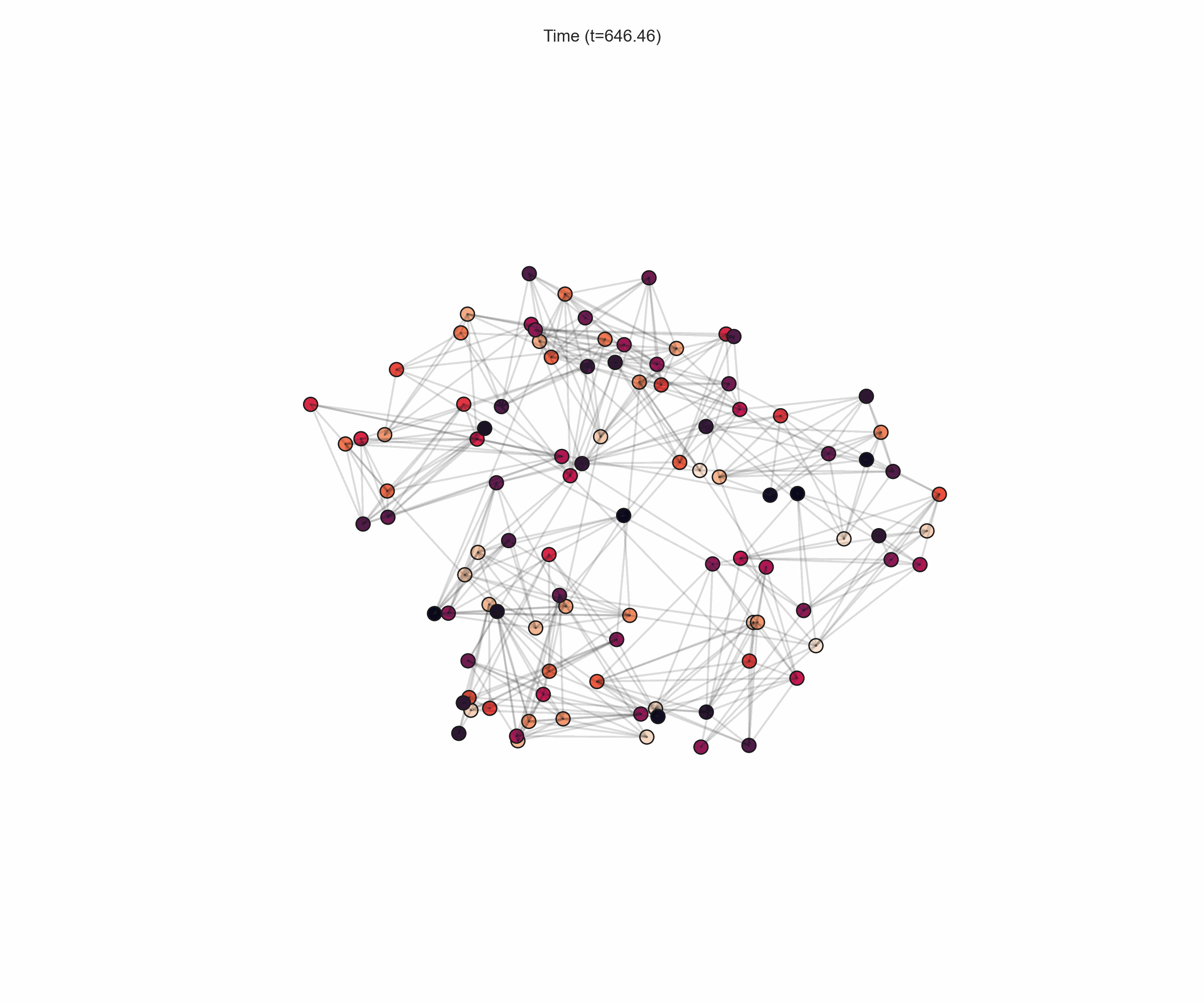}}
\hfill
\subfigure[$t=686$]{\includegraphics[trim={5cm 6cm 5cm 6cm},clip,width=0.16\textwidth]{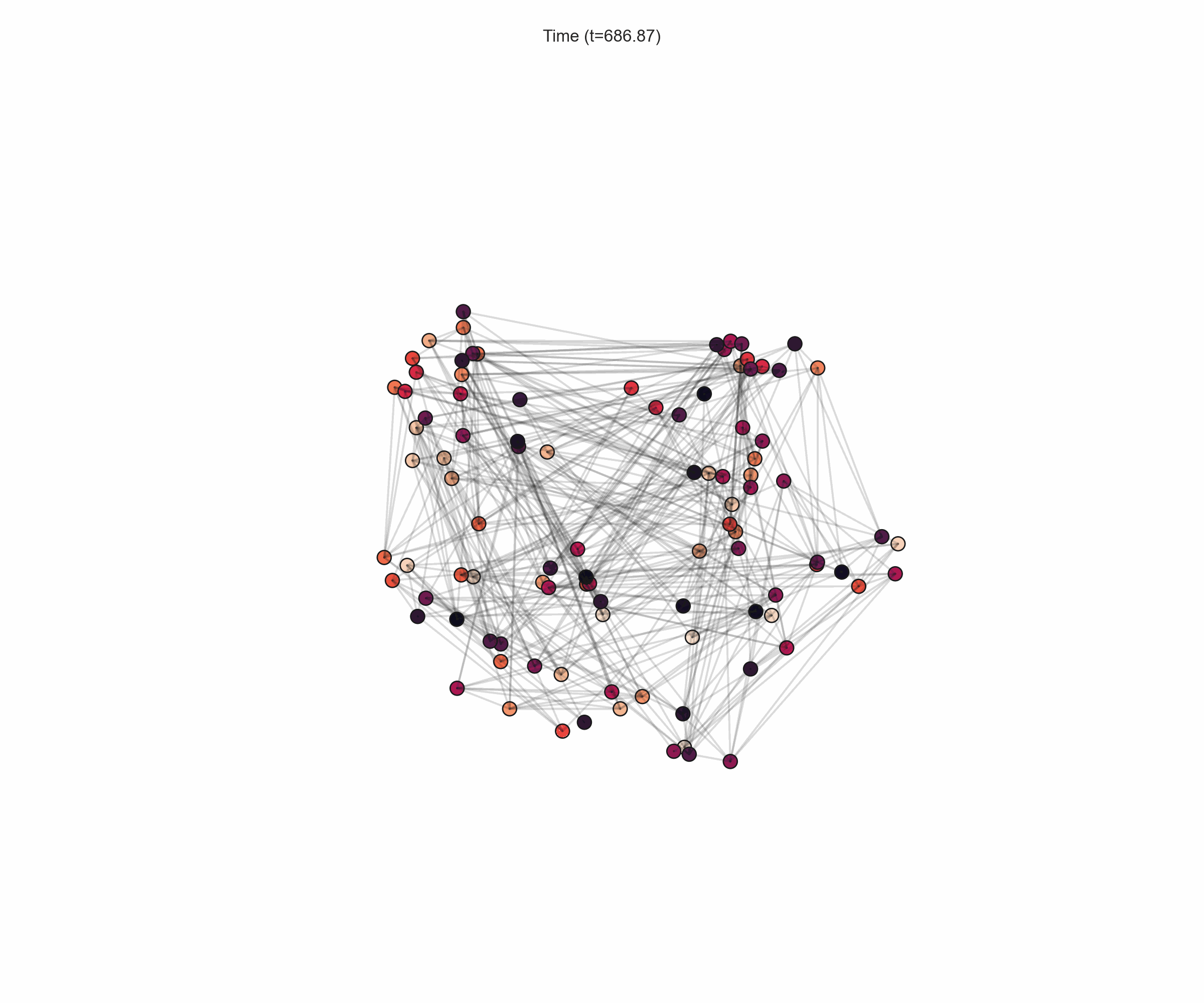}}
\hfill
\subfigure[$t=727$]{\includegraphics[trim={5cm 6cm 5cm 6cm},clip,width=0.16\textwidth]{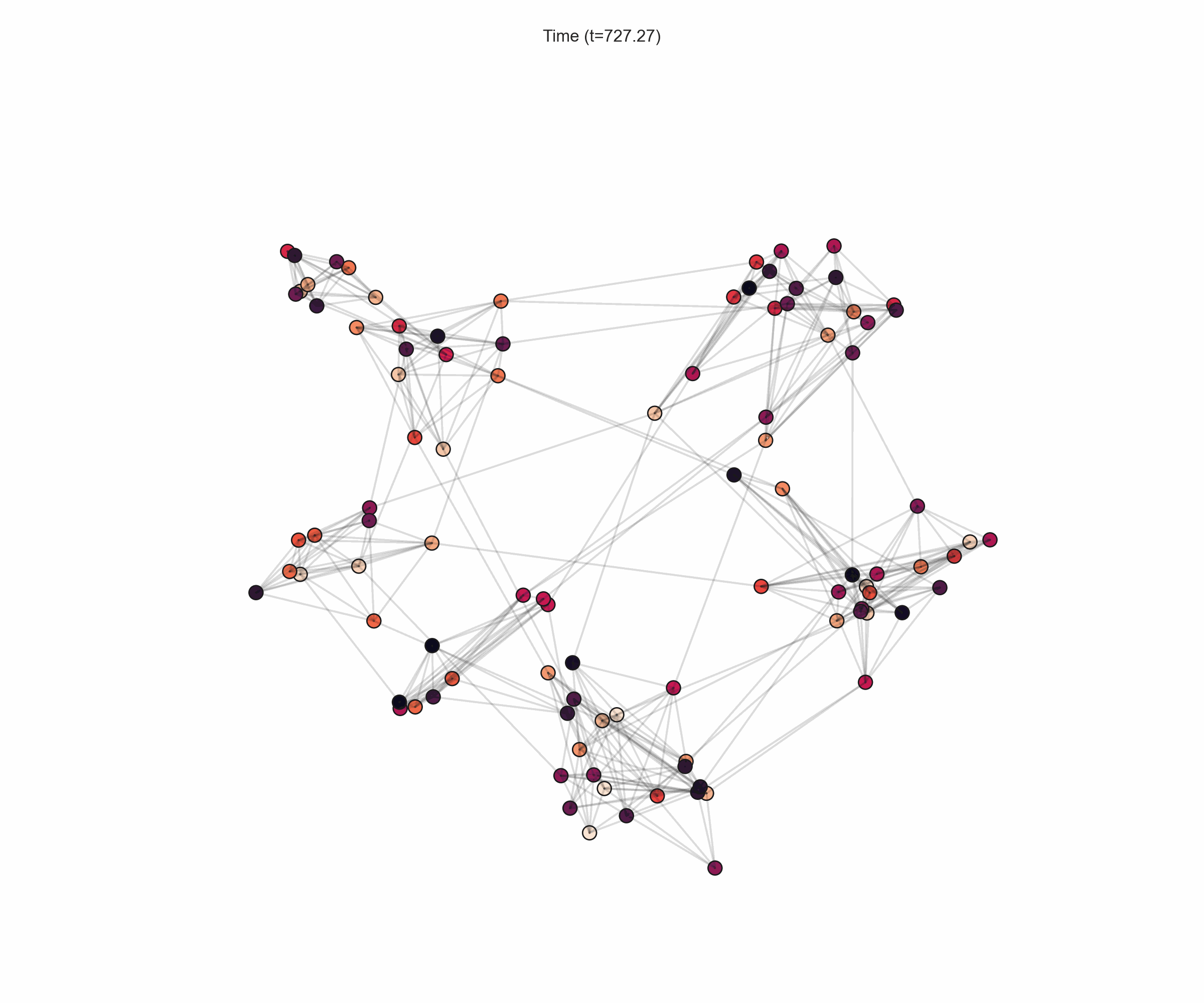}}
\caption{Snapshots of the continuous-time embeddings learned by \textsc{\modelname} for various time points over \textsl{Synthetic-$\beta$}.}\label{fig:appendix_visualization_synthetic_beta}
\end{figure*}
%%%%%%%%%%%%%
\begin{figure*}[!ht]
\centering
\subfigure[$t=78678$]{\includegraphics[trim={5cm 6cm 5cm 6cm},clip,width=0.16\textwidth]{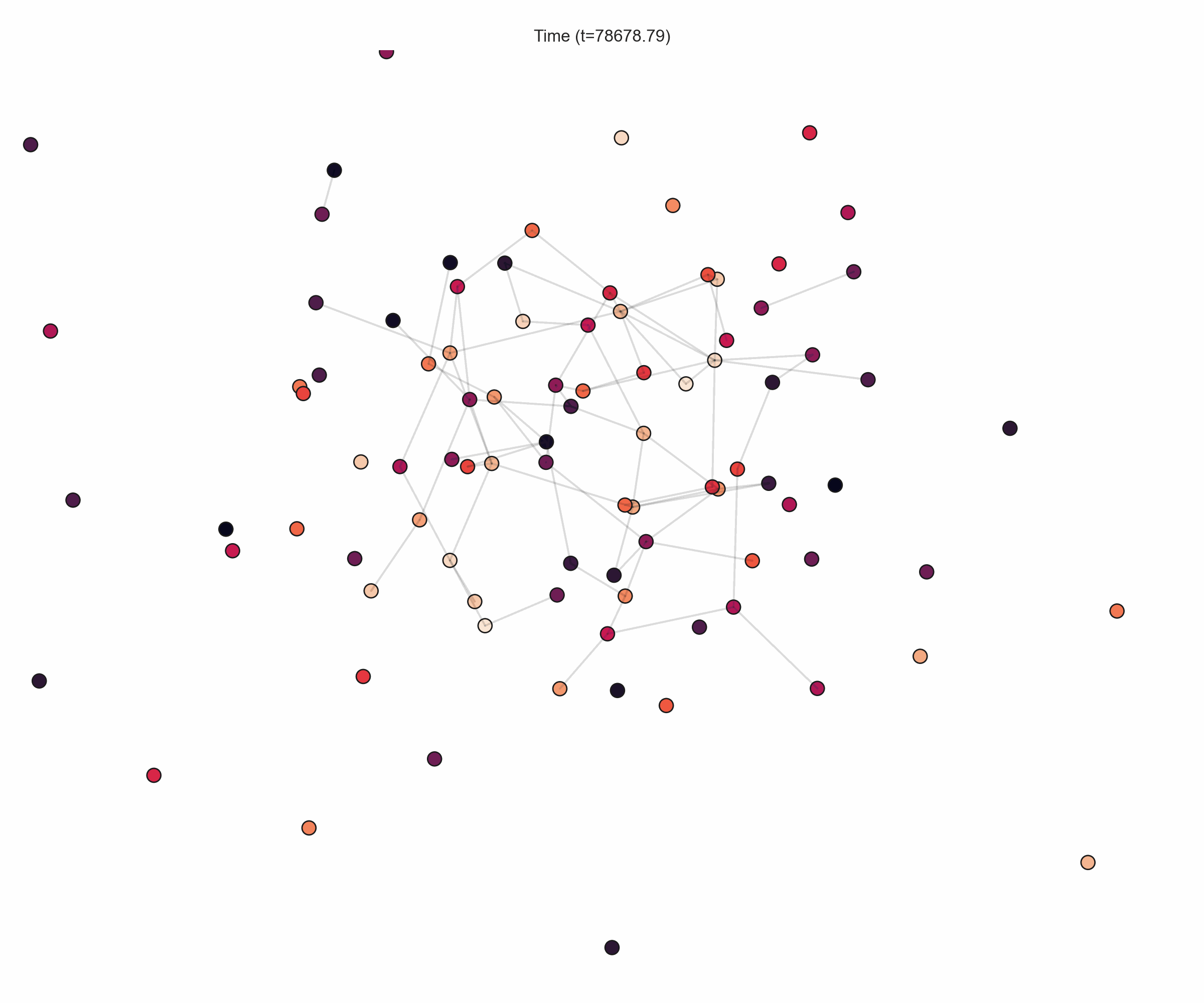}}
\hfill
\subfigure[$t=128557$]{\includegraphics[trim={5cm 6cm 5cm 6cm},clip,width=0.16\textwidth]{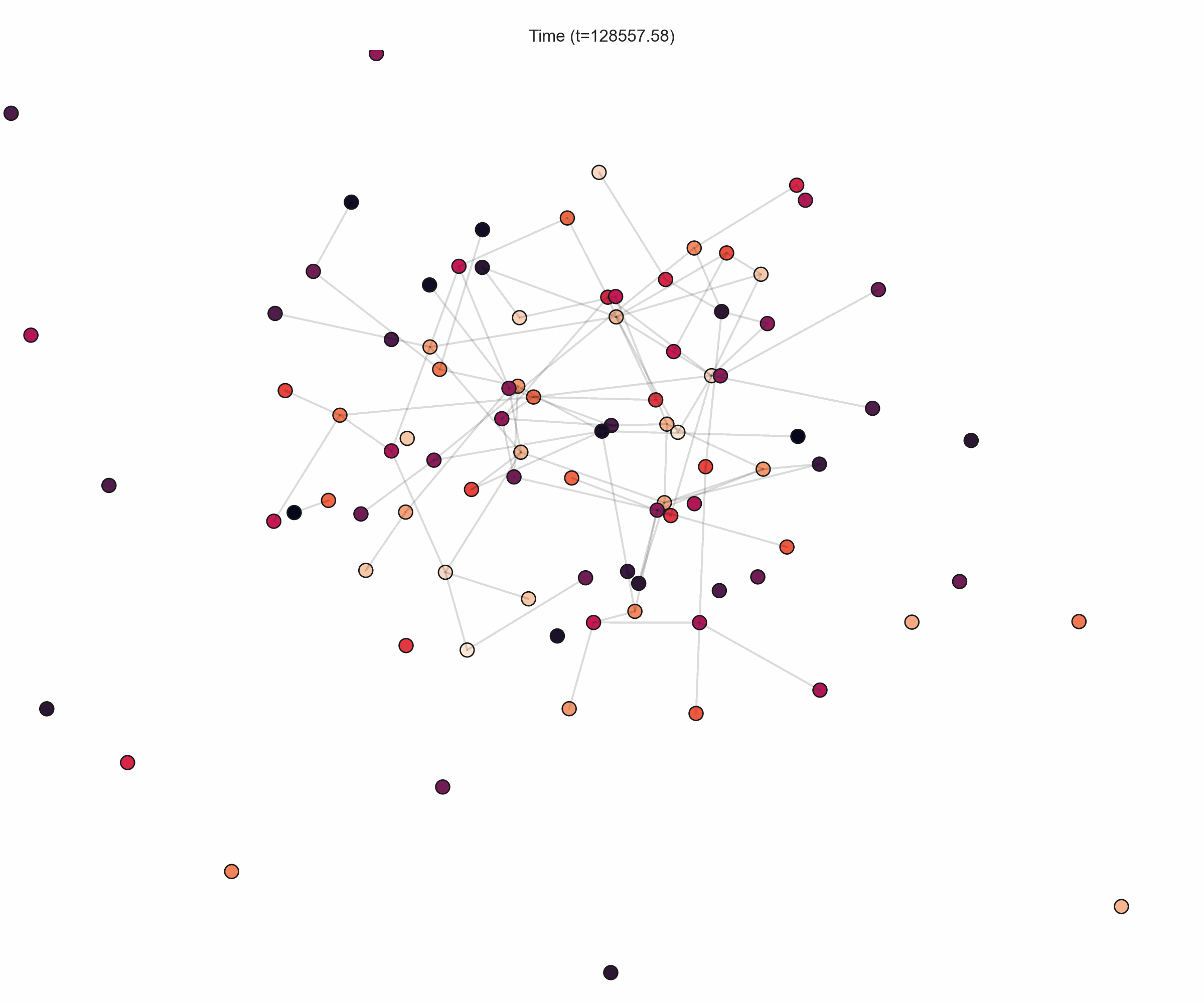}}
\hfill
\subfigure[$t=178436$]{\includegraphics[trim={5cm 6cm 5cm 6cm},clip,width=0.16\textwidth]{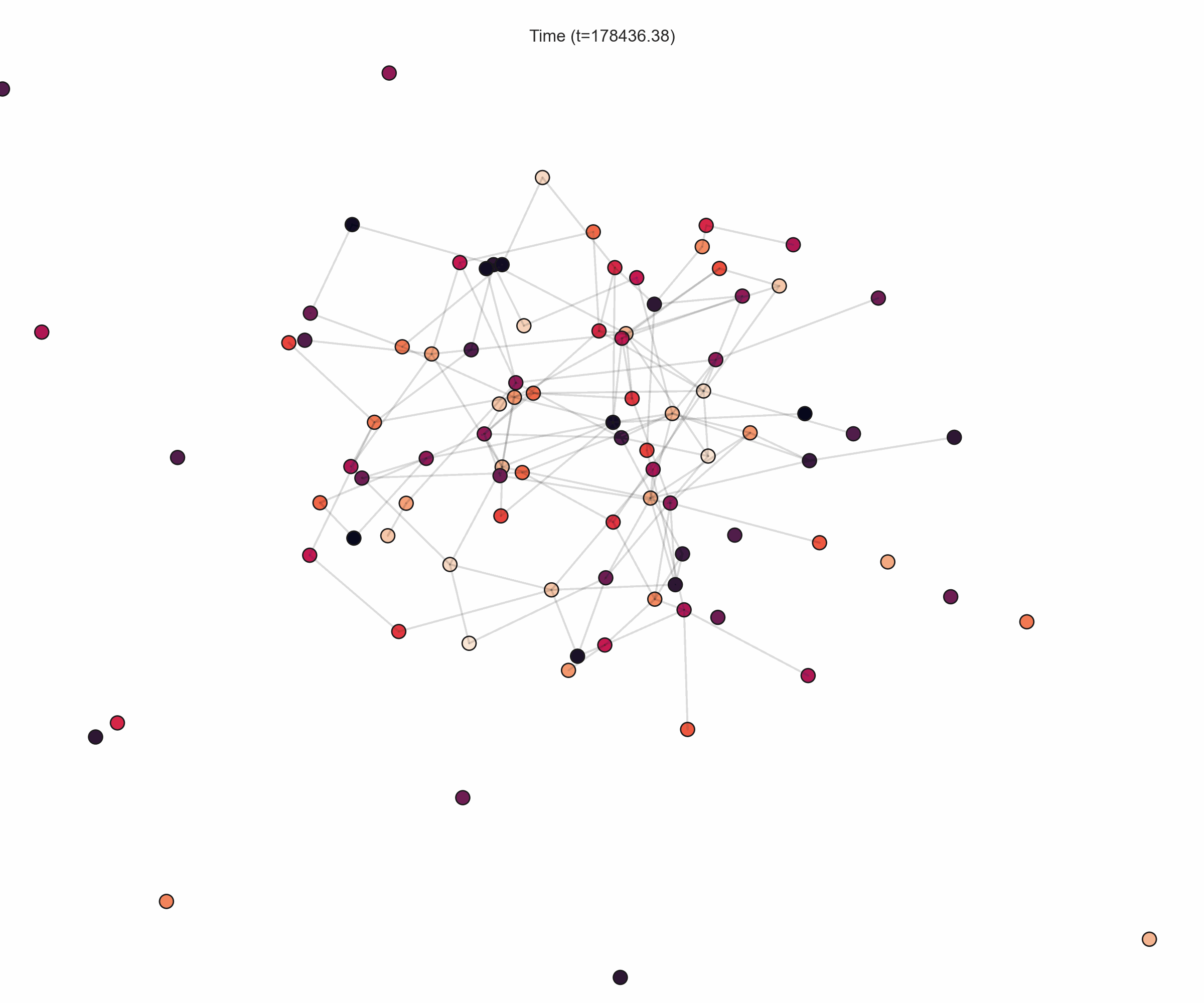}}
\hfill
\subfigure[$t=228315$]{\includegraphics[trim={5cm 6cm 5cm 6cm},clip,width=0.16\textwidth]{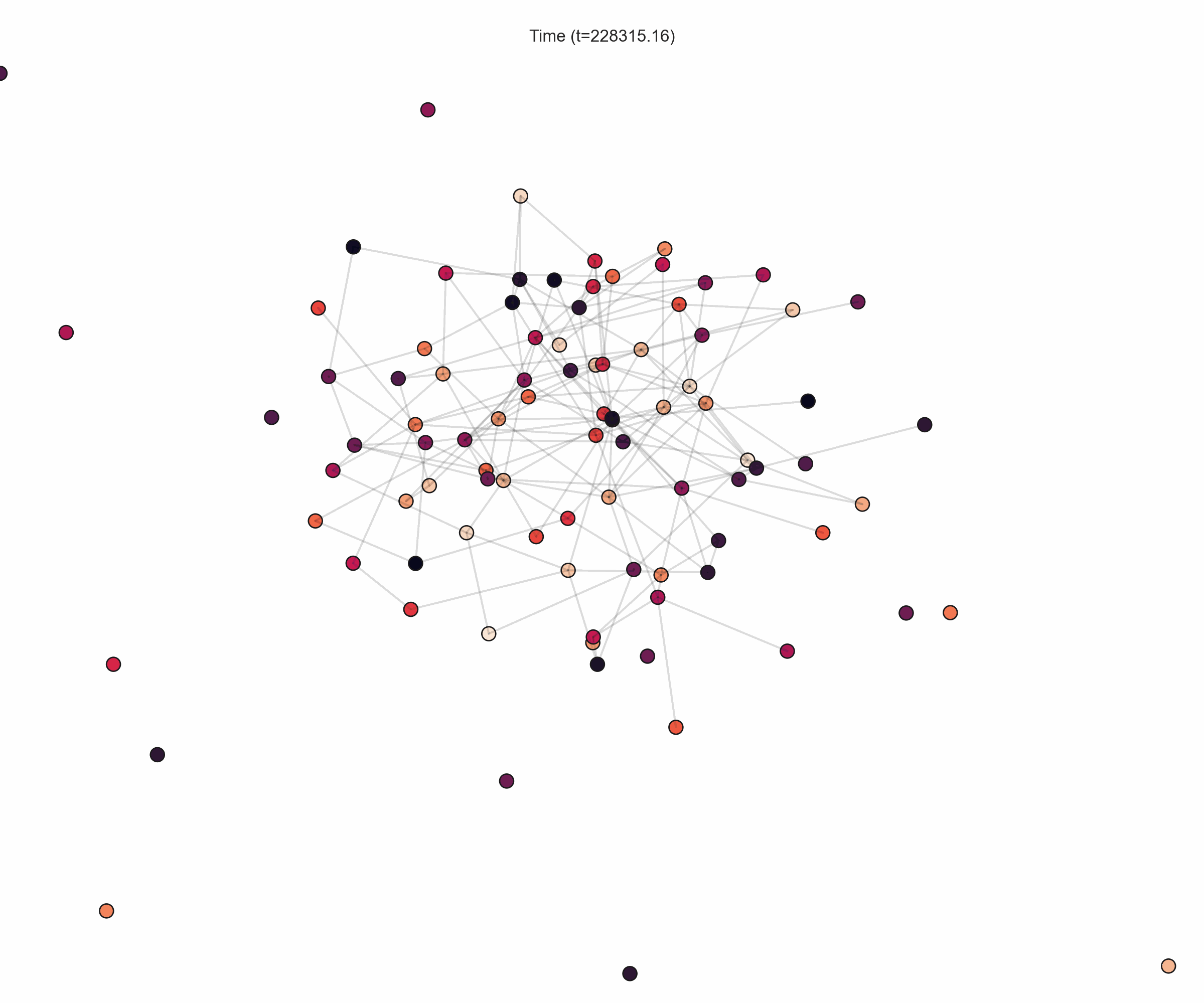}}
\hfill
\subfigure[$t=278193$]{\includegraphics[trim={5cm 6cm 5cm 6cm},clip,width=0.16\textwidth]{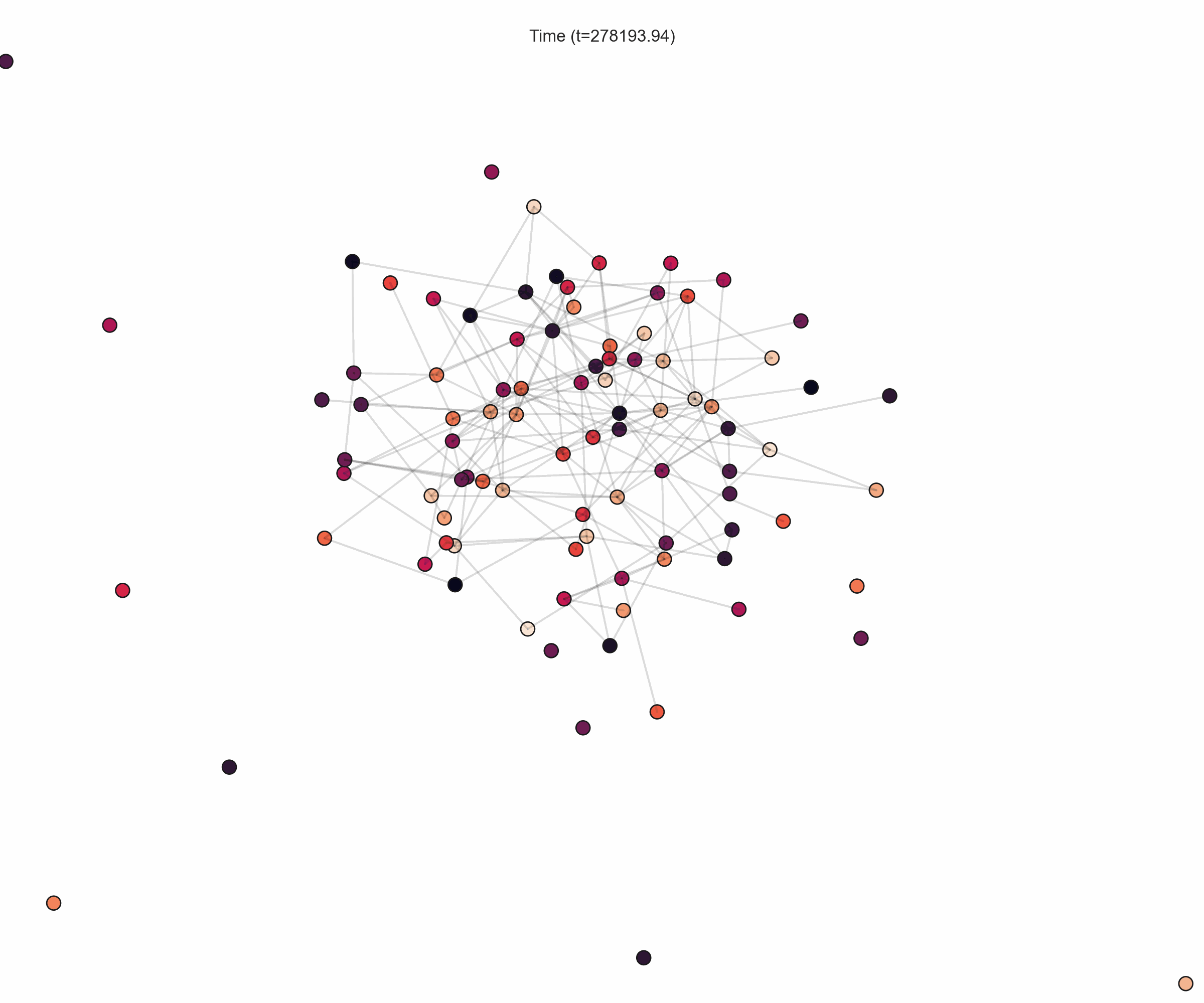}}
\hfill
\subfigure[$t=328072$]{\includegraphics[trim={5cm 6cm 5cm 6cm},clip,width=0.16\textwidth]{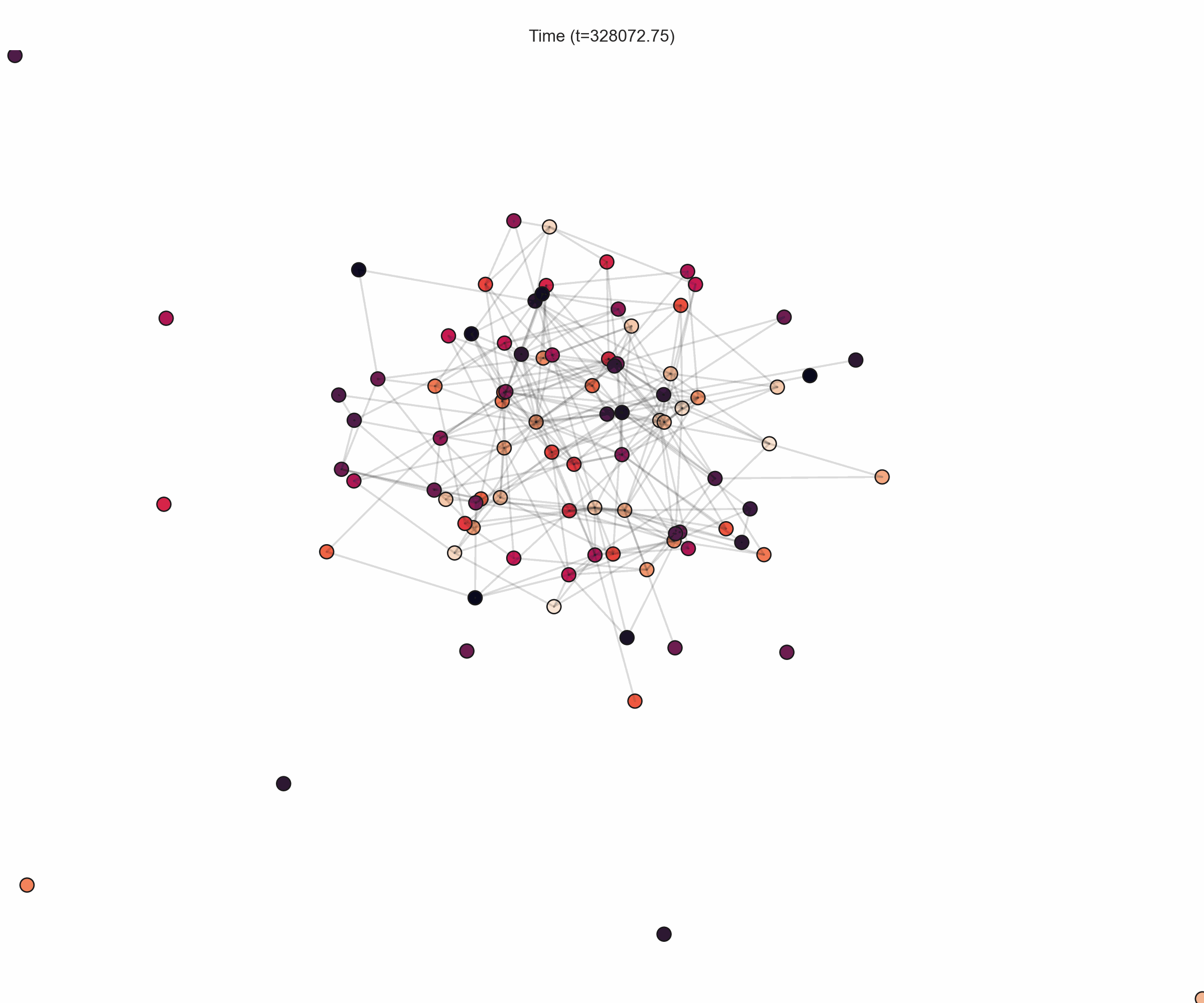}}
%%%%%%%%
\subfigure[$t=377951$]{\includegraphics[trim={5cm 6cm 5cm 6cm},clip,width=0.16\textwidth]{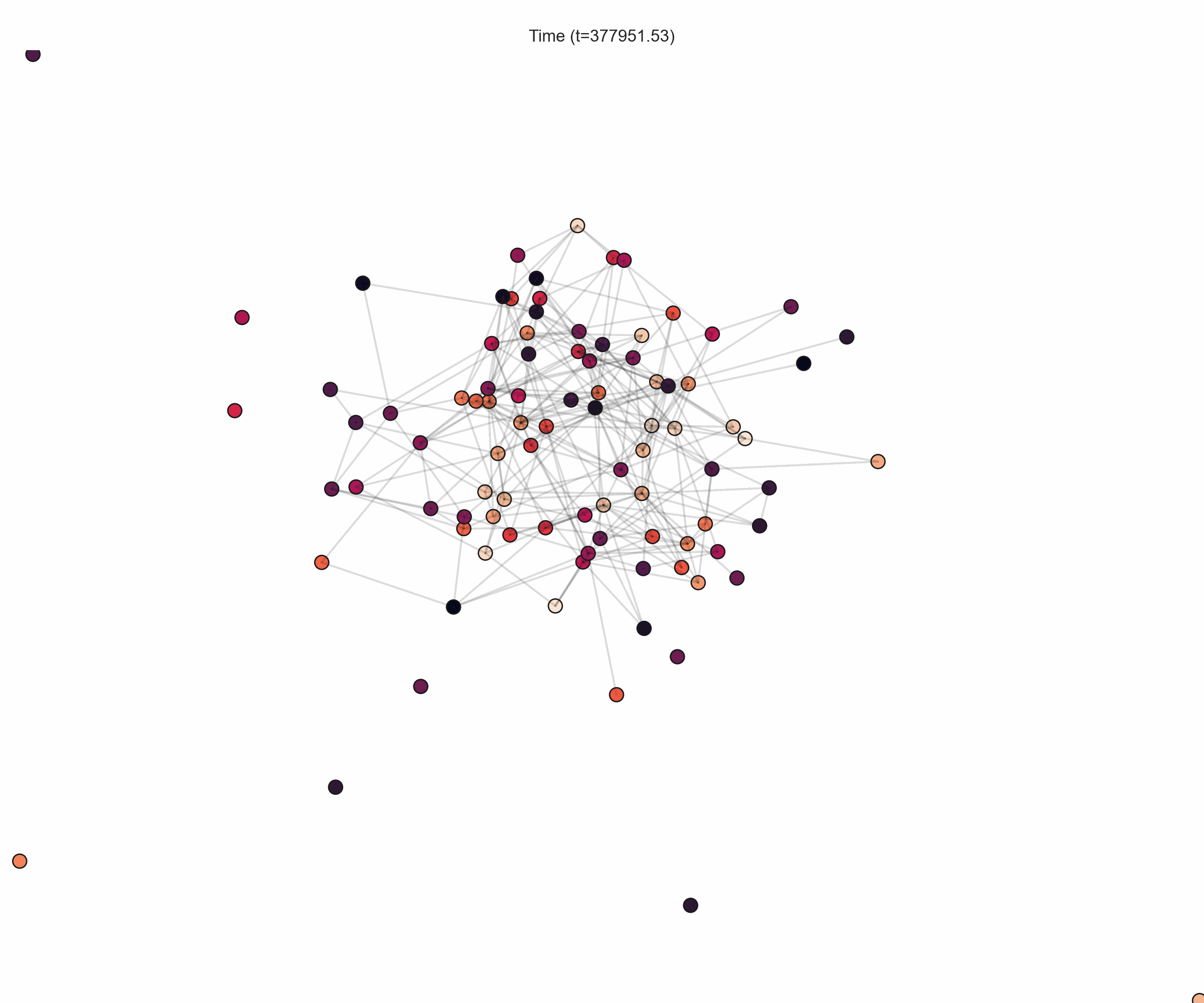}}
\hfill
\subfigure[$t=427830$]{\includegraphics[trim={5cm 6cm 5cm 6cm},clip,width=0.16\textwidth]{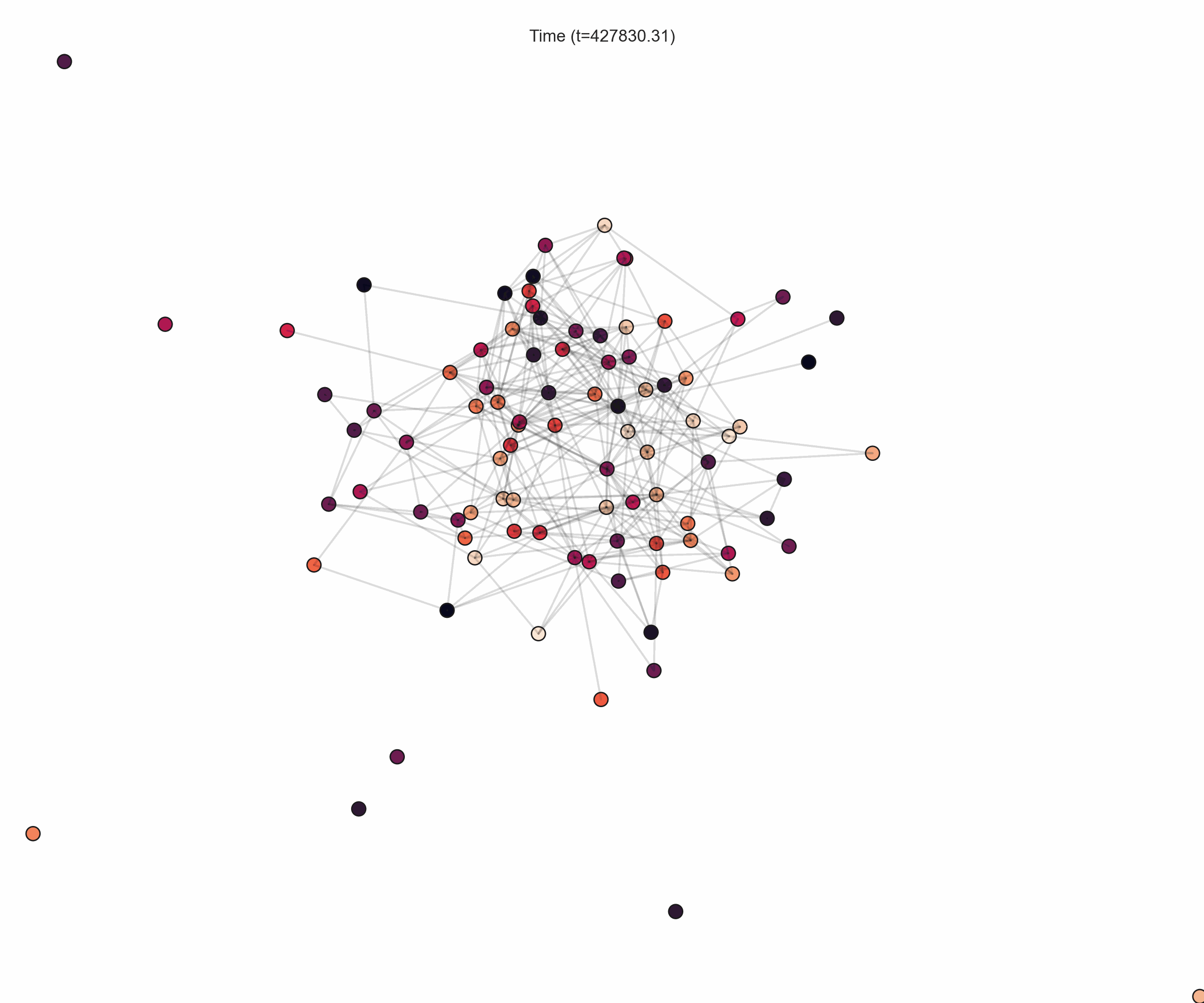}}
\hfill
\subfigure[$t=477709$]{\includegraphics[trim={5cm 6cm 5cm 6cm},clip,width=0.16\textwidth]{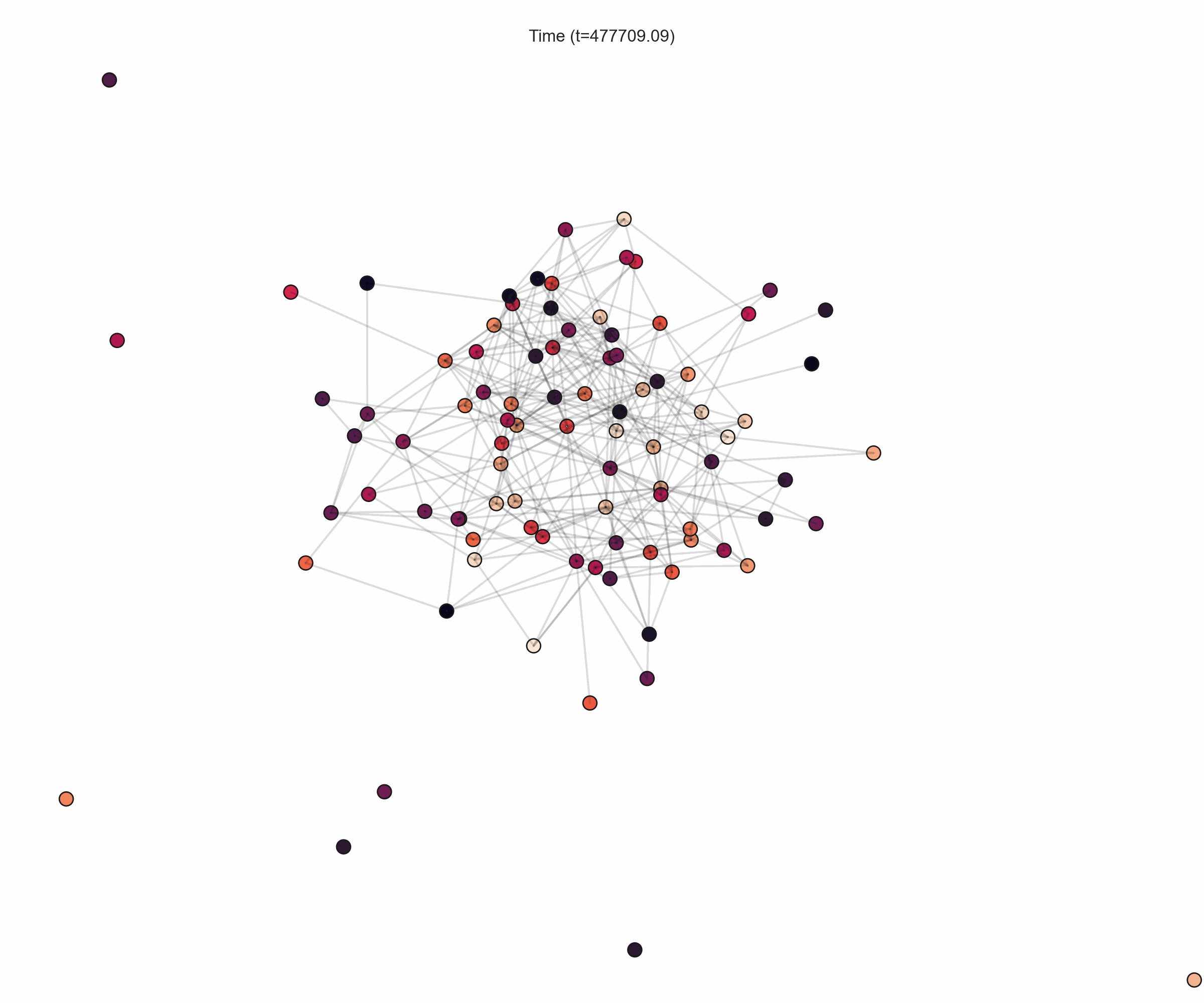}}
\hfill
\subfigure[$t=527587$]{\includegraphics[trim={5cm 6cm 5cm 6cm},clip,width=0.16\textwidth]{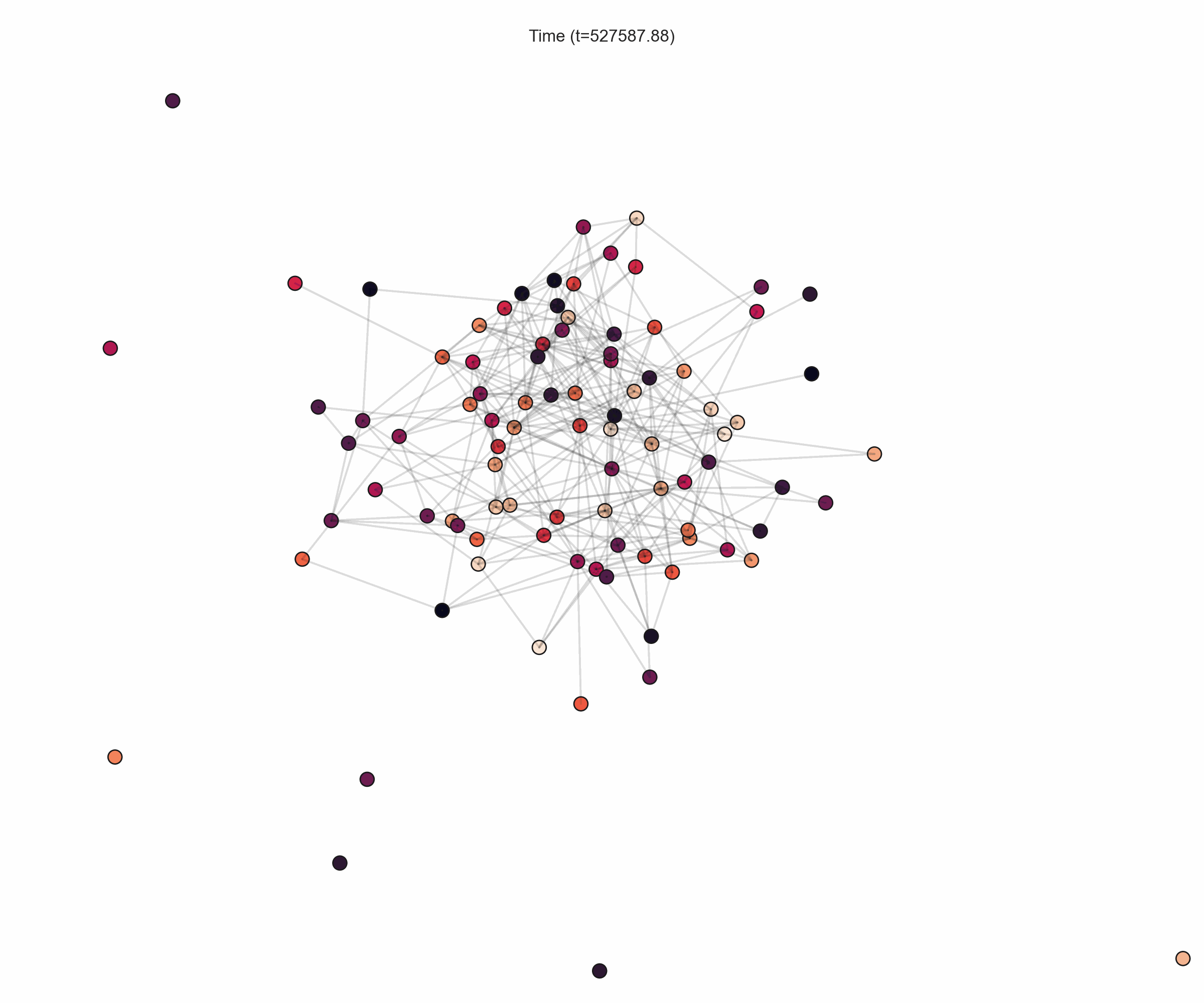}}
\hfill
\subfigure[$t=577466$]{\includegraphics[trim={5cm 6cm 5cm 6cm},clip,width=0.16\textwidth]{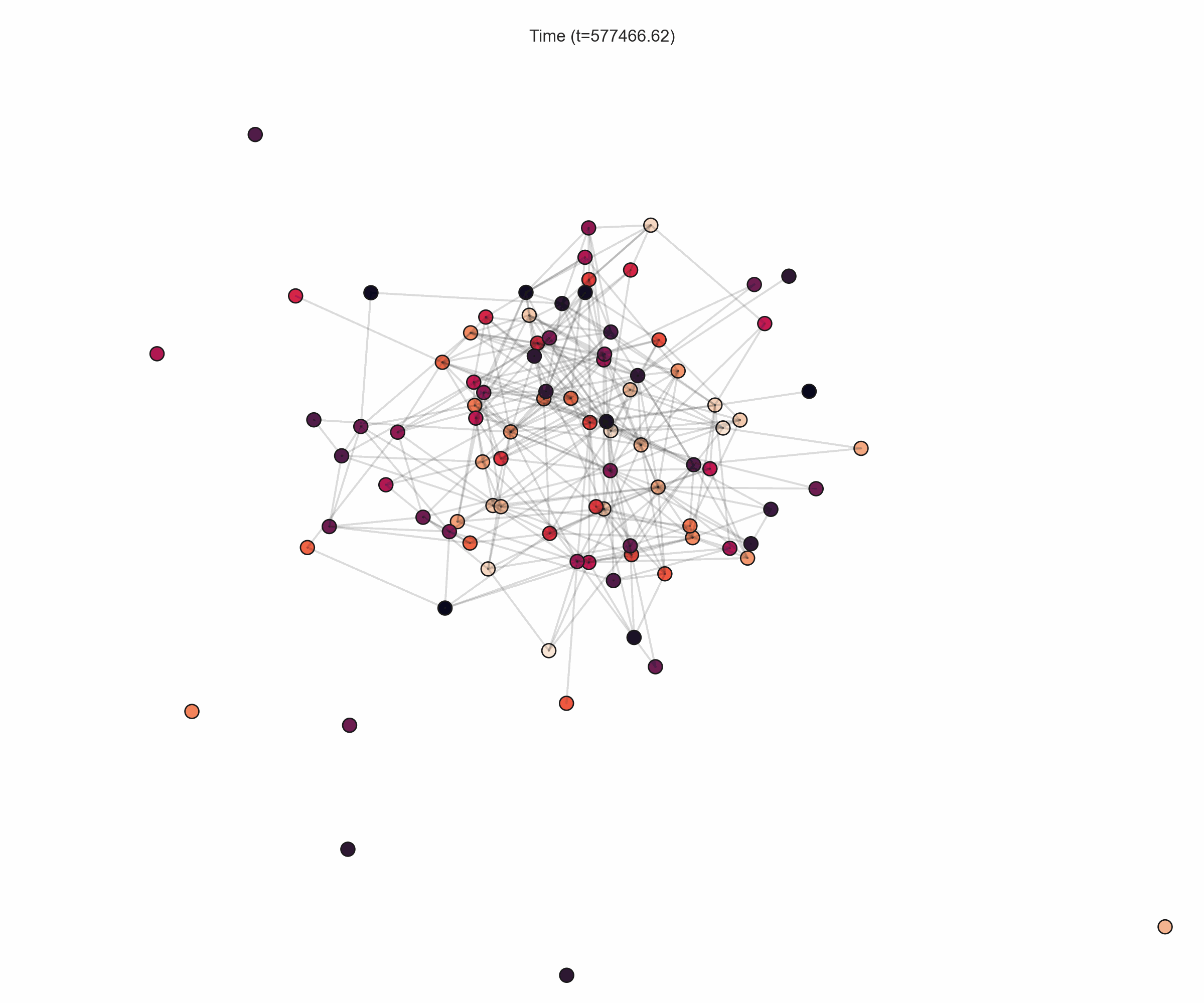}}
\hfill
\subfigure[$t=627345$]{\includegraphics[trim={5cm 6cm 5cm 6cm},clip,width=0.16\textwidth]{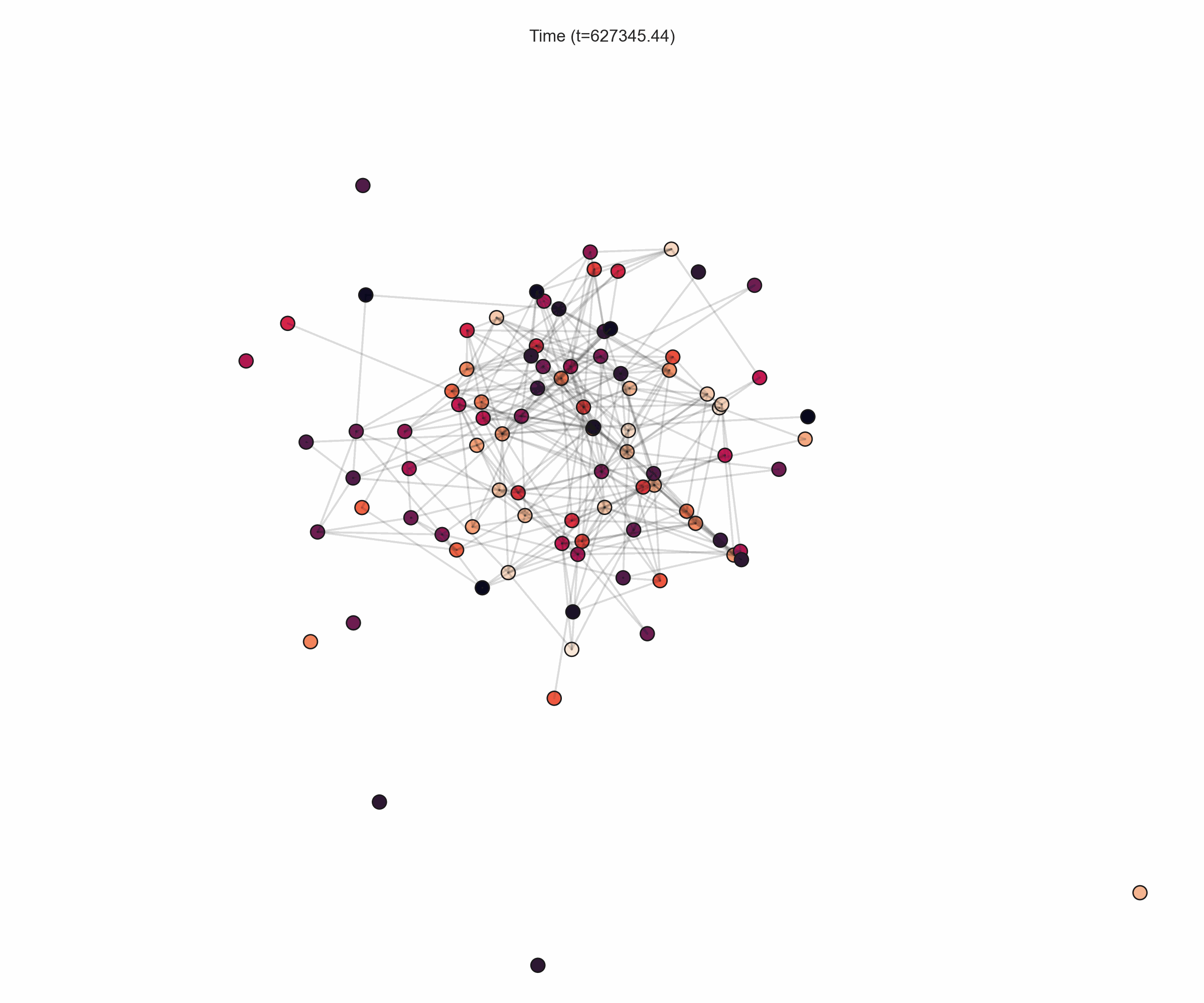}}
%%%%%%%%
\subfigure[$t=677224$]{\includegraphics[trim={5cm 6cm 5cm 6cm},clip,width=0.16\textwidth]{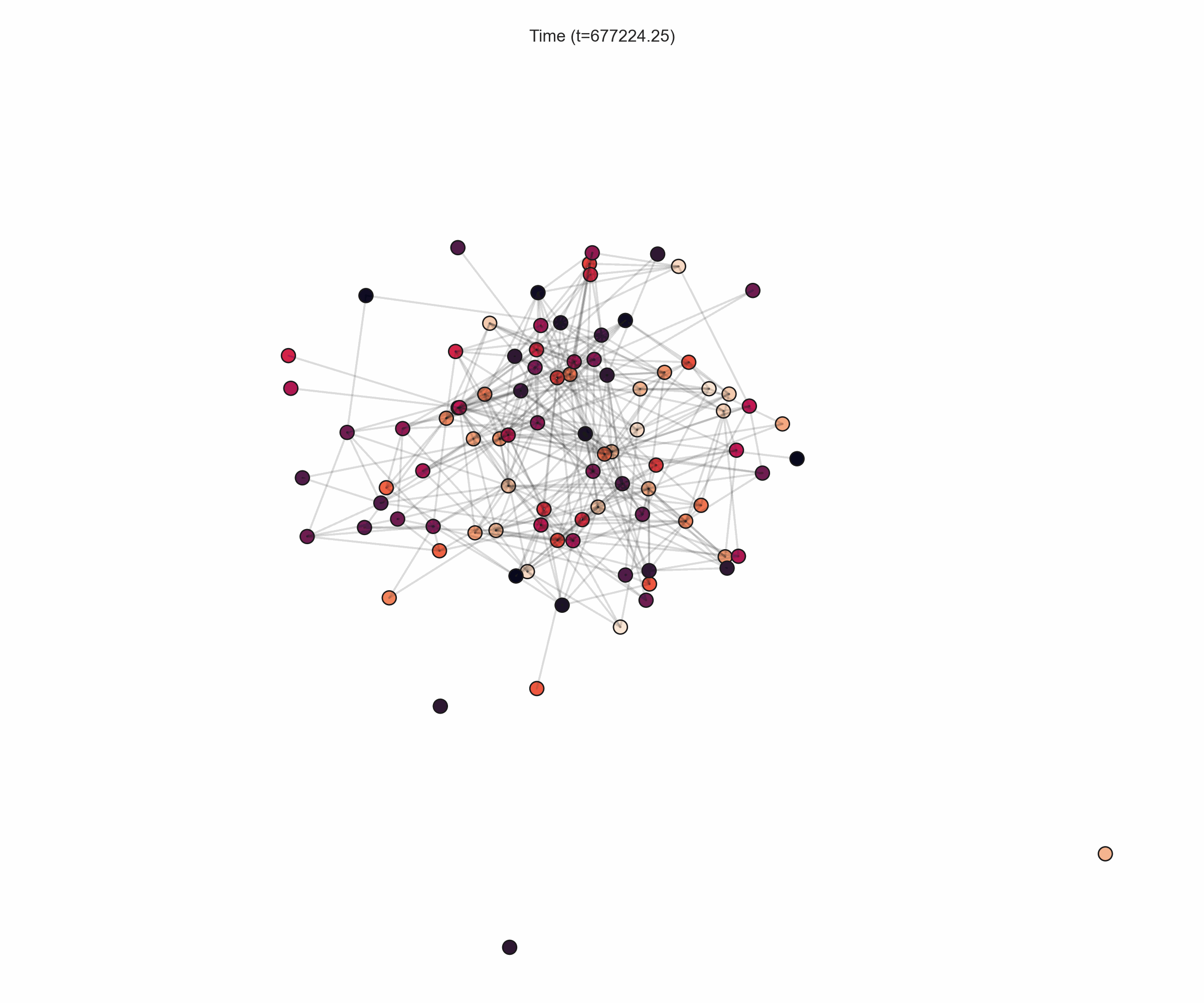}}
\hfill
\subfigure[$t=727103$]{\includegraphics[trim={5cm 6cm 5cm 6cm},clip,width=0.16\textwidth]{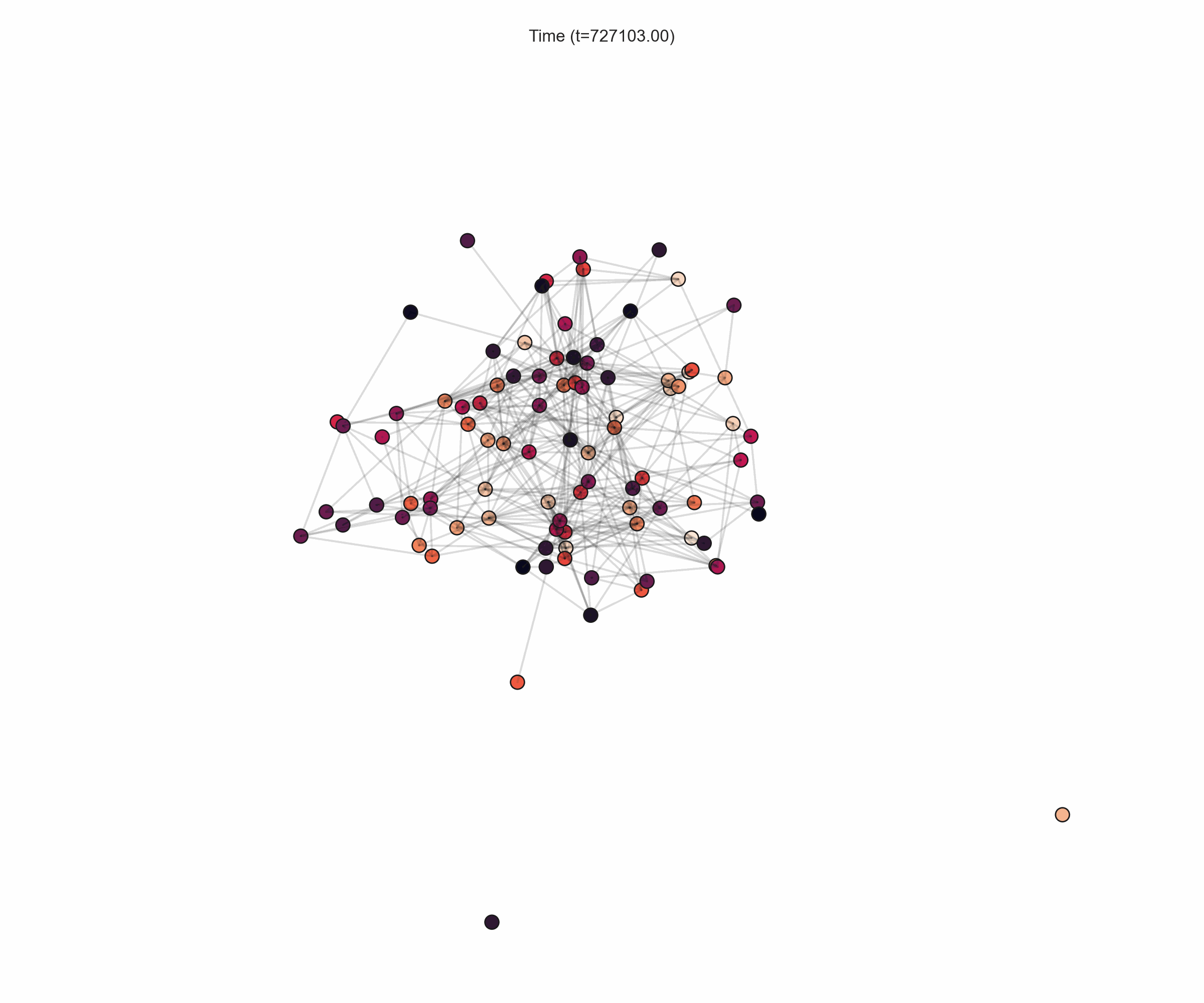}}
\hfill
\subfigure[$t=776981$]{\includegraphics[trim={5cm 6cm 5cm 6cm},clip,width=0.16\textwidth]{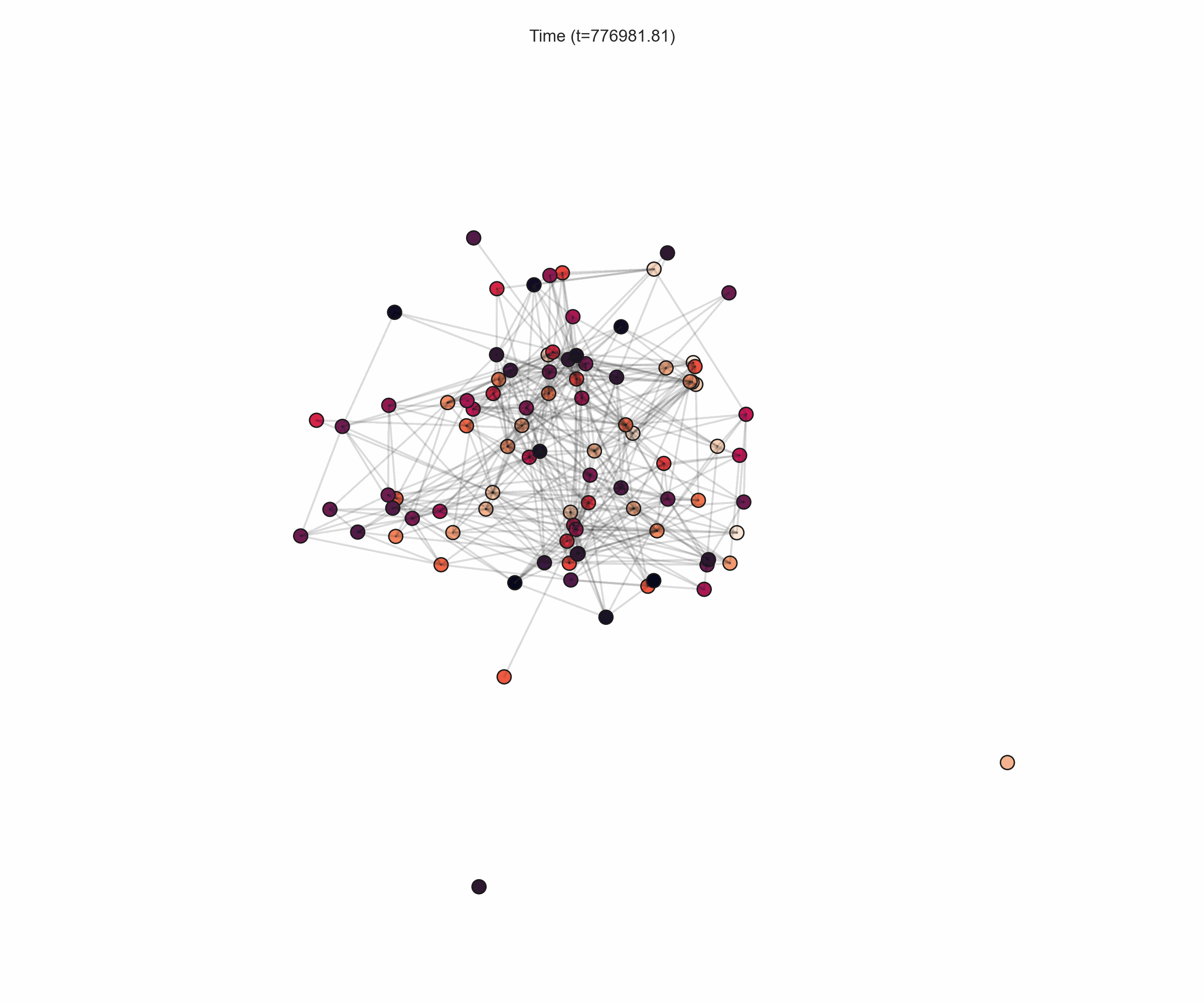}}
\hfill
\subfigure[$t=826860$]{\includegraphics[trim={5cm 6cm 5cm 6cm},clip,width=0.16\textwidth]{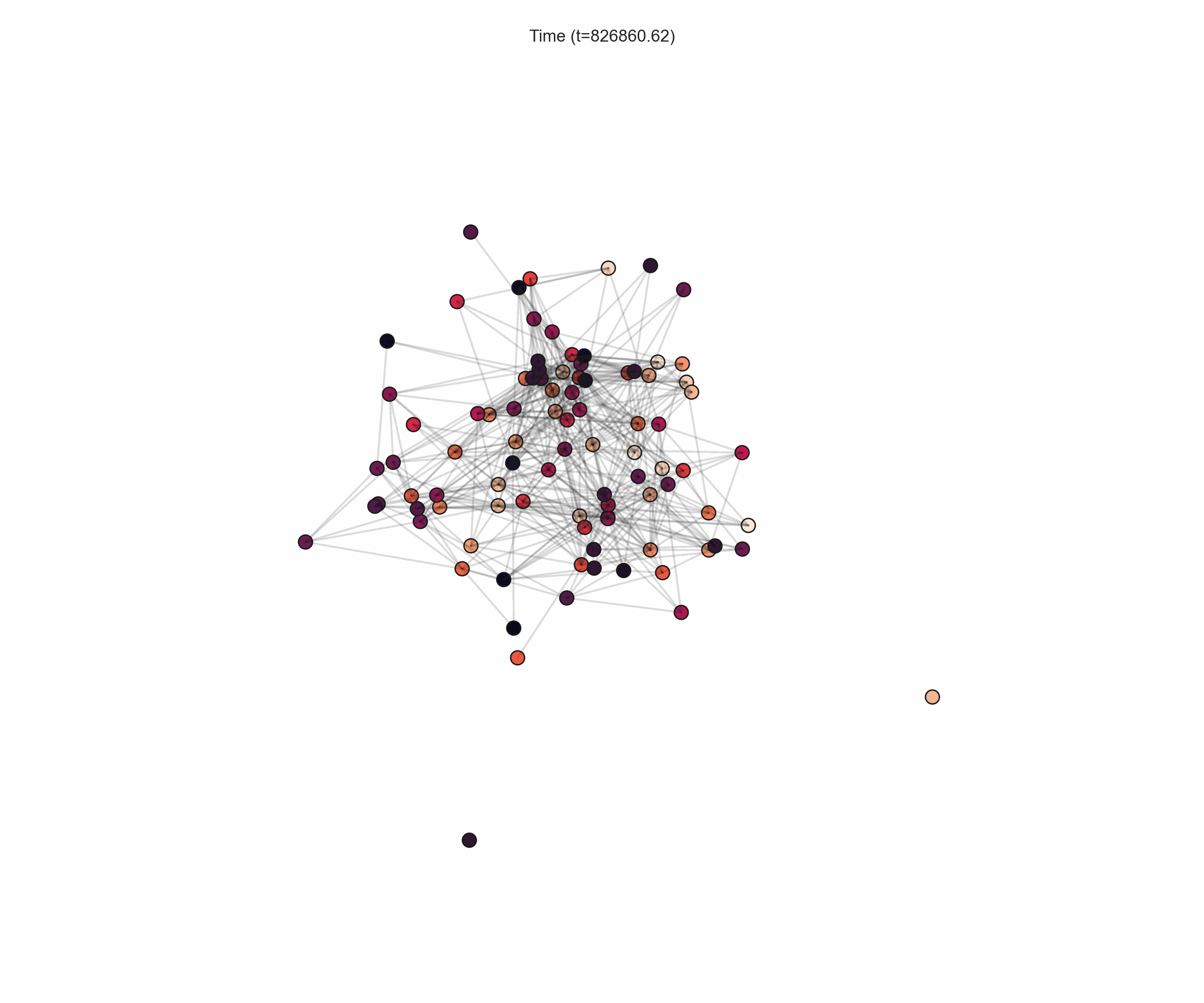}}
\hfill
\subfigure[$t=876739$]{\includegraphics[trim={5cm 6cm 5cm 6cm},clip,width=0.16\textwidth]{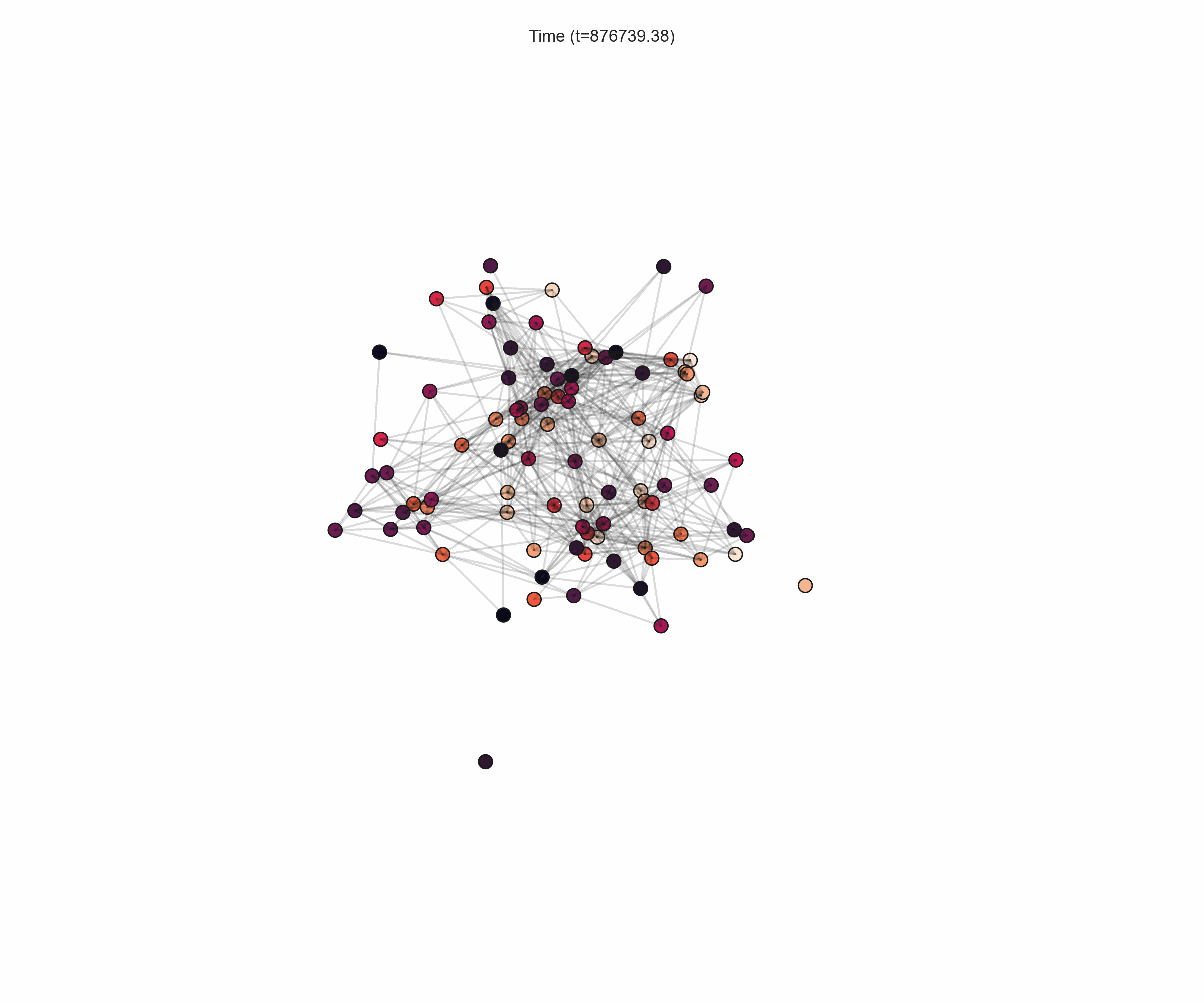}}
\hfill
\subfigure[$t=926618$]{\includegraphics[trim={5cm 6cm 5cm 6cm},clip,width=0.16\textwidth]{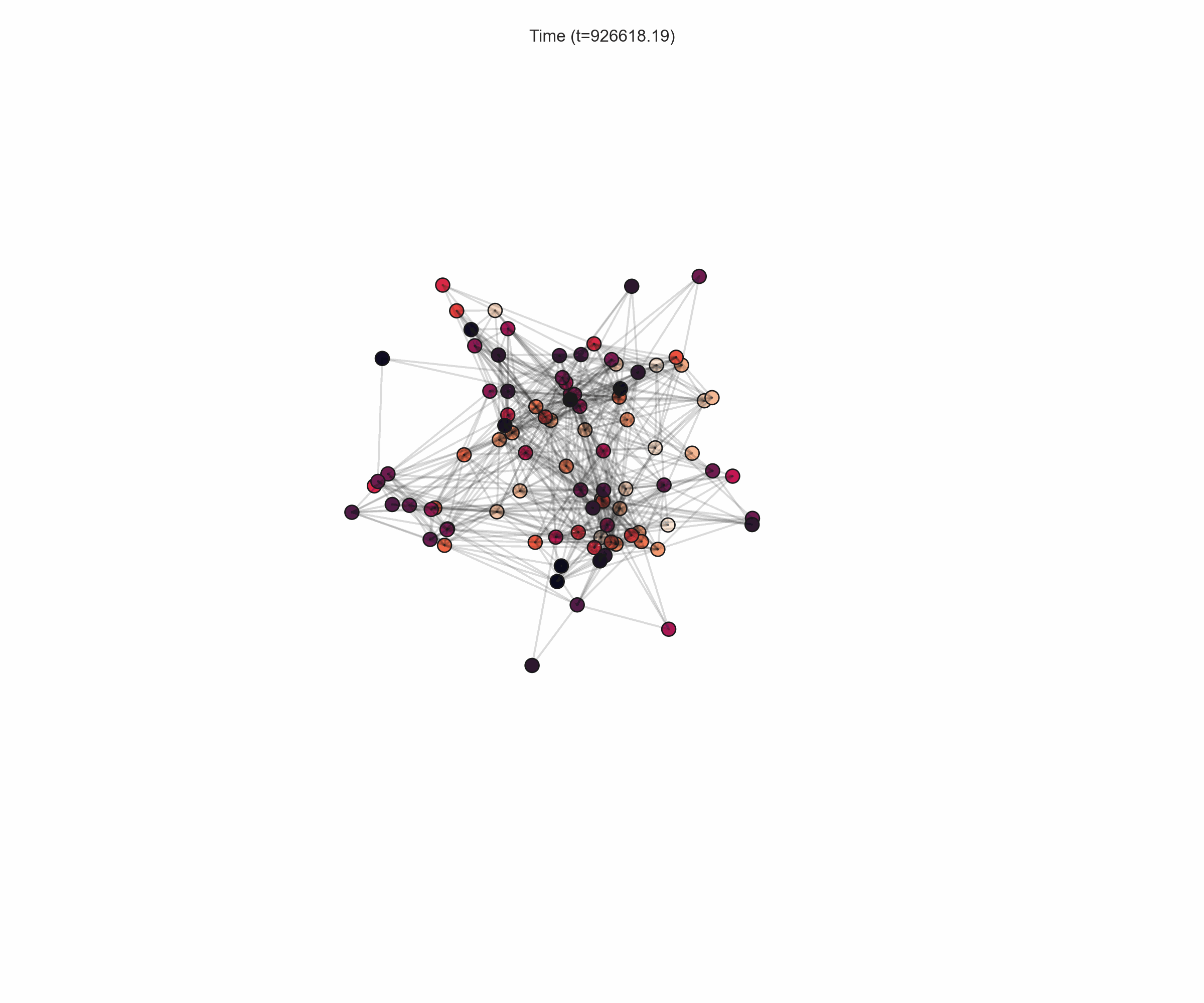}}
\caption{Snapshots of the continuous-time embeddings learned by \textsc{\modelname} for various time points over \textsl{Contacts}.}\label{fig:appendix_visualization_contacts}
\end{figure*}
%%%%%%%%%%%%%
\begin{figure*}[!ht]
\centering
\subfigure[$t=10719$]{\includegraphics[trim={5cm 6cm 5cm 6cm},clip,width=0.16\textwidth]{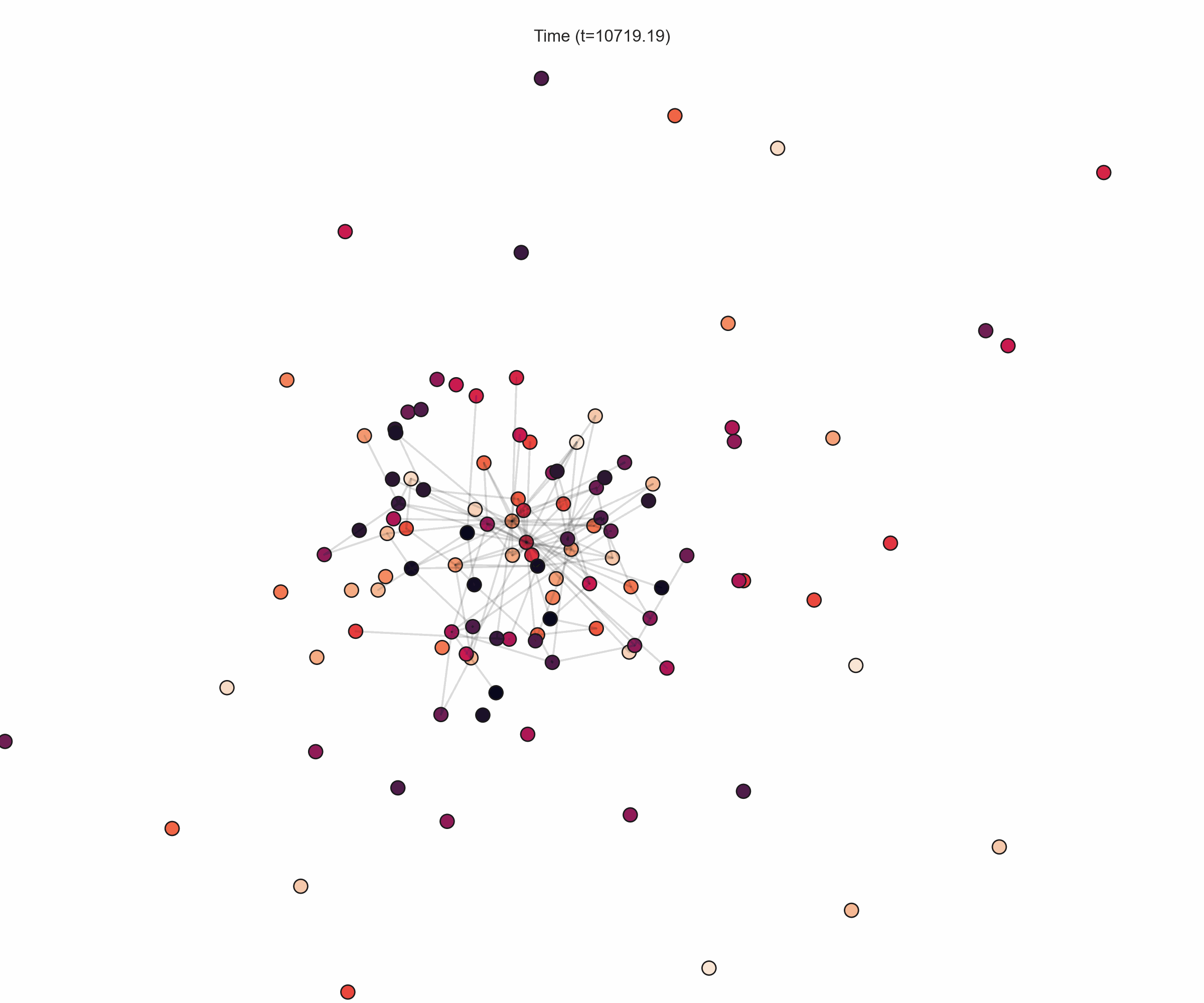}}
\hfill
\subfigure[$t=21438$]{\includegraphics[trim={5cm 6cm 5cm 6cm},clip,width=0.16\textwidth]{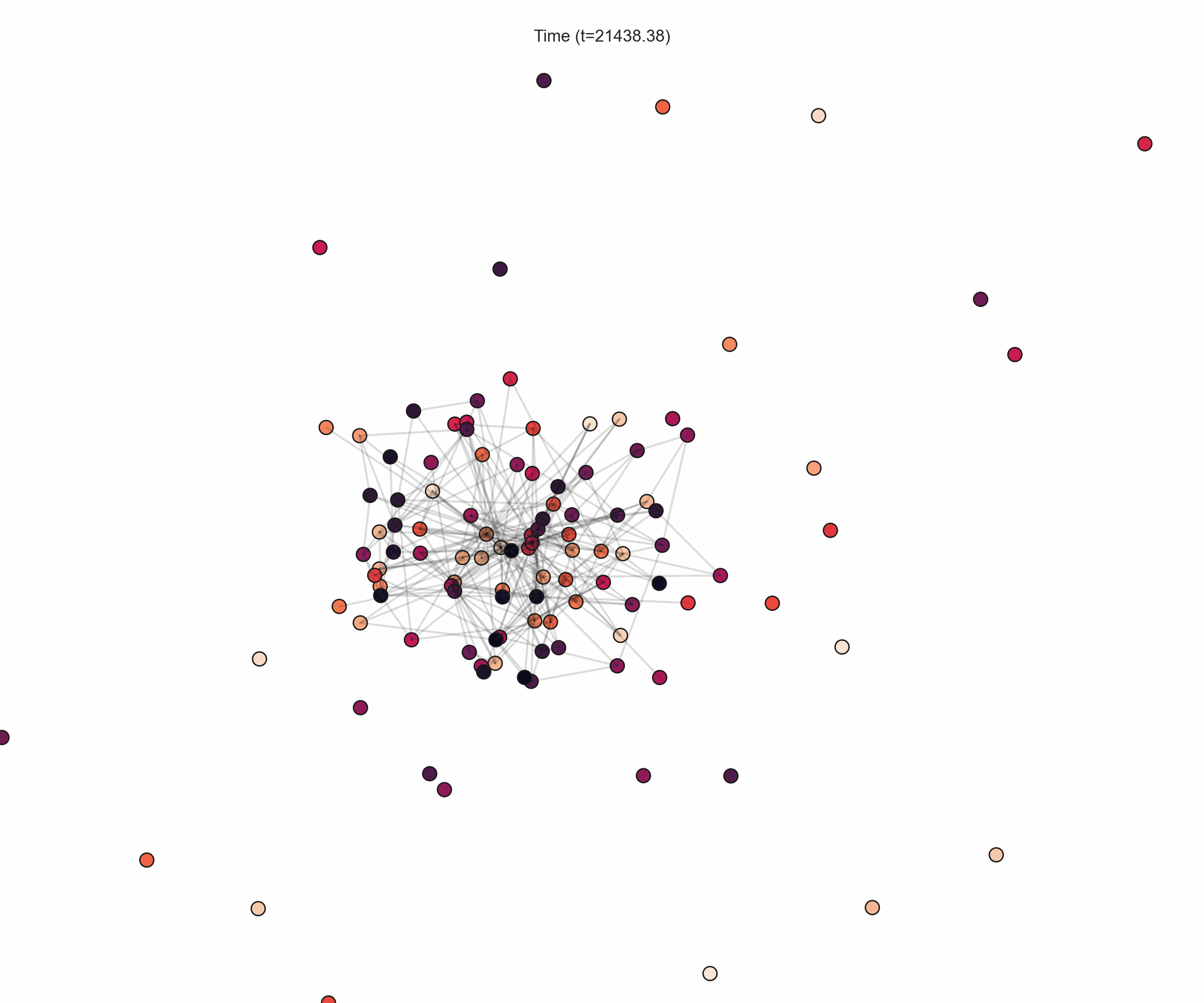}}
\hfill
\subfigure[$t=32157$]{\includegraphics[trim={5cm 6cm 5cm 6cm},clip,width=0.16\textwidth]{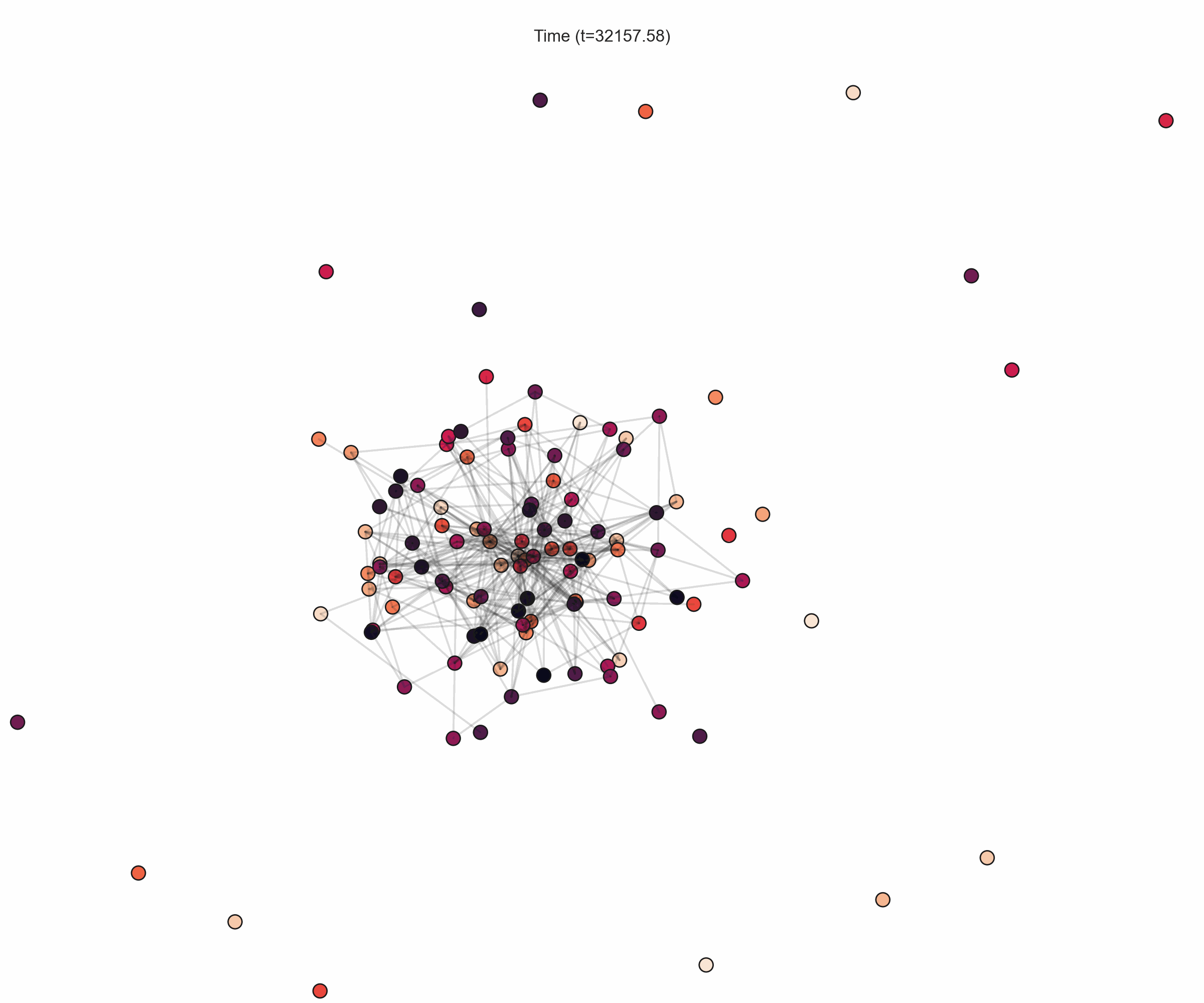}}
\hfill
\subfigure[$t=42876$]{\includegraphics[trim={5cm 6cm 5cm 6cm},clip,width=0.16\textwidth]{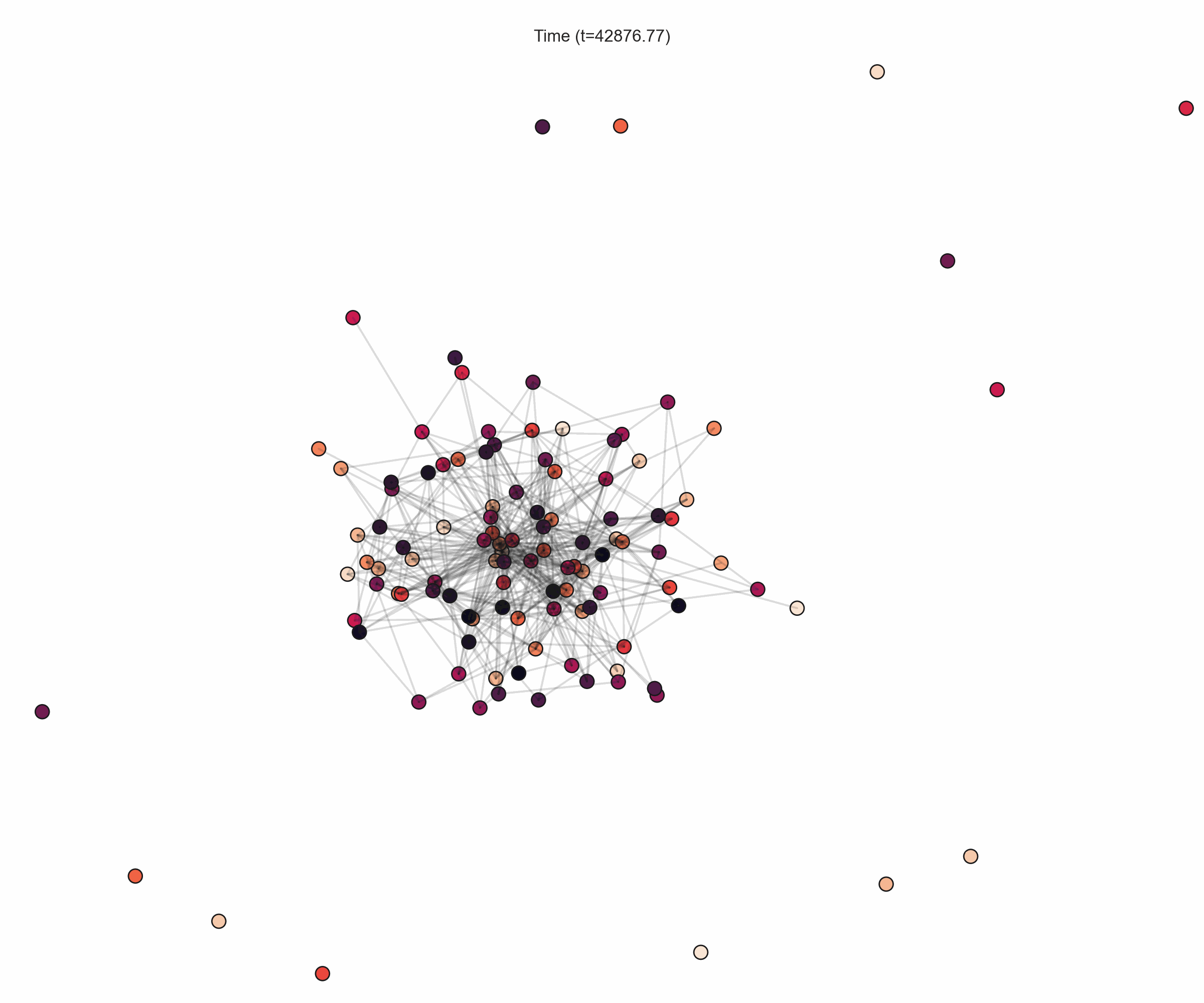}}
\hfill
\subfigure[$t=53595$]{\includegraphics[trim={5cm 6cm 5cm 6cm},clip,width=0.16\textwidth]{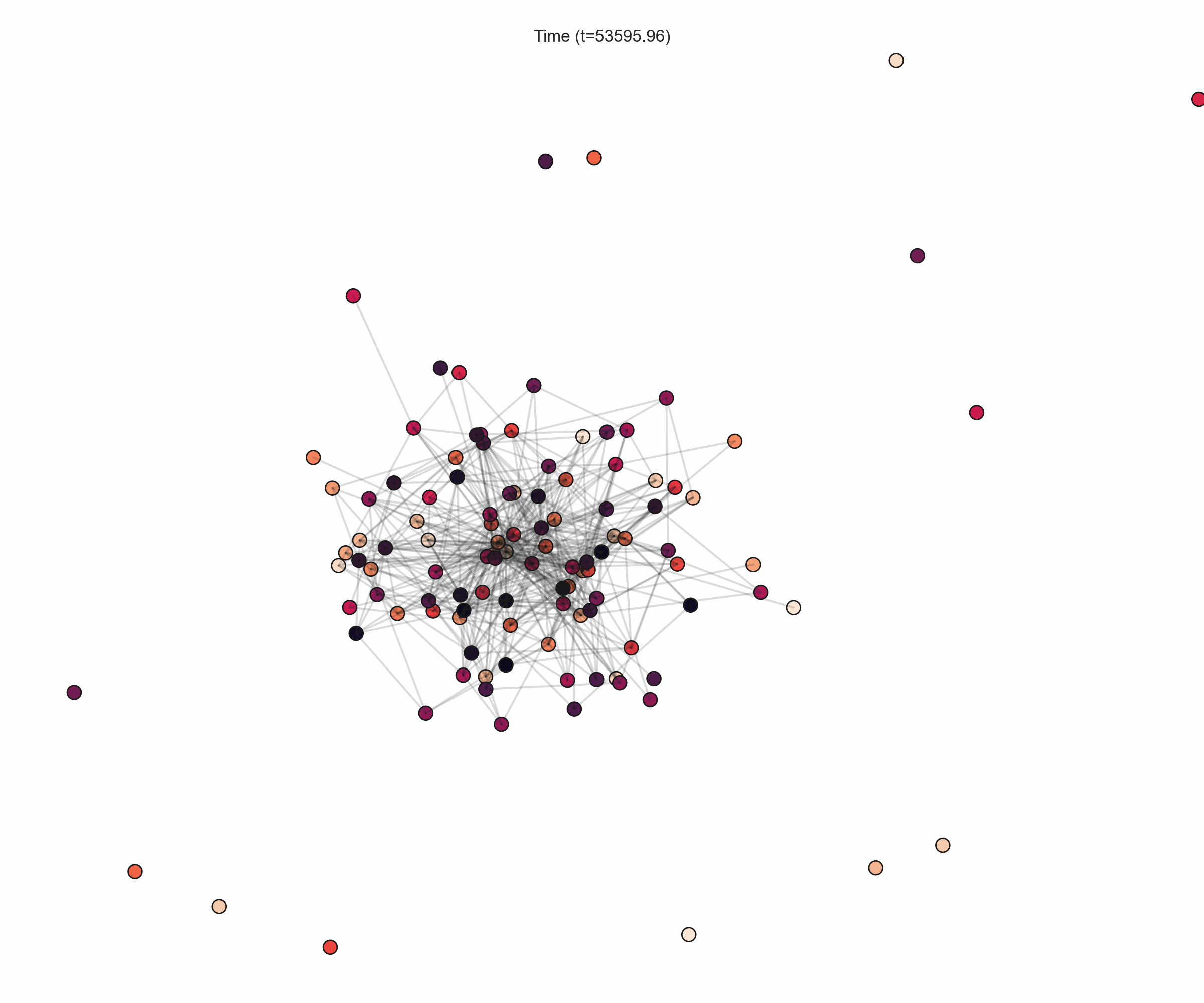}}
\hfill
\subfigure[$t=64315$]{\includegraphics[trim={5cm 6cm 5cm 6cm},clip,width=0.16\textwidth]{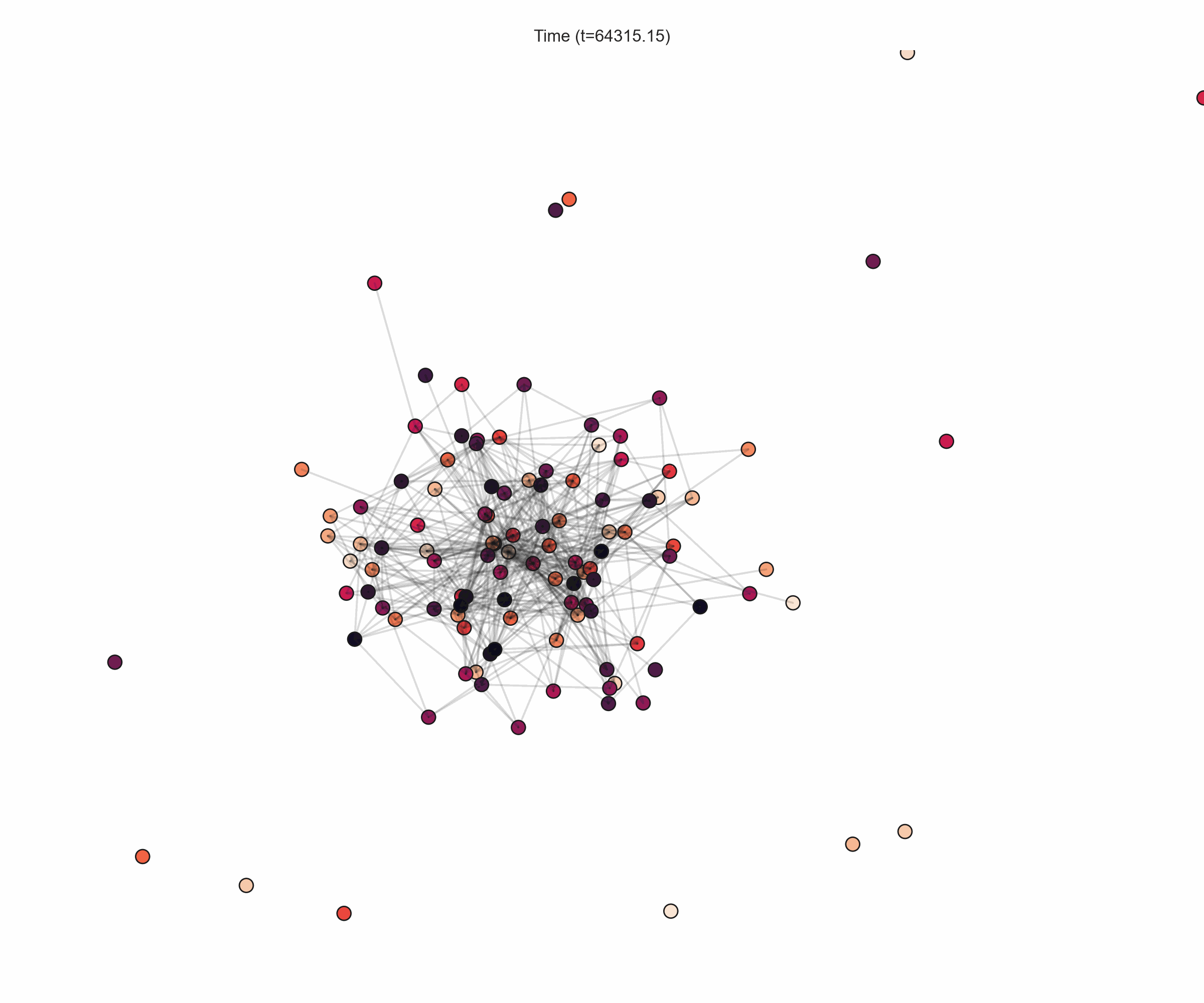}}
%%%%%%%%
\subfigure[$t=75034$]{\includegraphics[trim={5cm 6cm 5cm 6cm},clip,width=0.16\textwidth]{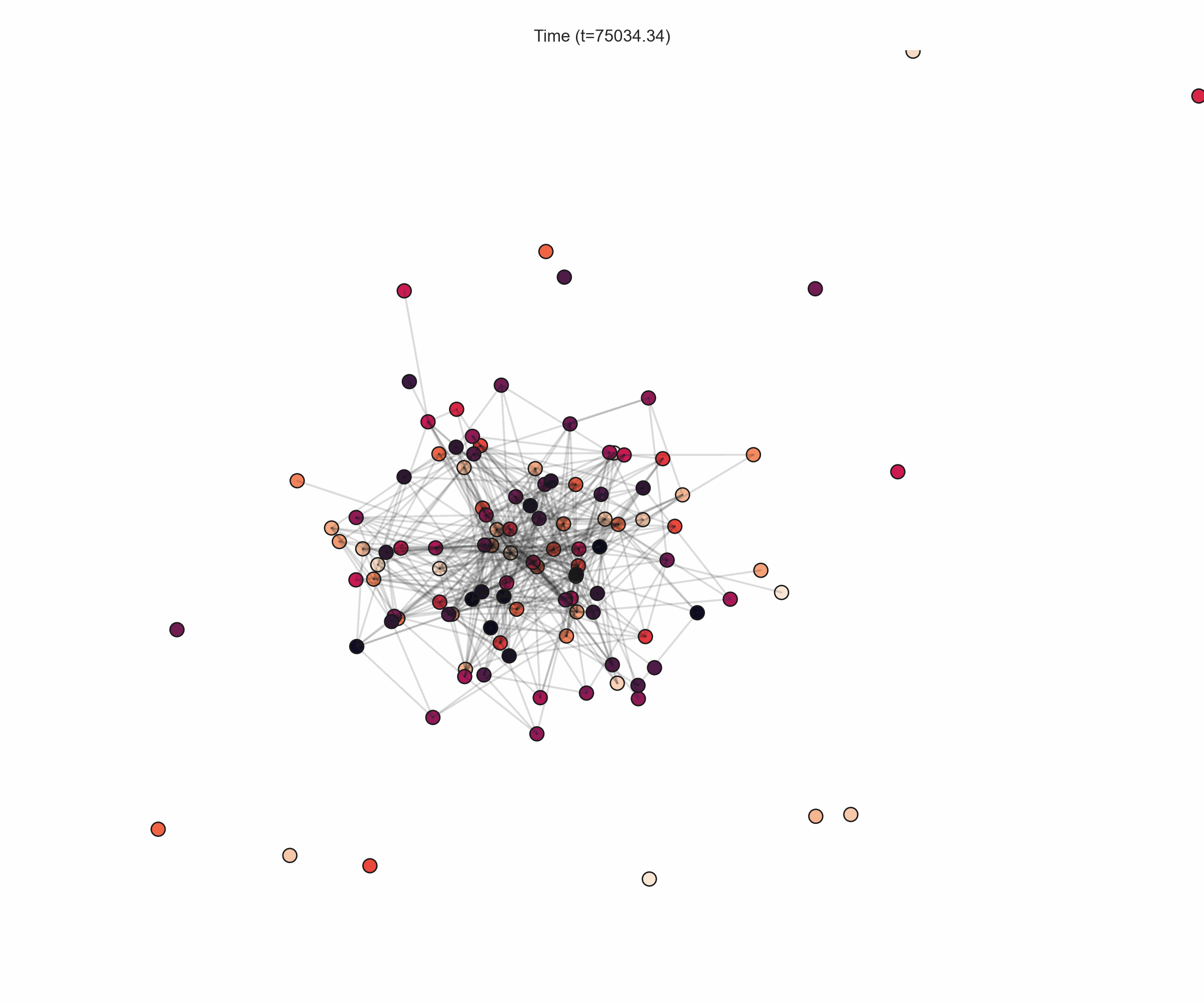}}
\hfill
\subfigure[$t=85753$]{\includegraphics[trim={5cm 6cm 5cm 6cm},clip,width=0.16\textwidth]{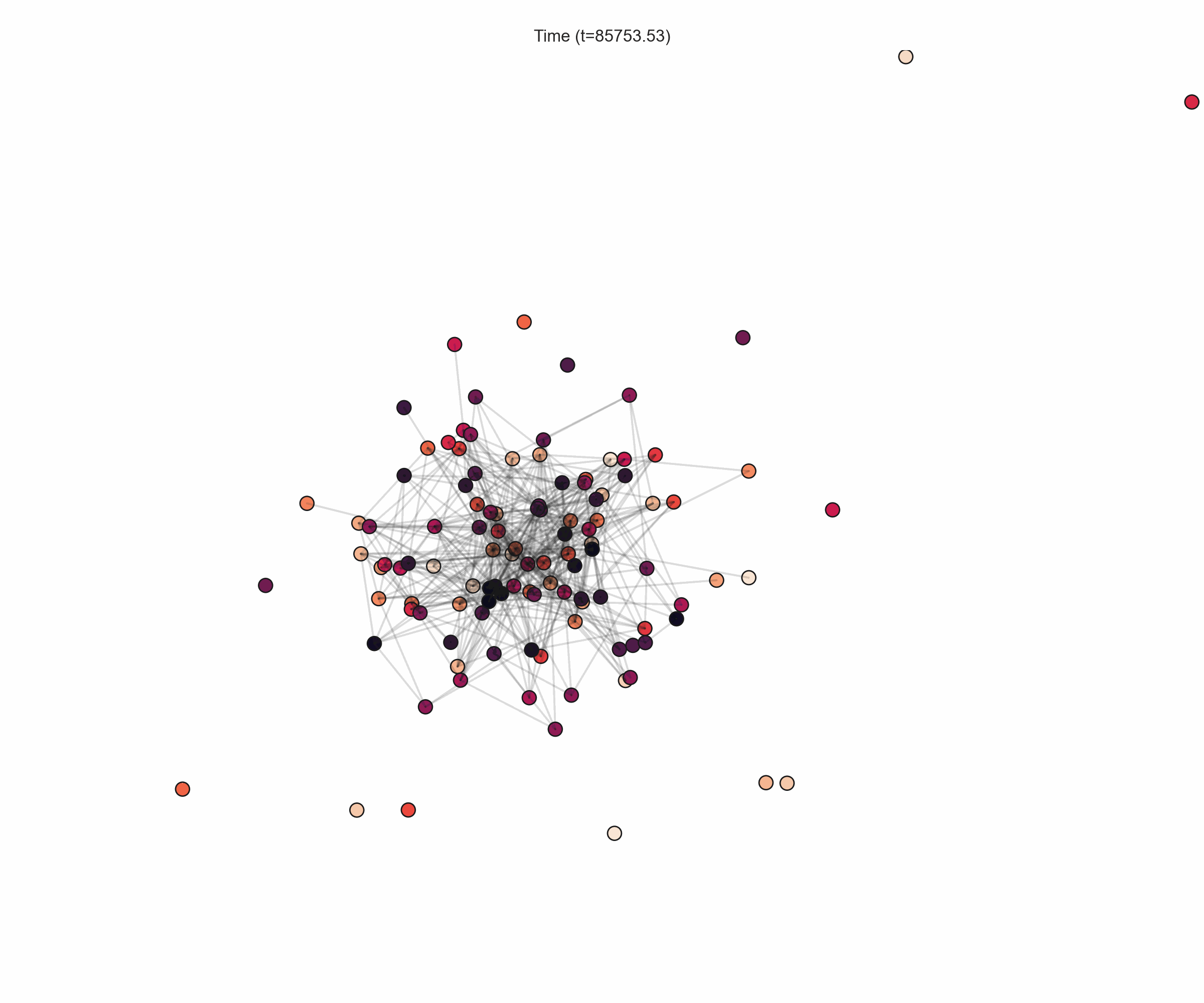}}
\hfill
\subfigure[$t=96472$]{\includegraphics[trim={5cm 6cm 5cm 6cm},clip,width=0.16\textwidth]{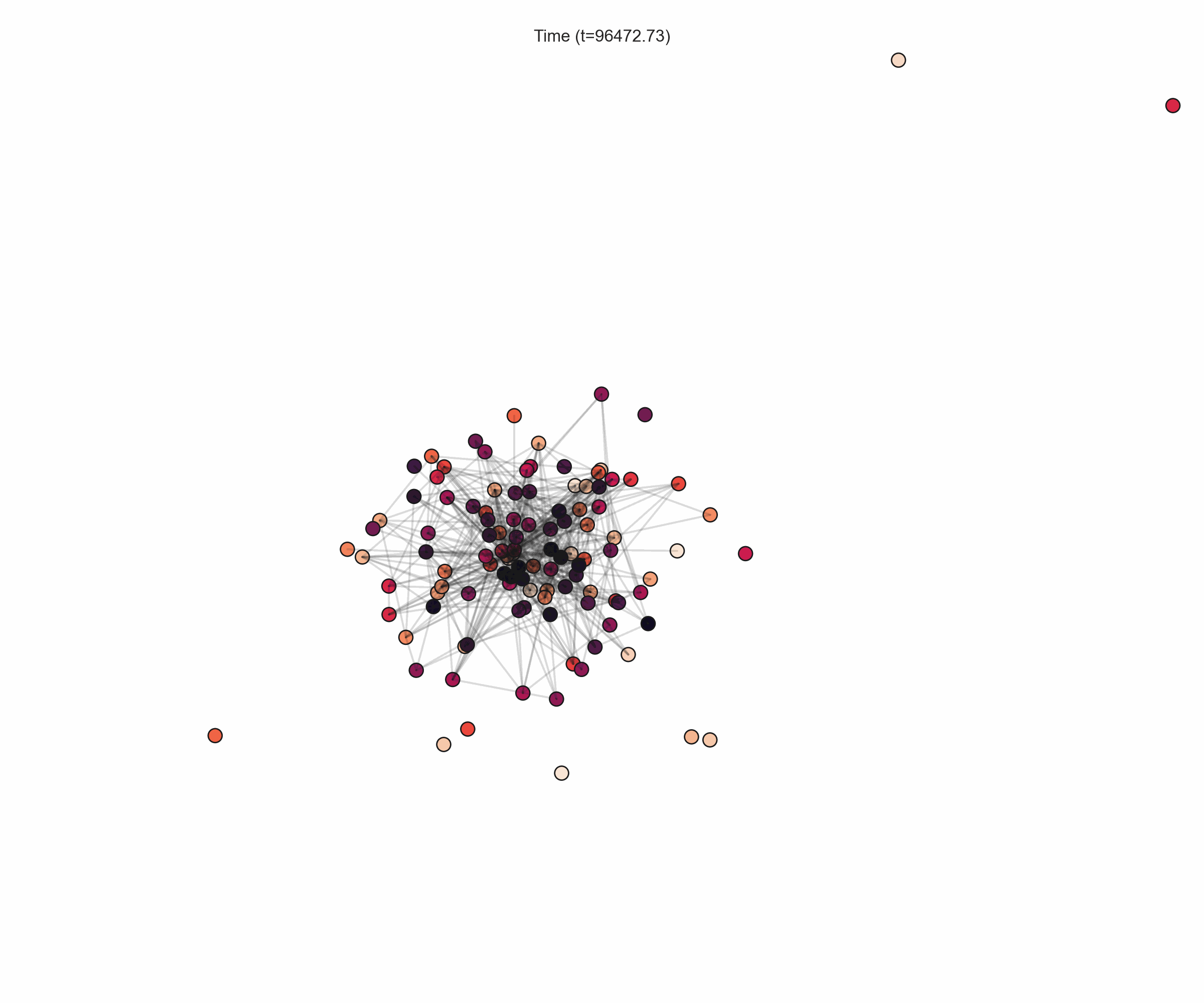}}
\hfill
\subfigure[$t=107191$]{\includegraphics[trim={5cm 6cm 5cm 6cm},clip,width=0.16\textwidth]{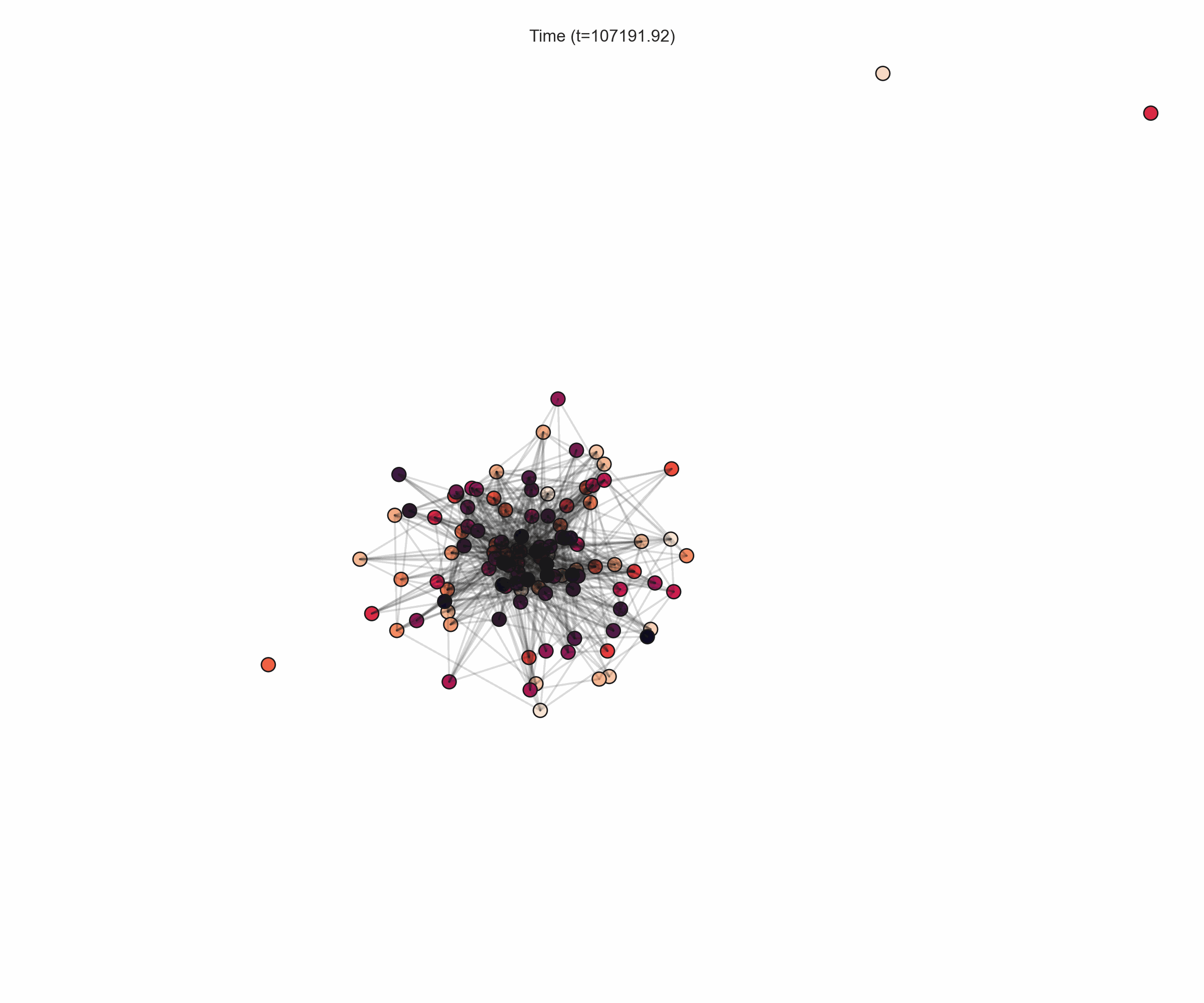}}
\hfill
\subfigure[$t=117911$]{\includegraphics[trim={5cm 6cm 5cm 6cm},clip,width=0.16\textwidth]{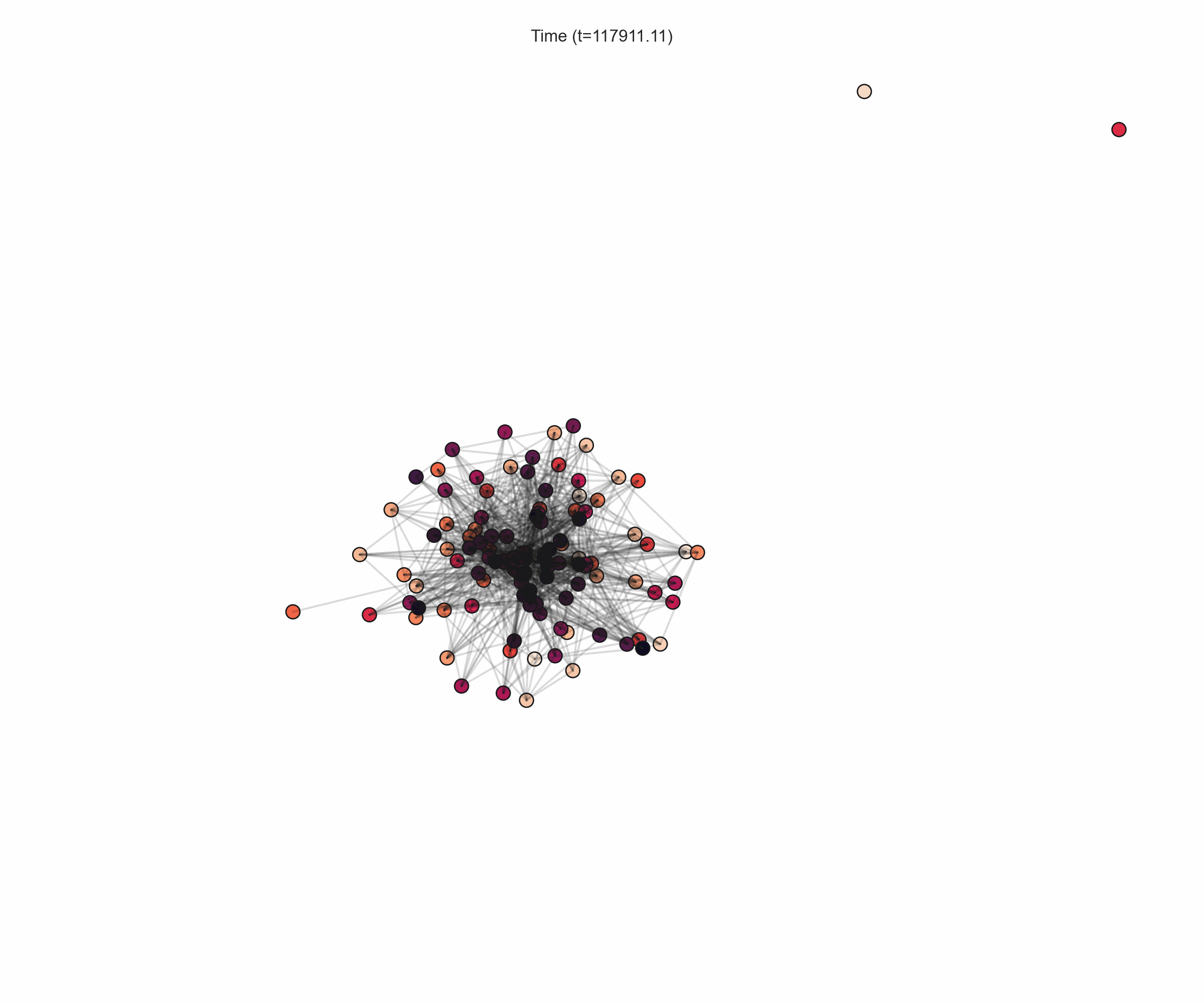}}
\hfill
\subfigure[$t=128630$]{\includegraphics[trim={5cm 6cm 5cm 6cm},clip,width=0.16\textwidth]{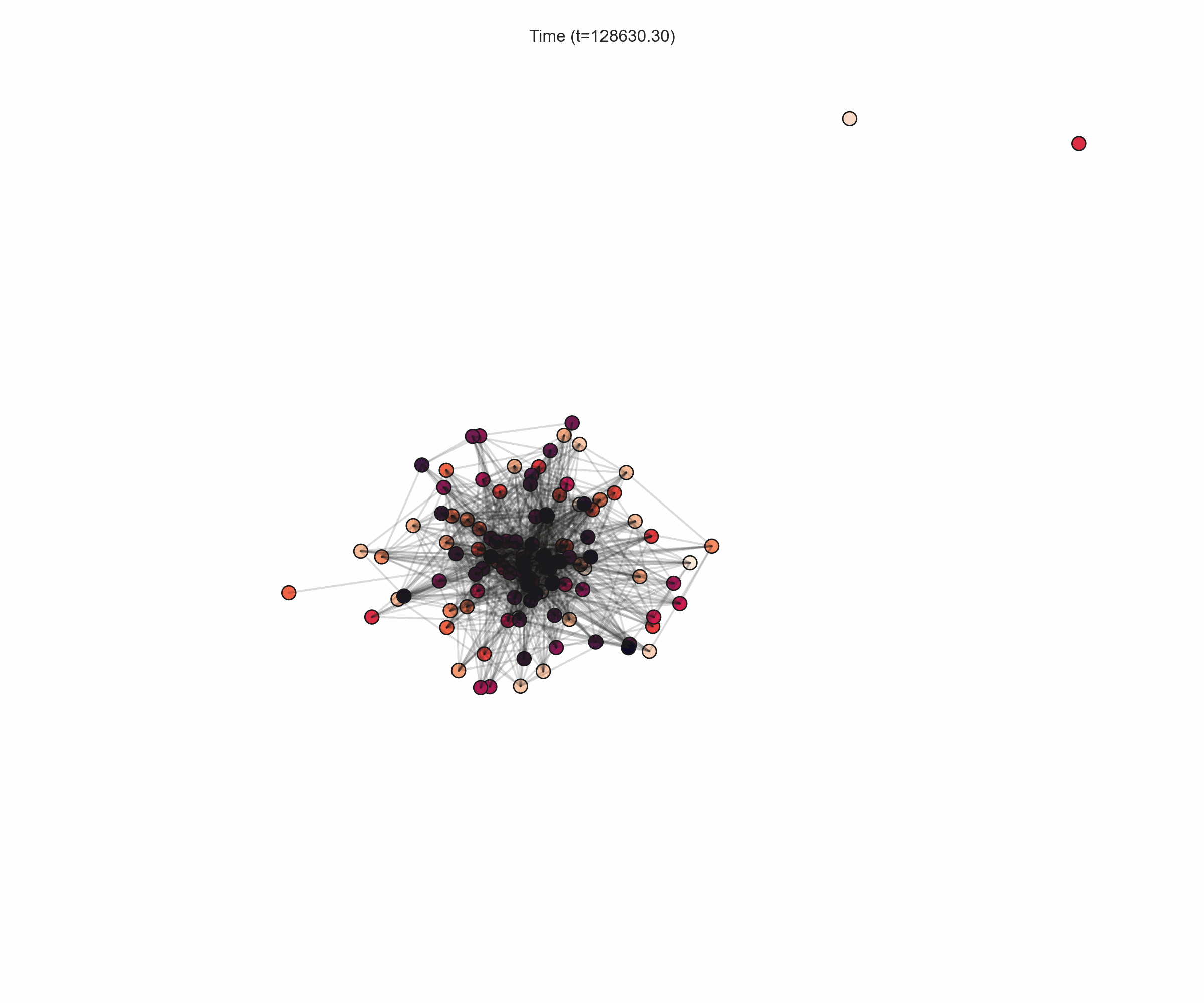}}
%%%%%%%%
\subfigure[$t=139349$]{\includegraphics[trim={5cm 6cm 5cm 6cm},clip,width=0.16\textwidth]{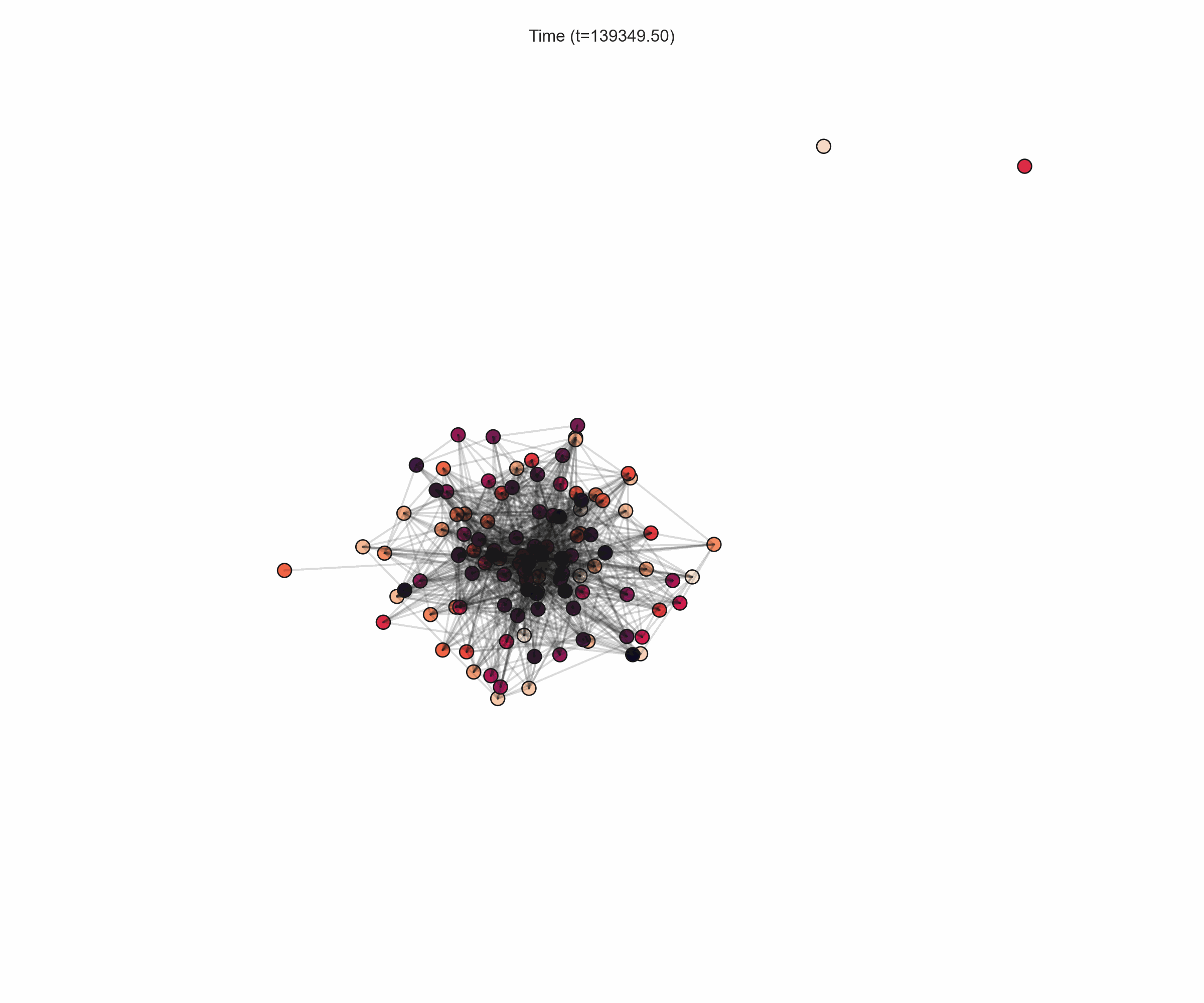}}
\hfill
\subfigure[$t=150068$]{\includegraphics[trim={5cm 6cm 5cm 6cm},clip,width=0.16\textwidth]{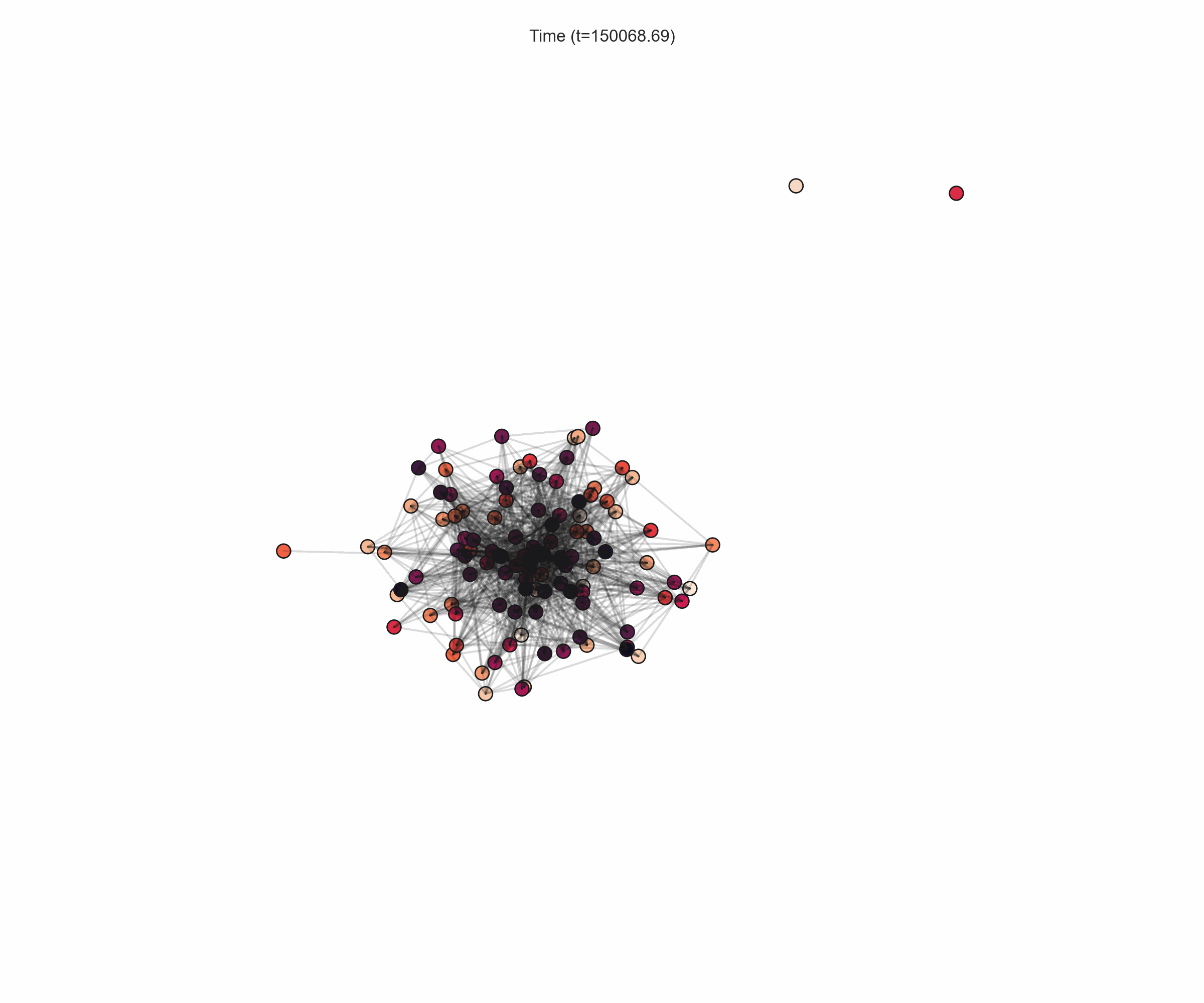}}
\hfill
\subfigure[$t=160787$]{\includegraphics[trim={5cm 6cm 5cm 6cm},clip,width=0.16\textwidth]{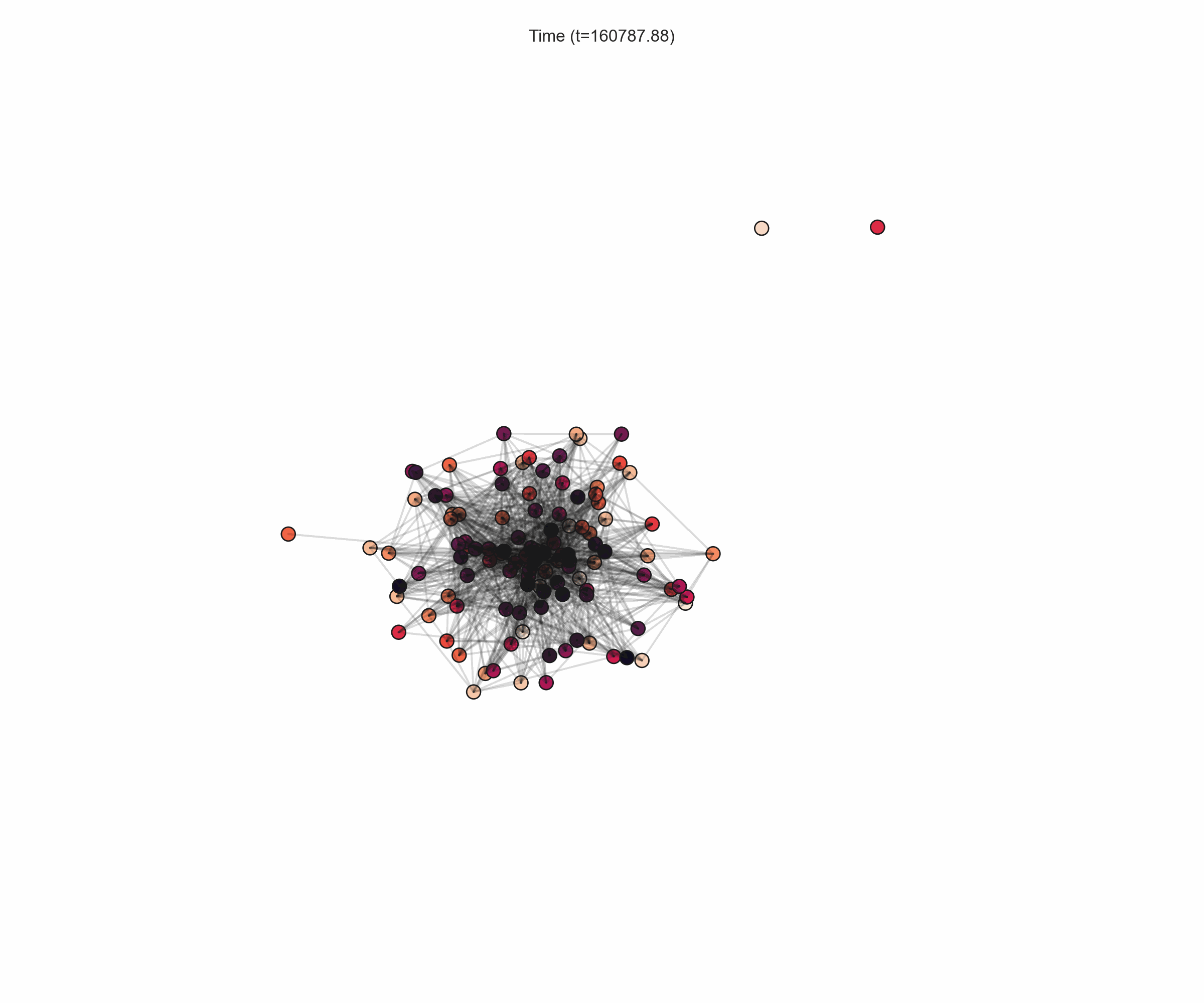}}
\hfill
\subfigure[$t=171507$]{\includegraphics[trim={5cm 6cm 5cm 6cm},clip,width=0.16\textwidth]{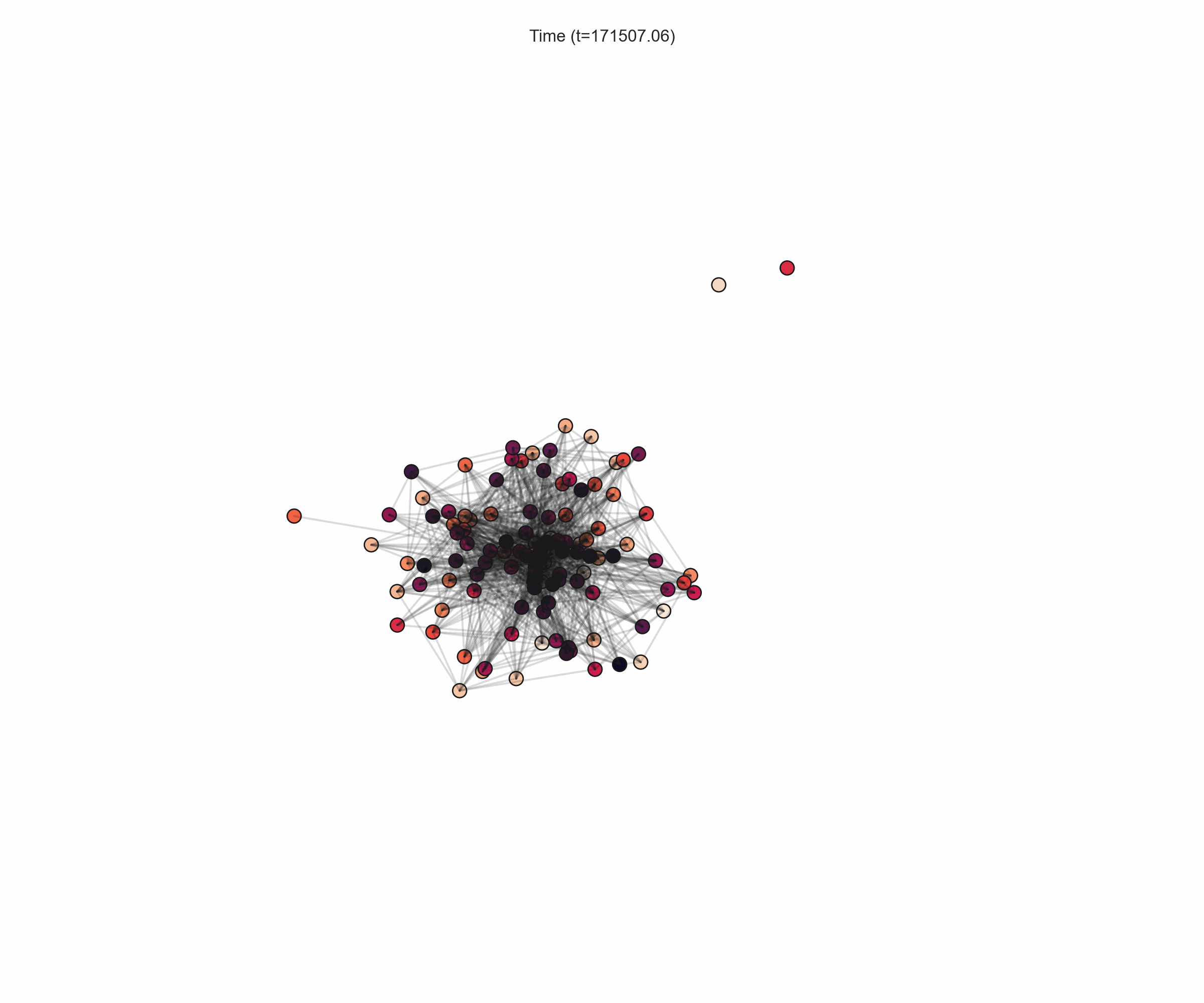}}
\hfill
\subfigure[$t=182226$]{\includegraphics[trim={5cm 6cm 5cm 6cm},clip,width=0.16\textwidth]{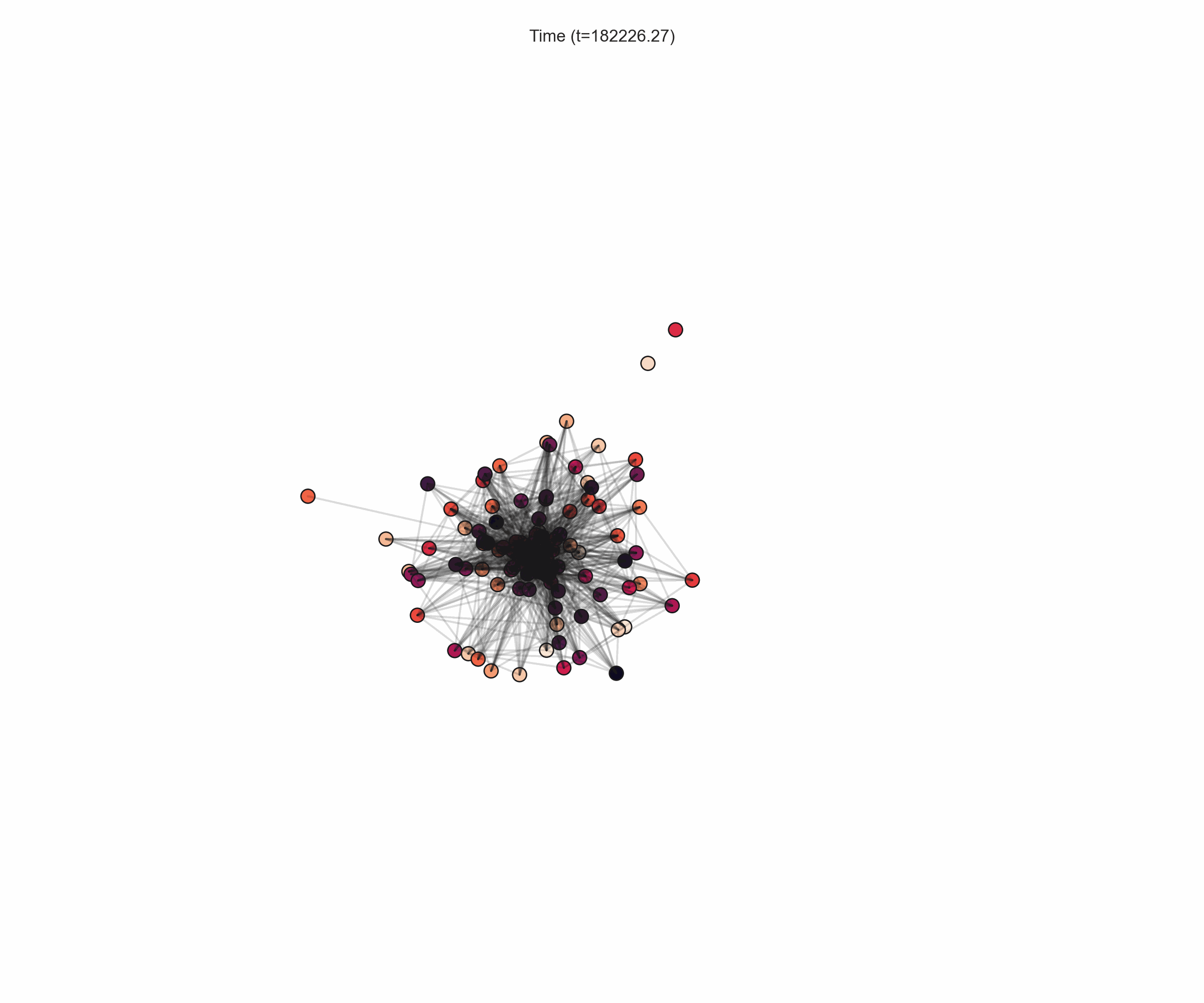}}
\hfill
\subfigure[$t=192945$]{\includegraphics[trim={5cm 6cm 5cm 6cm},clip,width=0.16\textwidth]{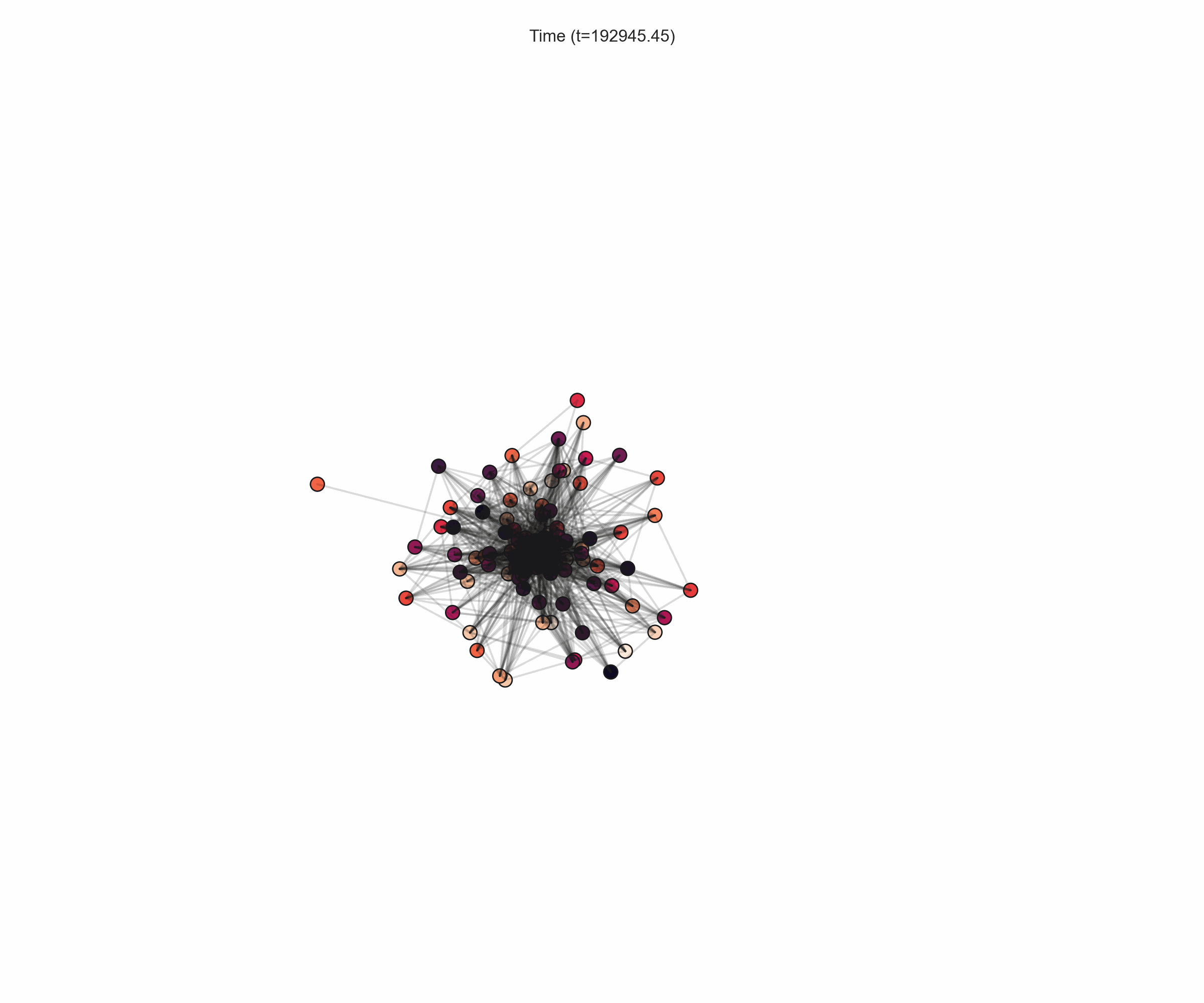}}
\caption{Snapshots of the continuous-time embeddings learned by \textsc{\modelname} for various time points over \textsl{HyperText}.}\label{fig:appendix_visualization_hypertext}
\end{figure*}
%%%%%%%%%%%%%
\begin{figure*}[!ht]
\centering
\subfigure[$t=3712333$]{\includegraphics[trim={5cm 6cm 5cm 6cm},clip,width=0.16\textwidth]{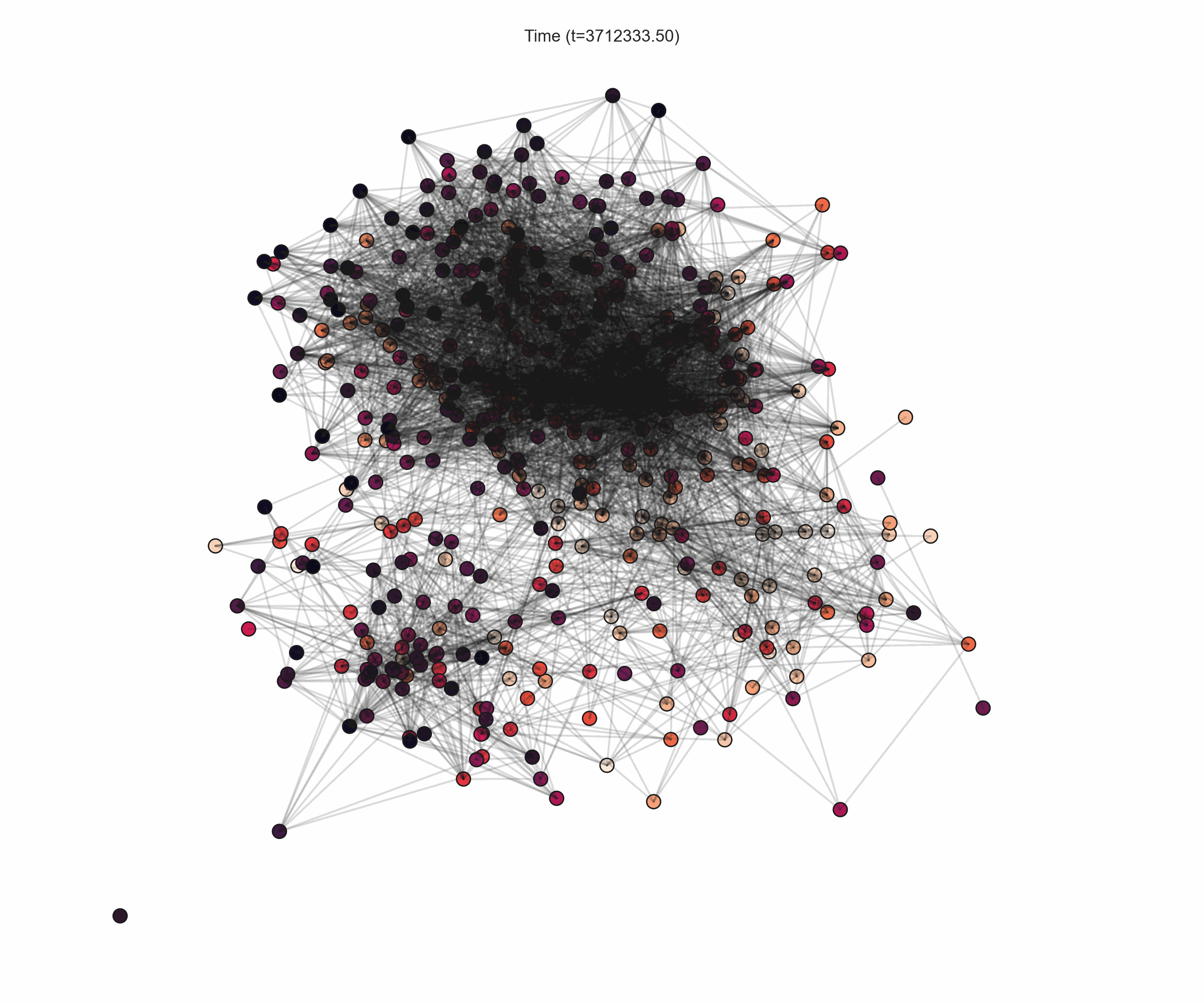}}
\hfill
\subfigure[$t=7424667$]{\includegraphics[trim={5cm 6cm 5cm 6cm},clip,width=0.16\textwidth]{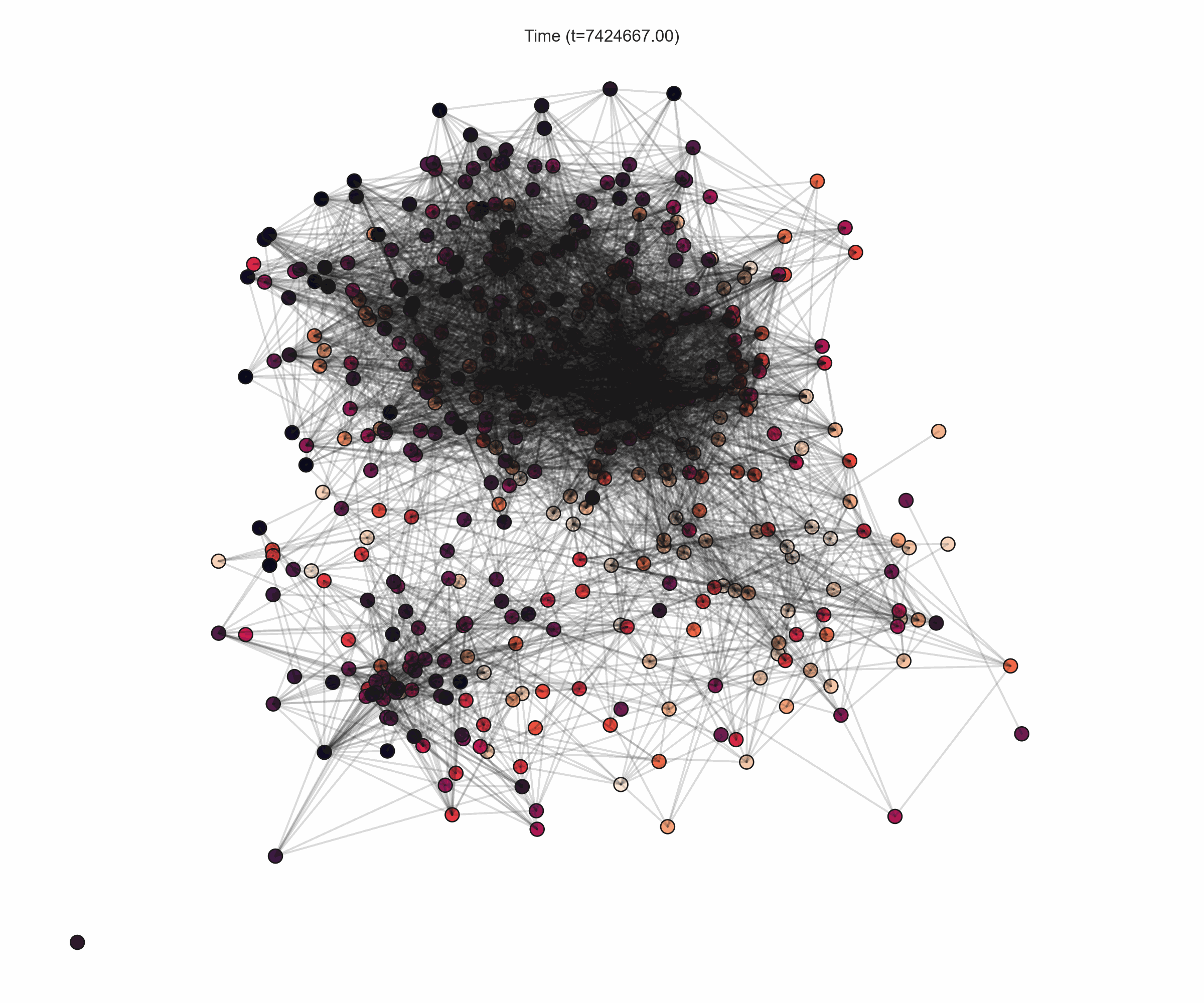}}
\hfill
\subfigure[$t=11137000$]{\includegraphics[trim={5cm 6cm 5cm 6cm},clip,width=0.16\textwidth]{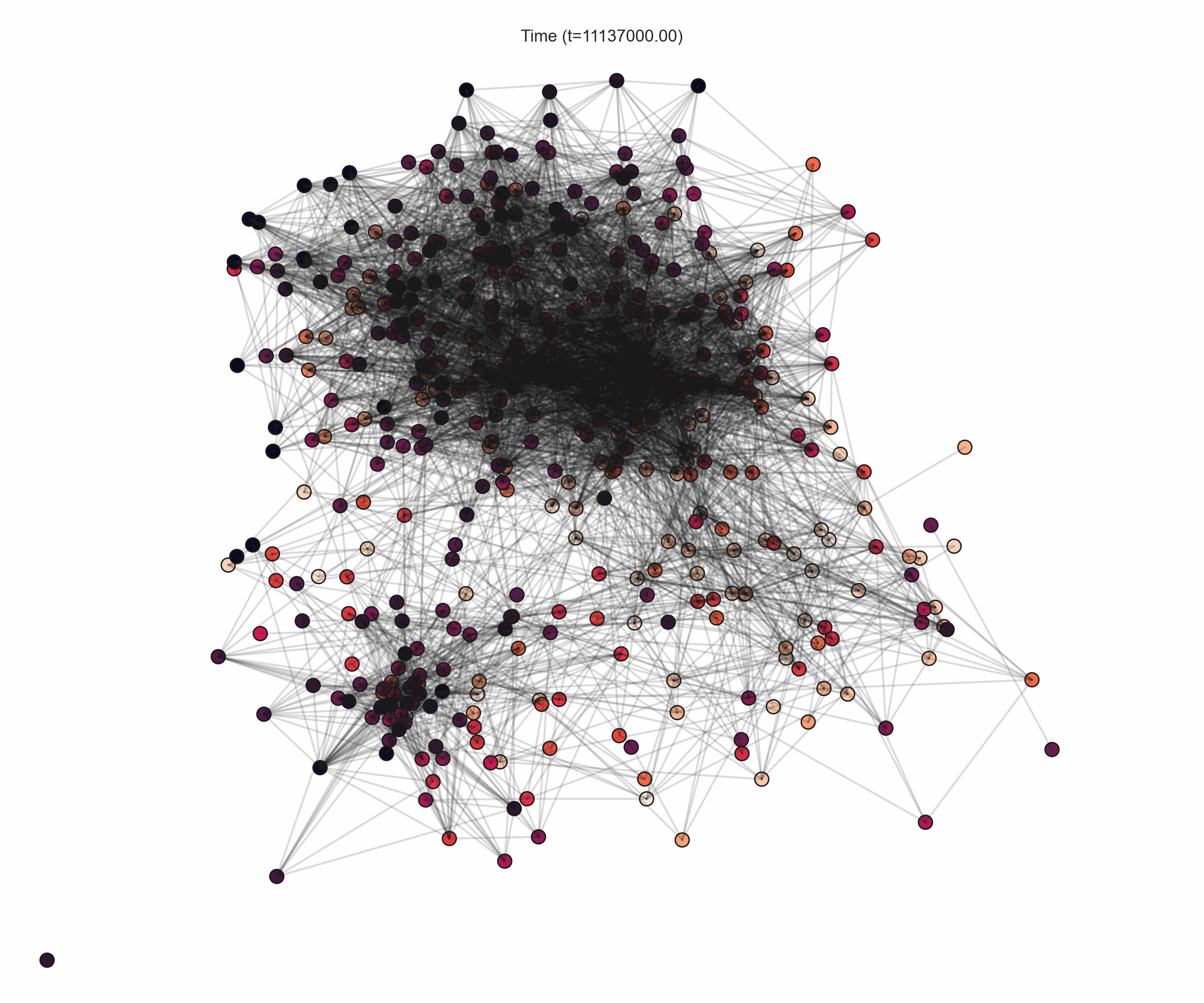}}
\hfill
\subfigure[$t=14849334$]{\includegraphics[trim={5cm 6cm 5cm 6cm},clip,width=0.16\textwidth]{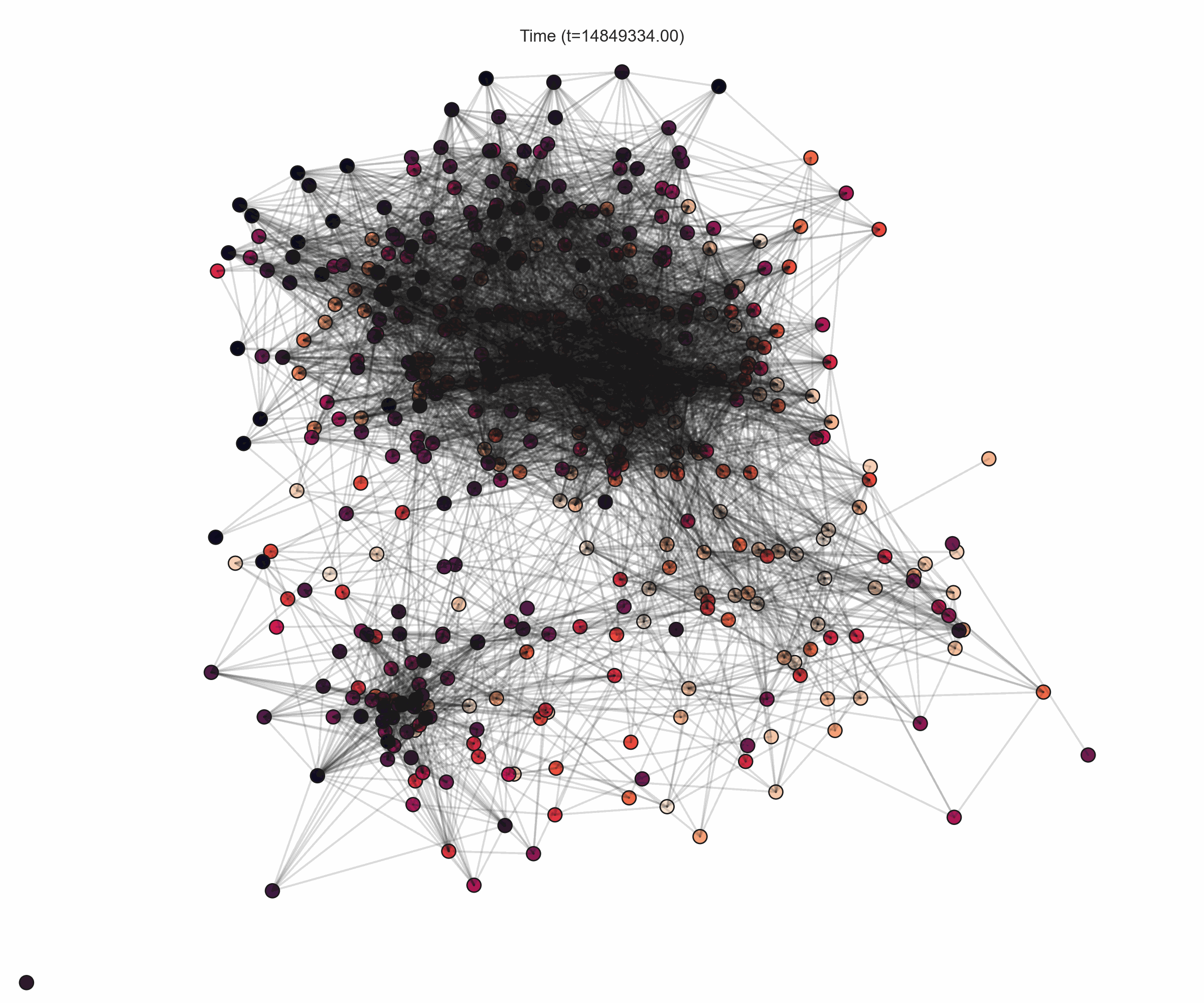}}
\hfill
\subfigure[$t=18561668$]{\includegraphics[trim={5cm 6cm 5cm 6cm},clip,width=0.16\textwidth]{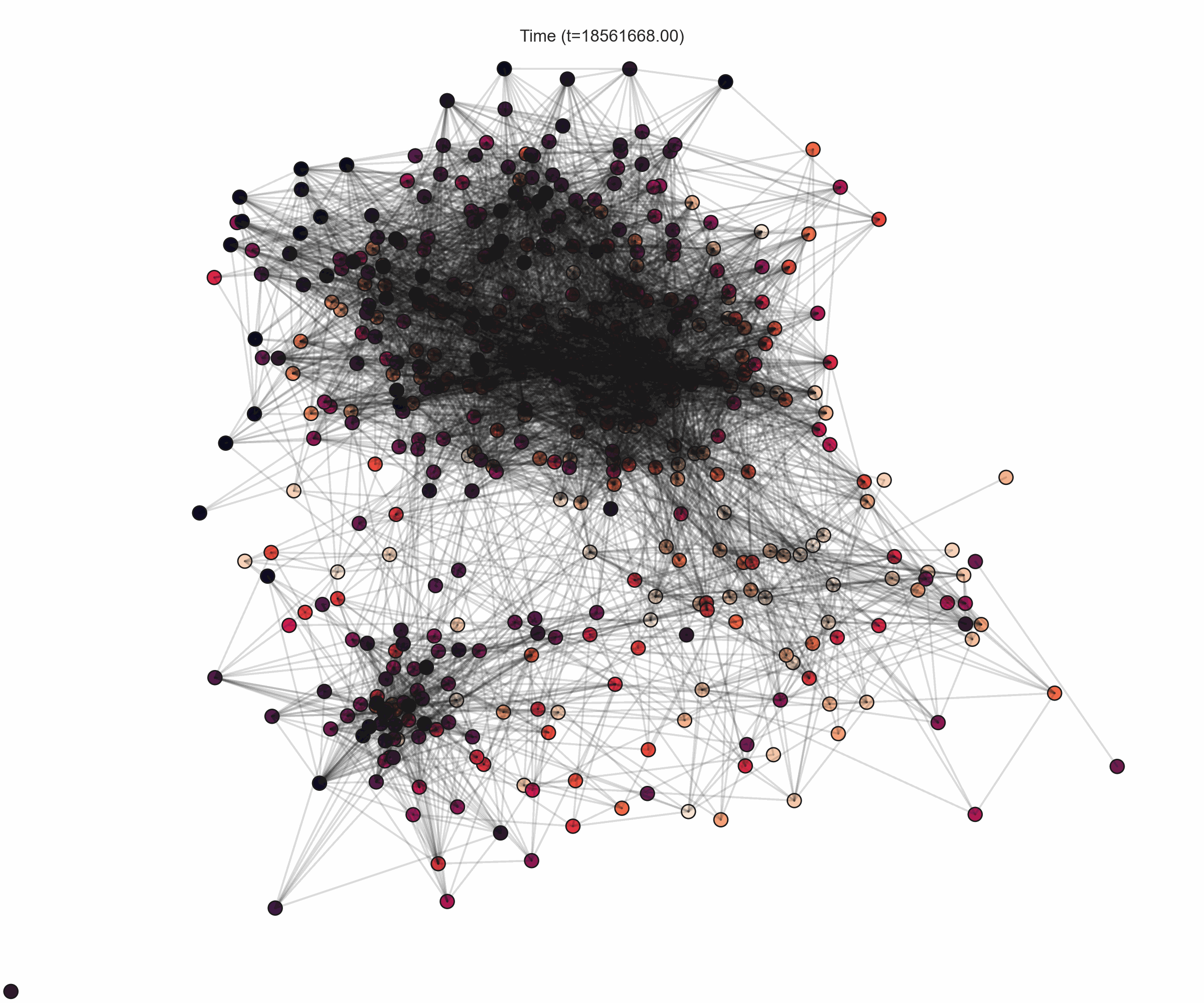}}
\hfill
\subfigure[$t=22274000$]{\includegraphics[trim={5cm 6cm 5cm 6cm},clip,width=0.16\textwidth]{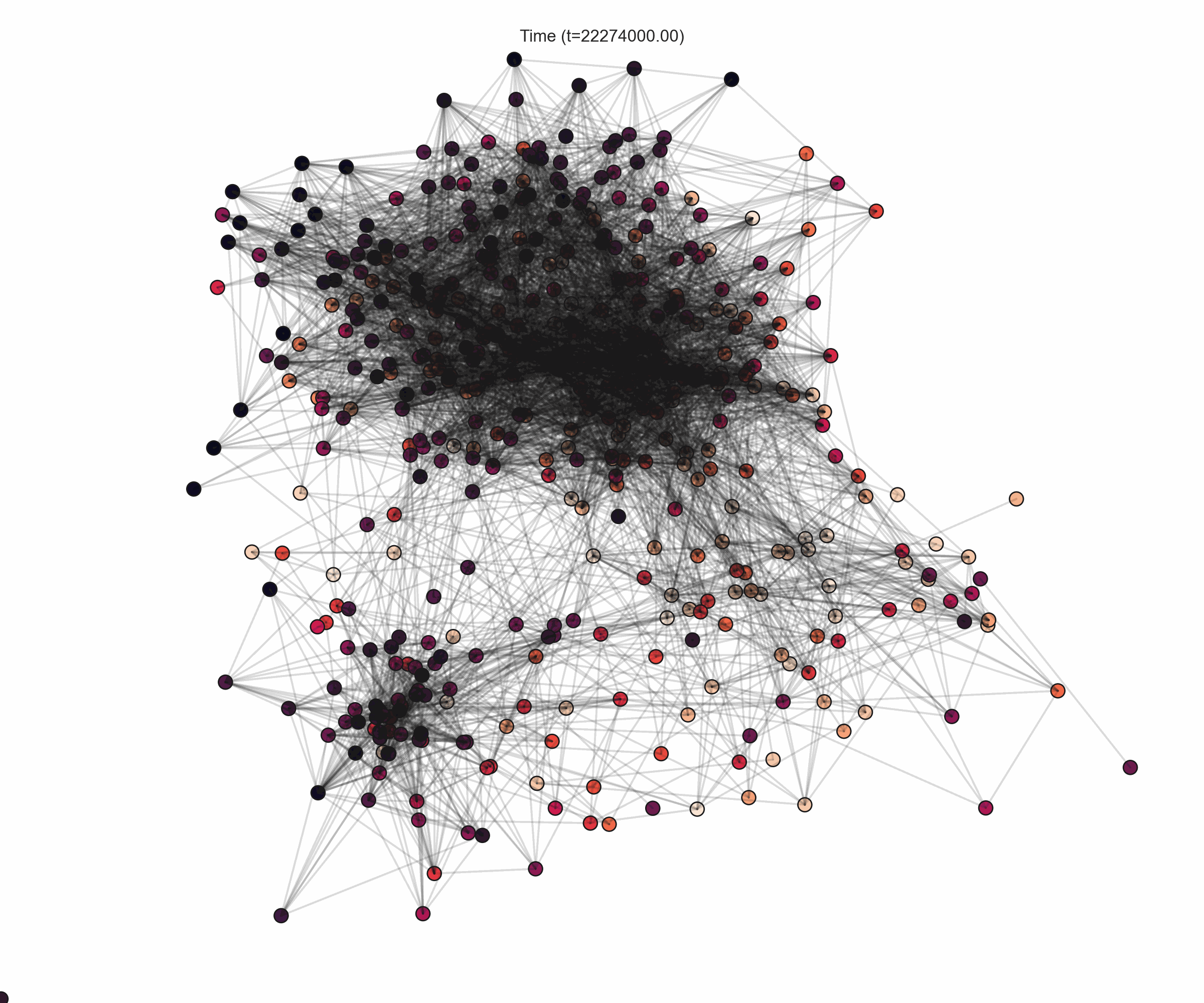}}
%%%%%%%%
\subfigure[$t=25986334$]{\includegraphics[trim={5cm 6cm 5cm 6cm},clip,width=0.16\textwidth]{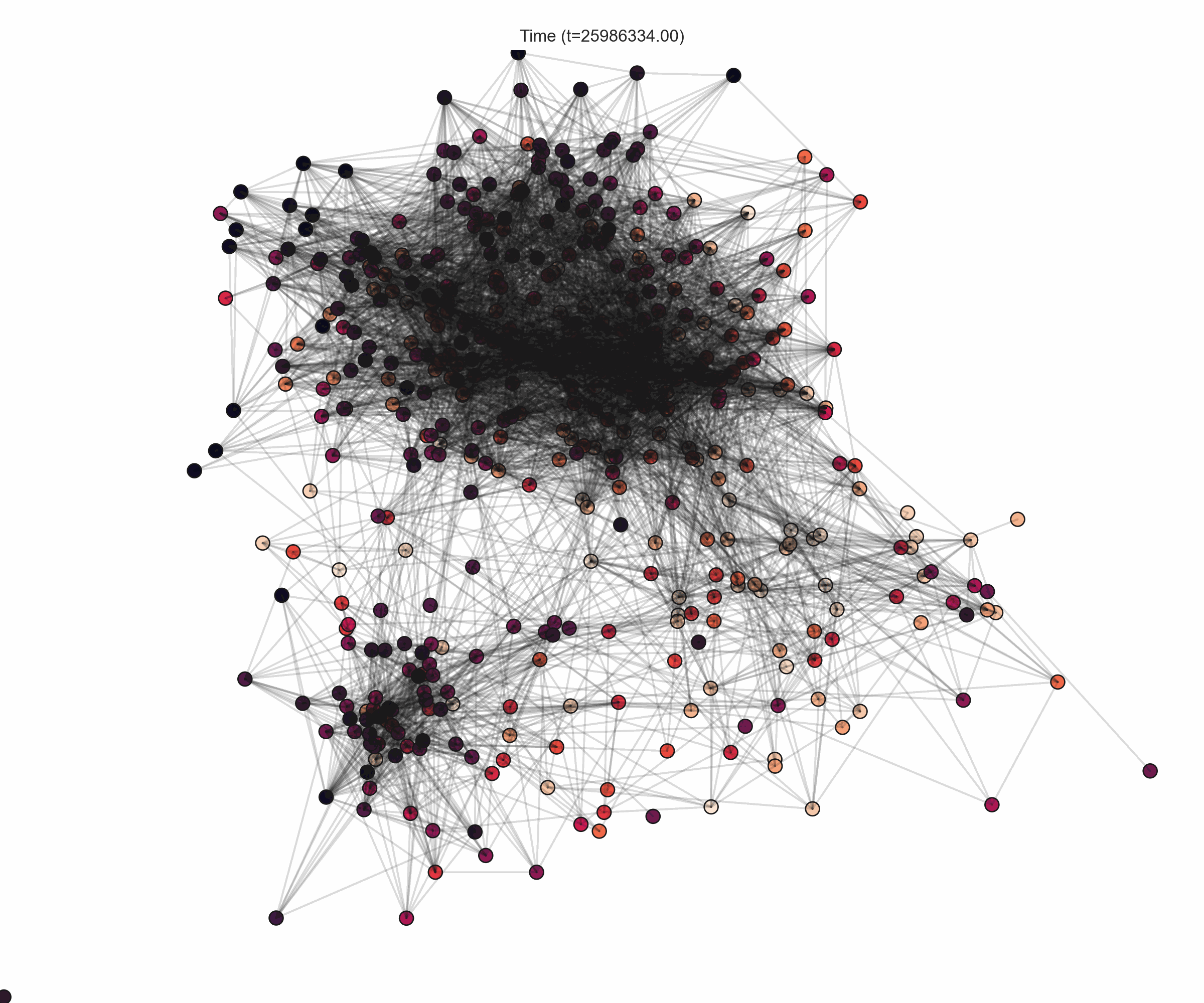}}
\hfill
\subfigure[$t=29698668$]{\includegraphics[trim={5cm 6cm 5cm 6cm},clip,width=0.16\textwidth]{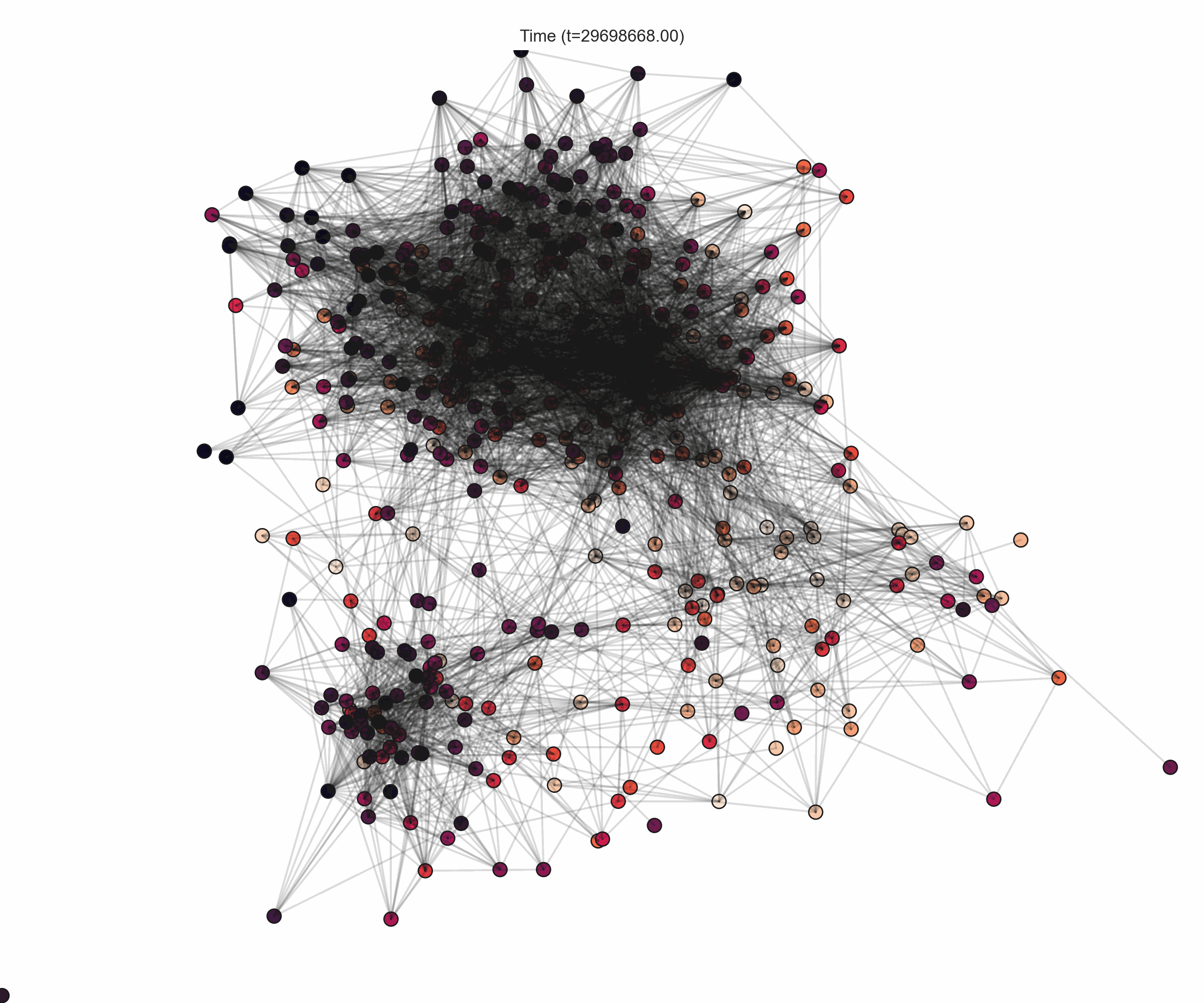}}
\hfill
\subfigure[$t=(t=33411000$]{\includegraphics[trim={5cm 6cm 5cm 6cm},clip,width=0.16\textwidth]{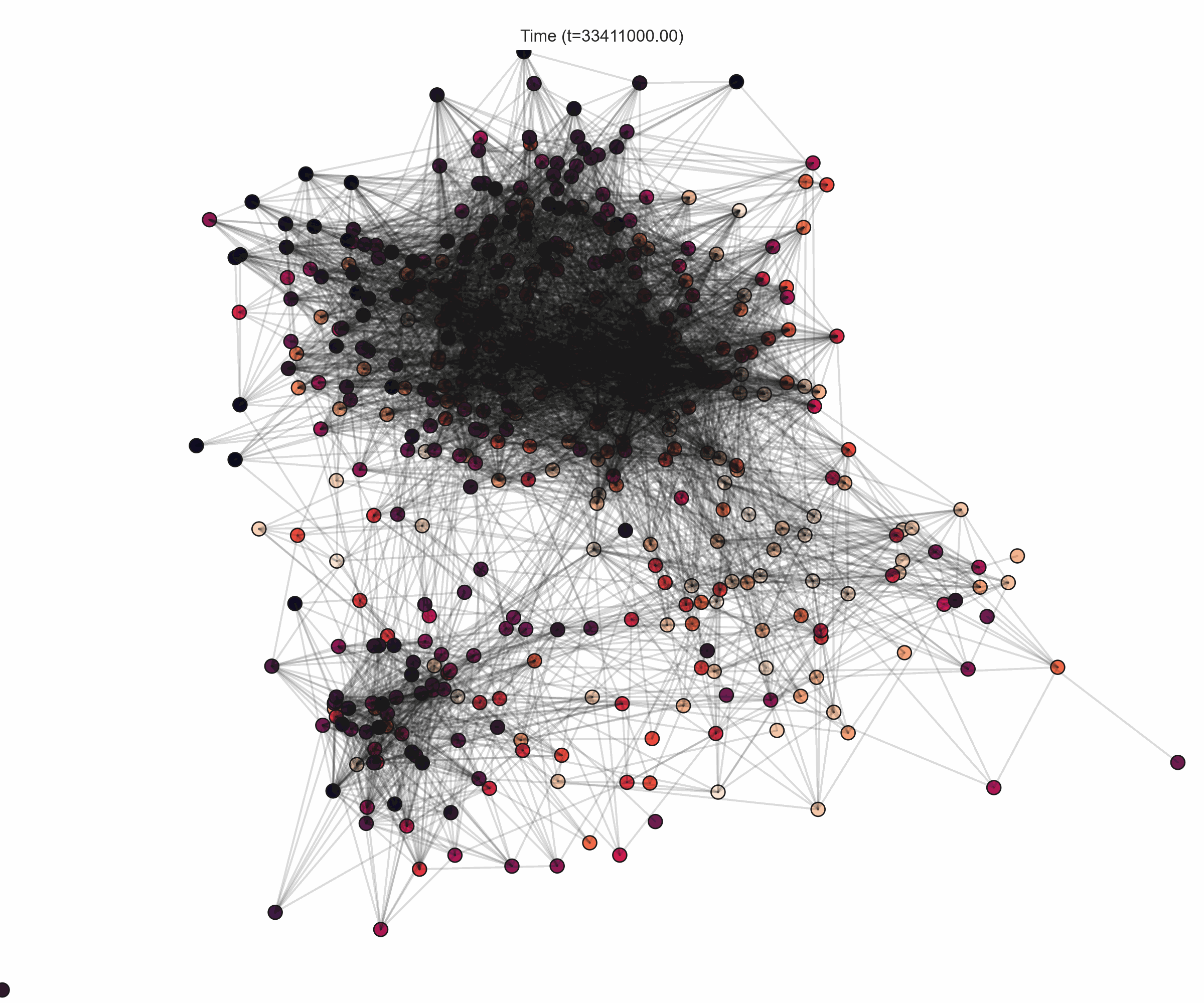}}
\hfill
\subfigure[$t=37123332$]{\includegraphics[trim={5cm 6cm 5cm 6cm},clip,width=0.16\textwidth]{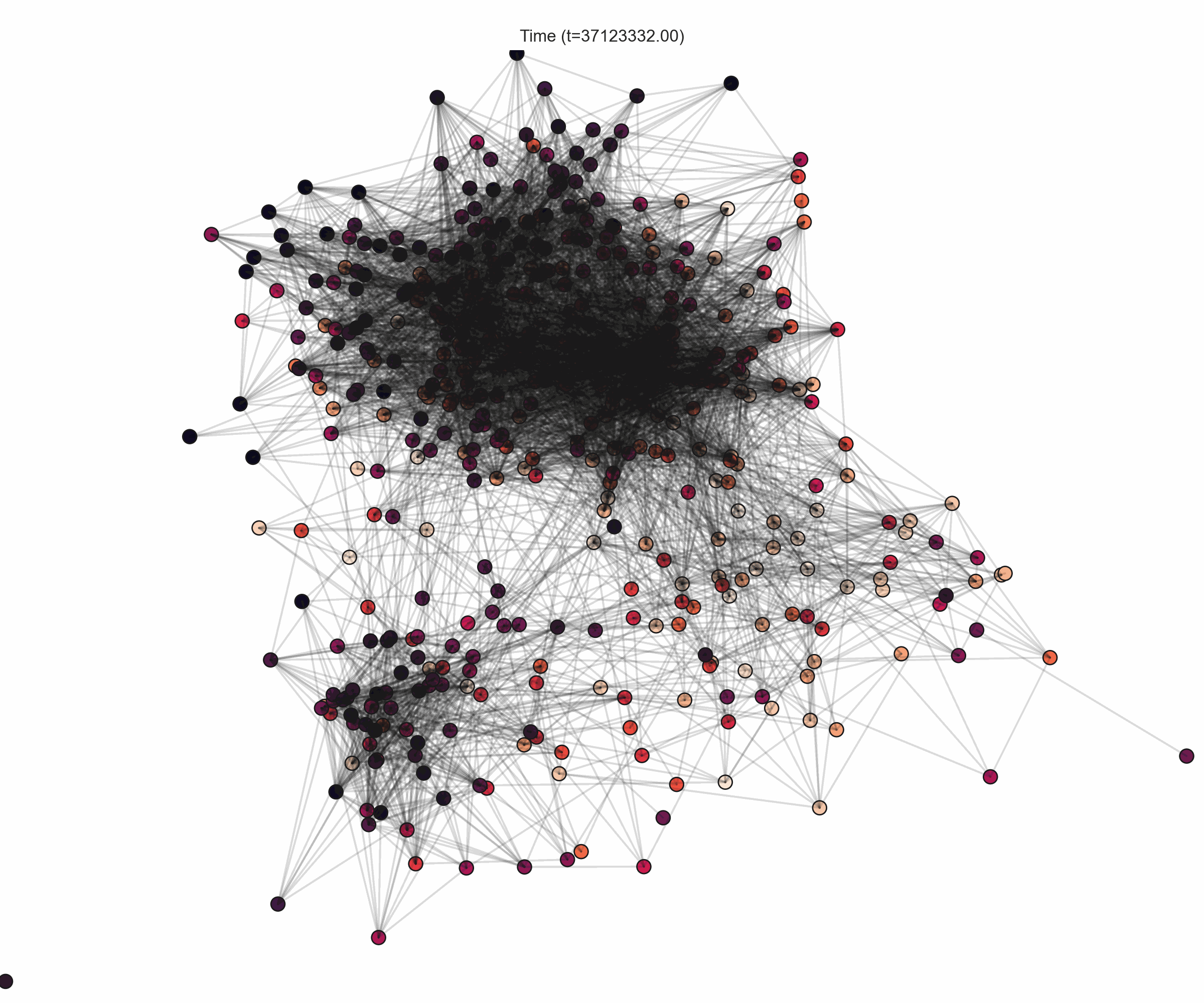}}
\hfill
\subfigure[$t=40835664$]{\includegraphics[trim={5cm 6cm 5cm 6cm},clip,width=0.16\textwidth]{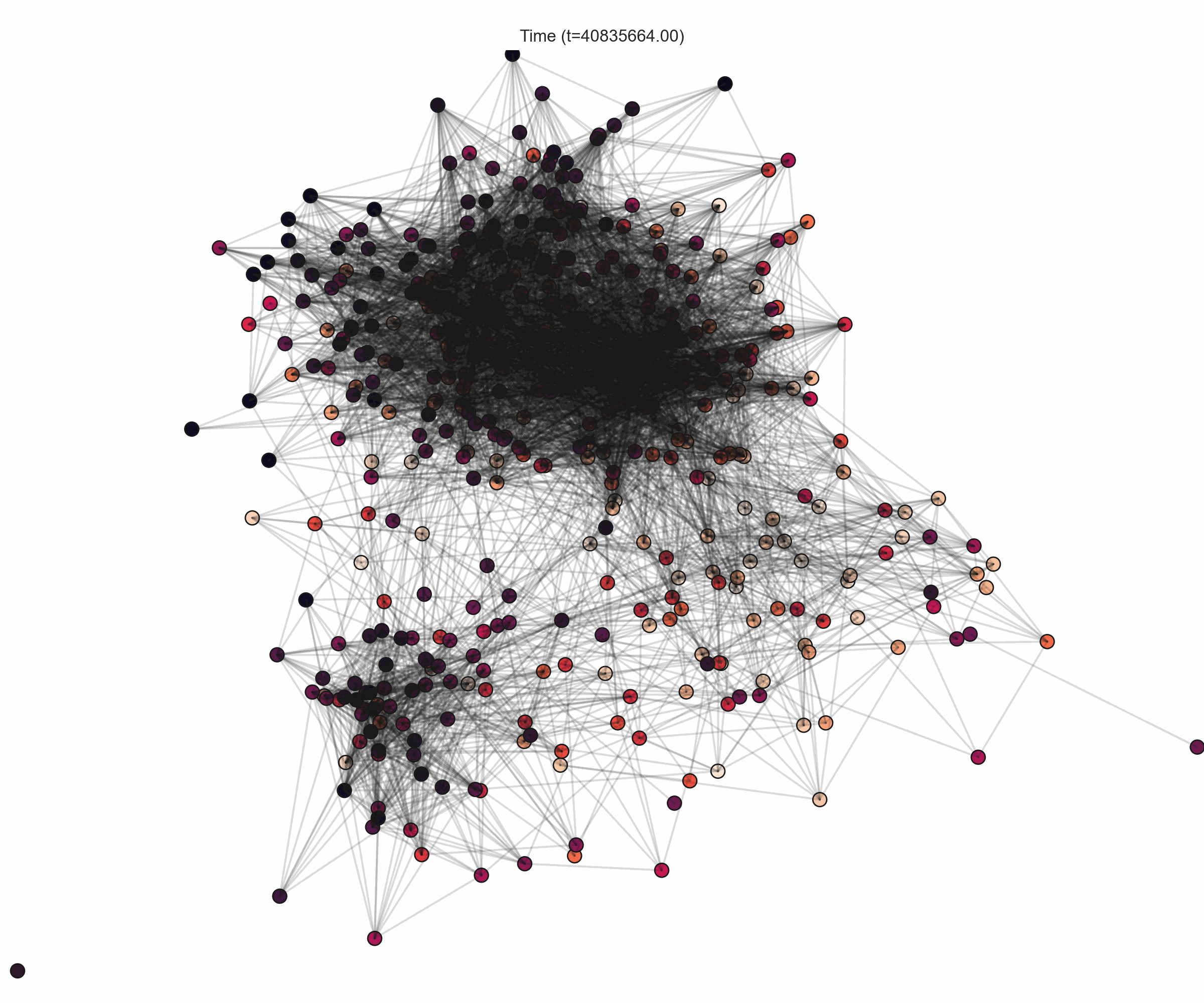}}
\hfill
\subfigure[$t=44548000$]{\includegraphics[trim={5cm 6cm 5cm 6cm},clip,width=0.16\textwidth]{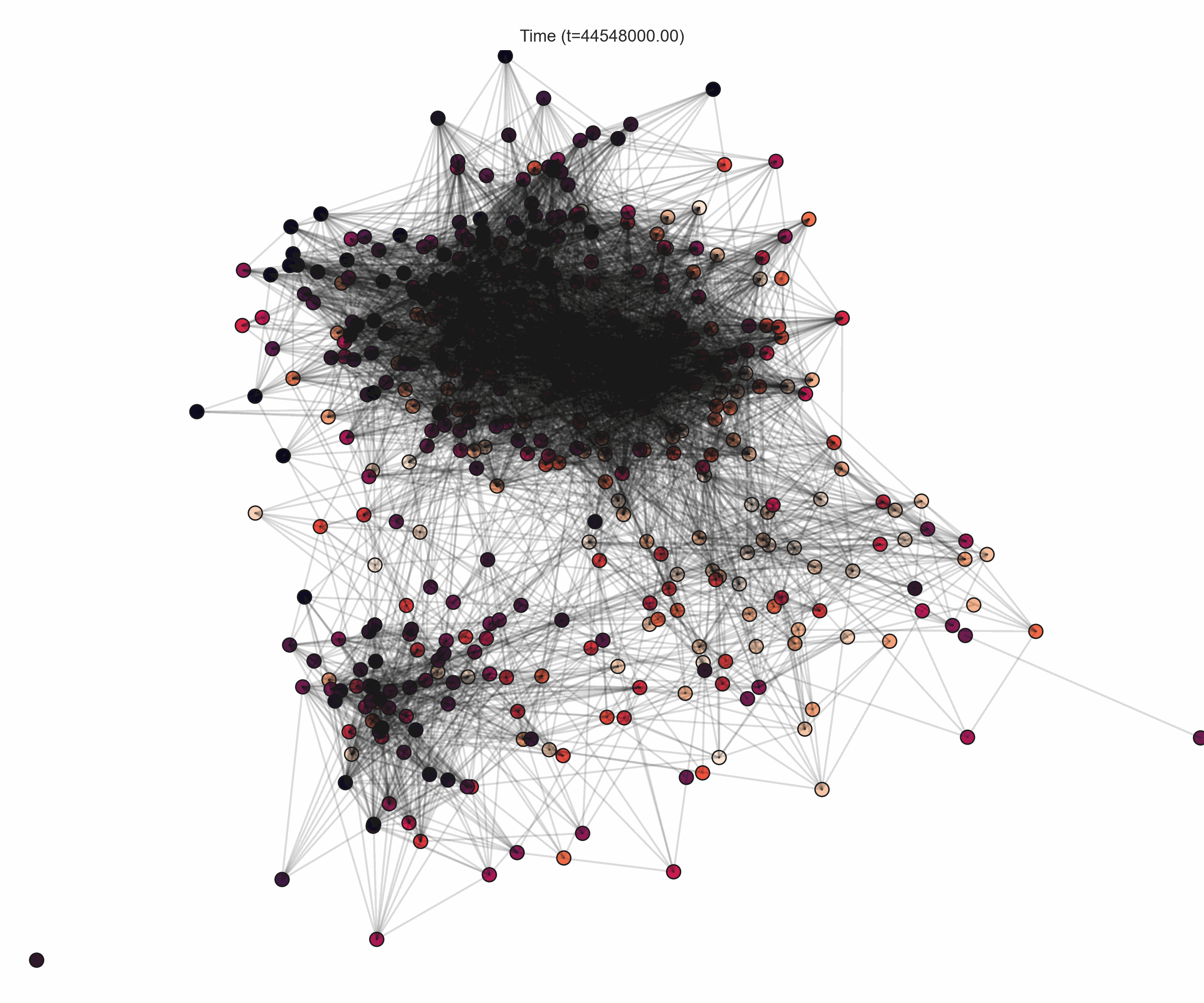}}
%%%%%%%%
\subfigure[$t=48260332$]{\includegraphics[trim={5cm 6cm 5cm 6cm},clip,width=0.16\textwidth]{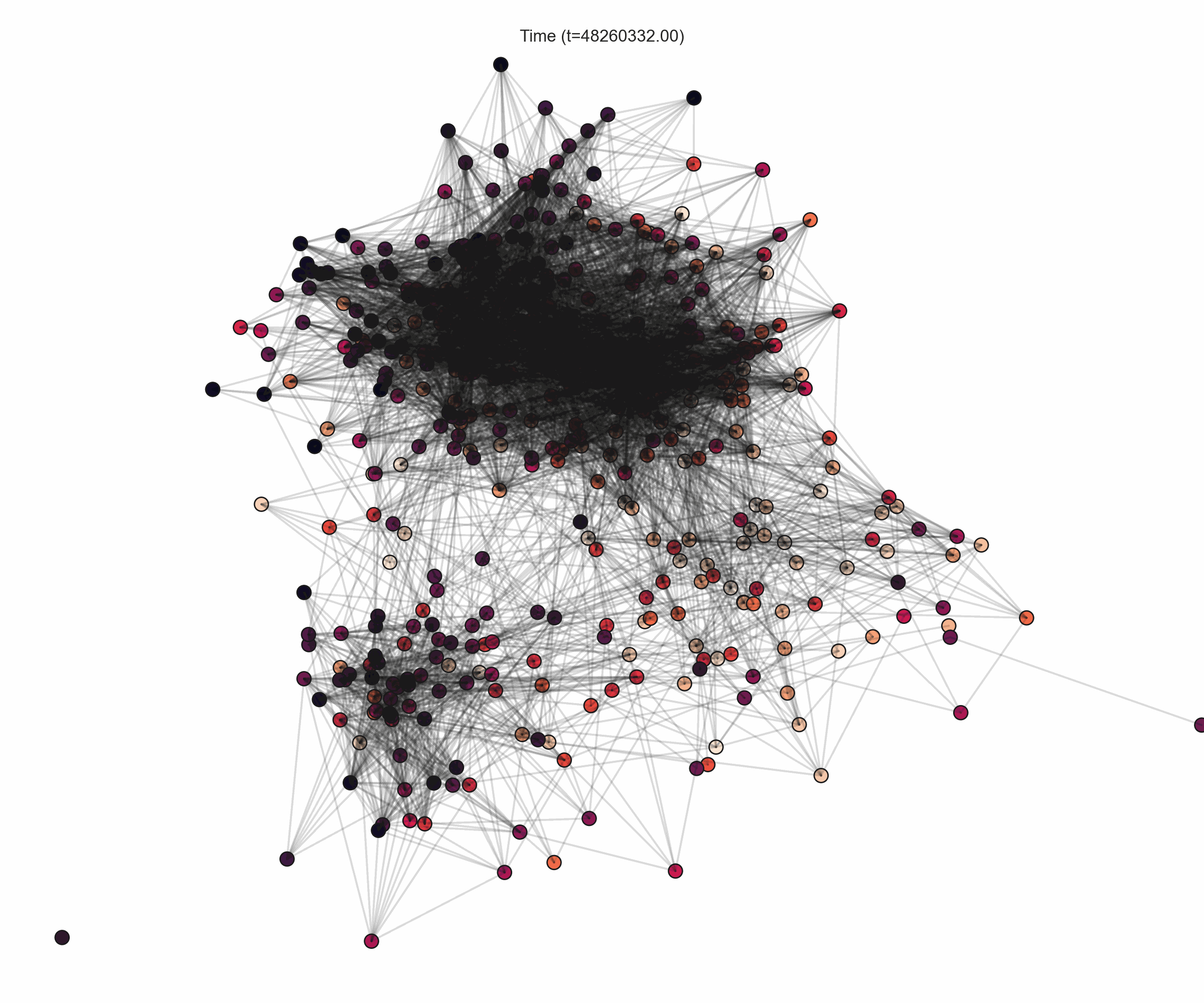}}
\hfill
\subfigure[$t=51972664$]{\includegraphics[trim={5cm 6cm 5cm 6cm},clip,width=0.16\textwidth]{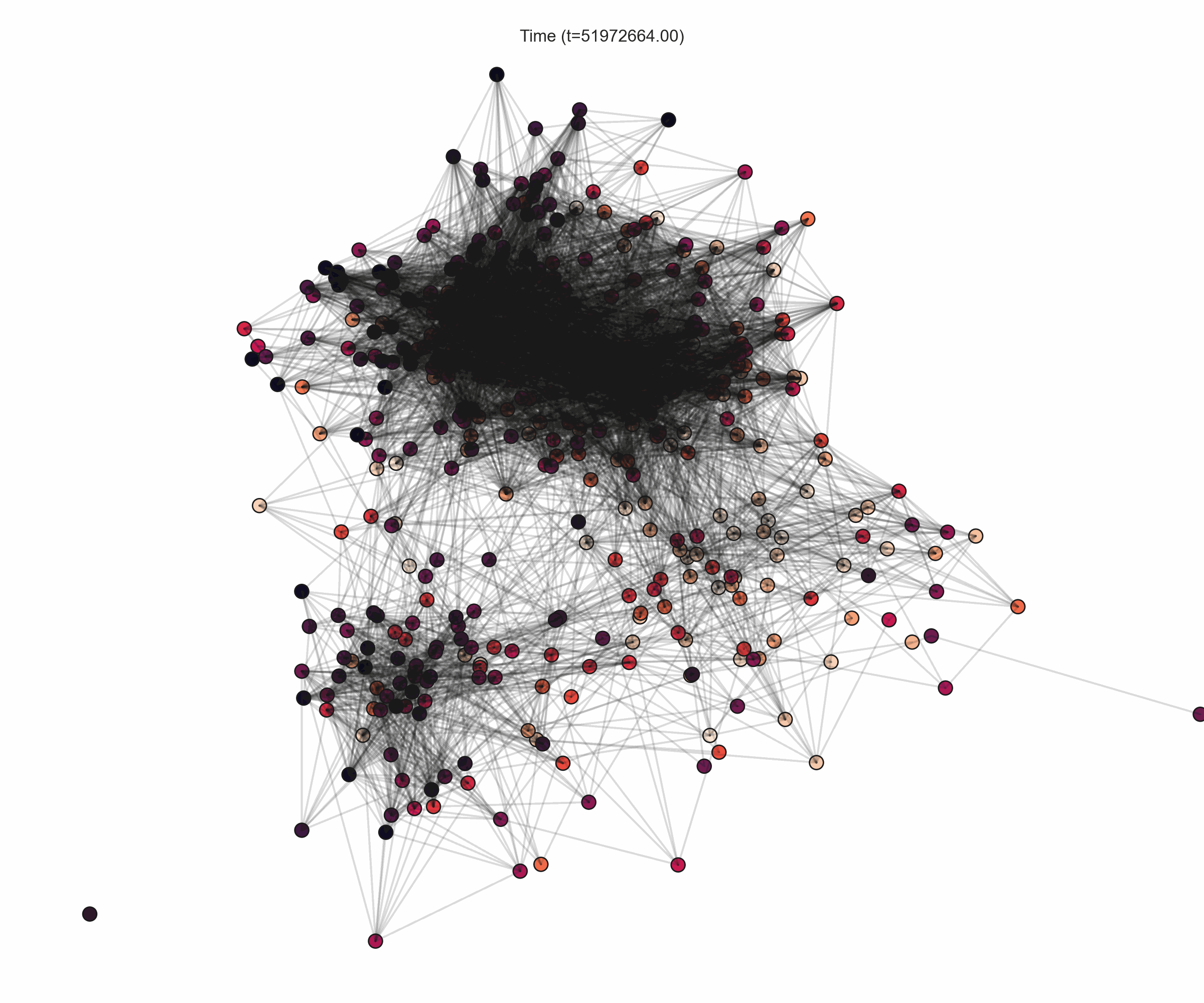}}
\hfill
\subfigure[$t=55685000$]{\includegraphics[trim={5cm 6cm 5cm 6cm},clip,width=0.16\textwidth]{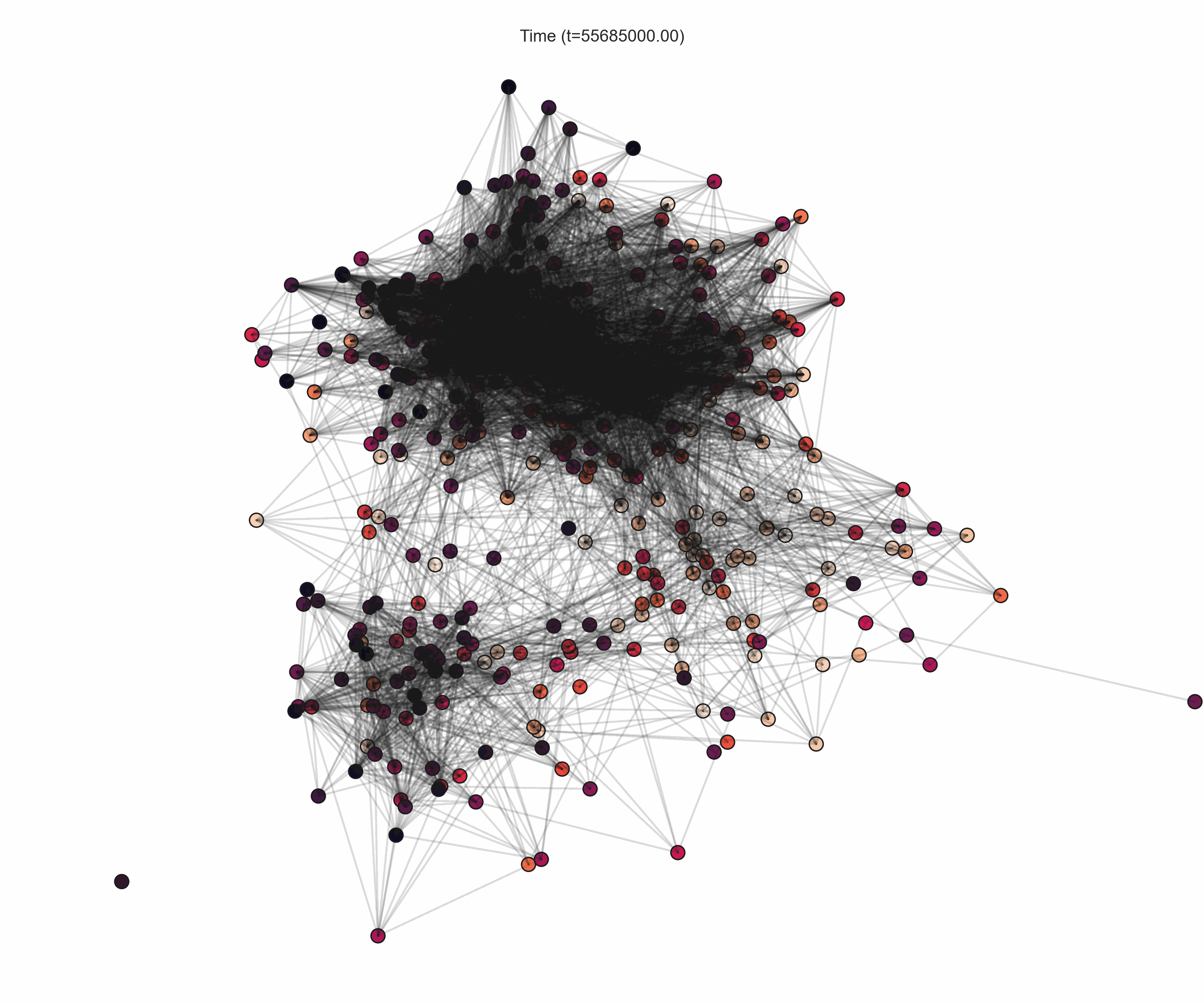}}
\hfill
\subfigure[$t=59397332$]{\includegraphics[trim={5cm 6cm 5cm 6cm},clip,width=0.16\textwidth]{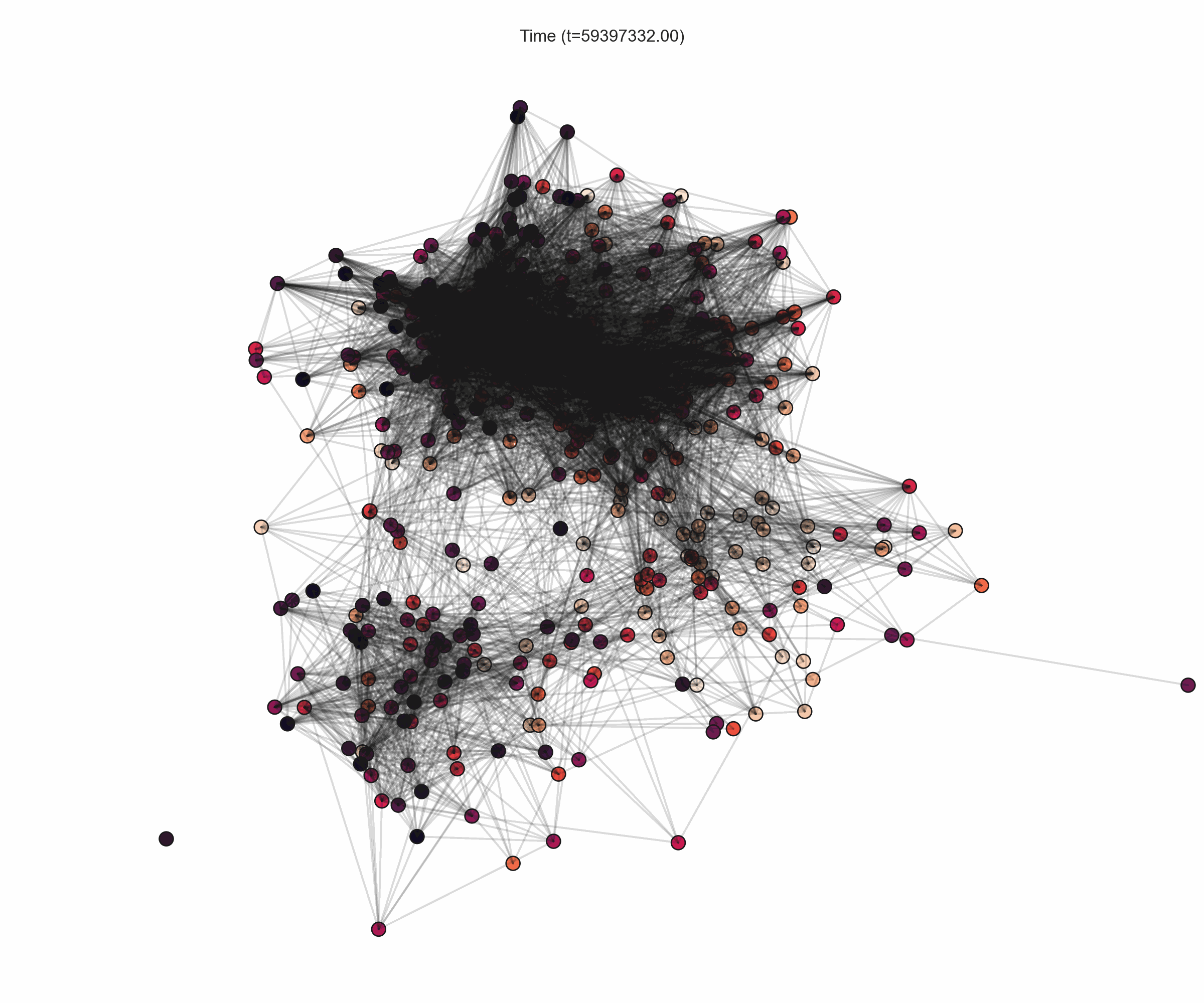}}
\hfill
\subfigure[$t=63109664$]{\includegraphics[trim={5cm 6cm 5cm 6cm},clip,width=0.16\textwidth]{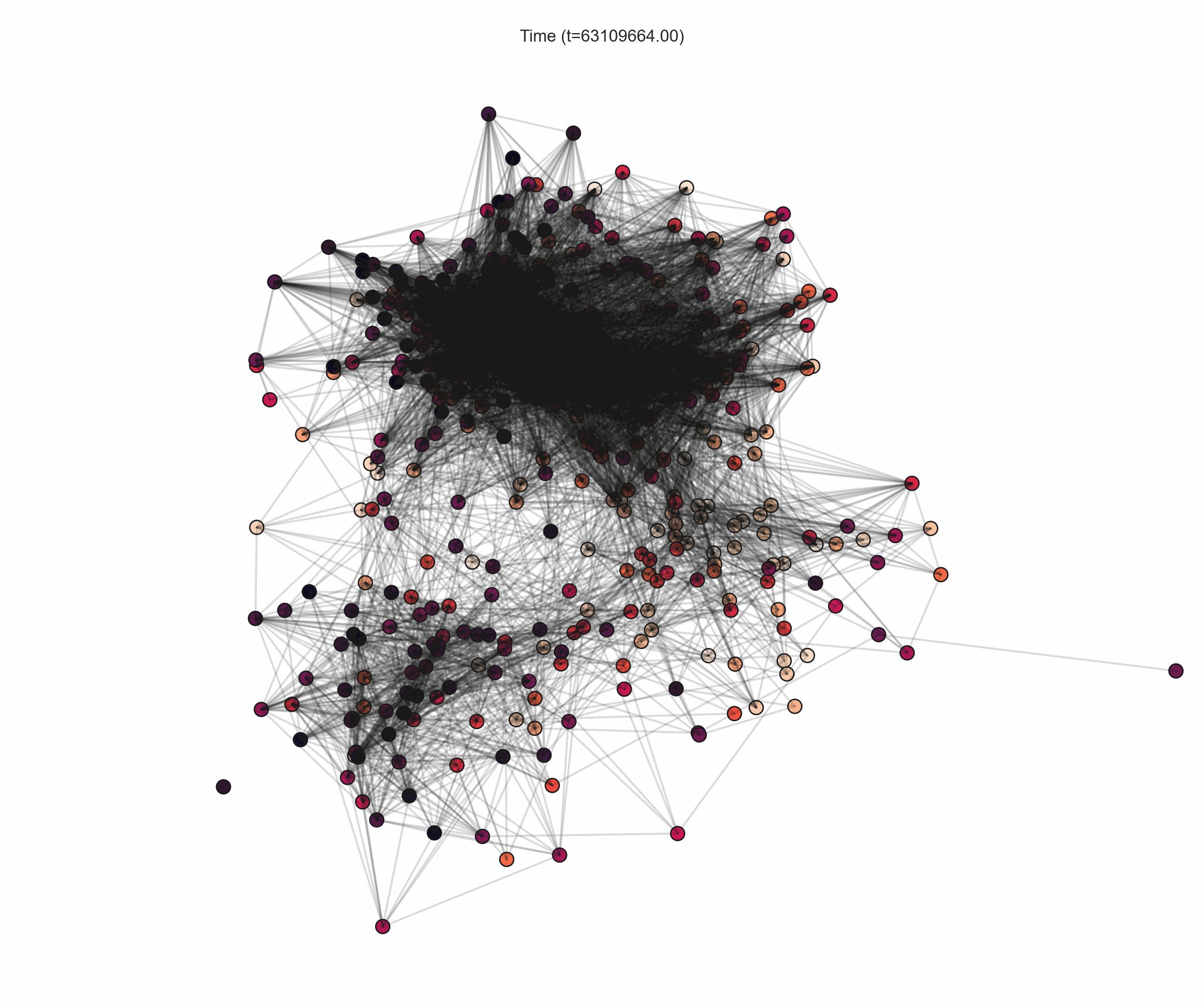}}
\hfill
\subfigure[$t=66822000$]{\includegraphics[trim={5cm 6cm 5cm 6cm},clip,width=0.16\textwidth]{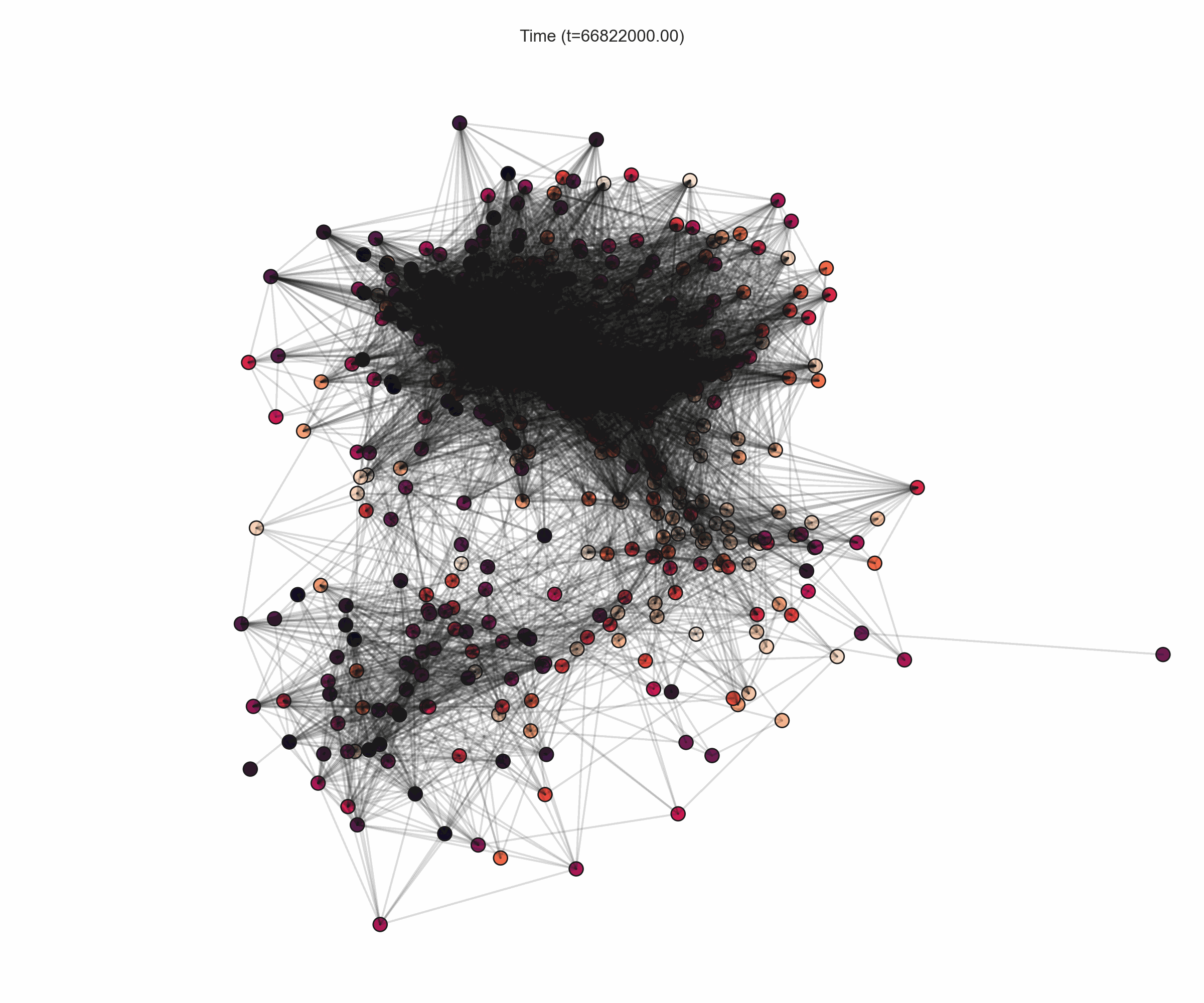}}
\caption{Snapshots of the continuous-time embeddings learned by \textsc{\modelname} for various time points over \textsl{Facebook}.}\label{fig:appendix_visualization_facebook}
\end{figure*}
%%%%%%%%%%%%%
\begin{figure*}[!ht]
\centering
\subfigure[$t=(t=1989.61$]{\includegraphics[trim={5cm 6cm 5cm 6cm},clip,width=0.16\textwidth]{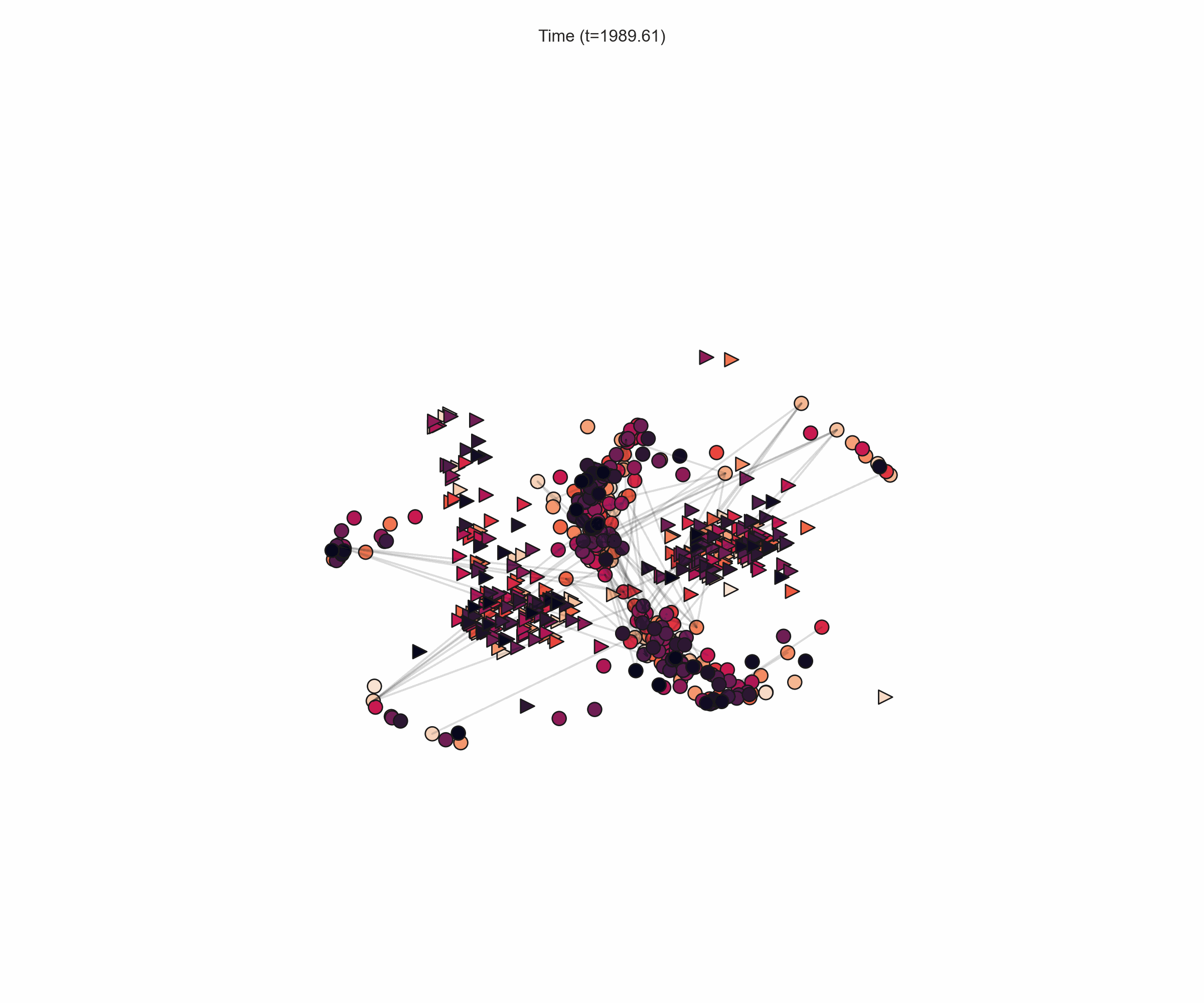}}
\hfill
\subfigure[$t=1990.21$]{\includegraphics[trim={5cm 6cm 5cm 6cm},clip,width=0.16\textwidth]{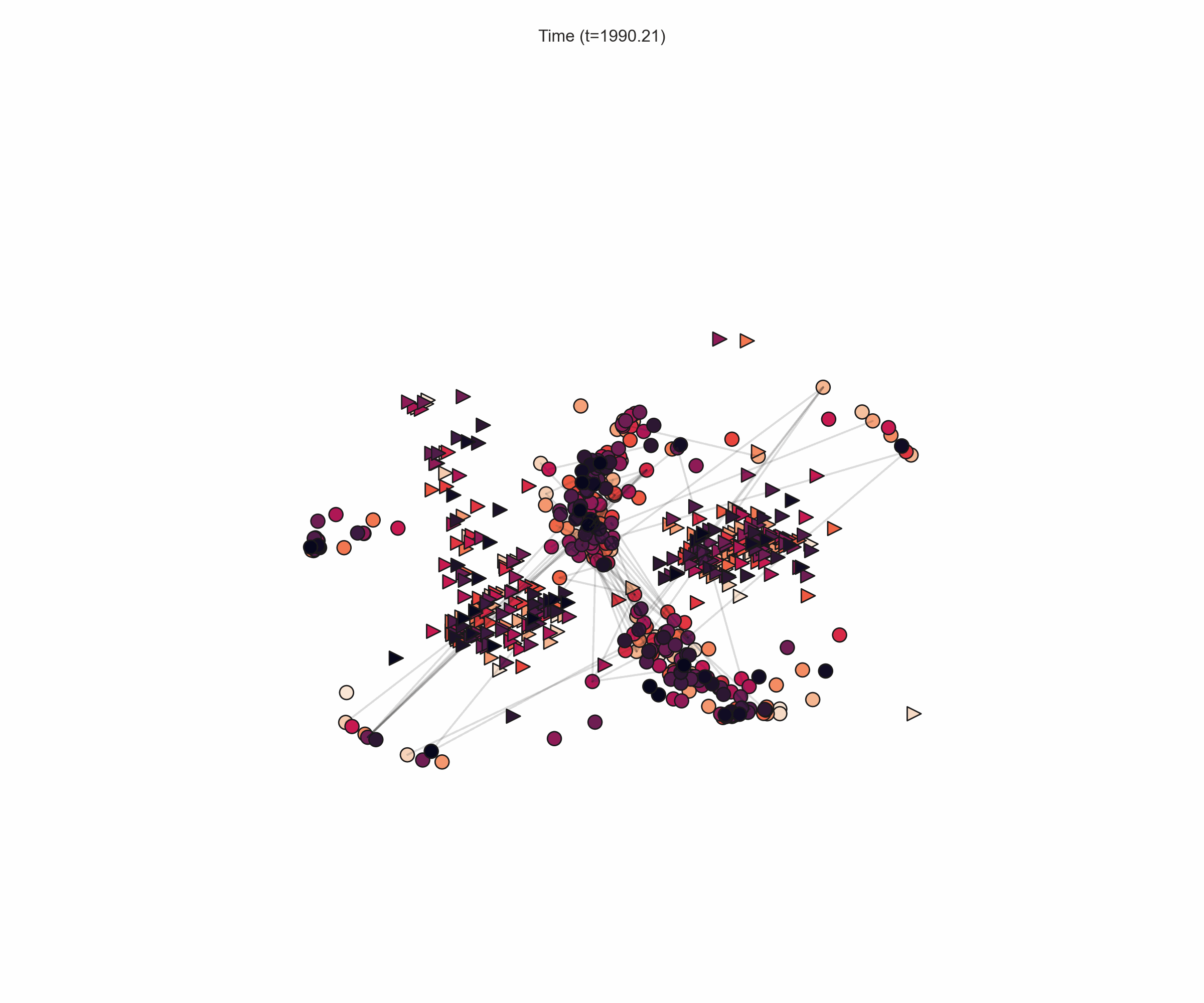}}
\hfill
\subfigure[$t=1990.82$]{\includegraphics[trim={5cm 6cm 5cm 6cm},clip,width=0.16\textwidth]{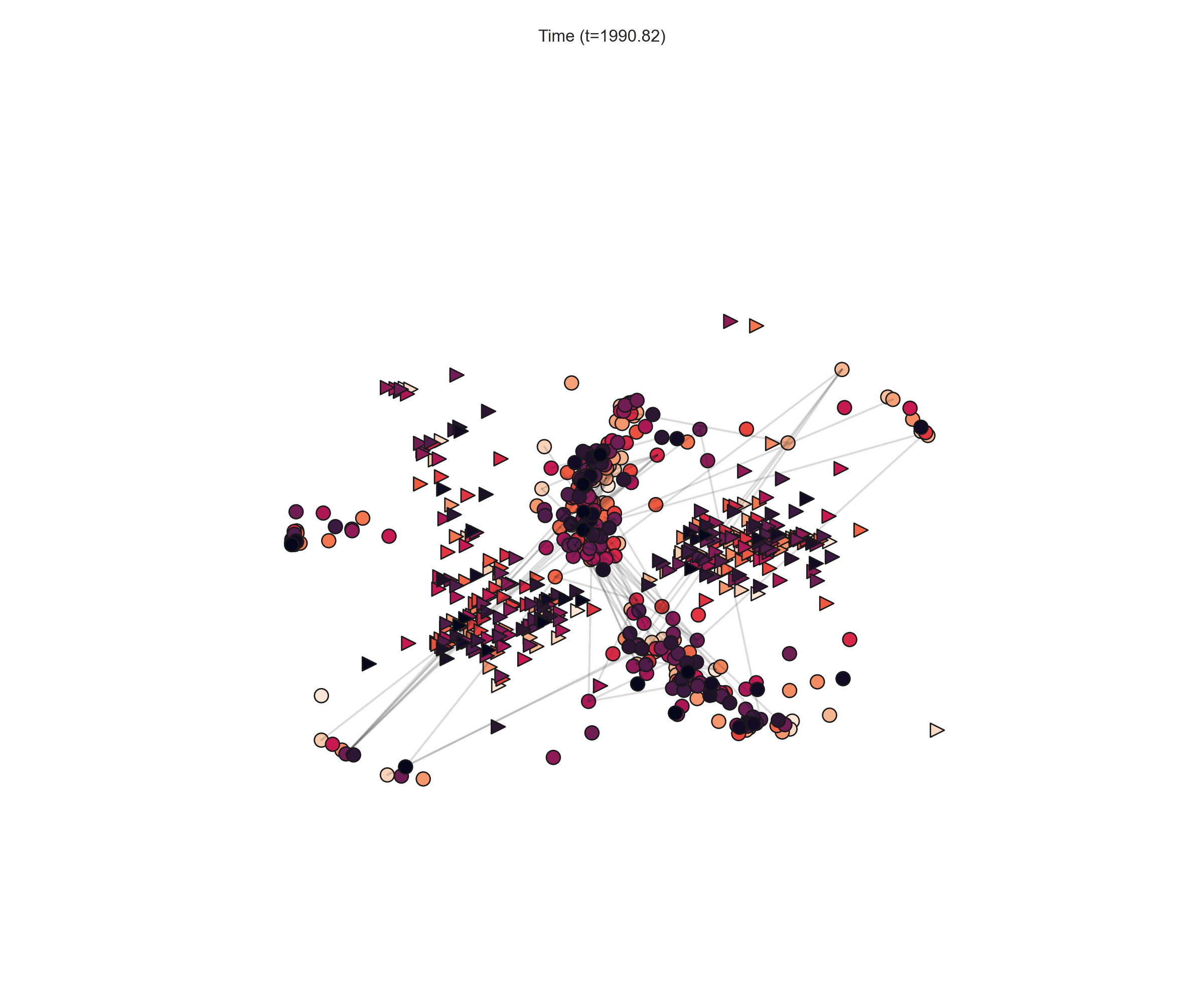}}
\hfill
\subfigure[$t=1991.42$]{\includegraphics[trim={5cm 6cm 5cm 6cm},clip,width=0.16\textwidth]{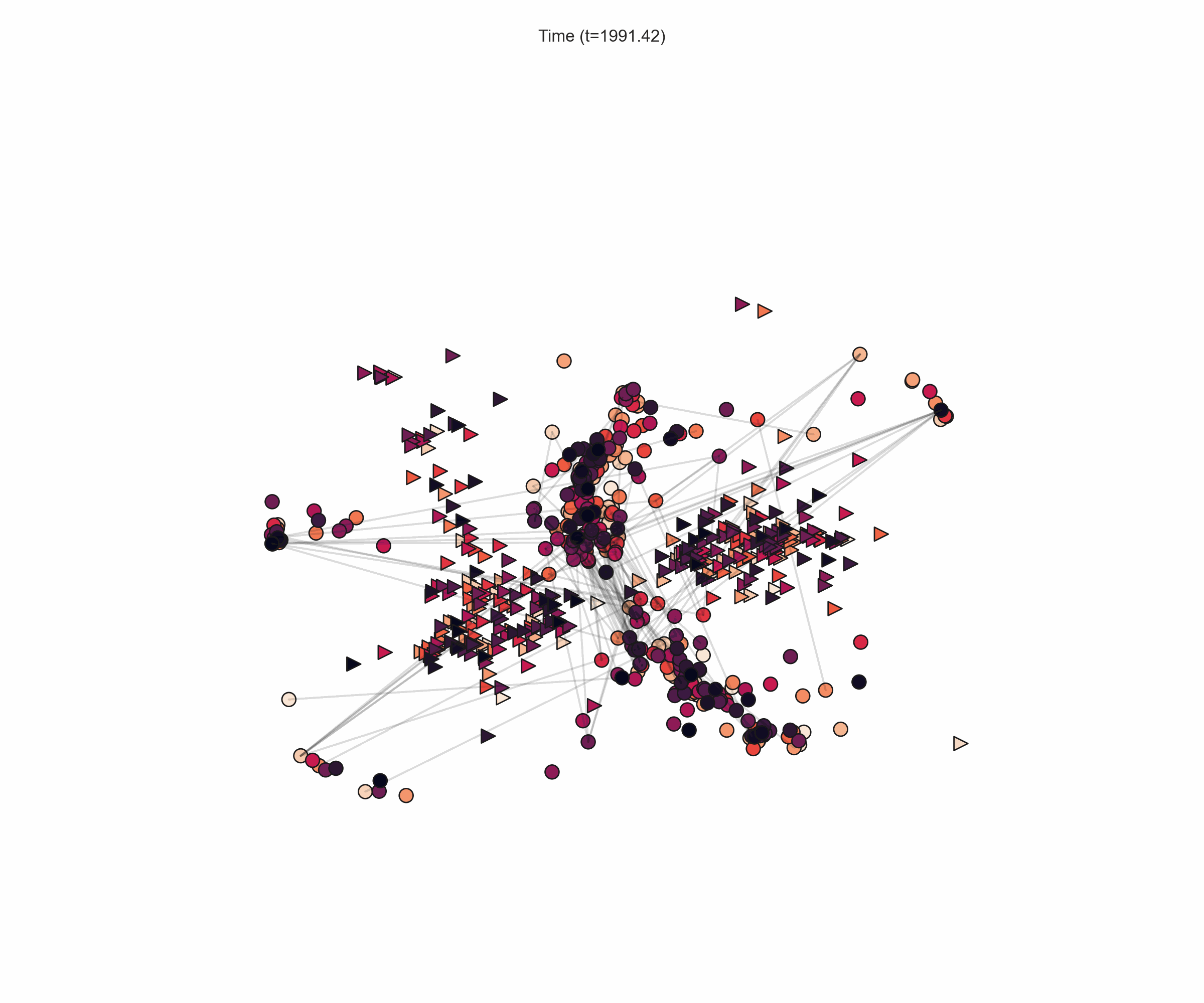}}
\hfill
\subfigure[$t=1992.03$]{\includegraphics[trim={5cm 6cm 5cm 6cm},clip,width=0.16\textwidth]{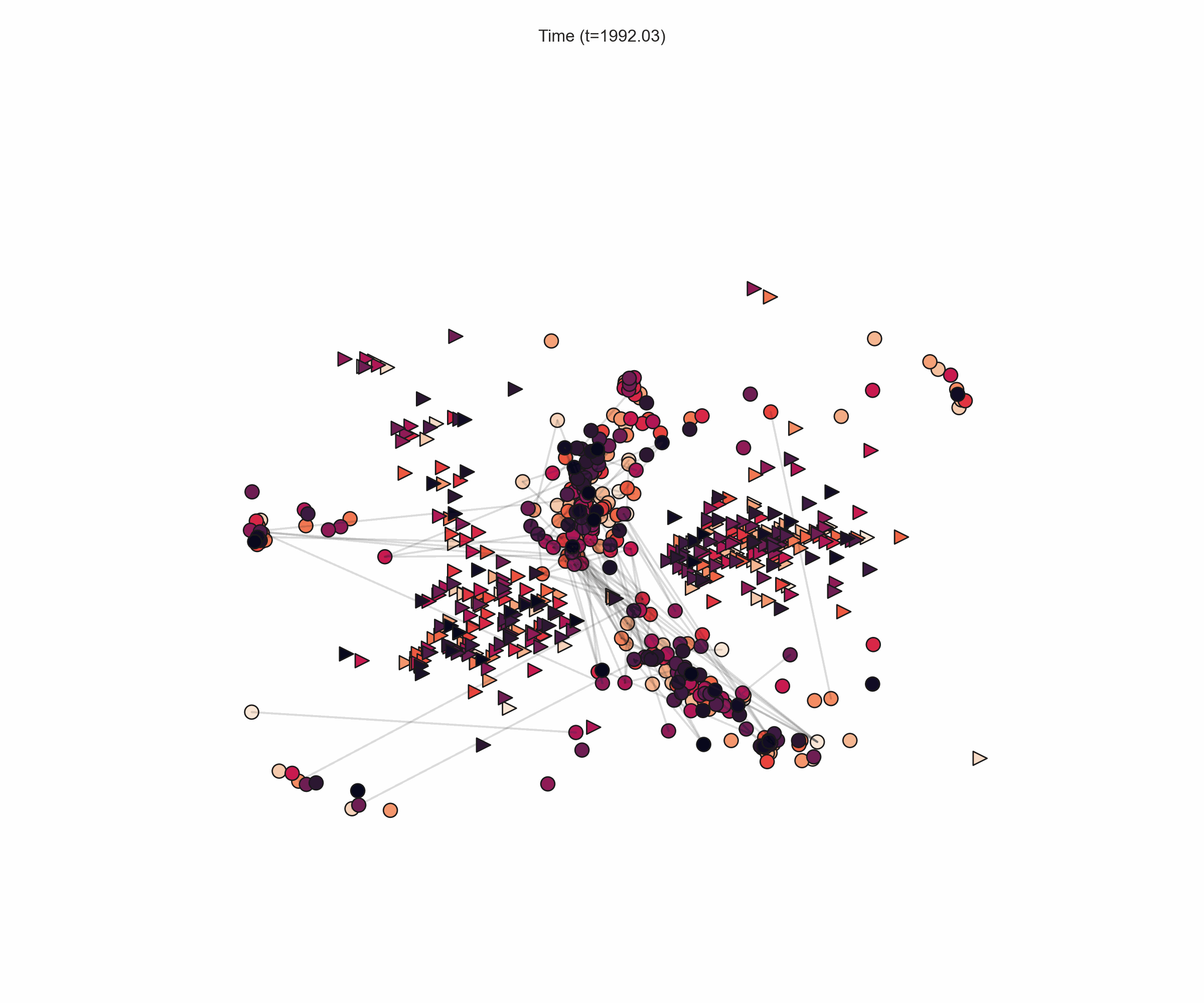}}
\hfill
\subfigure[$t=1992.64$]{\includegraphics[trim={5cm 6cm 5cm 6cm},clip,width=0.16\textwidth]{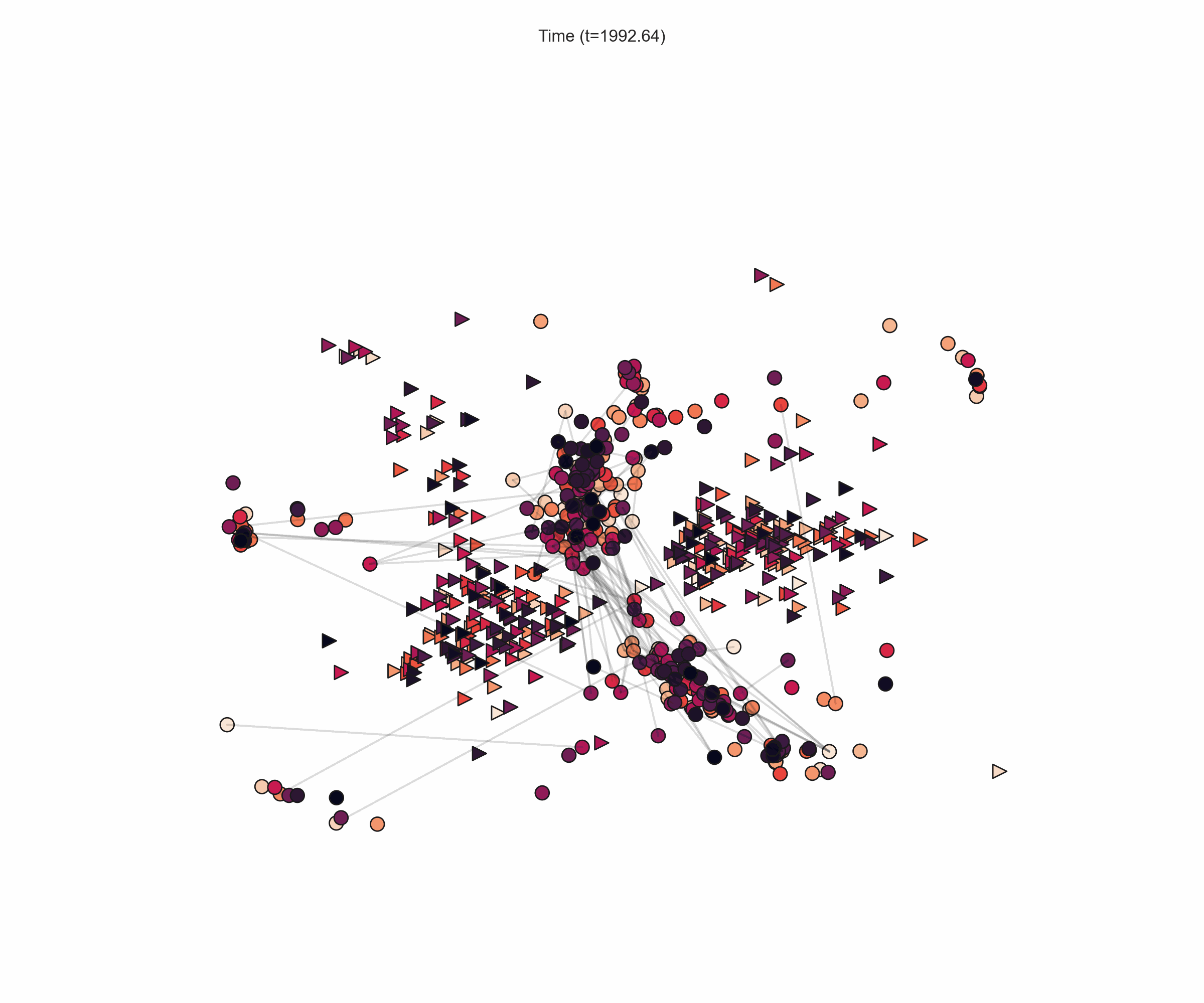}}
%%%%%%%%
\subfigure[$t=1993.24$]{\includegraphics[trim={5cm 6cm 5cm 6cm},clip,width=0.16\textwidth]{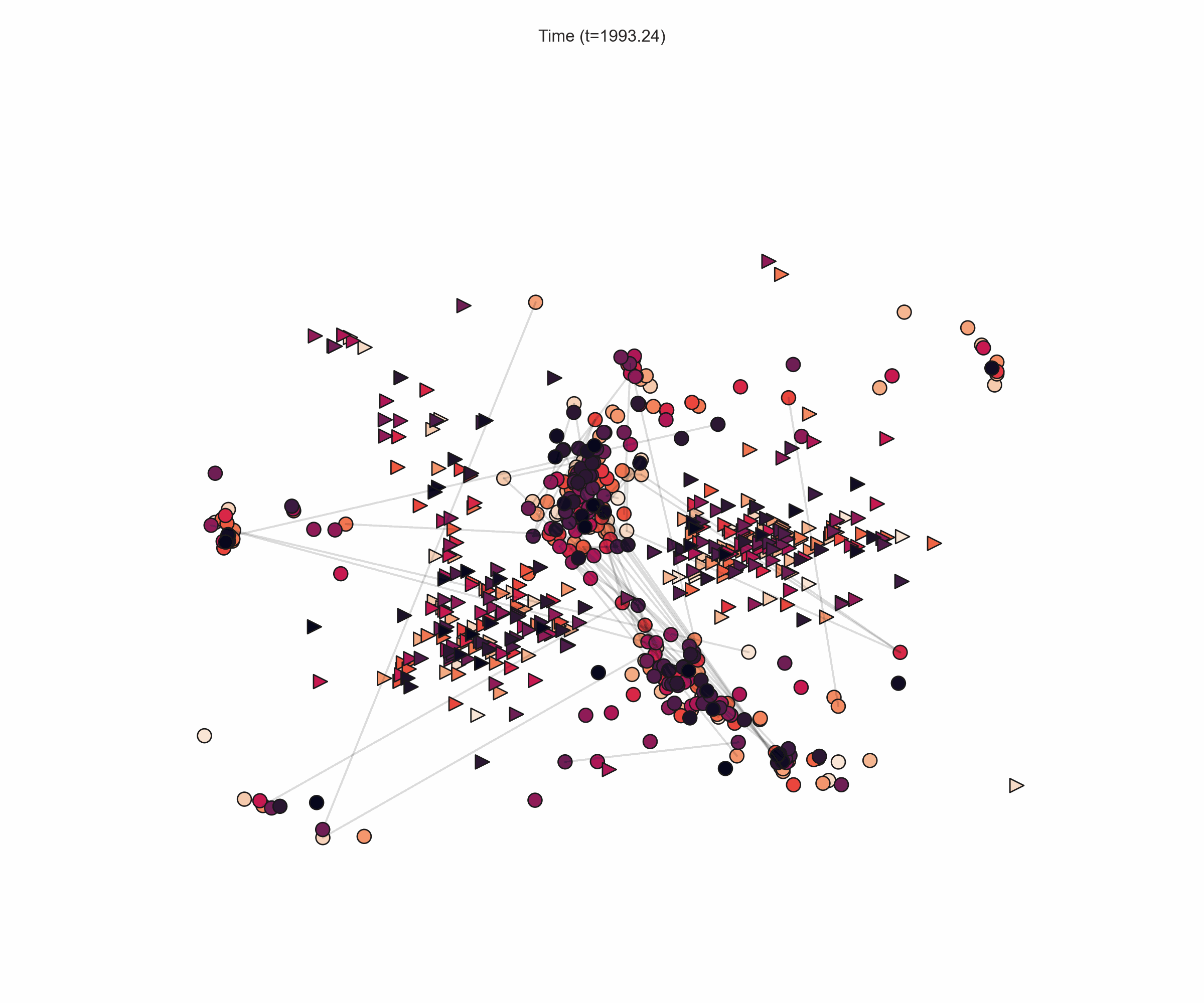}}
\hfill
\subfigure[$t=1993.85$]{\includegraphics[trim={5cm 6cm 5cm 6cm},clip,width=0.16\textwidth]{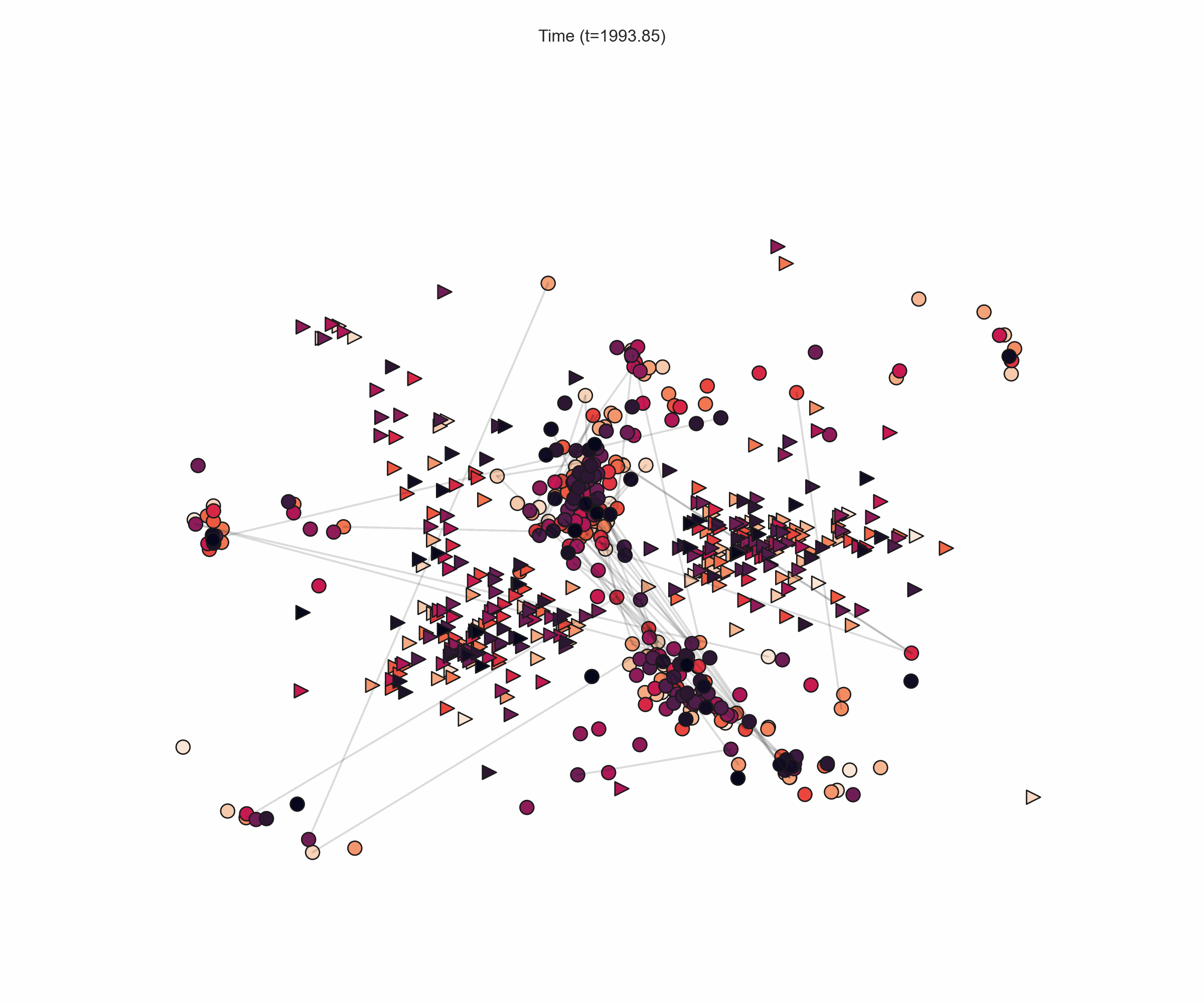}}
\hfill
\subfigure[$t=1994.45$]{\includegraphics[trim={5cm 6cm 5cm 6cm},clip,width=0.16\textwidth]{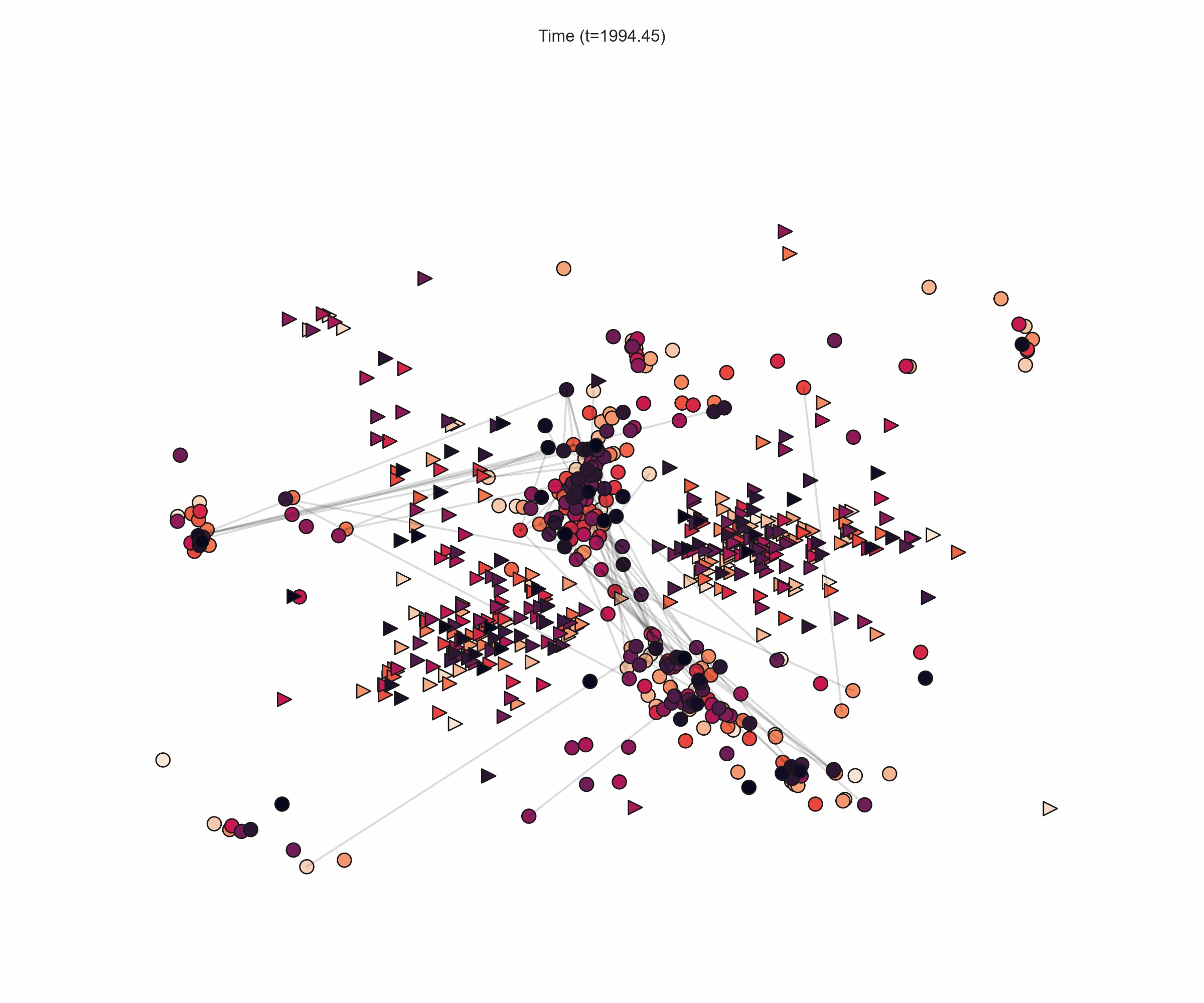}}
\hfill
\subfigure[$t=1995.06$]{\includegraphics[trim={5cm 6cm 5cm 6cm},clip,width=0.16\textwidth]{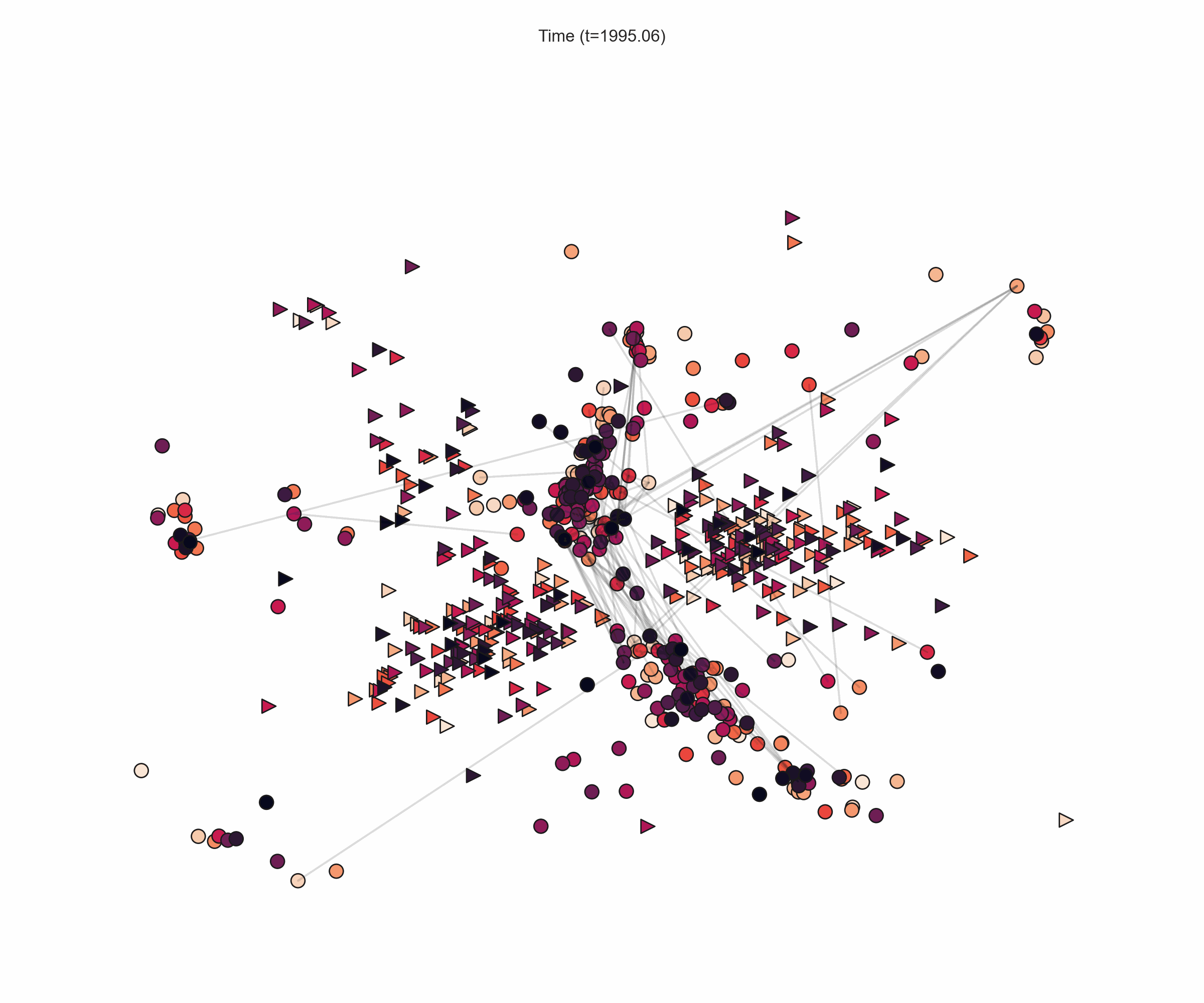}}
\hfill
\subfigure[$t=1995.67$]{\includegraphics[trim={5cm 6cm 5cm 6cm},clip,width=0.16\textwidth]{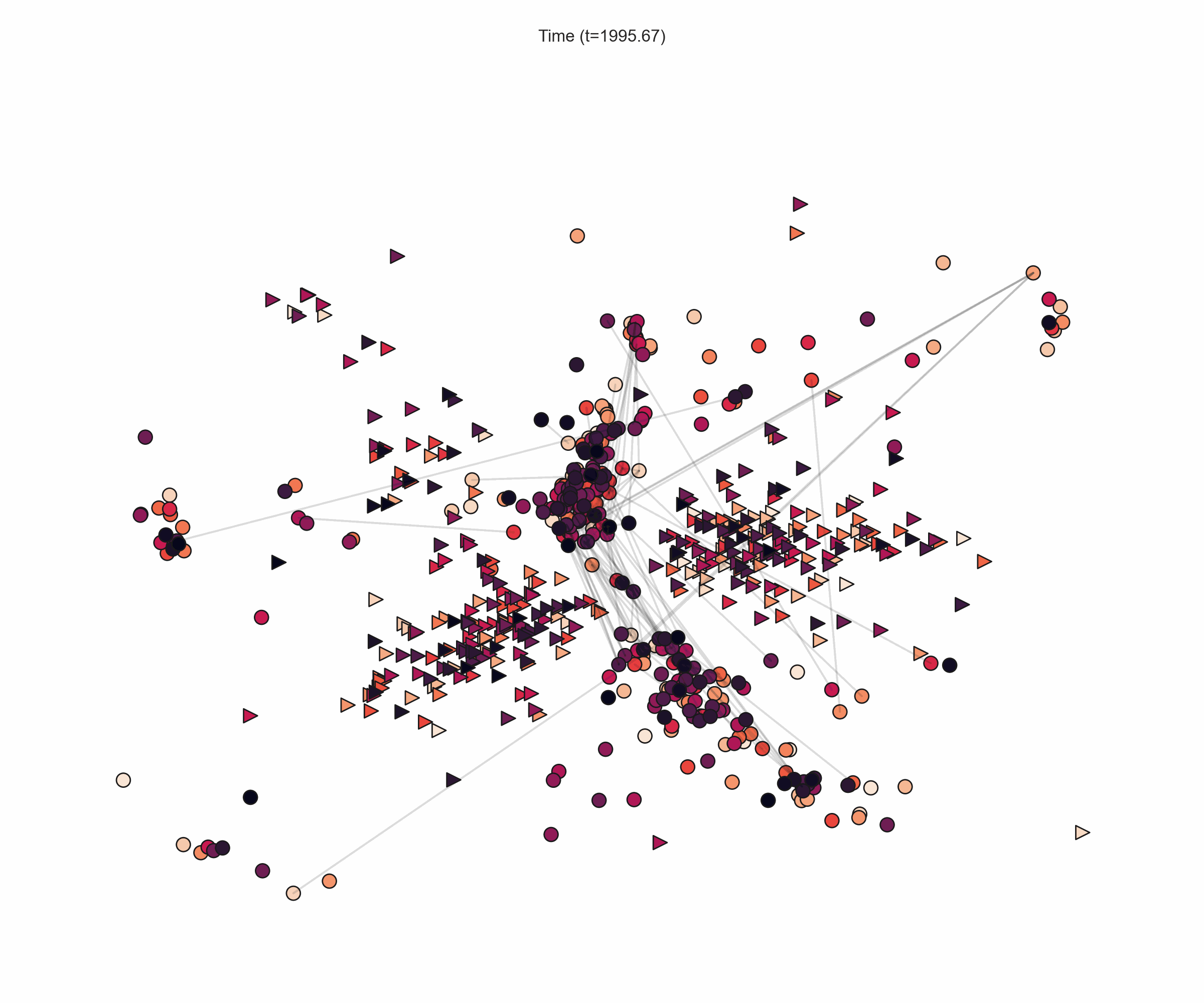}}
\hfill
\subfigure[$t=1996.27$]{\includegraphics[trim={5cm 6cm 5cm 6cm},clip,width=0.16\textwidth]{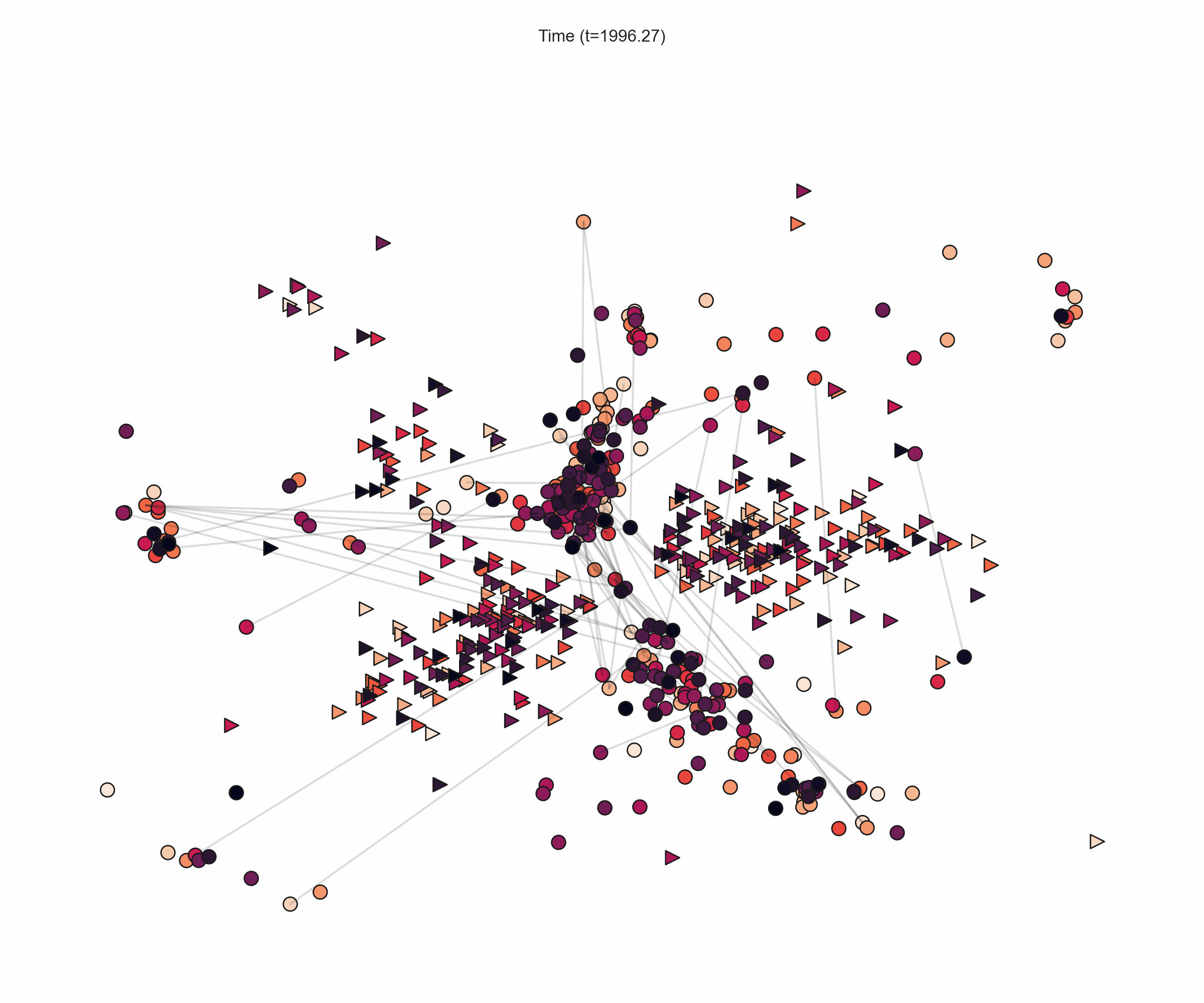}}
%%%%%%%%
\subfigure[$t=1996.88$]{\includegraphics[trim={5cm 6cm 5cm 6cm},clip,width=0.16\textwidth]{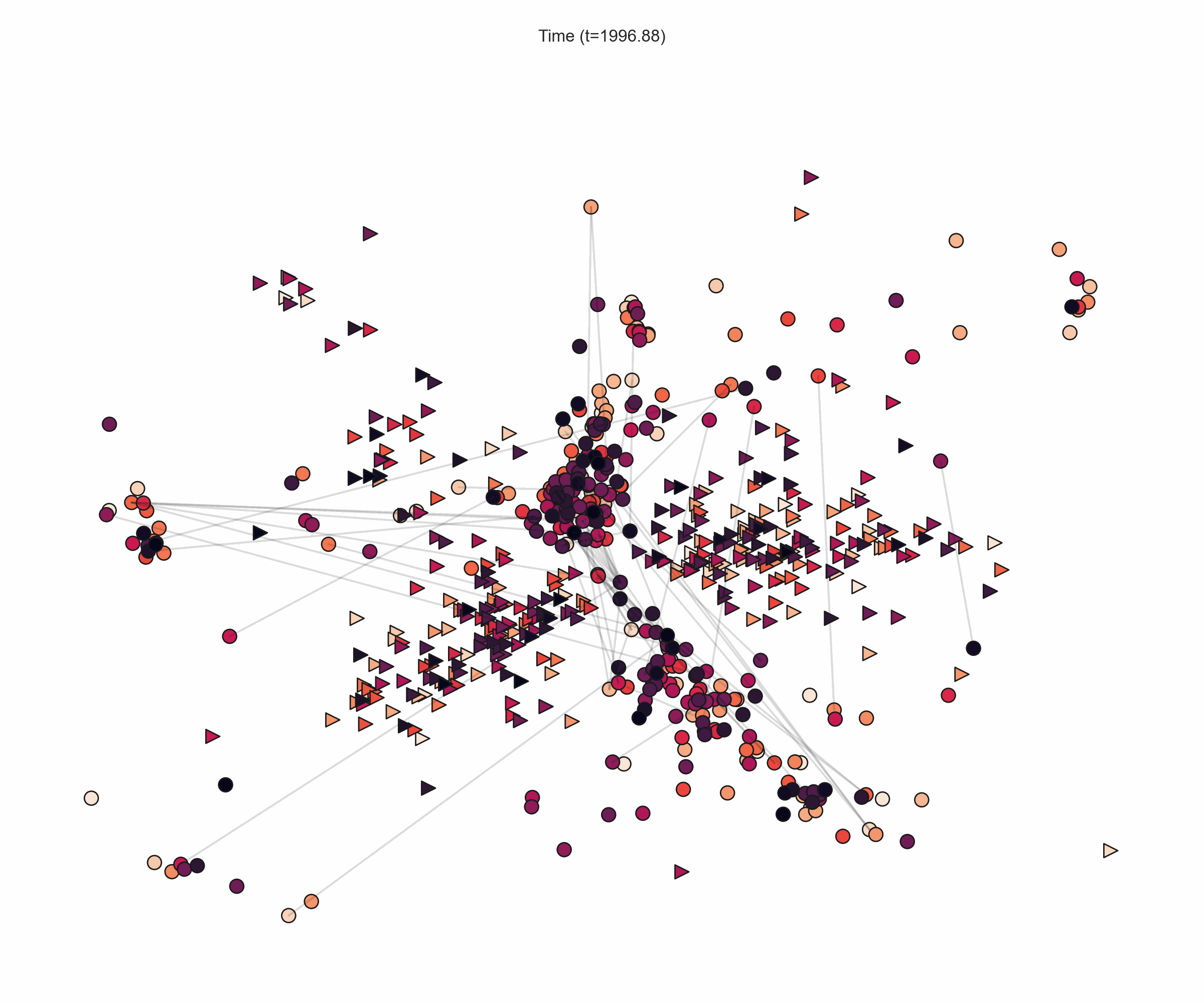}}
\hfill
\subfigure[$t=1997.48$]{\includegraphics[trim={5cm 6cm 5cm 6cm},clip,width=0.16\textwidth]{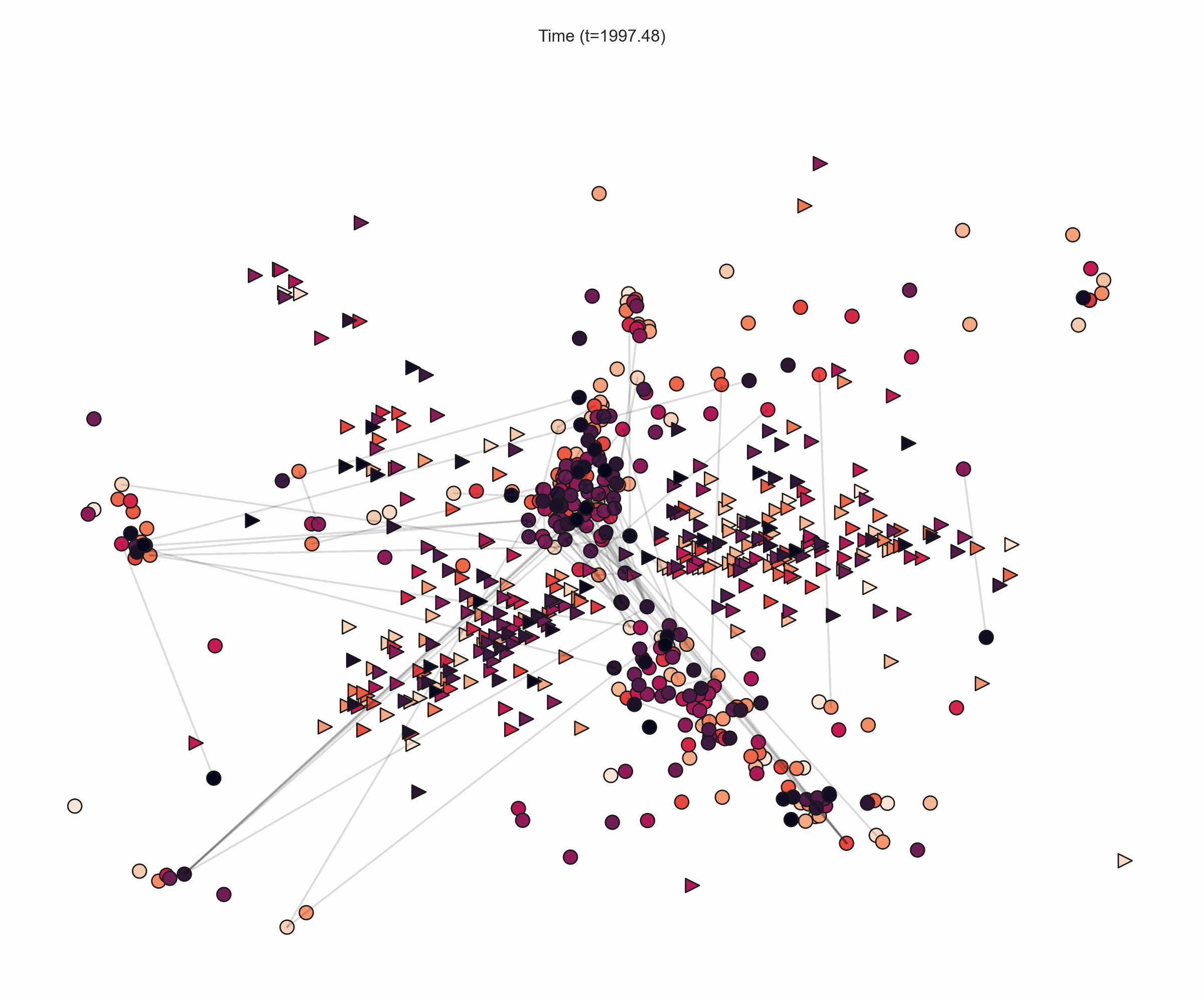}}
\hfill
\subfigure[$t=1998.09$]{\includegraphics[trim={5cm 6cm 5cm 6cm},clip,width=0.16\textwidth]{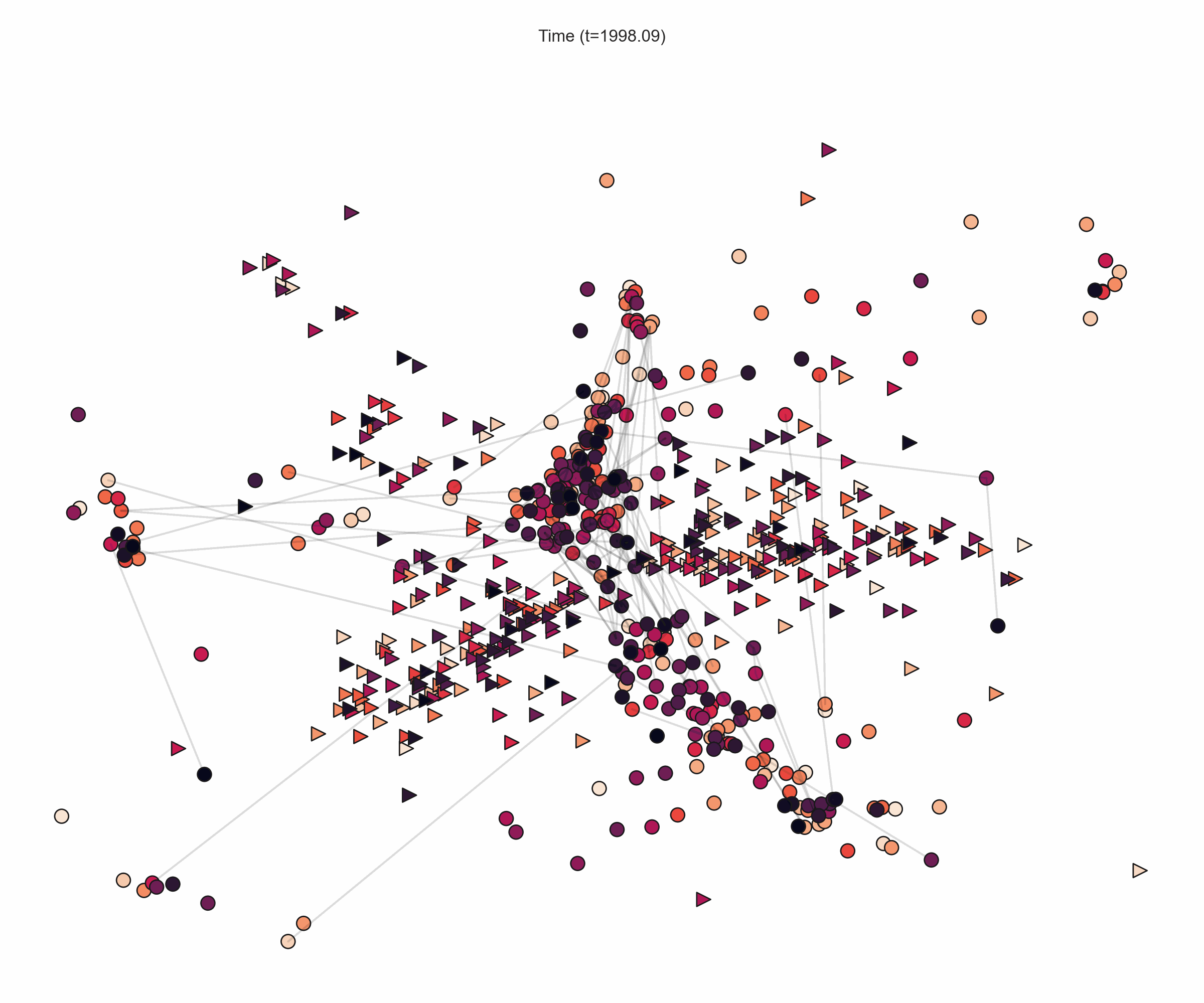}}
\hfill
\subfigure[$t=1998.70$]{\includegraphics[trim={5cm 6cm 5cm 6cm},clip,width=0.16\textwidth]{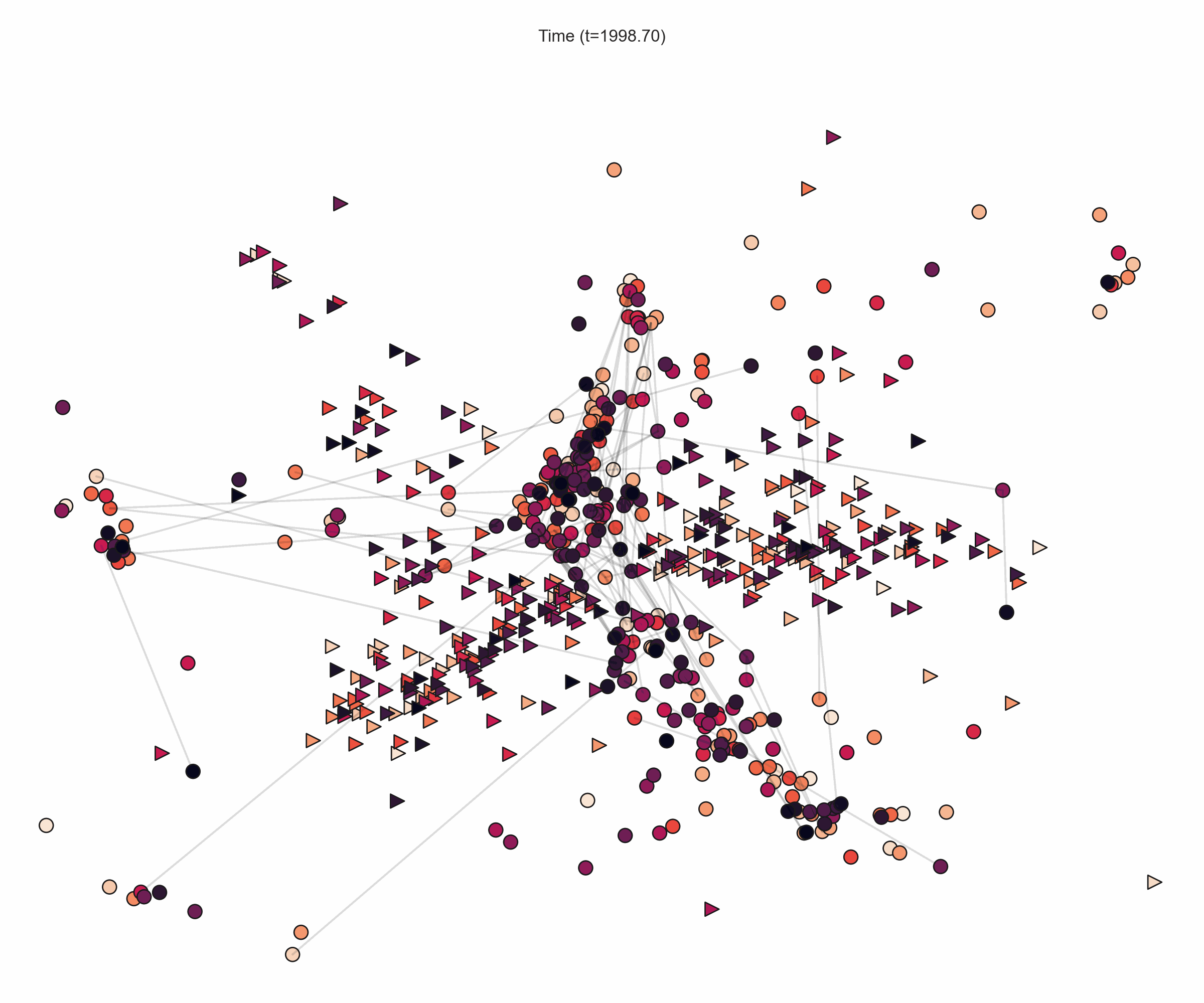}}
\hfill
\subfigure[$t=1999.30$]{\includegraphics[trim={5cm 6cm 5cm 6cm},clip,width=0.16\textwidth]{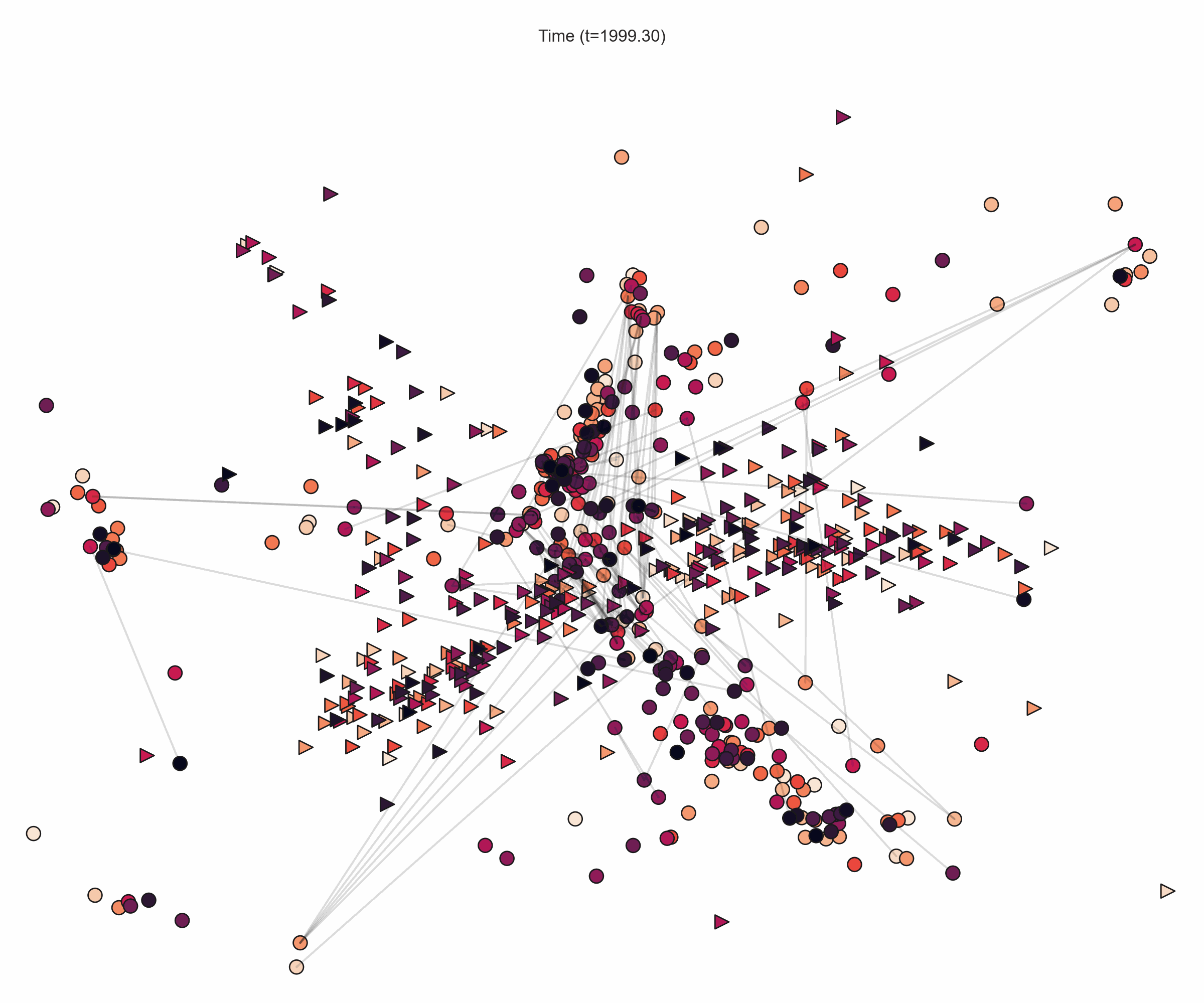}}
\hfill
\subfigure[$t=1991.91$]{\includegraphics[trim={5cm 6cm 5cm 6cm},clip,width=0.16\textwidth]{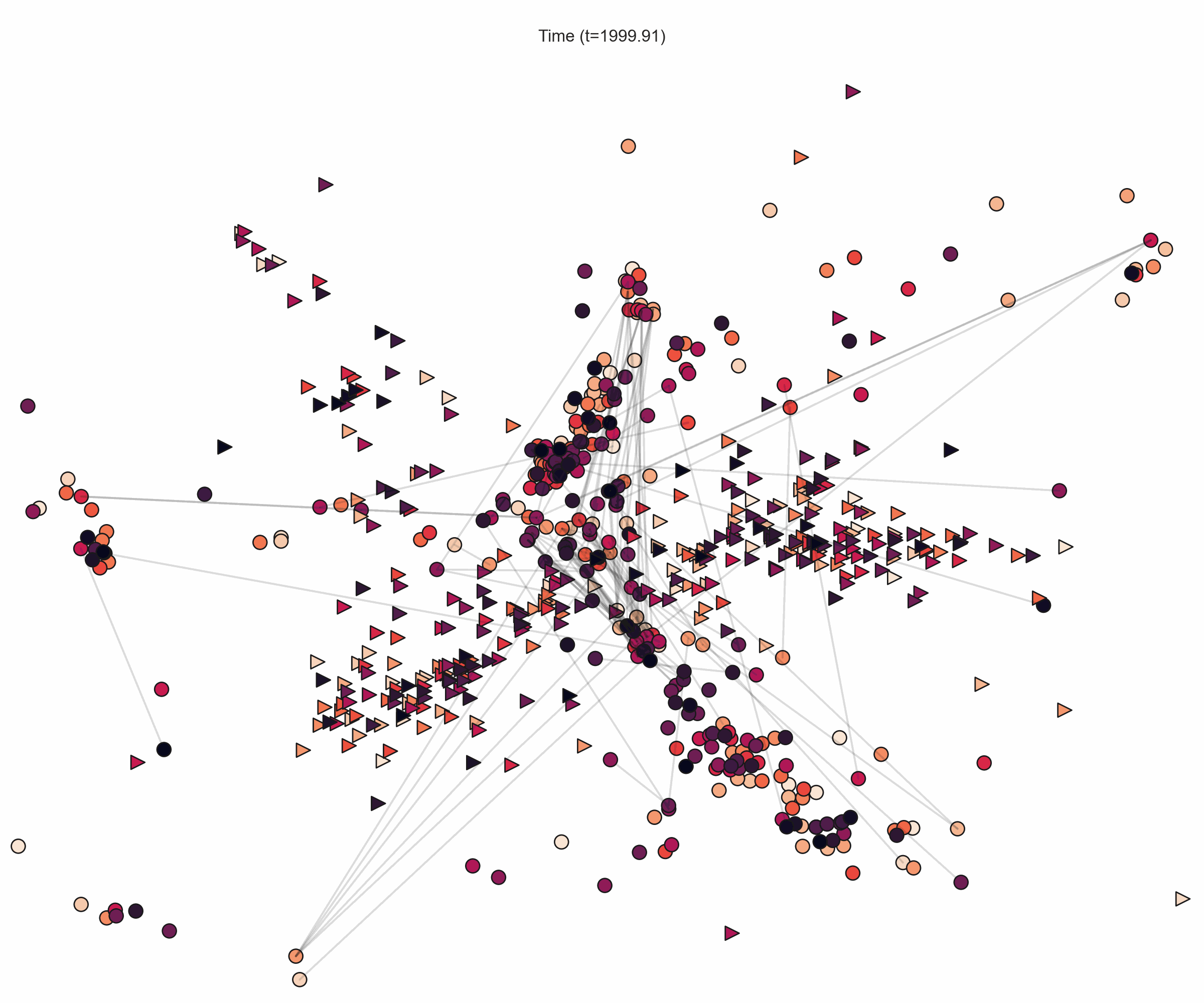}}
\caption{Snapshots of the continuous-time embeddings learned by \textsc{\modelname} for various time points over \textsl{NeurIPS}.}\label{fig:appendix_visualization_neurips}
\end{figure*}
%%%%%%%%%%%%%
\begin{figure*} %[!ht]
\centering
\subfigure[$t=(t=1435$]{\includegraphics[trim={5cm 6cm 5cm 6cm},clip,width=0.16\textwidth]{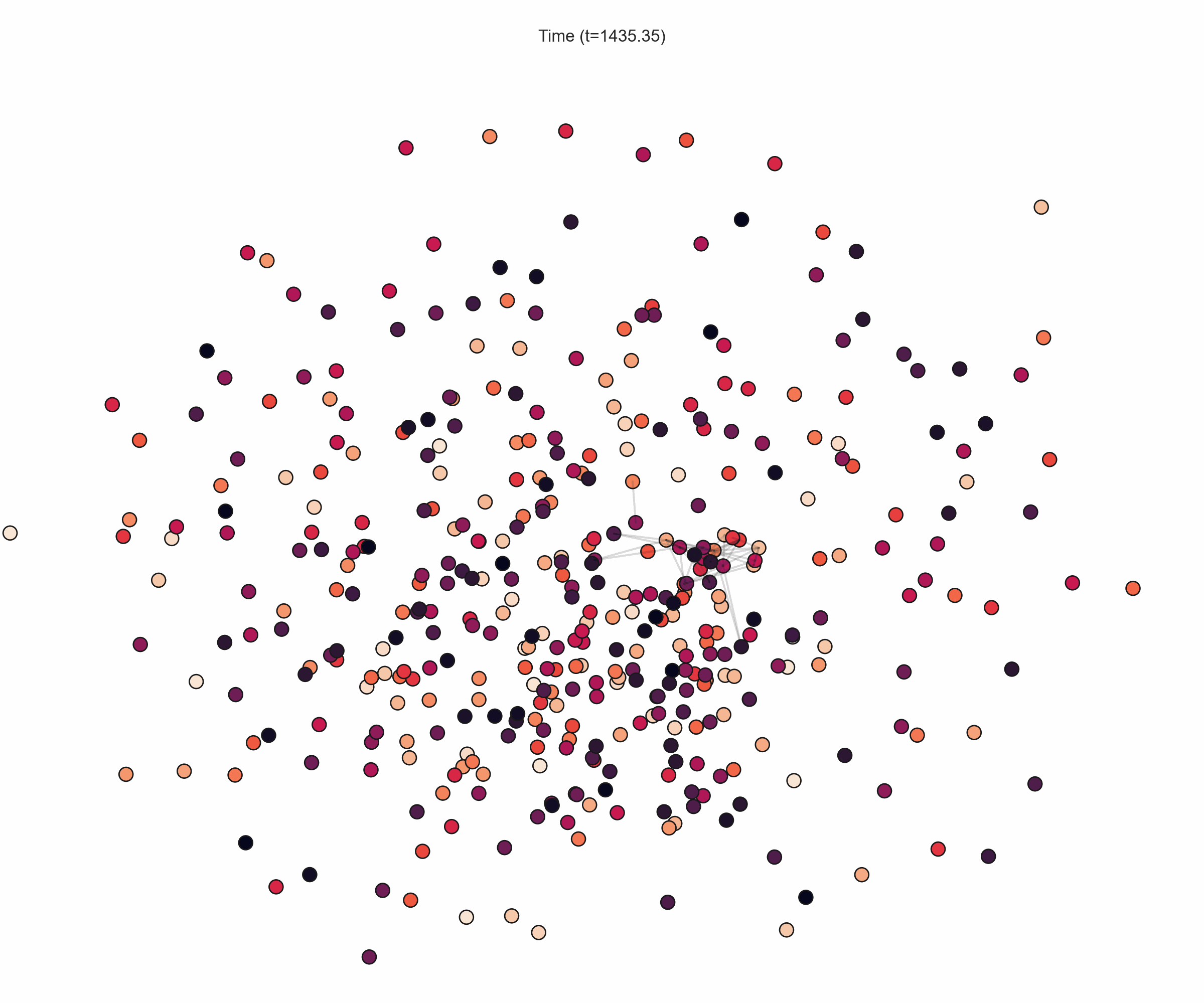}}
\hfill
\subfigure[$t=2870$]{\includegraphics[trim={5cm 6cm 5cm 6cm},clip,width=0.16\textwidth]{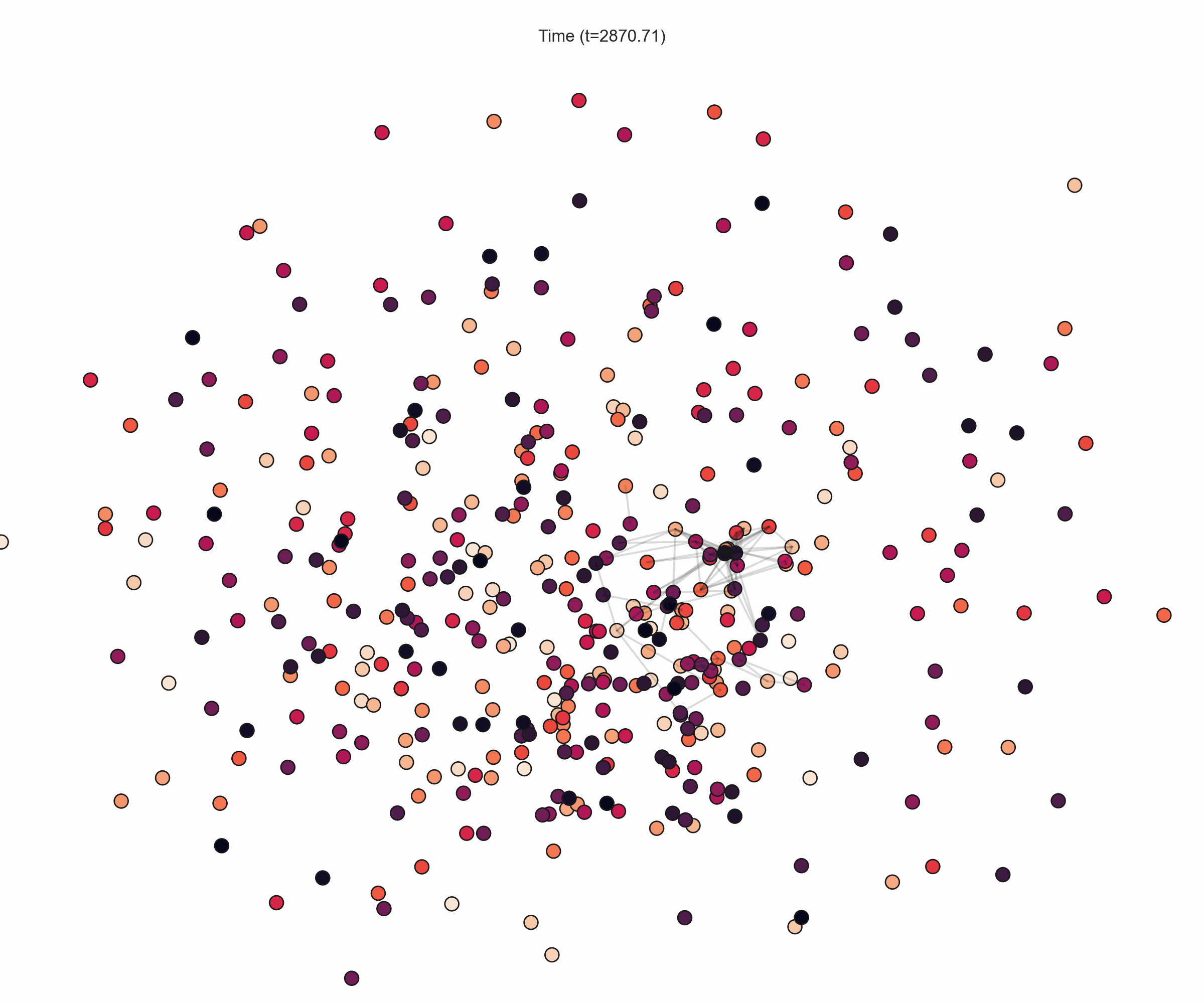}}
\hfill
\subfigure[$t=4306$]{\includegraphics[trim={5cm 6cm 5cm 6cm},clip,width=0.16\textwidth]{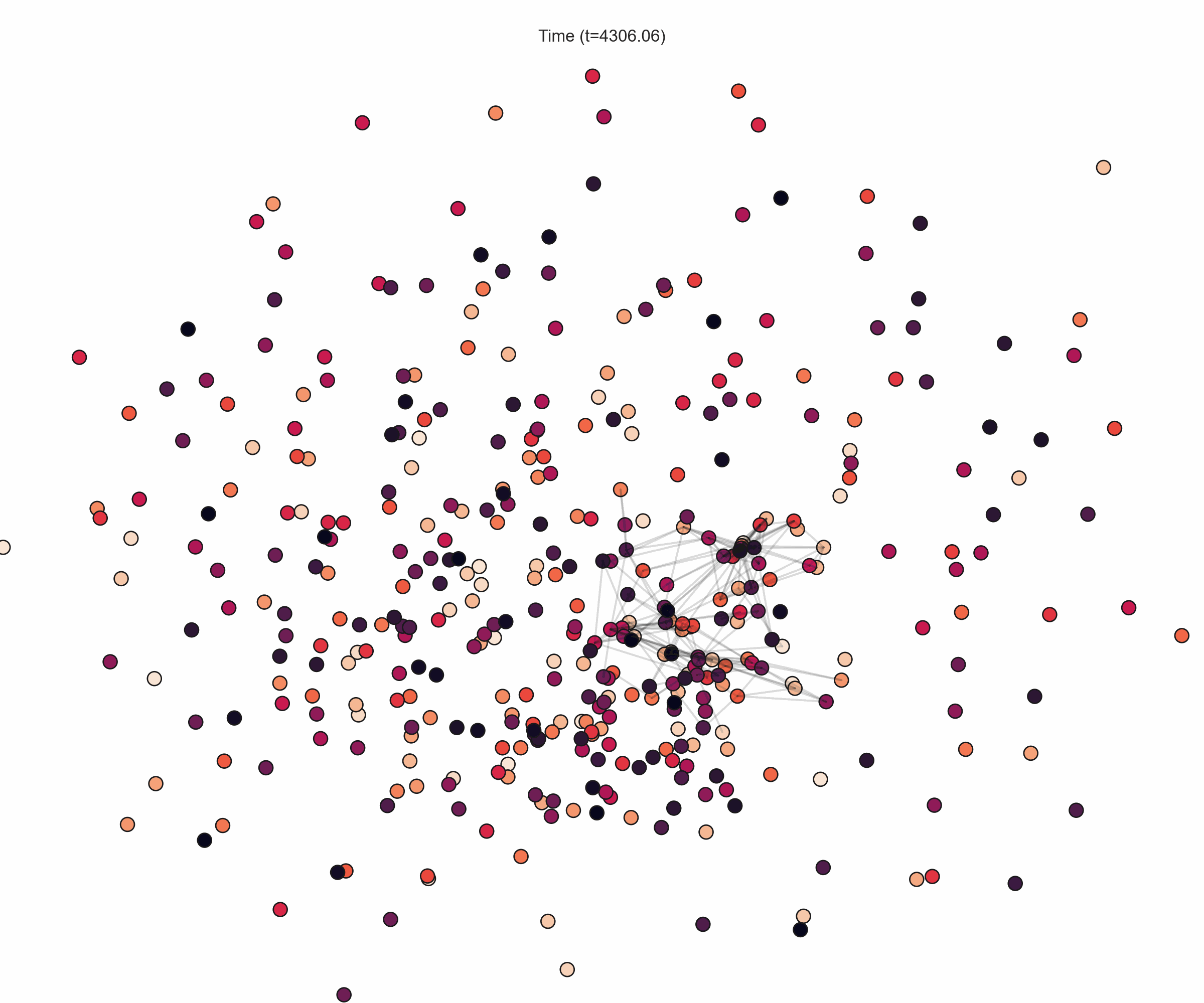}}
\hfill
\subfigure[$t=5741$]{\includegraphics[trim={5cm 6cm 5cm 6cm},clip,width=0.16\textwidth]{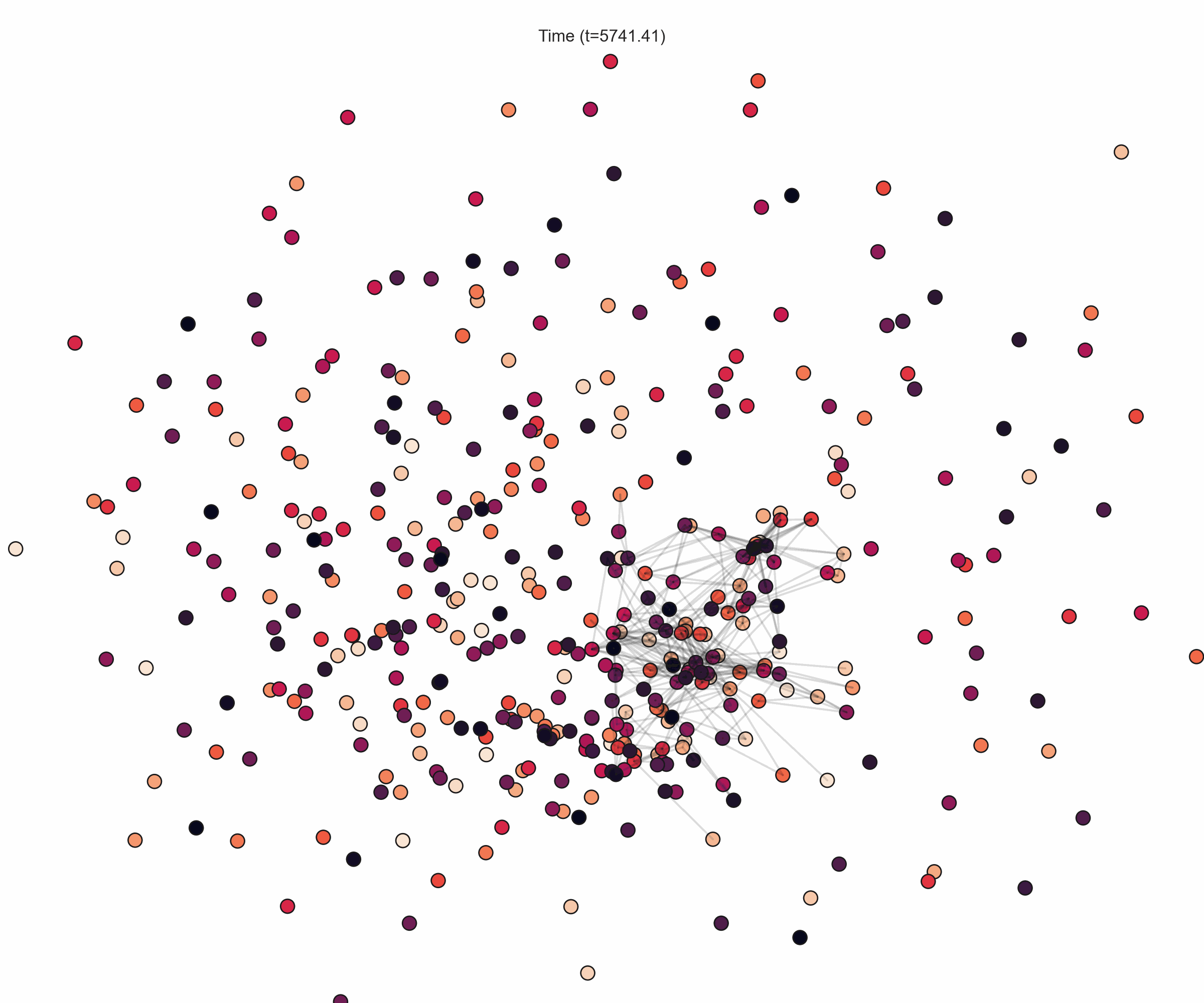}}
\hfill
\subfigure[$t=7176$]{\includegraphics[trim={5cm 6cm 5cm 6cm},clip,width=0.16\textwidth]{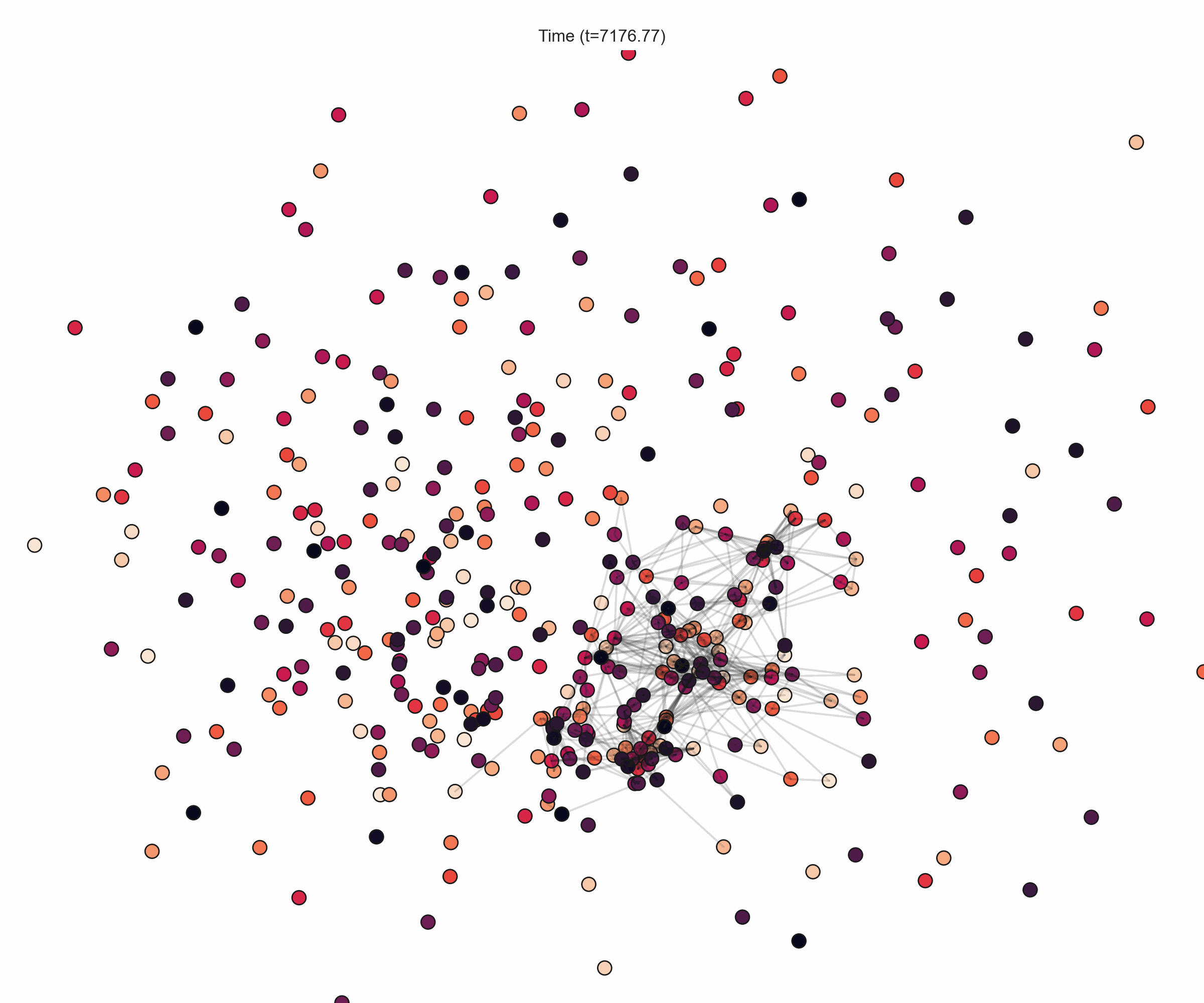}}
\hfill
\subfigure[$t=8612$]{\includegraphics[trim={5cm 6cm 5cm 6cm},clip,width=0.16\textwidth]{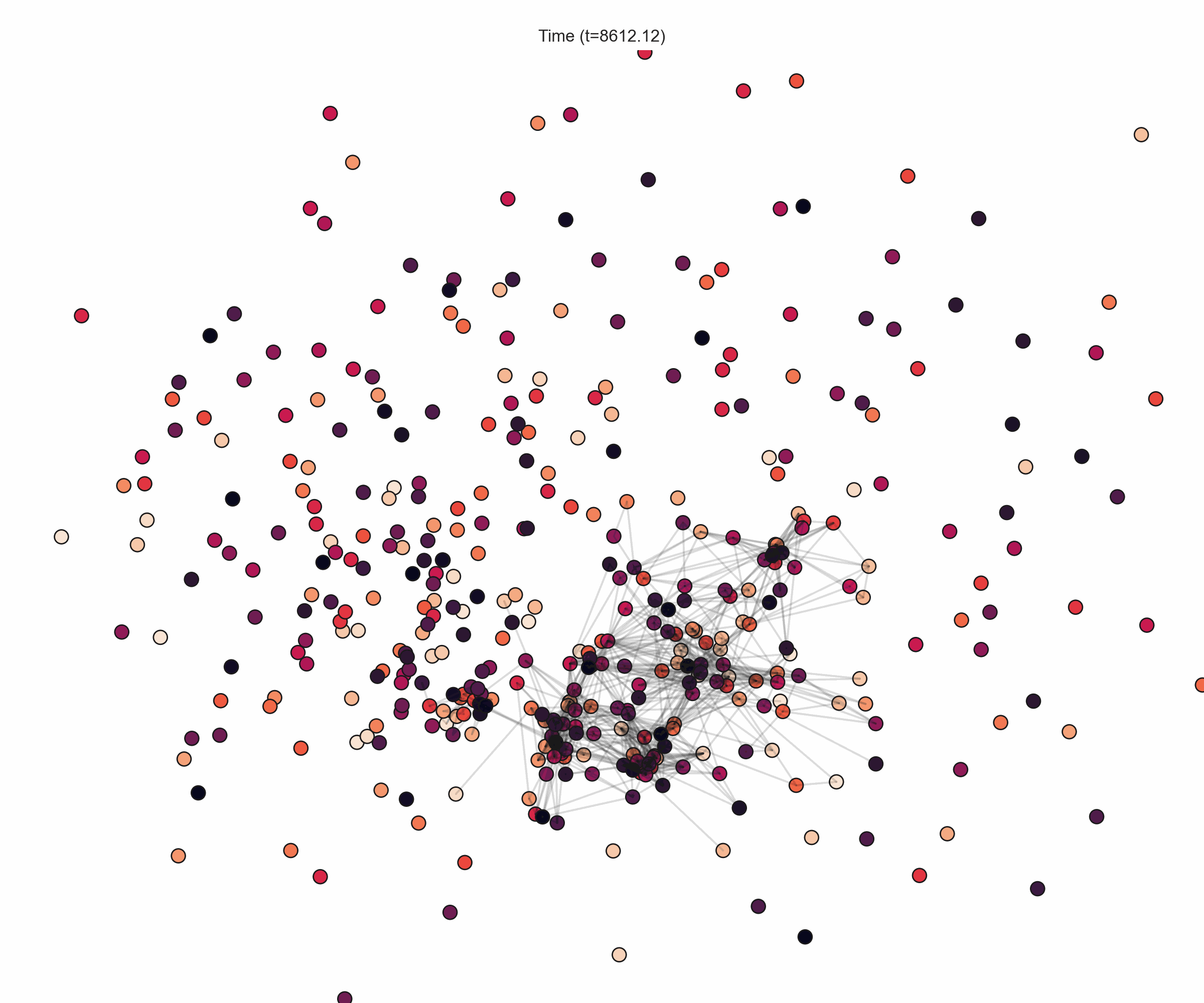}}
%%%%%%%%
\subfigure[$t=10047$]{\includegraphics[trim={5cm 6cm 5cm 6cm},clip,width=0.16\textwidth]{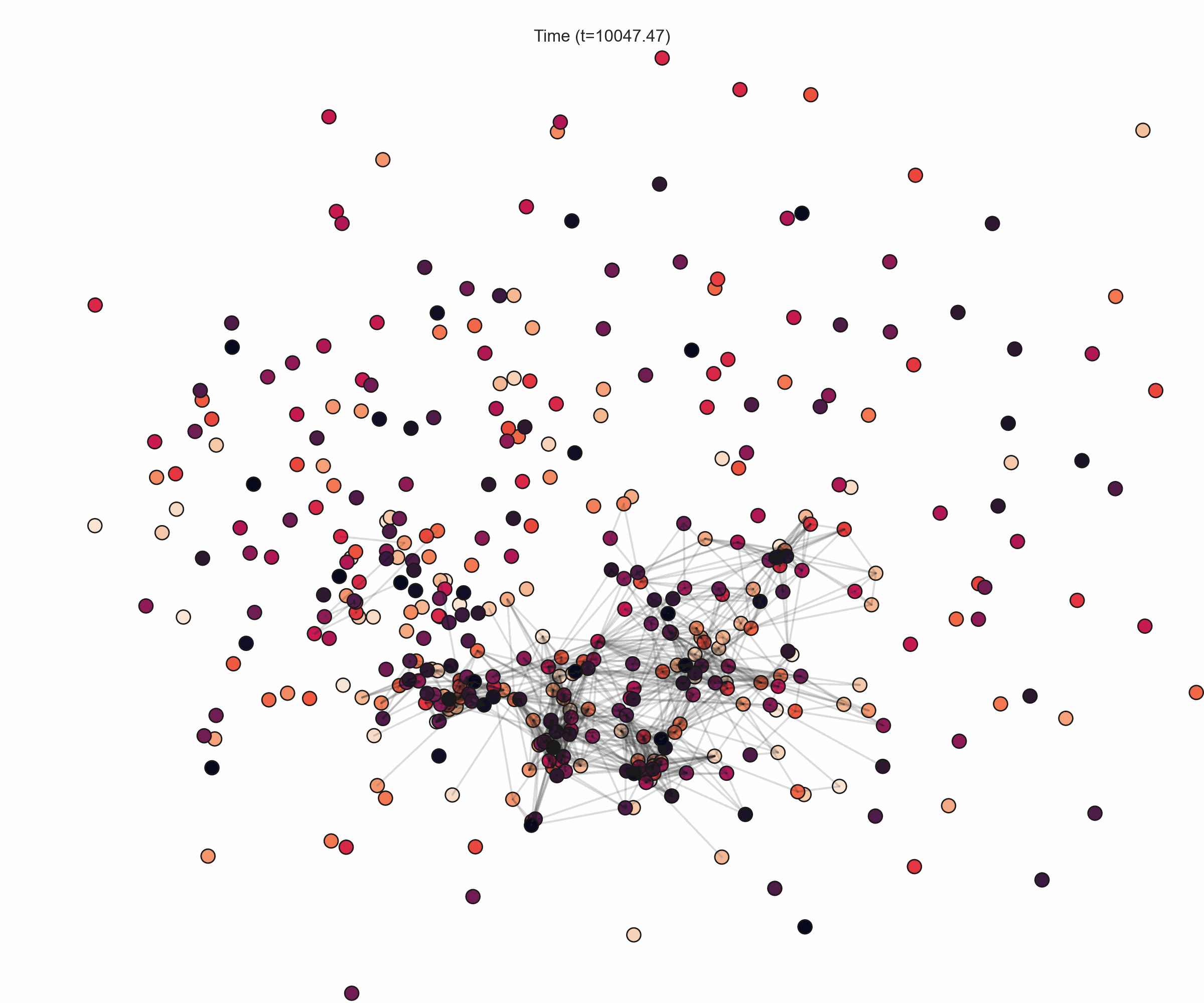}}
\hfill
\subfigure[$t=11482$]{\includegraphics[trim={5cm 6cm 5cm 6cm},clip,width=0.16\textwidth]{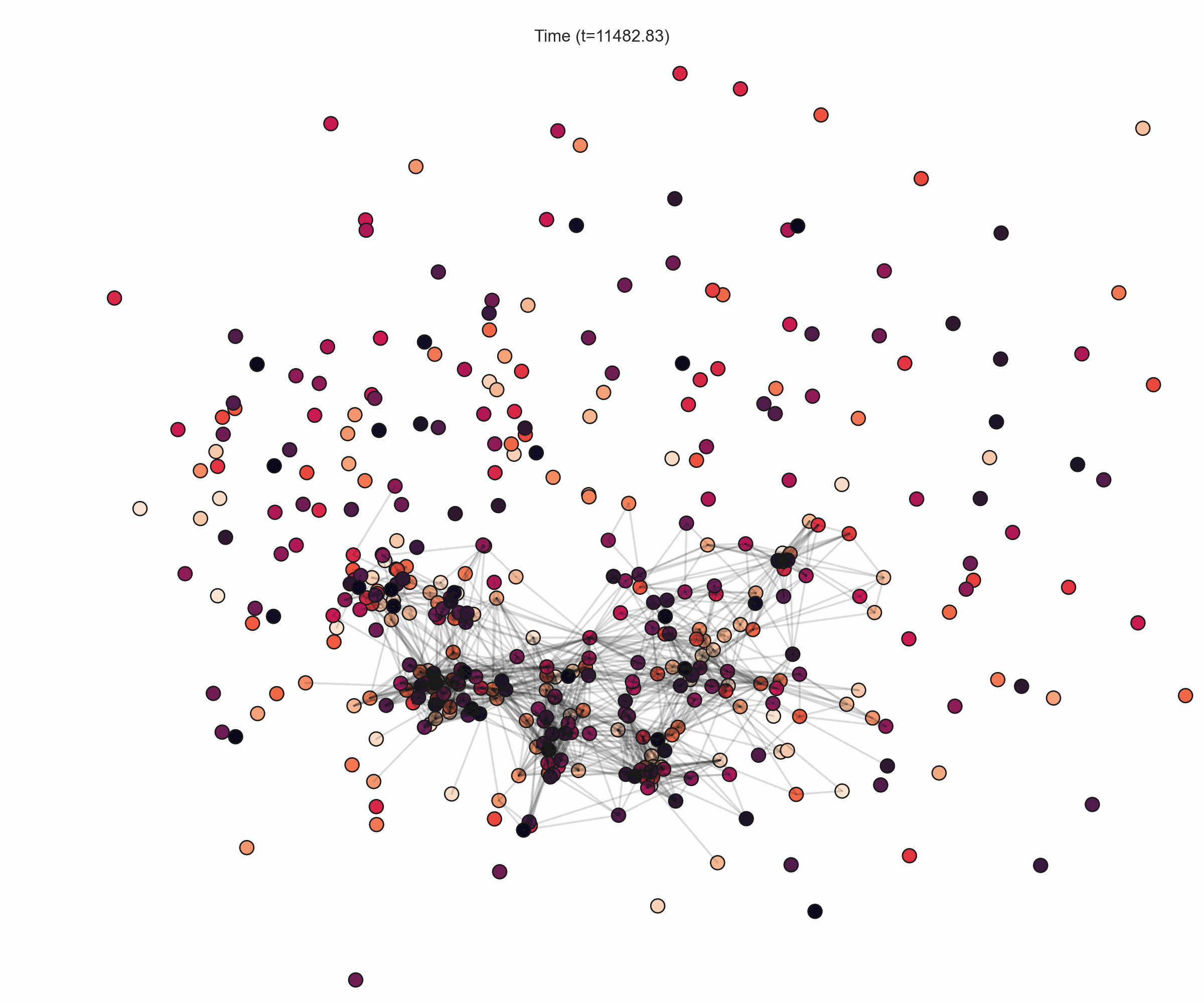}}
\hfill
\subfigure[$t=12918$]{\includegraphics[trim={5cm 6cm 5cm 6cm},clip,width=0.16\textwidth]{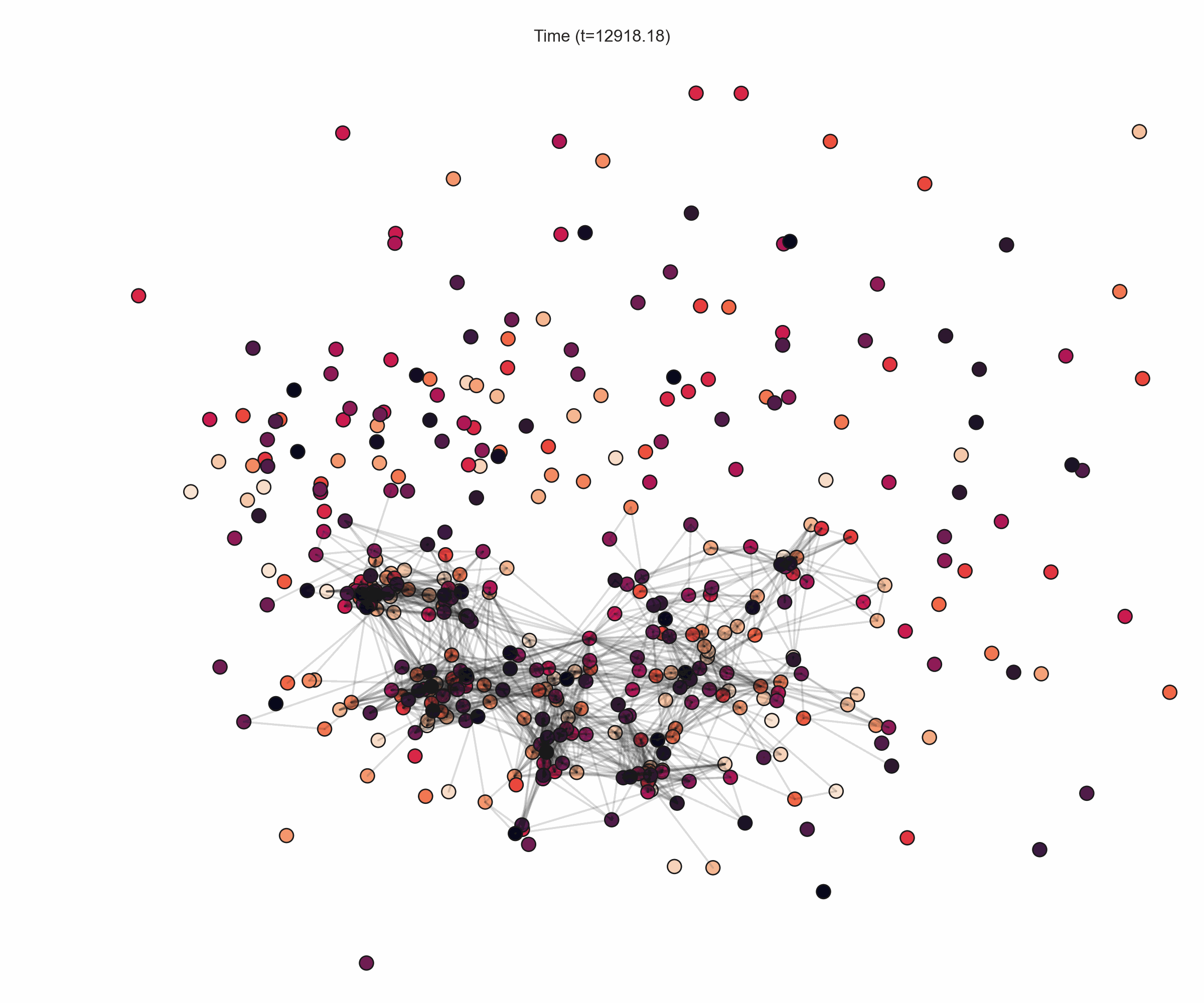}}
\hfill
\subfigure[$t=14353$]{\includegraphics[trim={5cm 6cm 5cm 6cm},clip,width=0.16\textwidth]{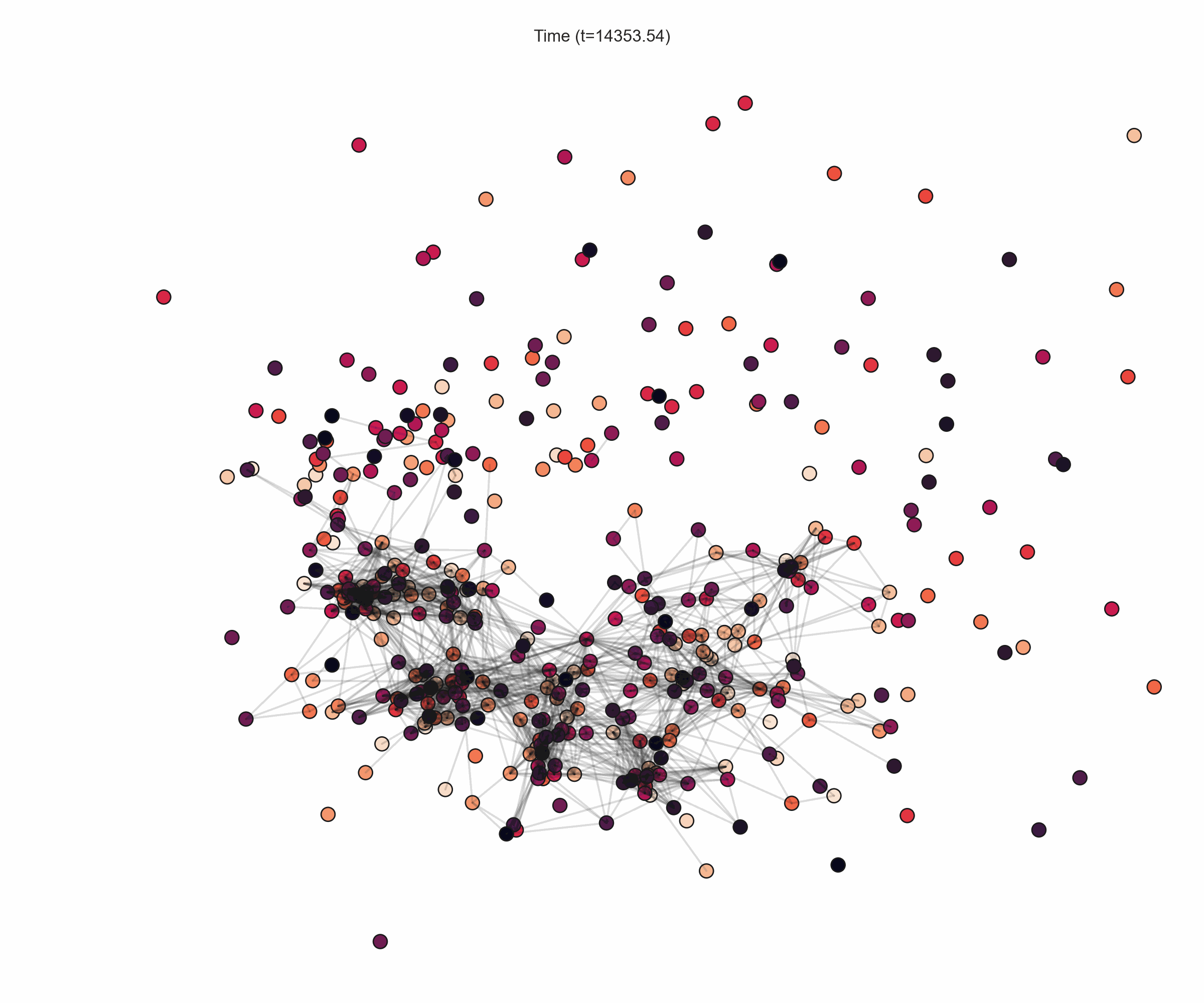}}
\hfill
\subfigure[$t=15788$]{\includegraphics[trim={5cm 6cm 5cm 6cm},clip,width=0.16\textwidth]{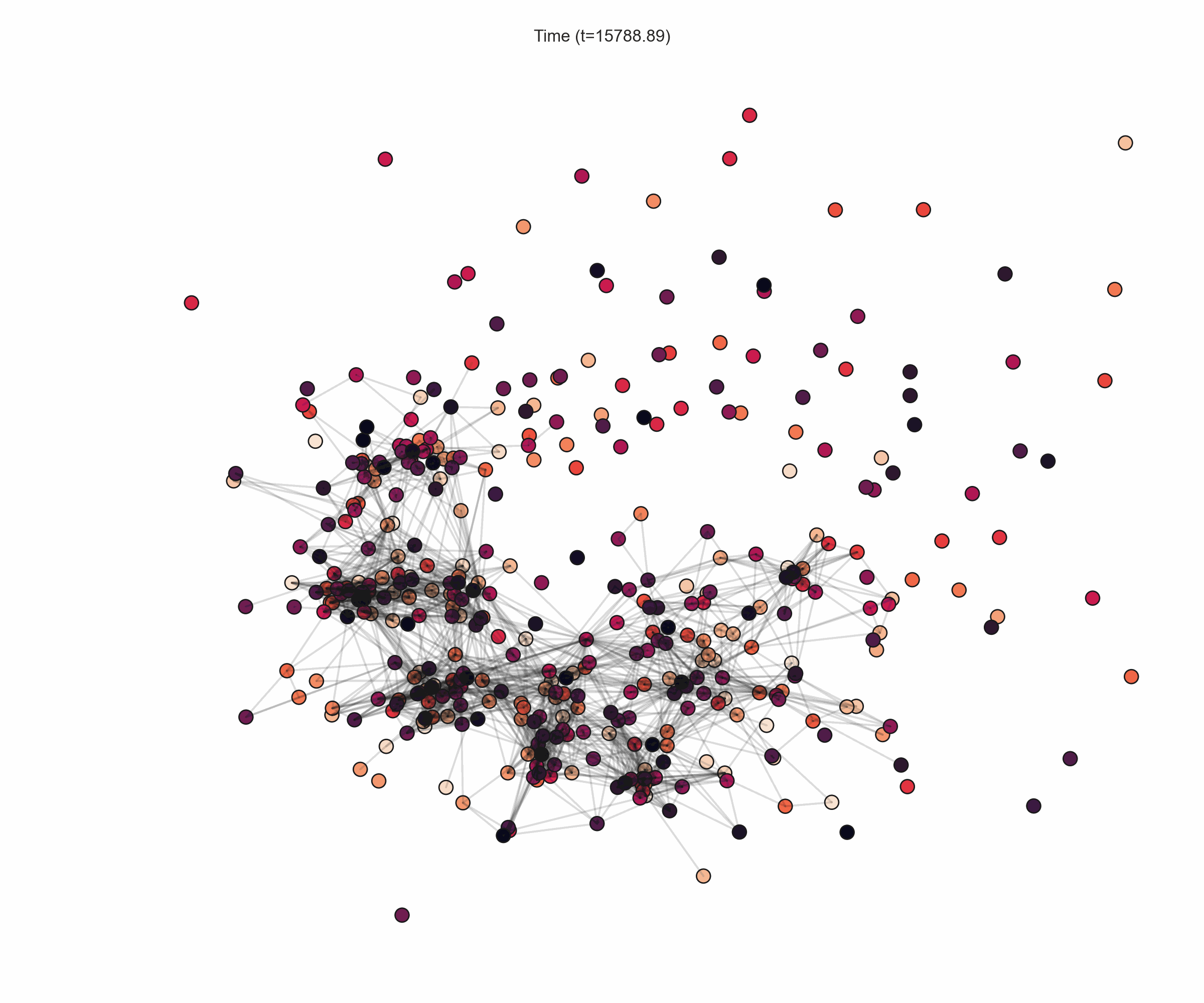}}
\hfill
\subfigure[$t=17224$]{\includegraphics[trim={5cm 6cm 5cm 6cm},clip,width=0.16\textwidth]{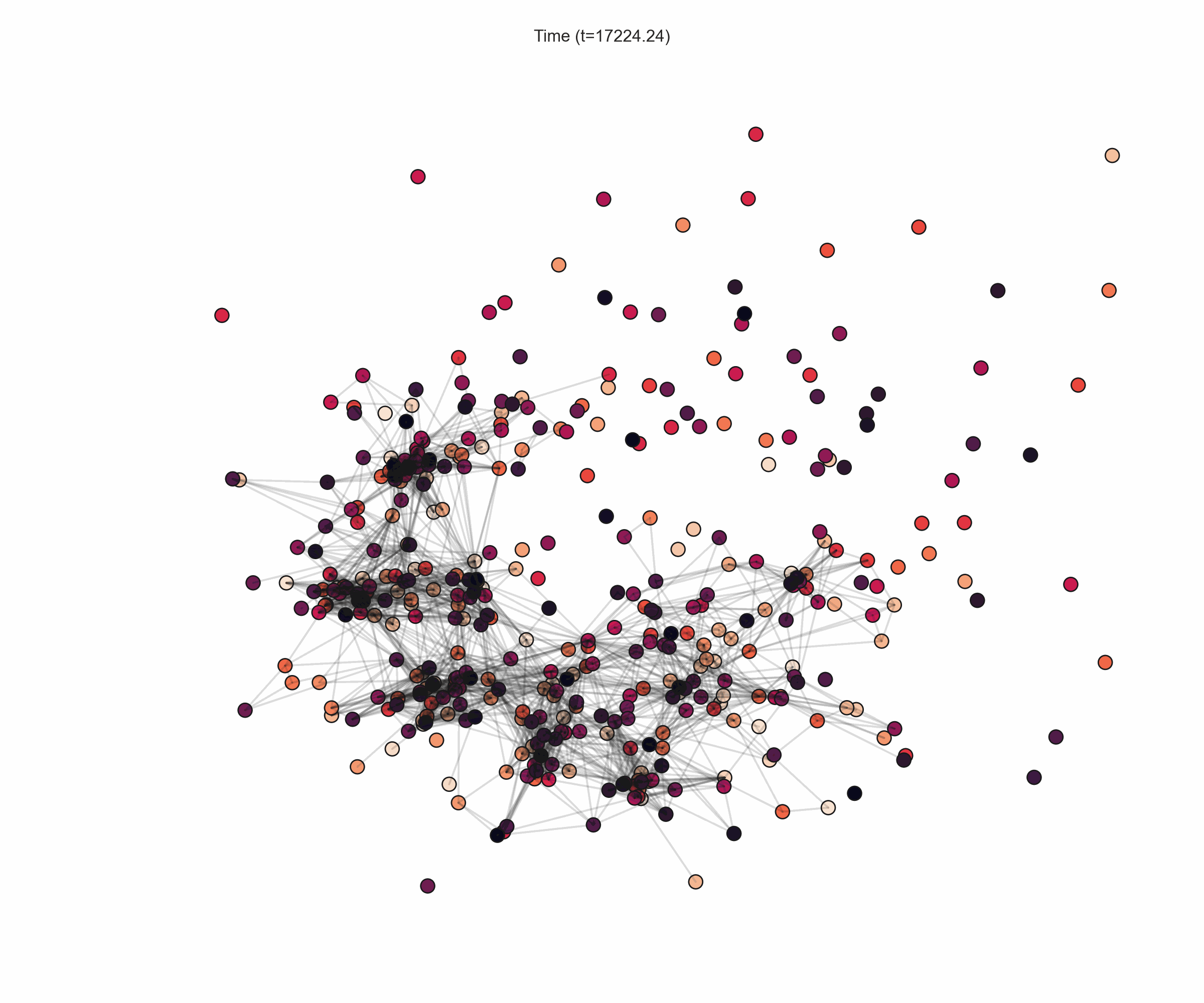}}
%%%%%%%%
\subfigure[$t=18659$]{\includegraphics[trim={5cm 6cm 5cm 6cm},clip,width=0.16\textwidth]{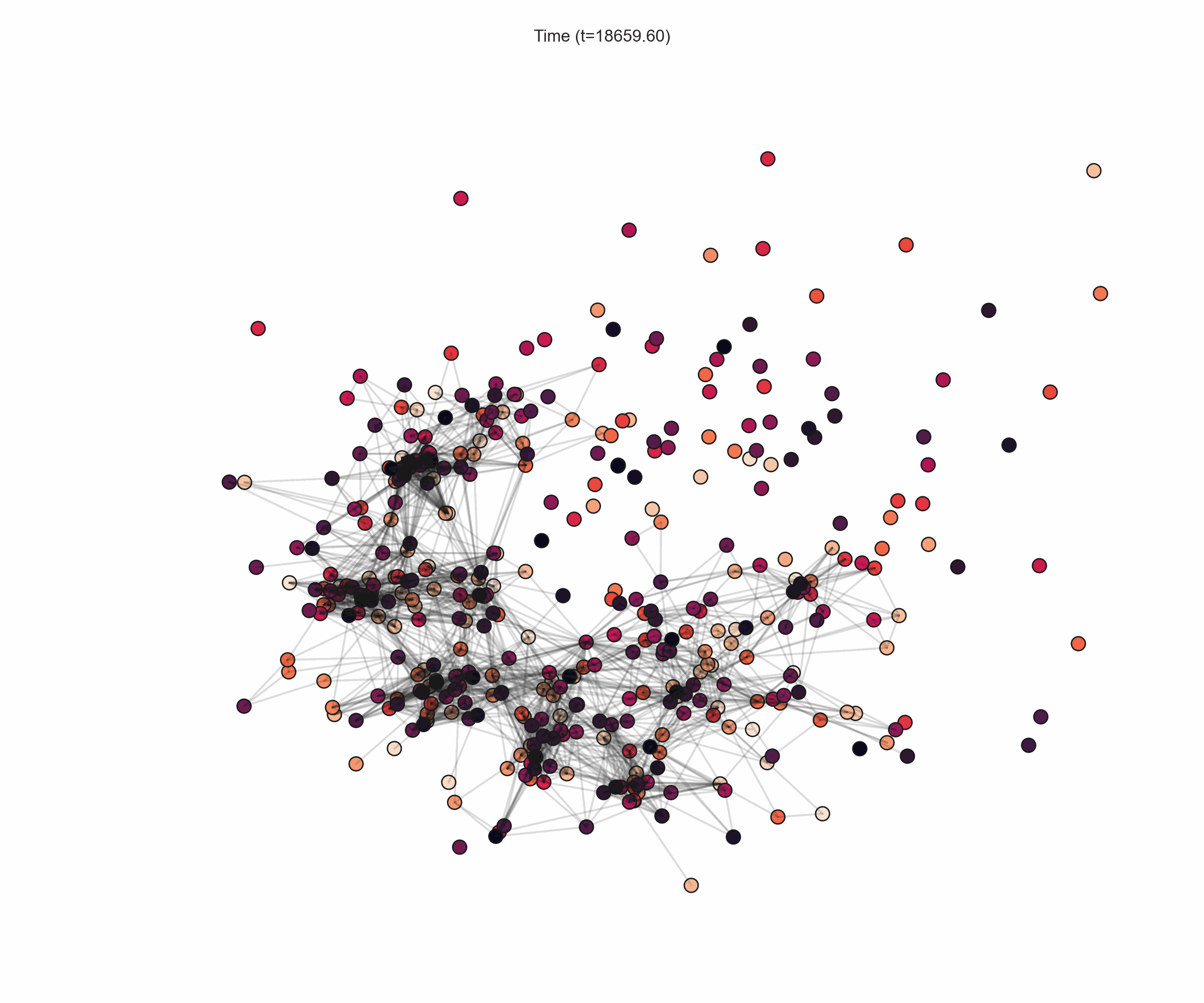}}
\hfill
\subfigure[$t=20094$]{\includegraphics[trim={5cm 6cm 5cm 6cm},clip,width=0.16\textwidth]{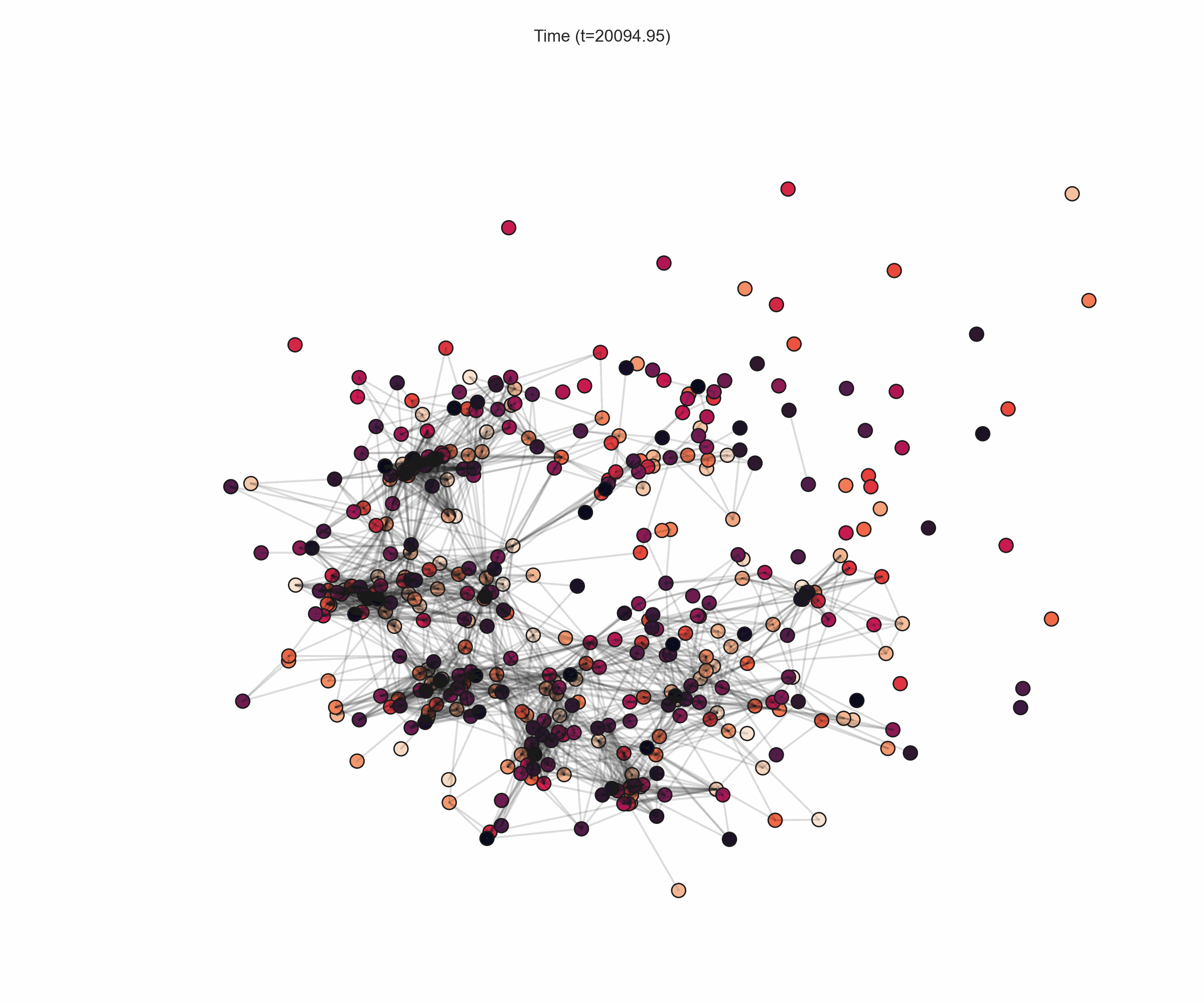}}
\hfill
\subfigure[$t=21530$]{\includegraphics[trim={5cm 6cm 5cm 6cm},clip,width=0.16\textwidth]{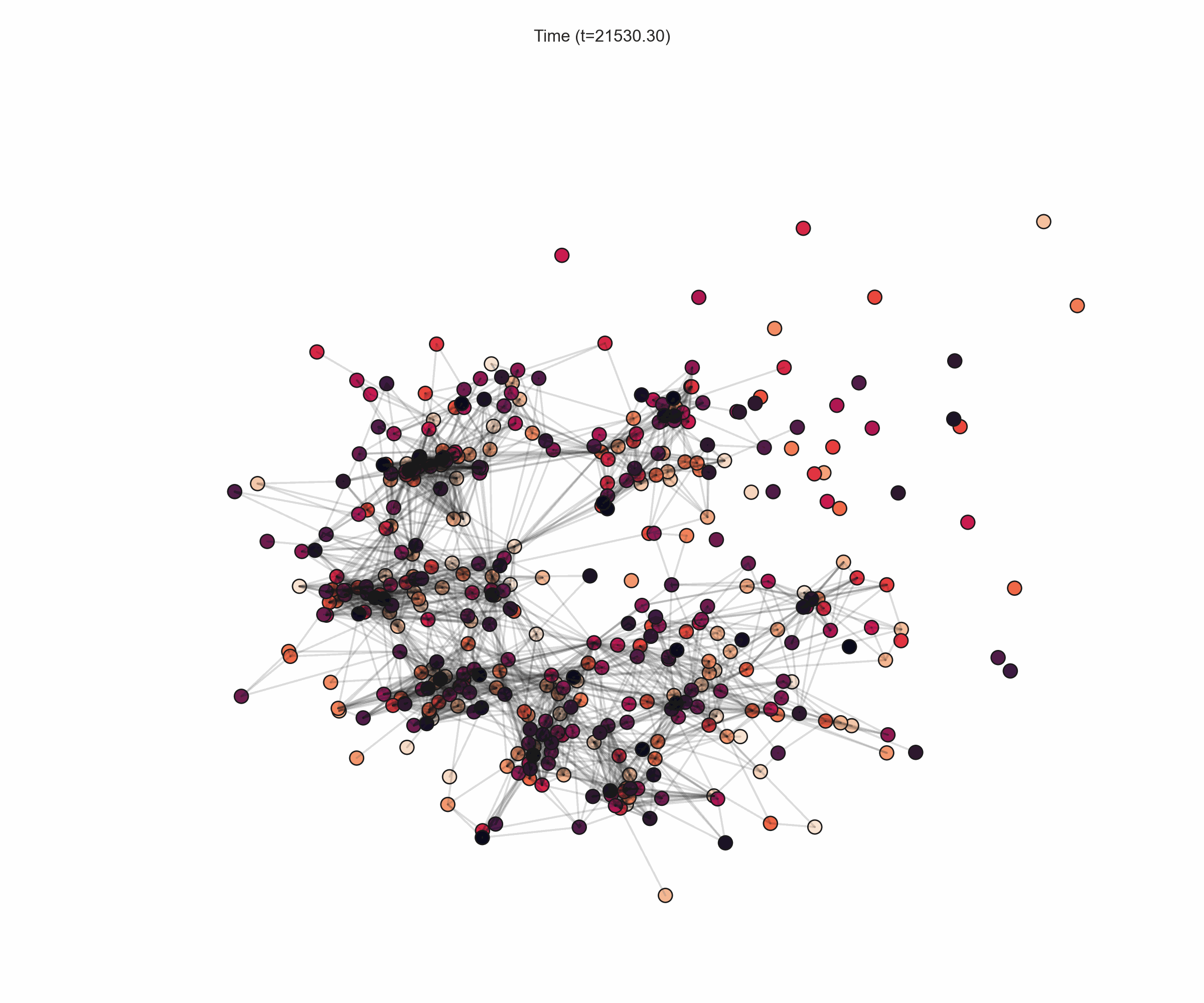}}
\hfill
\subfigure[$t=22965$]{\includegraphics[trim={5cm 6cm 5cm 6cm},clip,width=0.16\textwidth]{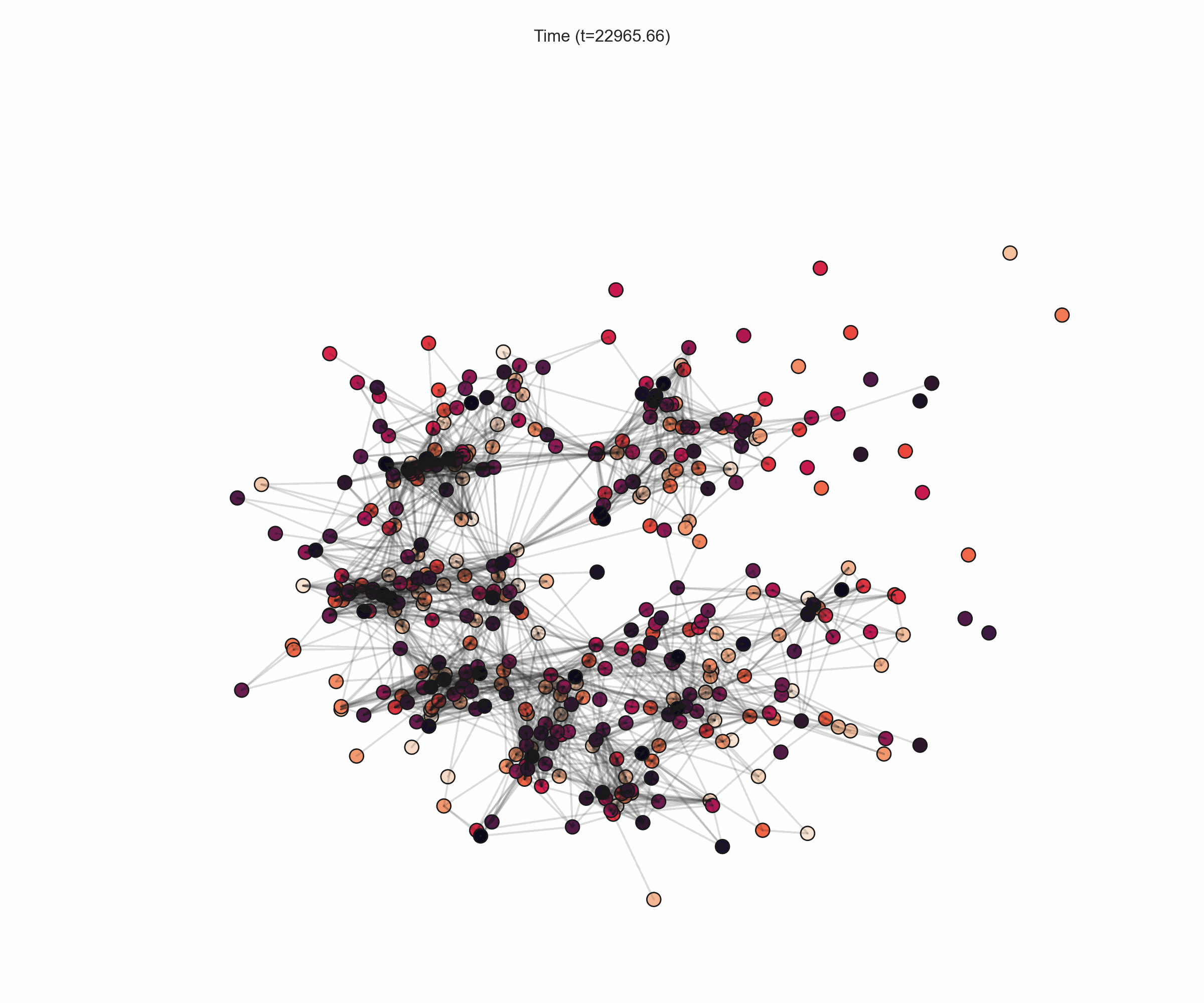}}
\hfill
\subfigure[$t=24401$]{\includegraphics[trim={5cm 6cm 5cm 6cm},clip,width=0.16\textwidth]{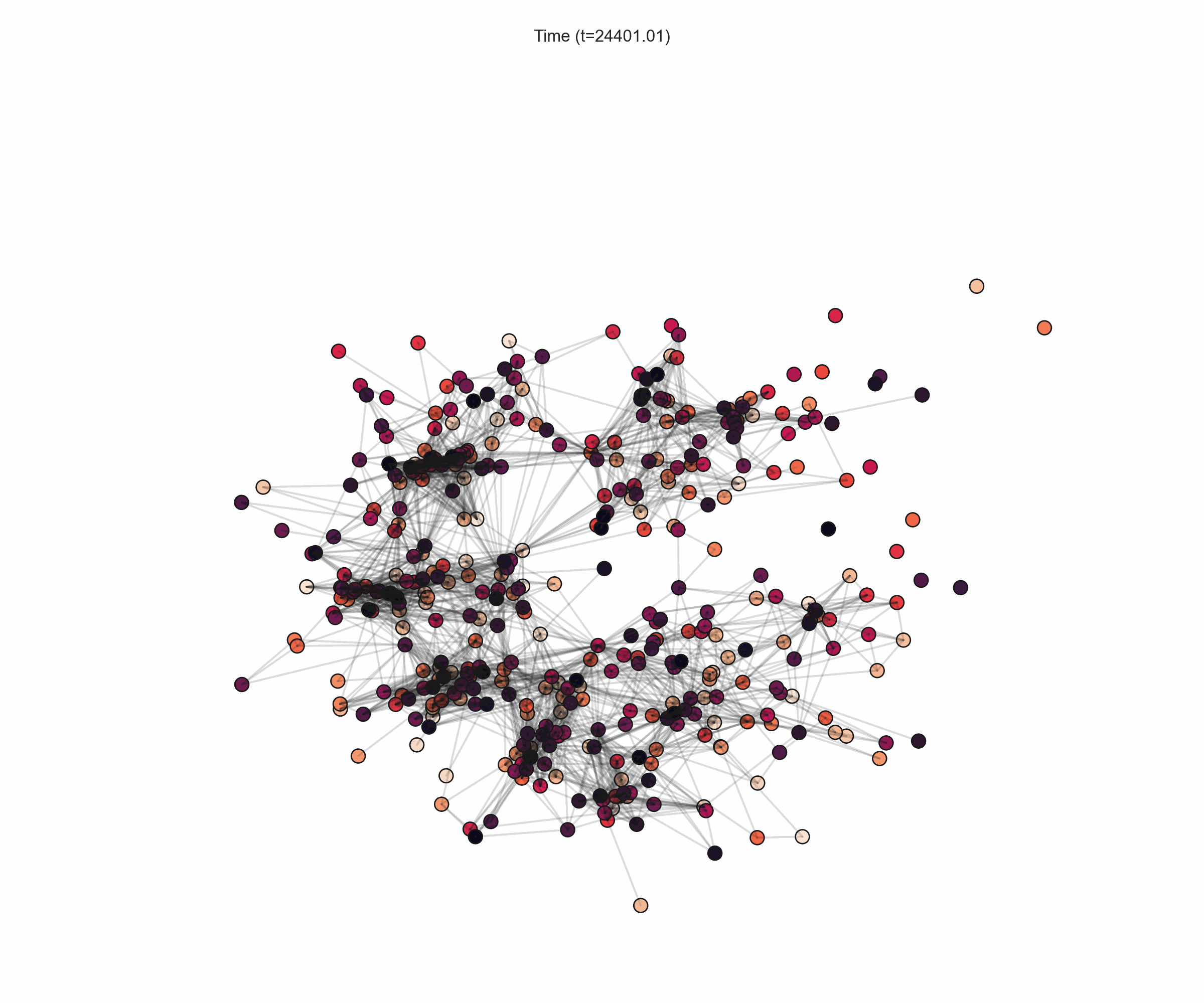}}
\hfill
\subfigure[$t=25836$]{\includegraphics[trim={5cm 6cm 5cm 6cm},clip,width=0.16\textwidth]{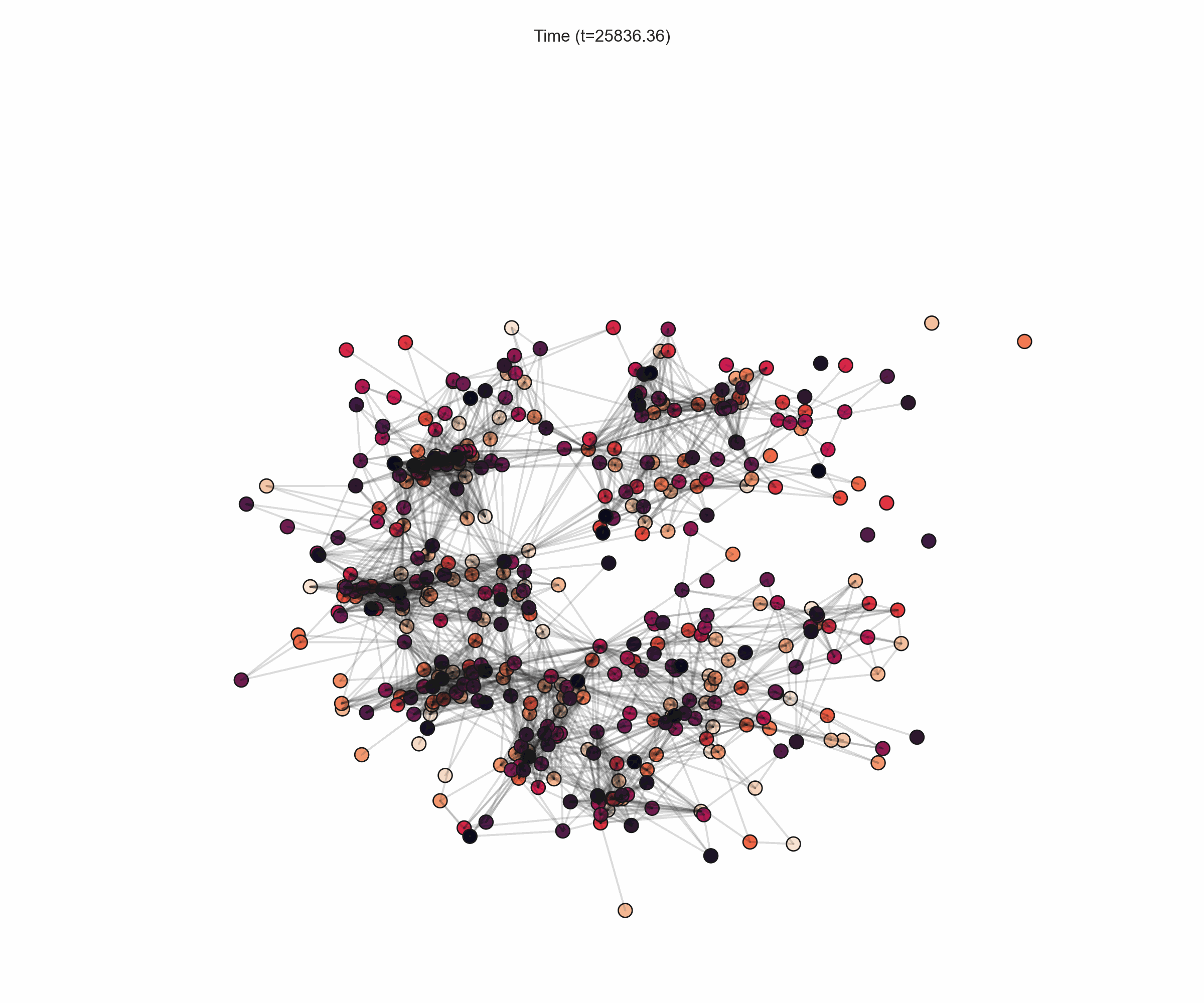}}
\caption{Snapshots of the continuous-time embeddings learned by \textsc{\modelname} for various time points over \textsl{Infectious}.}\label{fig:appendix_visualization_infectious}
\end{figure*}

\begin{figure*} %[!h]
\centering
\subfigure[\textsl{Synthetic-$\beta$}]{\includegraphics[width=0.49\textwidth]{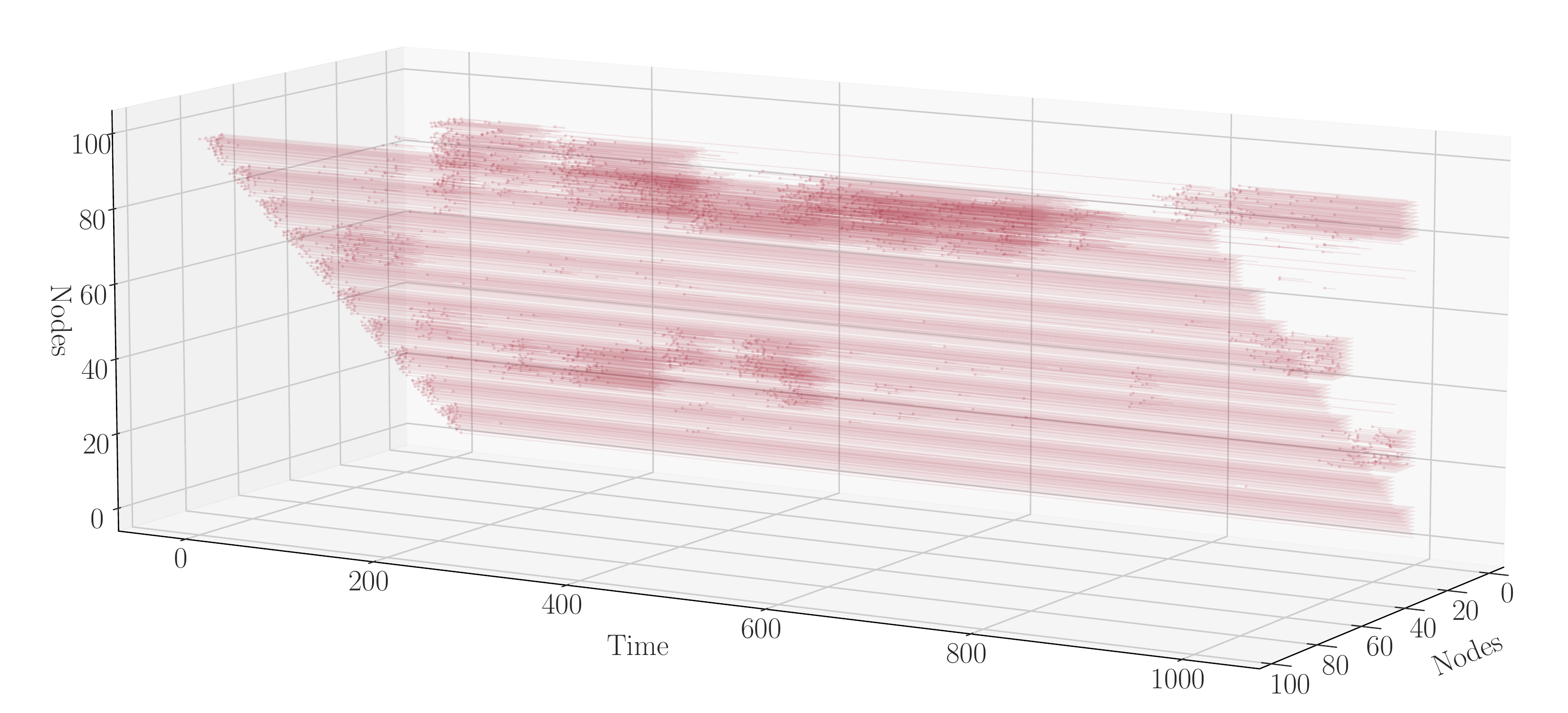}}
\hfill
\subfigure[\textsl{Synthetic-$\alpha$}]{\includegraphics[width=0.49\textwidth]{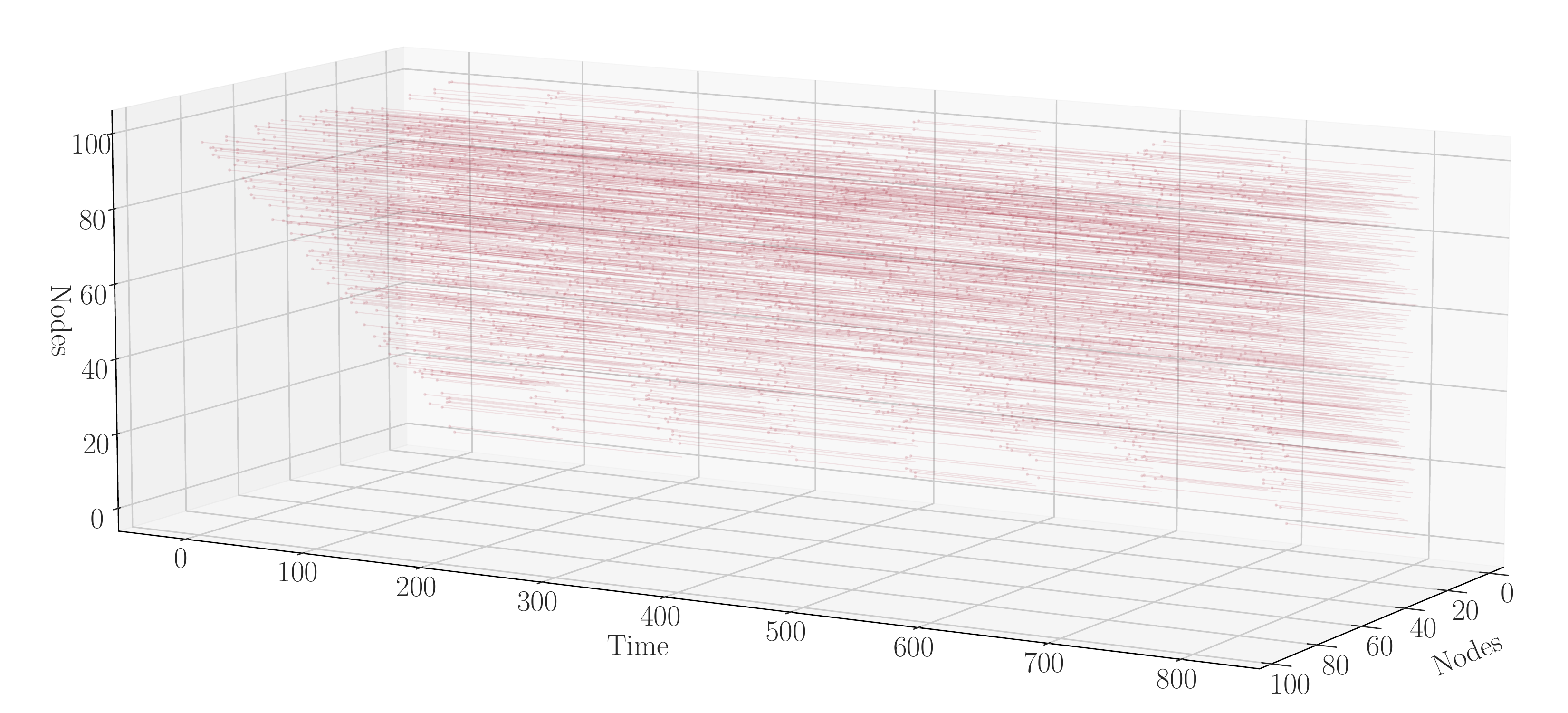}}
%%%%%%%
\subfigure[\textsl{Contacts}]
{\includegraphics[width=0.49\textwidth]{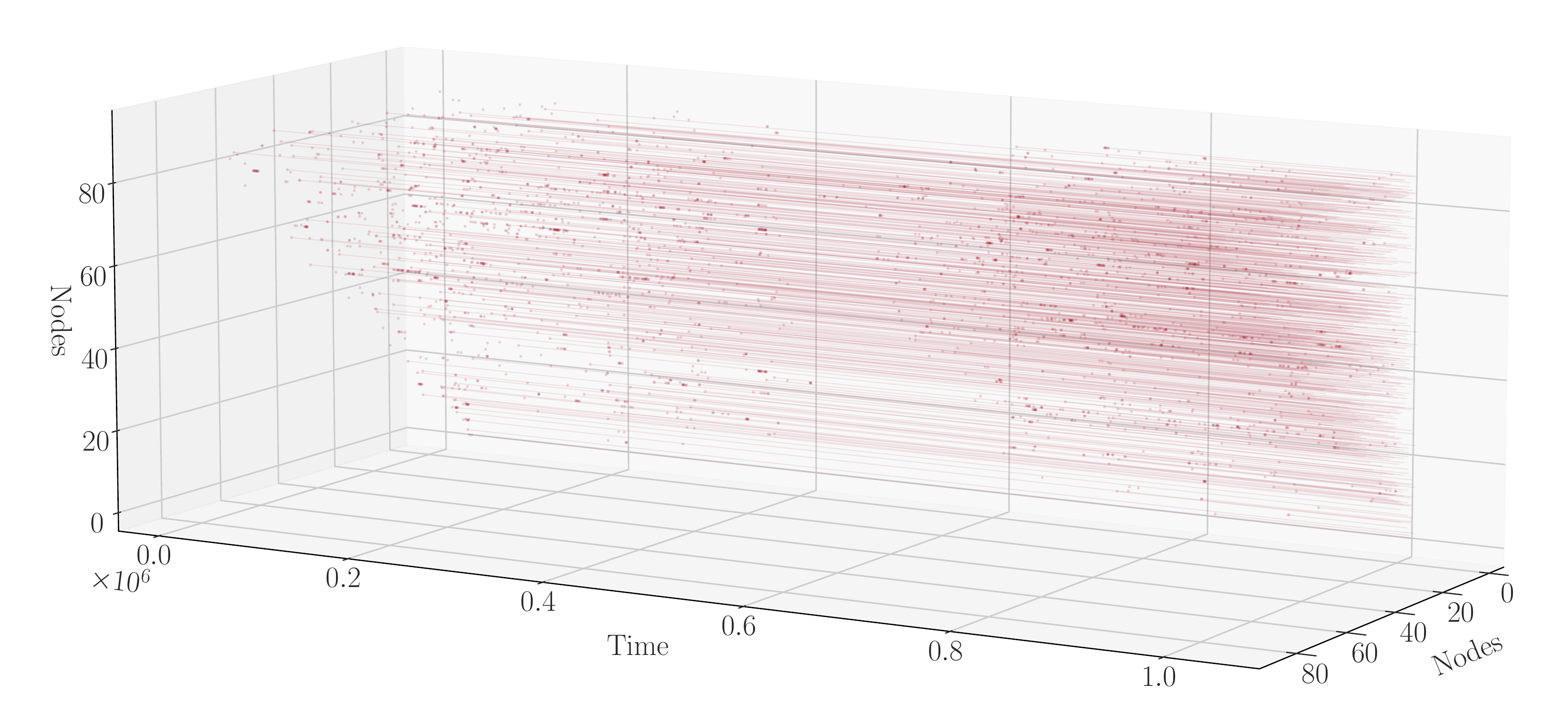}}
\hfill
\subfigure[\textsl{HyperText}]{\includegraphics[width=0.49\textwidth]{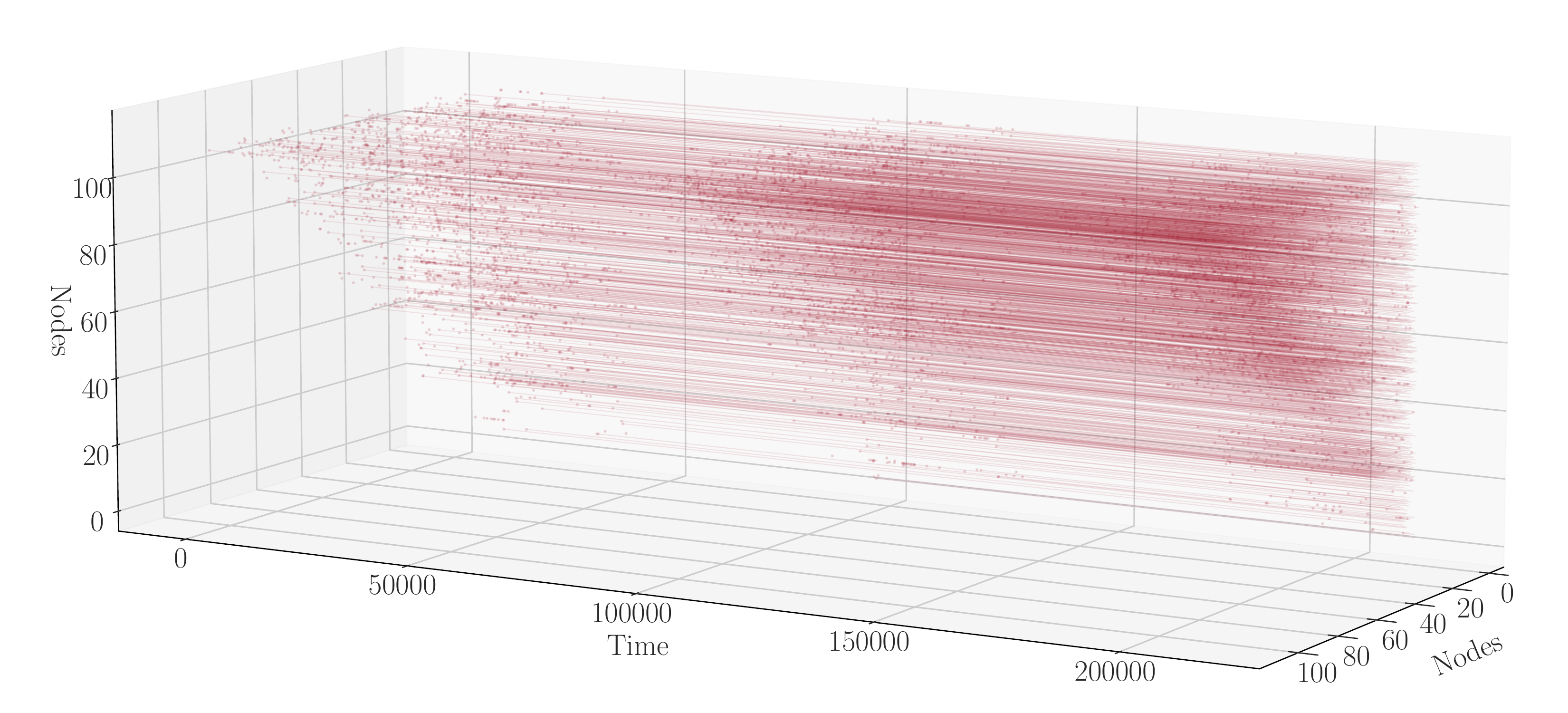}}
%%%%%%%%
\subfigure[\textsl{Facebook}]{\includegraphics[width=0.49\textwidth]{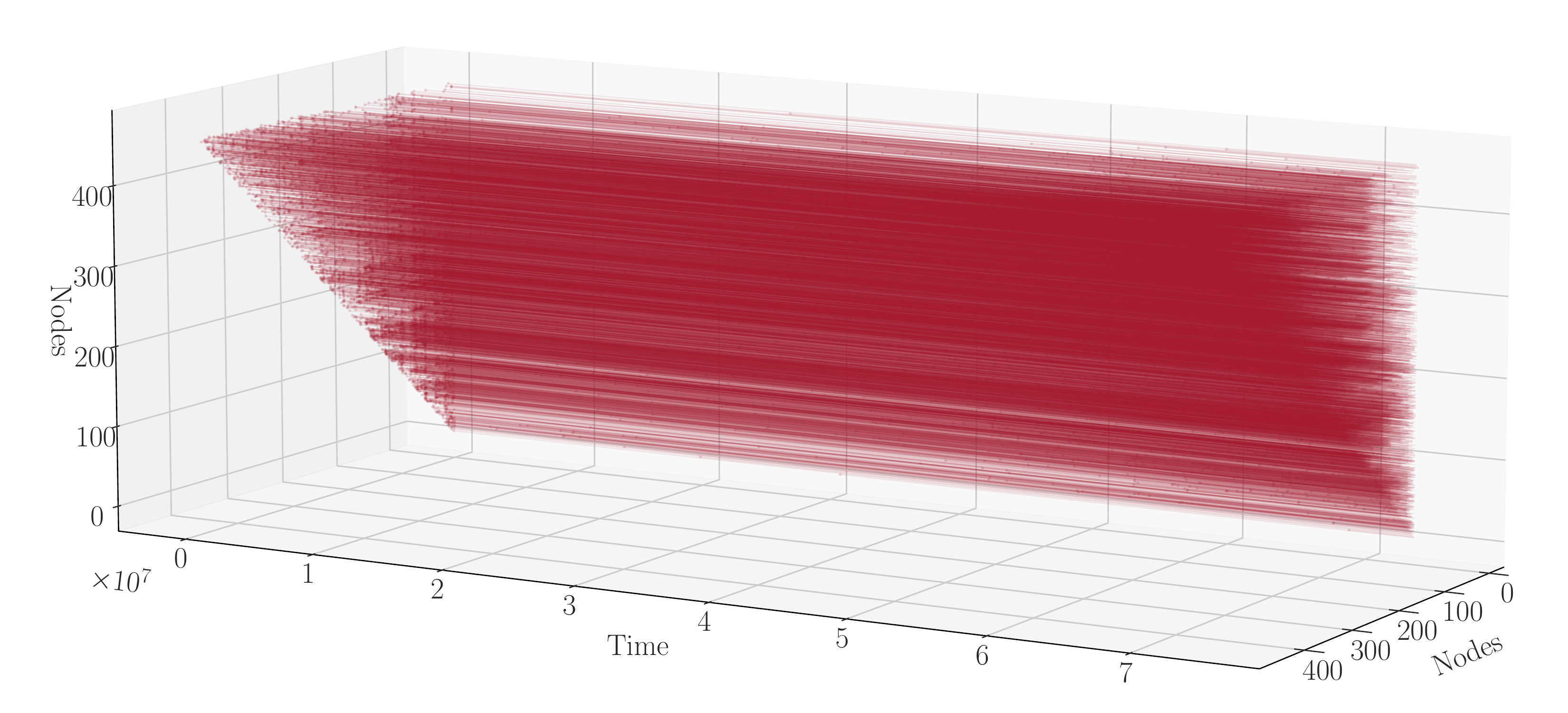}}
\hfill
\subfigure[\textsl{NeurIPS}]{\includegraphics[width=0.49\textwidth]{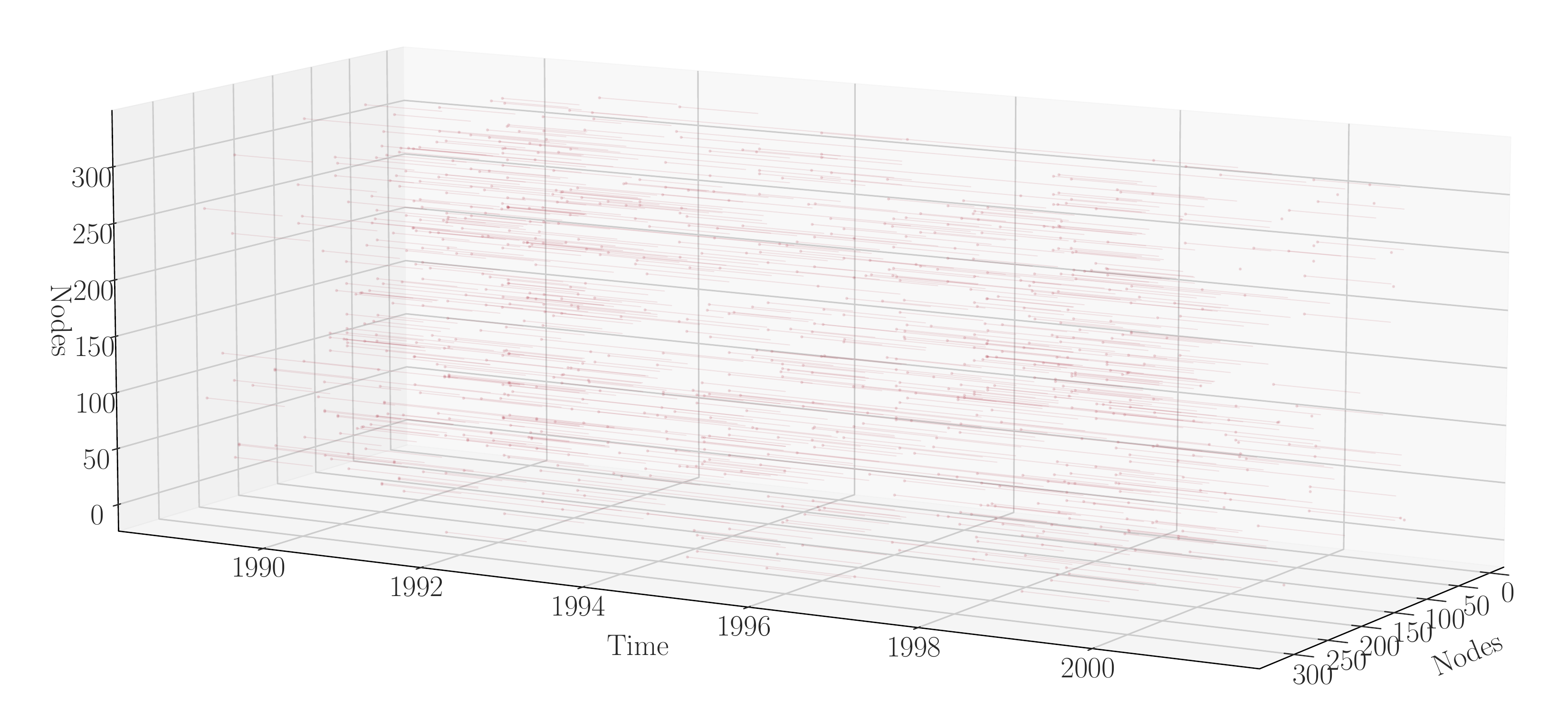}}
%%%%%%%
\subfigure[\textsl{Infectious}]{\includegraphics[width=0.49\textwidth]{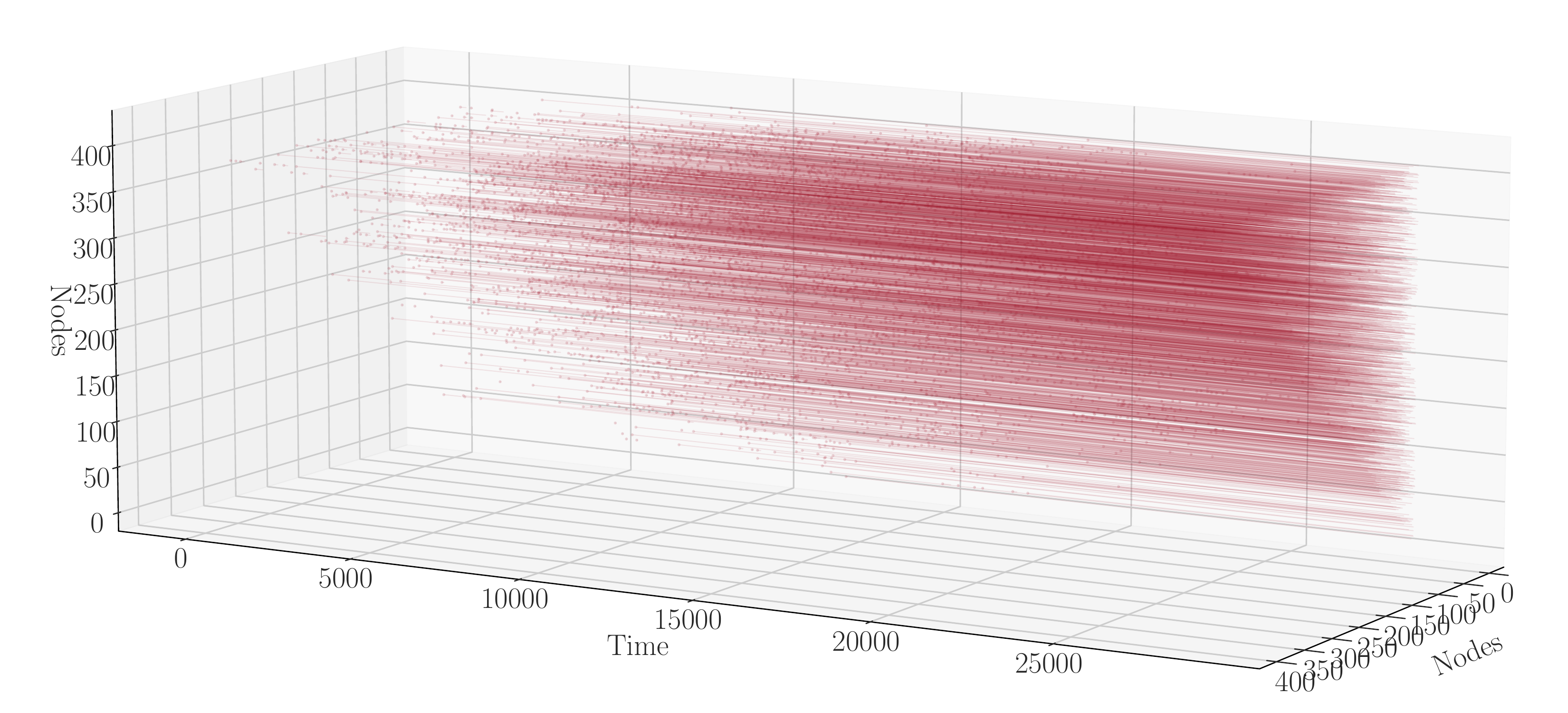}}
%%%%%%%
\caption{Intermittent time-persistent linkage structures of the networks.}
\label{fig:appendix_link_structure}
\end{figure*}

\clearpage
\appendix
\bibliography{references}

\end{document}